\documentclass[3p,number,times,fleqn]{elsarticle}

\usepackage{amssymb,amsmath,amsthm}
\usepackage%[only bigsqcap]
{stmaryrd}

\usepackage{bbm}
\usepackage{url}
\usepackage{xspace}
\usepackage{enumerate}
\usepackage{paralist}
\usepackage{multirow}
\usepackage{hyperref}

\usepackage{tikz}
\usetikzlibrary{shapes,decorations,decorations.pathmorphing,arrows,fit}
\usetikzlibrary{decorations.markings,calc,patterns}
\usetikzlibrary{decorations.pathreplacing,backgrounds}

\usepackage{misc}

\usepackage{pifont}

\newtheorem{theorem}{Theorem}
\newtheorem{definition}[theorem]{Definition}
\newtheorem{example}[theorem]{Example}
\newtheorem{proposition}[theorem]{Proposition}
\newtheorem{lemma}[theorem]{Lemma}

\newcommand{\A}{\mathcal{A}}

\newcommand{\I}{\mathcal{I}}
\newcommand{\J}{\mathcal{J}}

\renewcommand{\H}{\mathcal{H}}
\newcommand{\K}{\mathcal{K}}
\renewcommand{\L}{\mathcal{L}}
\newcommand{\M}{\mathcal{M}}

\newcommand{\R}{\mathcal{R}}
\renewcommand{\S}{\mathcal{S}}
\newcommand{\T}{\mathcal{T}}

\newcommand{\DLc}{\textsl{DL-Lite}\ensuremath{_{\textit{core}}}\xspace}
\newcommand{\DLcH}{\textsl{DL-Lite}\ensuremath{_{\textit{core}}^{\smash{\H}}}\xspace}

\newcommand{\ALC}{\ensuremath{\mathcal{ALC}}\xspace}

\newcommand{\ELUf}{\ensuremath{\mathcal{ELU}_{\mathit{rhs}}}\xspace}

\newcommand{\hALC}{\textsl{Horn}\ensuremath{\mathcal{ALC}}\xspace}

\newcommand{\EL}{\ensuremath{\mathcal{EL}}\xspace}
\newcommand{\ELI}{\ensuremath{\mathcal{ELI}}\xspace}

\newcommand{\LogSpace}{\textsc{LogSpace}\xspace}

\newcommand{\PTime}{\textsc{P}\xspace}

\newcommand{\ExpTime}{\textsc{ExpTime}\xspace}
\newcommand{\NExpTime}{\textsc{NExpTime}\xspace}

\renewcommand{\succ}{\mathsf{succ}}
\newcommand{\sig}{\mathsf{sig}}
\newcommand{\ind}{\mathsf{ind}}

\newcommand{\type}{\boldsymbol{t}}

\newcommand{\q}{\boldsymbol q}

\newcommand{\TWAPA}{2APTA\xspace}
\newcommand{\TWAPAs}{2APTAs\xspace}

\newcommand{\TWABA}{2ABTA\xspace}
\newcommand{\TWABAs}{2ABTAs\xspace}

\newcommand{\Start}{\textit{Start}}
\newcommand{\Row}{\textit{Row}}
\newcommand{\End}{\textit{End}}
\newcommand{\first}{\textit{first}}
\newcommand{\halt}{\textit{halt}}

\newcommand{\Mod}{\boldsymbol{M}}

\newcommand{\fosh}{forest-shaped\xspace}

\newcommand{\RCQ}{rCQ\xspace}
\newcommand{\RUCQ}{rUCQ\xspace}

\newcommand{\avec}[1]{\boldsymbol{#1}}

\newcommand{\types}{\mathsf{type}}
\newcommand{\compl}{\mathsf{compl}}

\newcommand{\ToneCQ}{\mathcal{T}_{\text{CQ}}^{1}}
\newcommand{\TtwoCQ}{\mathcal{T}_{\text{CQ}}^{2}}
\newcommand{\ACQ}{\A_{\text{CQ}}}
\newcommand{\SigmaCQ}{\Sigma_{\text{CQ}}}
\newcommand{\KoneCQ}{\K_{\text{CQ}}^{1}}
\newcommand{\KtwoCQ}{\K_{\text{CQ}}^{2}}
\newcommand{\KonerCQ}{\K_{\text{rCQ}}^{1}}
\newcommand{\KtworCQ}{\K_{\text{rCQ}}^{2}}
\newcommand{\TonerCQ}{\T_{\text{rCQ}}^{1}}
\newcommand{\TtworCQ}{\T_{\text{rCQ}}^{2}}
\newcommand{\ArCQ}{\A_{\text{rCQ}}}
\newcommand{\SigmarCQ}{\Sigma_{\text{rCQ}}}

\newcommand{\up}[1]{#1^{\uparrow\Gamma}}
\newcommand{\down}[1]{#1^{\downarrow\Gamma}}
\newcommand{\Iup}{\up{\I}}
\newcommand{\Idown}{\down{\I}}
\newcommand{\Kup}{\up{\K}}

\newcommand{\Dup}{\up{D}}

\newcommand{\Cuup}{\up{C}}

\newcommand{\Tup}{\up{\T}}

\newcommand{\Jup}{\up{\J}}
\newcommand{\Jdown}{\down{\J}}

\tikzset{ %
  point/.style={thick,circle,draw=black,minimum size=1.3mm,inner sep=0pt},%
  constant/.style={fill=black},%
  rowind/.style={fill=gray},%
  cwitness/.style={circle,draw=black,minimum size=3.5mm,inner sep=0pt,
    label=center:{$\land$}},%
  dwitness/.style={minimum width=1.1cm, minimum height=0.3cm},%
  homomorphism/.style={line width=0.1cm,-latex},%
  role/.style={-latex, semithick},%
  gen/.style={role, decorate, decoration={snake, amplitude=0.3mm,segment
      length=2mm, post length=1mm}},%
  subtree/.style={isosceles triangle, draw, very thick, subtreecolor,
    anchor=north, outer sep=0.1, isosceles triangle stretches, minimum
    width=1cm, minimum height=0.6cm, shape border rotate=90}, %
  backward/.style={rectangle, draw=black, fill=gray!50, minimum size=1mm, inner
    ysep=4pt, inner xsep=7pt, outer sep=0.05cm, rounded corners=2mm},%
  sbound/.style={rectangle, draw=black, fill=gray!10, minimum size=1mm, inner
    ysep=4pt, inner xsep=7pt, outer sep=0.05cm},%
  infinite/.style={rectangle, draw=black, fill=gray!25, minimum size=1mm, inner
    ysep=4pt, inner xsep=7pt, outer sep=0.05cm, minimum width=1.5cm},%
  trans/.style={-stealth', semithick},%
}

\colorlet{subtreecolor}{black!60}

\def \tikzdots[#1]{
  \begin{scope}[shift={#1}]
    \node at (0,0.15) {$\cdot$}; \node {$\cdot$}; \node at (0,-0.15) {$\cdot$};
  \end{scope}
}

\journal{Artificial Intelligence}

\begin{document}

\begin{frontmatter}



\title{Query Inseparability for \ALC Ontologies}


\author[bz]{Elena Botoeva}
\ead{botoeva@inf.unibz.it}

\author[br]{Carsten Lutz}
\ead{clu@informatik.uni-bremen.de}

\author[bbk]{Vladislav Ryzhikov}
\ead{vlad@dcs.bbk.ac.uk}

\author[liv]{Frank Wolter}
\ead{wolter@liverpool.ac.uk}

\author[bbk]{Michael Zakharyaschev}
\ead{michael@dcs.bbk.ac.uk}

\address[bz]{KRDB Research Centre,
  Free University of Bozen-Bolzano, Italy}
\address[br]{Fachbereich Informatik,
  University of Bremen, Germany}
\address[bbk]{Department of Computer Science and Information Systems,
  Birkbeck, University of London, UK}
\address[liv]{Department of Computer Science,
  University of Liverpool, UK}

\begin{abstract}
  We investigate the problem whether two $\ALC$ ontologies are
  indistinguishable (or inseparable) by means of queries in a given signature, which is fundamental for ontology engineering
  tasks such as ontology versioning, modularisation, update, and forgetting. We consider both
  knowledge base (KB) and TBox inseparability. For KBs, we give
  model-theoretic criteria in terms of (finite partial) homomorphisms and
  products and prove that this problem is undecidable for conjunctive
  queries (CQs), but 2\ExpTime-complete for unions of CQs (UCQs). The
  same results hold if (U)CQs are replaced by rooted (U)CQs, where every variable is
  connected to an answer variable. We also show that inseparability by CQs is still undecidable
  if one KB is given in the lightweight DL $\mathcal{EL}$ and if no restrictions
  are imposed on the signature of the CQs. We also consider the problem whether two $\ALC$ TBoxes give the same answers
  to any query over any ABox in a given signature and show that,
  for CQs, this problem is undecidable, too. We then develop model-theoretic criteria for \hALC TBoxes and show using tree automata that, in contrast, inseparability
  becomes decidable and 2\ExpTime-complete, even \ExpTime-complete when restricted to (unions of) rooted CQs.
\end{abstract}

% \begin{keyword}
% %% keywords here, in the form: keyword \sep keyword
% Description logic \sep knowledge base \sep conjunctive query \sep query inseparability \sep computational complexity \sep tree automaton.
% %% MSC codes here, in the form: \MSC code \sep code
% %% or \MSC[2008] code \sep code (2000 is the default)
% %\MSC \textcolor{red}{what is this?} \sep code
% \end{keyword}

\end{frontmatter}

% For research notes, remove the comment character in the line below.
% \researchnote

% !TEX root = aij-insep.tex

%-----------------------------------------------
%-----------------------------------------------
\section{Introduction}

In recent years, data access using description logic (DL) TBoxes has
become one of the most important applications of
DLs (see, e.g.,~\cite{PLCD*08,BienvenuO15,KontchakovRZ13} and references therein), where the underlying
idea is to use a TBox to specify semantics and background knowledge
for the data (stored in an ABox) and thereby derive more complete
 answers to queries.  A major research effort has led to the development of
efficient querying algorithms and tools for a number of DLs ranging from
DL-Lite~\cite{CDLLR07,CalvaneseGLLPRRRS11,Rodriguez-MuroKZ13} via more
expressive Horn DLs such as
Horn\ALC~\cite{EiterOSTX12,TrivelaSCS15}
to DLs with full Boolean constructors including 
\ALC and extensions such as $\mathcal{SHIQ}$~\cite{KolliaG13,ZhouGNKH15}. 

% The query language that dominate in the area are
% conjunctive queries (CQs) and unions of conjunctive queries (UCQs).

While query answering with DLs is now well-developed, this is much
less the case for reasoning services that support ontology engineering
when ontologies are used to query data. Important ontology
engineering tasks include ontology versioning~\cite{NoyM02-promptdiff,OntoView,RSDT08,HORROCKS,KonevL0W12},
ontology modularisation~\cite{DBLP:series/lncs/5445,KutzML10,GrauHKS08,KWZ10,RomeroKGH16},
ontology revision and update~\cite{GiacomoLPR09,LiuLMW11,WangWT10,WangWT15},
and forgetting in ontologies~\cite{KonevWW09,WangWTP10,LutzW11,WangWTPA14,KoopmannS14,NikitinaR14,KoopmannS15}. 
A fundamental reasoning problem in all these tasks is to \emph{compare} two ontologies.
For example, in ontology versioning, the user is interested in comparing two versions of an ontology
and understanding the relevant difference between them. In ontology modularisation, the relevant consequences 
of the full ontology should be preserved when it is replaced by a module. In ontology revision and update, one typically
minimises the relevant difference between the updated or revised ontology and the original ontology 
while taking into account new knowledge. In ontology forgetting, one constructs a new ontology, which is indistinguishable 
from the original ontology with respect to a signature of interest. The relevant consequences that should be considered
when comparing two ontologies depend on the application. 
%From a purely logic-based perspective, thecomopway the ontologies are comparedFrom a purelyregarding the whether in how far a module% can be used,
In the context of querying data via ontologies, it is natural to consider the answers the ontologies give to queries. 
Then, in ontology versioning, the relevant difference between two versions of an ontology is based on the queries 
that receive distinct answers with respect to the ontology versions. In ontology modularisation, it is the answers to queries 
that should be preserved when a module is extracted from an ontology. In ontology update or revision, the difference between 
the answers to queries over the updated or revised ontology and the original one should be minimised 
when constructing update or revision operators. Similarly, in forgetting, it is the
answers to queries which should be preserved under appropriate forgetting operators.
Thus, in the context of query answering, the fundamental
relationship between ontologies is not whether they are logically
equivalent (have the same models), but whether they give the same
answers to any relevant query. To illustrate, consider the following simple TBox
$$
\T= \{\textit{Book} \sqsubseteq \exists \textit{author}.\neg \textit{Book}\}
$$
saying that every book has an author who is not a book.  Clearly, $\T$ is not logically equivalent to the TBox
$$
\T' = \{\textit{Book} \sqsubseteq \exists \textit{author}.\top\},
$$
which only states that every book has an author. However, if one takes as the query language the popular classes of conjunctive queries (CQs) or unions of CQs (UCQs),
then no matter what the data is, every query will have the same answers independently of  whether one uses $\T$ or $\T'$. Intuitively, the reason
is that the `positive' information given by $\T$ coincides with the `positive' information given by $\T'$. If the main purpose of
the ontology is answering UCQs, it is thus more important to know
that $\T$ can be safely replaced by $\T'$ without affecting the answers to UCQs than to establish that $\T$ and $\T'$
are not logically equivalent.

In most ontology engineering applications for ontology-based data access, the relevant
class $\mathcal{Q}$ of queries can be further restricted to those given in a finite signature of relevant concept and role names. For example, to
establish that a subset $\mathcal{M}$ of an ontology $\mathcal{O}$ is a module of $\Omc$,  one should
not require that $\mathcal{M}$ and $\mathcal{O}$ give the same answers to all queries in $\mathcal{Q}$, but only to those that are in the signature of $\mathcal{M}$. Similarly, in the versioning context,
often only the answers to queries in $\mathcal{Q}$ given in a small signature containing  a fraction of the concept and role names
of the ontology are relevant for the application, and so for the difference that should be presented to a user.

The resulting entailment problem can be formalised in two ways. Recall that, in DL, a knowledge base (KB)
$\K=(\T,\A)$ consists of a TBox $\T$ and an ABox $\A$. Now, given a class $\mathcal{Q}$ of queries, KBs $\K_{1}$ and
$\K_{2}$, and a signature $\Sigma$ of relevant concept and role names,
we say that $\K_{1}$ \emph{$\Sigma$-$\mathcal{Q}$ entails} $\K_{2}$ if
the answers to any $\Sigma$-query in $\mathcal{Q}$ over $\K_{2}$ are
contained in the answers to the same query over $\K_{1}$. Further,
$\K_{1}$ and $\K_{2}$ are \emph{$\Sigma$-$\mathcal{Q}$ inseparable} if
they $\Sigma$-$\mathcal{Q}$ entail each other. Since a KB
includes an ABox, this notion of entailment is appropriate if
the data is known while the ontology engineering task is completed and does not change frequently.
This is the case for many real-world ontologies, which not only provide a conceptual model of the domain of interest, 
but also introduce the individuals relevant for the domain and their properties.  
In addition to versioning, modularisation, revision, update, and forgetting, applications of
$\Sigma$-KB entailment and $\Sigma$-KB inseparability also include knowledge exchange~\cite{ArenasBCR13,ArenasBCR16,BotoevaKRWZ16}, where a user wants to transform a KB $\K_{1}$ given in a signature $\Sigma_{1}$
to a KB $\K_{2}$ in a new signature $\Sigma_{2}$ connected to $\Sigma_{1}$ using a mapping $\mathcal{M}$, also known as
an ontology alignment or ontology matching~\cite{ShvaikoE13}. The condition that
the target KB $\K_{2}$ is a sound and complete representation of $\K_{1}$ under $\mathcal{M}$ 
with respect to the answers to a class $\mathcal{Q}$ of relevant queries can then be formulated as
the condition that $\K_{1} \cup \mathcal{M}$ and $\K_{2}$ are $\Sigma_{2}$-$\mathcal{Q}$ inseparable~\cite{BotoevaKRWZ16}. 
The following simple example illustrates the notion of KB inseparability.
\begin{example}\label{ex:lecture}\em
Suppose we are given the KBs $\K_{1}=(\T_{1},\A)$ and $\K_{2}=(\T_{2},\A)$, where
\begin{align*}
& \Tmc_{1}= \{\textit{Lecturer}\sqsubseteq \forall \textit{teaches}.(\textit{Undergraduate} \sqcup \textit{Graduate})\}, \qquad \T_{2}=\emptyset,\\
& \A= \{\textit{Lecturer}(a), \, \textit{teaches}(a,b)\}.
\end{align*}
Then $\K_{1}$ and $\K_{2}$ are $\Sigma$-CQ inseparable, for any signature $\Sigma$. However, they are not $\Sigma$-UCQ inseparable for the signature $\Sigma$ containing the concept names $\textit{Undergraduate}$ and $\textit{Graduate}$.
To see this, consider the $\Sigma$-UCQ
$$
\q(x) = \textit{Undergraduate}(x) \vee \textit{Graduate}(x).
$$
Clearly, $b$ is an answer to $\q(x)$ over $\K_{1}$, but not over $\K_{2}$.
%Observe that even for $\Sigma'= \{{\sf Undergraduate},{\sf Graduate}\}$ the KBs $\K_{1}$ and $\K_{2}$ are not $\Sigma'$-UCQ-inseparable
%as they are distinguished by the Boolean UCQ $\exists x {\sf Undergraduate}(x) \vee \exists x {\sf Graduate}(x)$.
\end{example}

KB entailment and inseparability are appropriate if the data is known and does not change frequently. If, however, the data is not known or tends to change,
it is not KBs that should be compared, but TBoxes. Given a pair
$\Theta=(\Sigma_{1},\Sigma_{2})$ that specifies a relevant signature
$\Sigma_{1}$ for ABoxes and a relevant signature $\Sigma_{2}$ for queries, we say that a
TBox $\T_{1}$ \emph{$\Theta$-$\mathcal{Q}$ entails} a TBox $\T_{2}$
if, for every $\Sigma_{1}$-ABox $\A$, the KB $(\T_{1},\A)$
$\Sigma_{2}$-$\mathcal{Q}$ entails $(\T_{2},\A)$. TBoxes $\T_{1}$ and
$\T_{2}$ are \emph{$\Theta$-$\mathcal{Q}$ inseparable} if they
$\Theta$-$\mathcal{Q}$ entail each other.
\begin{example}\label{ex:2}\em
Consider again the TBoxes $\T_{1}$ and $\T_{2}$ from Example~\ref{ex:lecture}. Clearly, $\T_{1}$ and $\T_{2}$ are not $(\Sigma_{0},\Sigma_{1})$-UCQ inseparable
for $\Sigma_{0}=\{\textit{Lecturer},\textit{teaches}\}$ and $\Sigma_{1}=\{\textit{Undergraduate},\textit{Graduate}\}$
as we have seen a $\Sigma_0$-ABox $\A$ for which $(\T_{1},\A)$ and $(\T_{2},\A)$ are not
$\Sigma_{1}$-UCQ inseparable. Notice, however, that $\T_{1}$ and $\T_{2}$ are both $(\Sigma_{0},\Sigma_{0})$-UCQ and $(\Sigma_{1},\Sigma_{1})$-UCQ inseparable.
On the other hand, it is not difficult to see that
$\T_{1}$ and $\T_{2}$ are $(\Sigma_{0},\Sigma_{1})$-CQ inseparable. The situation
changes drastically if the ABox can contain additional role names, for instance $\textit{hasFriend}$. Indeed, suppose
$\Sigma_{2} = \Sigma_{0} \cup \Sigma_{1} \cup \{\textit{hasFriend}\}$.
Then $\T_{1}$ and $\T_{2}$ are $(\Sigma_{2},\Sigma_{2})$-CQ separable by the ABox $\A'$ shown in the picture below and the CQ 
$$
\q'(x) = \exists y \exists z\, \big( \mathit{teaches}(x,y) \land \mathit{Undergraduate}(y) \land \mathit{hasFriend}(y,z) \land \mathit{Graduate}(z) \big)
$$ 
since $a$ is returned as an answer to $\q'(x)$ over $(\T_{1},\A')$ but not over $(\T_{2},\A')$. (This example is a variant of the well-known~\cite[Example~4.2.5]{Scha94b}.)
\begin{center}
\begin{tikzpicture}[xscale=1.5, yscale=1.5]
  \foreach \name/\x/\y/\lab/\conc/\wh in {%
    a/-2/0.5/a/\mathit{Lecturer}/above,%
    b/-1/-0.3/b/\mathit{Undergraduate}/below, %
    c/0/0.5/c/{}/above, %
    d/1.5/0.5/d/\mathit{Graduate}/below%
  }{ \node[inner sep=1, outer sep=0, label=\wh:{\scriptsize$\conc$}] (\name)
    at (\x,\y) {\footnotesize $\lab$}; }

  \foreach \from/\to/\wh/\lab in {%
    a/b/below/\mathit{teaches}, a/c/above/\mathit{teaches}, b/c/below/\mathit{hasFriend},
    c/d/above/\mathit{hasFriend}%
  }{ \draw[role] (\from) -- node[\wh,sloped] {\scriptsize$\mathit{\lab}$} (\to); }

  \node[anchor=east] at (-2.5,0.2) {$\A'$:};

  \begin{scope}[xshift=6cm, yshift=0.2cm]
    \foreach \name/\x/\conc/\wh in {%
      x/-1.5//right,%
      y/0/{\footnotesize\mathit{Undergraduate}}/below, %
      z/1.5/{\footnotesize\mathit{Graduate}}/below%
    }{ \node[inner sep=1, outer sep=0, label=\wh:{\footnotesize$\conc$}]
      (\name) at (\x,0) {\small $\name$ }; }

    \foreach \from/\to/\wh in {%
      x/y/right%
    }{ \draw[role] (\from) -- node[above] {\footnotesize$\mathit{teaches}$} (\to); }

    \foreach \from/\to/\wh in {%
      y/z/right%
    }{ \draw[role] (\from) -- node[above] {\footnotesize$\mathit{hasFriend}$} (\to); }

    \node[anchor=east] at (-2,0) {$\q'(x)$:};
  \end{scope}
\end{tikzpicture}
\end{center}
\end{example}

%%
%\begin{example}
%Consider again the TBoxes $\T_{1}$ and $\T_{2}$ from above. Clearly $\T_{1}$ and $\T_{2}$ are not $(\Sigma_{0},\Sigma_{1})$-UCQ inseparable
%for the signatures $\Sigma_{0}=\{{\sf Lecturer},{\sf teaches}\}$ and $\Sigma_{1}=\{{\sf Undergraduate},{\sf Graduate}\}$
%since we have found an ABox $\A$ already using only $\{{\sf Lecturer},{\sf teaches}\}$ for which $(\T_{1},\A)$ and $(\T_{2},\A)$ are not
%$\Sigma_{1}$-UCQ-inseparable. Notice that $\T_{1}$ and $\T_{2}$ are, however, both $(\Sigma_{0},\Sigma_{0})$-UCQ inseparable and $(\Sigma_{1},\Sigma_{1})$-UCQ inseparable.
%It is not difficult to see that, on the other hand,
%$\T_{1}$ and $\T_{2}$ are $(\Sigma_{0},\Sigma_{1})$-CQ-inseparable. The situation
%changes drastically if the ABox can contain additional tables such as a table about friends. Let
%$$
%\Sigma_{2} = \Sigma_{0} \cup \Sigma_{1} \cup \{{\sf has\_friend}\}
%$$
%Then $\T_{1}$ and $\T_{2}$ are not $(\Sigma_{2},\Sigma_{2})$-CQ inseparable since for the ABox
%$$
%\A'=\{{\sf teaches}(a,b_{1}),{\sf teaches}(a,b_{2}),{\sf has\_friend}(b_{1},b_{2}),
%{\sf has\_friend}(b_{2},b_{3}),{\sf Graduate}(b_{3}),{\sf Undergraduate}(b_{1})\}
%$$
%the $\Sigma_{2}$-CQ $q$ asking whether $a$ teaches somebody who is an undergraduate who has a friend who is a graduate is
%entailed by $(\T_{1},\A')$ but not by $(\T_{2},\A')$.
%\end{example}

%Applications of this notion include
%data-oriented TBox versioning, TBox modularisation, and TBox
%forgetting~\cite{KWZ10}.

In this paper, we investigate entailment and inseparability for KBs and TBoxes and for queries that are CQs or UCQs.
In practice, the majority of queries are \emph{rooted} in the sense that every variable is connected to an answer variable. We therefore also consider the classes of rooted CQs (\RCQ{s}) and UCQs (\RUCQ{s}).
So far, query entailment and inseparability have been studied for Horn DL KBs~\cite{BotoevaKRWZ14}, $\mathcal{EL}$
TBoxes~\cite{LutzW10,KonevL0W12}, DL-Lite TBoxes~\cite{KontchakovPSSSWZ09}, and also for OBDA
specifications, that is, DL-Lite TBoxes with mappings~\cite{BienvenuR16}; for a recent survey see~\cite{BotoevaKLRWZ16}.
No results are yet available for non-Horn DLs (neither in the KB nor in the TBox case) and for
expressive Horn DLs in the TBox case. In particular, query entailment
in non-Horn DLs has had the reputation of being a technically challenging problem. Here, we make first steps towards understanding query entailment and inseparability in these cases.
To begin with, we give model-theoretic characterisations of these notions
for $\mathcal{ALC}$ and \hALC in terms of (finite partial) homomorphisms and products of interpretations.
The obtained characterisations together with various types of automata are then used to investigate the
computational complexity of deciding query entailment and inseparability. Our main
results on KB and TBox inseparabilities are summarised in Tables~\ref{table:kb} and~\ref{table:TBox}, respectively:
%
% \begin{table}[h]
% \centering
% \caption{KB query inseparability.}
% \label{table:kb}
% \begin{tabular}{|c|c|c|c|c|}
% \hline
% Queries &  $\mathcal{ALC}$ and $\mathcal{ALC}$ & $\mathcal{ALC}$ and $\mathcal{EL}$ & \hALC and \hALC & $\mathcal{EL}$ and $\mathcal{EL}$ \\
% \hline
% CQ and rCQ &  undecidable   &  undecidable & \ExpTime-complete &  \multirow{2}{*}{PTime}\\
% \cline{1-4}
% UCQ and rUCQ & 2\ExpTime-complete & in 2\ExpTime & \ExpTime-complete & \\
% \hline
% \end{tabular}
%
\begin{table}[h]
\centering
\caption{KB query inseparability.}
\label{table:kb}
\begin{tabular}{|c|c|c|}
\hline
Queries &  $\mathcal{ALC}$ and $\mathcal{ALC}$ & $\mathcal{ALC}$ and $\mathcal{EL}$ \\
\hline
CQ and rCQ &  undecidable   &  undecidable \\
\hline
UCQ and rUCQ & 2\ExpTime-complete & in 2\ExpTime \\
\hline
\end{tabular}
%
%\end{figure}
%
%\begin{figure}[h]
\caption{TBox query inseparability.}%
\label{table:TBox}
\begin{tabular}{|c|c|c|c|}
\hline
Queries &  $\mathcal{ALC}$ and $\mathcal{ALC}$ & $\mathcal{ALC}$ and $\mathcal{EL}$ & \hALC and \hALC \\
\hline
CQs &  undecidable  &  undecidable & 2\ExpTime-complete \\
\hline
rCQs & undecidable  &  undecidable & \ExpTime-complete \\
\hline
\end{tabular}
\end{table}

% \caption{TBox query inseparability.}%
% \label{table:TBox}
% \begin{tabular}{|c|c|c|c|c|}
% \hline
% Queries &  $\mathcal{ALC}$ and $\mathcal{ALC}$ & $\mathcal{ALC}$ and $\mathcal{EL}$ & \hALC and \hALC & $\mathcal{EL}$ and $\mathcal{EL}$ \\
% \hline
% CQs &  undecidable  &  undecidable & 2\ExpTime-complete &  \ExpTime-complete\\
% \hline
% rCQs & undecidable  &  undecidable & \ExpTime-complete & \ExpTime-complete\\
% \hline
% \end{tabular}
% \end{table}

Three of these results came as a real surprise to us. First, it turned out that
CQ and rCQ inseparability between \ALC KBs is undecidable, even if one of the
KBs is formulated in the lightweight DL $\EL$ and without any
signature restriction. This should be contrasted with the decidability
of subsumption-based entailment between \ALC TBoxes \cite{GhiLuWo-06} (and even theories in guarded fragments of FO~\cite{JungLM0W17})
and of CQ entailment between \hALC KBs \cite{BotoevaKRWZ14}. The
second surprising result is that inseparability between \ALC KBs becomes
decidable when CQs are replaced with UCQs or \RUCQ{s}. In fact, we
show that inseparability is 2\ExpTime-complete for both UCQs and \RUCQ{s}.
An even more fine-grained picture is obtained by considering entailment
instead of inseparability. It turns out that (r)CQ entailment of \hALC KBs by \ALC KBs
coincides with (r)UCQ entailment of \hALC KBs by \ALC KBs and is 2\ExpTime-complete, but
that in contrast (r)CQ entailment of \ALC KBs by \hALC KBs is undecidable.

For \ALC TBoxes, CQ and rCQ entailment as well as CQ and rCQ inseparability are undecidable as well. We
obtain decidability for \hALC TBoxes (where CQ und UCQ entailments coincide) using the fact that non-entailment
is always witnessed by tree-shaped ABoxes. % We then reduce the entailment problem
% to the emptiness problem for a mix of two-way alternating and
% non-deterministic top-down automata on finite trees.
As another surprise, CQ inseparability of \hALC TBoxes is 2\ExpTime-complete while
rCQ-entailment is only \ExpTime-complete. This applies to CQ entailment and rCQ entailment as well.
This result should be contrasted with the \EL case, where both problems are \ExpTime-complete
\cite{LutzW10}. Table~\ref{table:TBox} does not contain any results in the UCQ case, as the decidability
of UCQ entailment and inseparability between \ALC TBoxes remains open.

%All upper bounds and most lower
%bounds hold also for inseparability in place of entailment.
%
%A model-theoretic foundation for these results is a characterisation
%of query entailment between KBs and TBoxes in terms of (partial)
%homomorphisms, which, in particular, enables the use of tree automata
%techniques to establish the upper bounds in Figures~\ref{table:kb}
%and~\ref{table:TBox}.

We now discuss the structure and contributions of this paper in more detail.
Section~2 defines the DLs we are interested in, which range from $\mathcal{EL}$ to  \hALC and \ALC. It also introduces query answering for DL KBs and provides basic completeness results and homomorphism characterisations for
query answering. Section~3 defines query entailment and inseparability between DL KBs. It  provides illustrating examples
and characterises UCQ entailment in terms of finite partial homomorphisms between models of KBs. To characterise CQ entailment, products of KB models are also required. The  difference between the characterisations will play a crucial role
in our algorithmic analysis of entailment. In some important cases later on in the paper, finite partial homomorphisms are replaced by full homomorphisms using, for example, automata-theoretic techniques and, in particular, Rabin's result that any tree automaton that accepts some tree accepts already a regular tree. This move from finite partial homomorphisms to full homomorphisms is non-trivial
and crucial for our decision procedures.

In Section~4, we prove the undecidability of (r)CQ entailment of an $\mathcal{ALC}$ KB by an $\mathcal{EL}$ KB using a reduction of an undecidable tiling problem.
The direction is important, as we prove later that (r)CQ entailment of an $\mathcal{EL}$ KB by an $\mathcal{ALC}$ KB is decidable (in 2\ExpTime). We also
prove undecidability of CQ inseparability between $\mathcal{EL}$ and $\mathcal{ALC}$ KBs. The model-theoretic characterisation of (r)CQ entailment
via products and finite homomorphisms is crucial for these proofs. We then use a `hiding technique' replacing concept names by complex concepts
to extend the undecidability results to the full signature. Thus, for example, even without any restriction on the signature it is undecidable
whether two $\mathcal{ALC}$ KBs are (r)CQ inseparable.

In Section~5, we first show that, in the (r)UCQ case, partial homomorphisms can be replaced by full homomorphisms in the model-theoretic characterisation
of rUCQ entailment between $\mathcal{ALC}$ KBs if one considers regular tree-shaped models of the KBs. This result is then used to
encode the UCQ entailment problem into an emptiness problem for two-way alternating parity automata on infinite trees (2APTAs). Using results from
automata theory we then obtain a 2\ExpTime upper bound for (r)UCQ entailment between $\mathcal{ALC}$ KBs and a characterisation of (r)UCQ
entailment with full homomorphisms that does not require the restriction to regular tree-shaped models. We prove that the 2\ExpTime upper bound is tight
by a reduction of the word problem for alternating Turing machines. Finally, we show using the hiding technique that the 2\ExpTime lower bounds
still hold without restrictions on the signature.

In Section~6, we introduce query entailment and inseparability between TBoxes and prove that the undecidability results for (r)CQ entailment and (r)CQ inseparability
can be lifted from KBs to TBoxes. In this case, however, undecidability without any restrictions regarding the signatures remains open. In Section~7,
we develop model-theoretic criteria for (r)CQ entailment of \hALC TBoxes by \ALC TBoxes. The crucial observation is that it suffices to consider tree-shaped ABoxes
when searching for counterexamples to (r)CQ entailment between TBoxes. This allows us to use, in Section~8, automata on trees to decide (r)CQ entailment.

In Section~8, we first prove an \ExpTime upper bound for rCQ entailment of \hALC TBoxes by \ALC TBoxes via an encoding into emptiness
problems for a mix of two-way alternating B\"uchi automata and non-deterministic top-down tree automata on finite trees (that represent tree-shaped ABoxes).
As satisfiability of \hALC TBoxes is \ExpTime-hard already, this bound is tight. We then consider arbitrary (not necessarily rooted) CQs and extend
the previous encoding into emptiness problems for tree automata to this case, thereby obtaining a 2\ExpTime upper bound. Here, it is non-trivial
to show that this bound is tight. We use a reduction of alternating Turing machines to prove the corresponding 2\ExpTime lower bound (also for CQ inseparability).

We conclude in Section~9 by discussing open problems. A small number of proofs that follow ideas presented in the main paper are deferred to the appendix. An extended abstract with initial results that led to this paper was presented at IJCAI 2016~\cite{BotoevaLRWZ16}.

%%% Local Variables:
%%% mode: latex
%%% TeX-master: "aij-insep"
%%% TeX-PDF-mode: t
%%% save-place: t
%%% End:

% !TEX root = aij-insep.tex

\section{Preliminaries}
\label{sec:introdesc}

In DL, knowledge is represented by means of concepts and roles that are defined inductively starting from a 
countably infinite set \NC of \emph{concept names} and a countably-infinite set \NR of \emph{role names}, and using a
set of concept and role constructors~\cite{BCMNP03}.  Different sets
of concept and role constructors give rise to different DLs.

We begin by introducing the description logic \ALC.
The concept constructors available in \ALC
are shown in Table~\ref{tab:syntax-semantics}, where $R$ is a
role name and $C$, $D$ are concepts.
A concept built using these
constructors is called an \emph{\ALC-concept}. \ALC does not have any role constructors.
\begin{table}[tb]
  \begin{center}
    \small
    \leavevmode
    \begin{tabular}{|l|c|c|}
      \hline Name &Syntax&Semantics\\
      \hline\hline   & &\\[-1em]
      top concept    & $\top$ & $\Delta^\Imc$\\
      \hline         & &\\[-1em]
      bottom concept & $\bot$ & $\emptyset$ \\
      \hline         & &\\[-1em]
      negation       & $\neg C$ & $\Delta^\Imc \setminus C^\Imc$\\
      \hline         & &\\[-1em]
      conjunction    & $C\sqcap D$ & $C^\Imc\cap D^\Imc$\\
      \hline         & &\\[-1em]
      disjunction    & $C\sqcup D$ & $C^\Imc\cup D^\Imc$\\
      \hline         & &\\[-1em]
      existential restriction & $\exists R . C$ & $\{~ d \in \Delta^\Imc \mid \exists e \in C^\Imc \, (d,e) \in R^\Imc~ \}$ \\
      \hline         & &\\[-1em]
      universal restriction & $\forall R . C$ & $\{~ d \in \Delta^\Imc \mid \forall e \in \Delta^\Imc \, \big((d,e) \in R^\Imc \to e \in C^\Imc\big) ~\}$ %\\ \hline & &\\[-1em]
      \\ \hline
    \end{tabular} %see text \\ \hline
    \caption{Syntax and semantics of \ALC.}
    \label{tab:syntax-semantics}
  \end{center}
\end{table}
An \emph{\ALC TBox} is a finite set of \emph{\ALC concept inclusions} (CIs) of the form $C \sqsubseteq D$ and
\emph{\ALC concept equivalences} (CEs) $C \equiv D$. (A CE $C \equiv D$ will be regarded as an
abbreviation for the two CIs $C \sqsubseteq D$ and $D \sqsubseteq C$.)
The \emph{size} $|\Tmc|$ of a TBox $\Tmc$ is the number of occurrences of symbols in $\Tmc$.

The semantics of TBoxes is given by \emph{interpretations}
$\Imc=(\Delta^\Imc,\cdot^\Imc)$, where the \emph{domain} $\Delta^\Imc$ is a
non-empty set and the \emph{interpretation function} $\cdot^\Imc$ maps
each concept name $A\in\NC$ to a subset $A^\Imc$ of $\Delta^\Imc$,
and each role name $R\in\NR$ to a binary relation $R^\Imc$ on
$\Delta^\Imc$. The extension of $\cdot^\Imc$ to arbitrary
concepts is defined inductively as shown in the third column of
Table~\ref{tab:syntax-semantics}.
We say that an interpretation \Imc \emph{satisfies} a CI $C
\sqsubseteq D$ if $C^\Imc \subseteq D^\Imc$, and that \Imc is a
\emph{model} of a TBox $\Tmc$ if $\Imc$ satisfies all the CIs in $\Tmc$.
A TBox is \emph{consistent} (or \emph{satisfiable}) if it has a
model. A concept $C$ is \emph{satisfiable with respect to~\Tmc} if there exists
a model $\Imc$ of $\Tmc$ such that $C^{\Imc}\neq\emptyset$.  A concept
$C$ is \emph{subsumed by a concept $D$ with respect to~\Tmc} ($\Tmc \models C
\sqsubseteq D$, in symbols) if every model \Imc of $\Tmc$ satisfies
the CI $C \sqsubseteq D$. For TBoxes $\Tmc_1$ and $\Tmc_2$, we write
$\Tmc_1\models \Tmc_2$ and say that \emph{$\Tmc_1$ entails $\Tmc_2$}
if $\Tmc_1\models \alpha$ for all $\alpha\in \Tmc_2$.  TBoxes
$\Tmc_{1}$ and $\Tmc_{2}$ are \emph{logically equivalent} if they have
the same models.  This is the case if and only if $\Tmc_{1}$ entails
$\Tmc_{2}$, and vice versa.

We next define two syntactic fragments of \ALC for which query answering (see below) is tractable in data complexity. 
The fragment of \ALC obtained by disallowing the constructors $\bot$, $\neg$, $\sqcup$ and
$\forall$ is known as \EL. Thus, $\mathcal{EL}$ concepts are
constructed using $\top$, $\sqcap$ and $\exists$ only \cite{BaBL05}.
A more expressive fragment with tractable query answering is Horn$\mathcal{ALC}$.
Following \cite{HuMS07,Kaza09}, we say, inductively, that a concept $C$ occurs positively in $C$ itself and,
if $C$ occurs positively (negatively) in $C'$, then
\begin{itemize}\itemsep 0cm
\item[--] $C$ occurs positively (respectively, negatively) in $C' \sqcup D$,
$C' \sqcap D$, $\exists R.C'$, $\forall R.C'$, $D \sqsubseteq C'$, and
\item[--] $C$ occurs
negatively (respectively, positively) in $\neg C'$ and $C' \sqsubseteq D$.
\end{itemize}
Now, we call an $\ALC$ TBox $\T$ \emph{Horn} if no concept of the form $C
\sqcup D$ occurs positively in $\T$, and no concept of the form $\neg
C$ or $\forall R.C$ occurs negatively in $\T$. In the DL \textsl{Horn}$\ALC$, only Horn TBoxes are allowed.

In DL, data is represented in the form of ABoxes. To introduce ABoxes,
we fix a countably-infinite set \NI of \emph{individual names}, which correspond to
individual constants in first-order logic. An \emph{assertion} is an
expression of the form $A(a)$ or $R(a,b)$, where $A$ is a concept
name, $R$ a role name, and $a,b$ individual names.  An
\emph{ABox} \Amc is a finite set of assertions. We call the pair
$\Kmc=(\Tmc,\Amc)$ of a TBox $\Tmc$ in a DL $\Lmc$ and an ABox $\Amc$
an \emph{$\Lmc$ knowledge base} (KB, for short). By ${\sf ind}(\A)$
and $\ind(\K)$, we denote the set of individual names in $\A$ and
$\K$, respectively.

To interpret ABoxes $\Amc$, we consider interpretations $\Imc$ that map all
individual names $a\in \ind(\A)$ to elements $a^{\Imc}\in \Delta^{\Imc}$ in
such a way that $a^{\Imc}\not=b^{\Imc}$ if $a\ne b$ (thus, we adopt the
\emph{unique name assumption}). It is to be noted that we do not assume all the individual names from \NI to be interpreted in $\Imc$. Sometimes, we make the \emph{standard name assumption}, that is, set $a^{\I}=a$, for all the relevant $a$. Both assumptions are without loss of generality as it is well known, and easy to check, that in $\mathcal{ALC}$ the certain answers to (unions of) 
conjunctive queries, as defined below, do not depend on the unique name assumption.  We say that $\Imc$ \emph{satisfies} assertions $A(a)$ and $R(a,b)$ if $a^\Imc \in A^\Imc$ and, respectively, $(a^\Imc,b^\Imc) \in R^\Imc$.  It is a \emph{model} of an ABox \Amc if it
satisfies all the assertions in \Amc, and it is a \emph{model} of a KB $\Kmc=(\Tmc,\Amc)$ if it is a model
of both \Tmc and \Amc. We say that \Kmc is \emph{consistent} (or
\emph{satisfiable}) if it has a model.  We apply the TBox terminology introduced above to KBs as well. For example, KBs $\K_{1}$ and $\K_{2}$ are
\emph{logically equivalent} if they have the same models (or, equivalently,
entail each other).

We next introduce query answering over KBs, starting with conjunctive
queries~\cite{GliHoLuSa-JAIR08,CalvaneseGLLR06,CalvaneseEO07}.
An \emph{atom} takes the form $A(x)$ or $R(x,y)$, where $x,y$ are from
a set of \emph{individual variables} \NV, $A$ is a concept name, and $R$ a role
name.  A \emph{conjunctive query} (or CQ) is an expression of the form
$\q(\avec{x}) = \exists \avec{y} \, \vp(\avec{x},\avec{y})$, where
$\avec{x}$ and $\avec{y}$ are disjoint sequences of variables and $\vp$
is a conjunction of atoms that only contain variables from $\avec{x}
\cup \avec{y}$---we (ab)use set-theoretic notation for sequences where
convenient. We often write $A(x) \in \q$ and $R(x,y) \in \q$ to
indicate that $A(x)$ and $R(x,y)$ are conjuncts of $\varphi$. We call
a CQ $\q(\avec{x}) = \exists \avec{y} \, \vp(\avec{x},\avec{y})$ \emph{rooted} (or an \RCQ) if every
$y \in \avec{y}$ is connected to some $x \in \avec{x}$ by a path in the undirected graph whose nodes
are the variables in $\q$ and edges are the pairs $\{u,v\}$ with
$R(u,v)\in\q$, for some $R$.
A \emph{union of CQs} (UCQ) is a disjunction $\q(\avec{x}) = \bigvee_i \q_i(\avec{x})$ of CQs $\q_i(\avec{x})$ with the same \emph{answer variables} $\avec{x}$; it is \emph{rooted} (\RUCQ) if all the $\q_i$ are rooted. If the sequence $\avec{x}$ is empty,
$\q(\avec{x})$ is called a \emph{Boolean} CQ or UCQ. Observe that no Boolean query is rooted.
\begin{example}
\em The CQ $\q(x_{1},x_{2})= \exists y_{1}\exists y_{2} (R(x_{1},y_{1}) \wedge S(x_{2},y_{2}))$ is an rCQ but
$\q(x_{1})= \exists x_{2}\exists y_{1}\exists y_{2} (R(x_{1},y_{1}) \wedge S(x_{2},y_{2}))$ is not an rCQ.
\end{example}
Given a UCQ $\q(\avec{x}) = \bigvee_i \q_i(\avec{x})$ with
$\avec{x}=x_{1},\ldots,x_{k}$ and a KB \Kmc, a sequence
$\avec{a}=a_{1},\ldots,a_{k}$ of individual names from $\Kmc$ is called a
\emph{certain answer to $\q(\avec{x})$ over \Kmc} if, for every model \Imc of
\Kmc, there exist a CQ $\q_i$ in $\q$ and a map (\emph{homomorphism}) $h$ of its
variables to $\Delta^\Imc$ such that $h(x_{j})= a_{j}^{\I}$, for $1\leq j
\leq k$, $A(z) \in \q_i$ implies $h(z) \in A^\Imc$, and $R(z,z') \in \q_i$
implies $(h(z),h(z')) \in R^\Imc$.  If this is the case, we write $\K
\models\q(\avec{a})$. For a Boolean UCQ $\q$, we say that the certain
answer to $\q$ over \Kmc is `yes' if $\K \models\q$ and `no' otherwise.  \emph{CQ} or
\emph{UCQ answering} means to decide---given a CQ or UCQ $\q(\avec{x})$, a KB
$\Kmc$ and a tuple $\avec{a}$ from $\ind(\K)$---whether $\K
\models\q(\avec{a})$.
\begin{example}\label{andrea}\em
To see that $a$ is a certain answer to the CQ $\q'(x)$ over the KB $\K = (\T_1,\A')$ from Example~\ref{ex:2}, 
we observe that, by the axiom of $\T_1$, we have $c \in \textit{Undergraduate}^\I$ or $c \in \textit{Graduate}^\I$ in any model $\I$  of $\K$. In the former case, the map $h_1$ with $h_{1}(x)=a$, $h_1(y)=c$ and $h_1(z)=d$ is a homomorphism from $\q'$ to $\I$, 
while in the latter one, $h_2$ with $h_{2}(x)=a$, $h_2(y)=b$ and $h_2(z)=c$ is such a homomorphism. 
%
%Let $\T= \{E \sqsubseteq A \sqcup B\}$.\nb{coordinate with the intro} Then $(\T,\A) \models \q(a)$, for the
%  ABox $\A$ and CQ $\q$ depicted below.
%This well-known example~\cite{Scha94b} shows that, for DLs with disjunction, `case distinctions'
%  might be required to answer queries.
%%
%\begin{center}
%\begin{tikzpicture}
%  \foreach \name/\x/\y/\conc/\wh in {%
%    a/-2/0.5//right,%
%    b/-1/-0.3/A/above, %
%    c/0/0.5/E/above, %
%    d/1.5/0.5/B/above%
%  }{ \node[inner sep=1, outer sep=0, label=\wh:{\footnotesize$\conc$}] (\name)
%    at (\x,\y) {\small $\name$}; }
%
%  \foreach \from/\to/\wh in {%
%    a/b/below, a/c/above, b/c/below, c/d/above%
%  }{ \draw[role] (\from) -- node[\wh,sloped] {\footnotesize$R$} (\to); }
%
%  \node[anchor=east] at (-2.5,0.2) {$\A$:};
%
%  \begin{scope}[xshift=6cm, yshift=0.2cm]
%    \foreach \name/\x/\conc/\wh in {%
%      x/-1.5//right,%
%      y_1/0/A/above, %
%      y_2/1.5/B/above%
%    }{ \node[inner sep=1, outer sep=0, label=\wh:{\footnotesize$\conc$}]
%      (\name) at (\x,0) {\small $\name$ }; }
%
%    \foreach \from/\to/\wh in {%
%      x/y_1/right, y_1/y_2/right%
%    }{ \draw[role] (\from) -- node[above] {\footnotesize$R$} (\to); }
%
%    \node[anchor=east] at (-2,0) {$\q(x)$:};
%  \end{scope}
%\end{tikzpicture}
%\end{center}
\end{example}

A \emph{signature}, $\Sigma$, is a finite set of concept and role names.  The
\emph{signature} $\sig(C)$ of a concept $C$ is the set of concept and role
names that occur in $C$, and likewise for TBoxes $\T$, CIs $C \sqsubseteq D$,
assertions $R(a,b)$ and $A(a)$, ABoxes $\Amc$, KBs~$\K$, UCQs~$\q$. Note that
individual names are not in any signature and, in particular, not in the
signature of an assertion, ABox or KB. We are often interested in concepts, TBoxes, KBs, and ABoxes formulated using a specific
signature $\Sigma$, in which case we use the terms $\Sigma$-\emph{concept},  $\Sigma$-\emph{TBox},  $\Sigma$-\emph{KB}, etc.
When dealing with $\Sigma$-KBs, it mostly suffices to consider $\Sigma$-\emph{interpretations} $\I$ where $X^{\I}=\emptyset$
for all concept and role names $X\not\in \Sigma$. A $\Sigma$-\emph{model} of a KB
is a $\Sigma$-interpretation that is a model of the KB. The $\Sigma$-\emph{reduct} $\J$ of an interpretation $\I$
is obtained from $\I$ by setting $\Delta^{\J}=\Delta^{\I}$, $A^{\J}=A^{\I}$ for all concept names $A\in \Sigma$,
$R^{\J}=R^{\I}$ for all role names $R\in \Sigma$, and $A^{\J}=R^{\J}=\emptyset$ for all remaining concept names $A$ 
and role names $R$.  

To compute the certain answers to queries over a KB $\K$, it is  convenient to work with a `small' subset $\Mod$ of $\sig(\K)$-models of $\K$
that is \emph{complete for} $\K$ in the sense that, for any UCQ $\q(\avec{x})$ and any $\avec{a} \subseteq \ind(\K)$, we have
$\K \models \q(\avec{a})$ iff $\I \models
\q(\avec{a})$ for all $\I\in \Mod$. We shall frequently use the following characterisation of complete sets of models  based on (partial) homomorphisms.

Suppose $\I$ and $\J$ are interpretations and $\Sigma$ a signature. A function
$h \colon \Delta^{\I} \to \Delta^{\J}$ is called a \emph{$\Sigma$-homomorphism}
if $u \in A^{\I}$ implies $h(u) \in A^{\J}$ and $(u,v) \in R^{\I}$ implies
$(h(u),h(v)) \in R^{\J}$, for all $u,v \in \Delta^{\I}$, $\Sigma$-concept names
$A$, and $\Sigma$-role names $R$. If $\Sigma$ is the set of all concept and
role names, then $h$ is called simply a \emph{homomorphism}.  We say that $h$
\emph{preserves a set $N$ of individual names} if $h(a^{\I}) = a^{\J}$, for all $a \in N$ that are defined in $\I$.
It is known from database theory that homomorphisms characterise CQ-containment~\cite{ChandraM77}.
To characterise completeness for KBs, we require finite partial homomorphisms.
An interpretation $\Imc$ is a \emph{subinterpretation} of an interpretation $\Jmc$ (\emph{induced by} a set $\Delta$) if
$\Delta = \Delta^{\I} \subseteq \Delta^{\J}$, $A^{\I}=A^{\J}\cap \Delta^{\I}$ for all concept names $A$,
$R^{\I} = R^{\J} \cap (\Delta^{\I} \times \Delta^{\I})$ for all role names $R$, and the interpretation $a^{\I}$ of an
individual name $a$ is defined exactly if $a^{\J}\in \Delta^{\I}$, in which case $a^{\I}=a^{\J}$.
%If $\Delta=\Delta^{\I}$, then we say that \emph{$\I$ is the subinterpretation of $\J$ induced by $\Delta$}.\nb{bad sentence?}
For a natural number $n$, we say that an interpretation $\I$ is \emph{$n\Sigma$-homomorphically embeddable into an interpretation $\J$} if,
for any subinterpretation $\I'$ of $\I$ with $|\Delta^{\I'}| \le n$, there is a $\Sigma$-homomorphism from $\I'$ to $\J$.
If $\Sigma$ is the set of all concept and role names, then we omit $\Sigma$ and speak about \emph{$n$-homomorphic embeddability}.
If we require all $\Sigma$-homomorphisms to preserve a set $N$
of individual names, then we speak about \emph{$n\Sigma$-homomorphic embeddability preserving $N$}. 
\begin{example}
\em Let $\Imc$ and $\J$ be interpretations whose domain is the set $\mathbb N$ of natural numbers and, for any $n,m \in \mathbb N$, we have 
$(n,m)\in R^{\I}$ if $m=n+1$, and $(n,m)\in R^{\J}$ if $n=m+1$. 
Then, for all $n\geq 0$, $\I$ is $n$-homomorphically embeddable into $\J$, 
but $\I$ is not homomorphically embeddable into $\J$. Now, let $a^{\I}=0$, $a^{\J}=m$,
and $N=\{a\}$. Then $\I$ is $(m+1)$-homomorphically embeddable into $\J$ preserving $N$, but $\I$ is not $(m+2)$-homomorphically
embeddable into $\J$ preserving $N$.
\end{example}  

\begin{proposition}\label{prop:char}
A set $\Mod$ of $\sig(\K)$-models of an $\mathcal{ALC}$ KB $\K$ is complete for $\K$ iff, for any model $\J$ of $\K$ and any $n>0$,
there is $\I\in \Mod$ such that $\I$ is $n$-homomorphically embeddable into $\J$ preserving $\ind(\K)$.
\end{proposition}
\begin{proof}
Let $\Sigma=\sig(\K)$ and let $\Mod$ be a class of $\Sigma$-models of $\K$.
Suppose first that $\Mod$ is not complete for $\K$.
Then there exist a UCQ $\q(\avec{x})$ and a tuple $\avec{a}$ from $\ind(\K)$ such that
$\K\not\models \q(\avec{a})$ but $\I \models \q(\avec{a})$ for all $\I\in \Mod$. Let $\J$ be a model of $\K$ such that
$\J\not\models \q(\avec{a})$ and let $n$ be the number of variables in $\q(\avec{x})$. For every $\I\in \Mod$, there exists
a subinterpretation $\I'$ of $\I$ with $|\Delta^{\I'}|\leq n$ and $\I'\models \q(\avec{a})$. No such $\I'$ is
homomorphically embeddable into $\J$ preserving $\avec{a}$, and so no $\I\in \Mod$ is $n$-homomorphically
embeddable into $\J$ preserving $\ind(\K)$.

Conversely, suppose there exists a model $\J$ of $\K$ and $n>0$ such that no $\I\in \Mod$ is $n$-homomorphically embeddable
into $\J$ preserving $\ind(\K)$. Let $\ind(\K)=\{a_{1},\dots,a_{k}\}$. For every finite
$\Sigma$-interpretation $\I$ with domain $\{u_{1},\dots,u_{m}\}$ such that $m\geq k$ and
$a_{i} =u_{i}$ ($1\leq i \leq k$), we define the \emph{canonical CQ} $\q_{\I}$ by taking 
$$
\q_{\I}(x_{1},\dots,x_{k}) ~=~ \exists x_{k+1}\cdots\exists x_{m} \, \Big( \bigwedge_{u_{i}\in A^{\I}\!,A\in \Sigma}A(x_{i}) \wedge
\bigwedge_{(u_{i},u_{j})\in R^{\I}\!,R\in \Sigma}R(x_{i},x_{j})\Big).
$$
%
%\varphi(x_{1},\dots,x_{k},x_{k+1},\ldots,x_{m})$ by taking
%
%$$
%\varphi(x_{1},\ldots,x_{k},x_{k+1},\dots,x_{m}) ~=~
%
Then there exists a homomorphism from $\I$ to $\J$ preserving $\ind(\K)$ iff $\J\models q_{\I}(a_{1},\ldots,a_{k})$.
Now pick for any $\I\in \Mod$ a subinterpretation $\I'$ of $\I$ with $\Delta^{\I'} \supseteq \ind(\K)$ and $|\Delta^{\I'}\setminus\ind(\K)|\leq n$
such that $\I'$ is not homomorphically embeddable into $\J$ preserving $\ind(\K)$. Let
$\q(x_{1},\ldots,x_{k})$ be the disjunction of all canonical CQs $\q_{\I'}(x_{1},\ldots,x_{k})$ determined by these $\I'$.
Then $\J\not\models\q(a_{1},\ldots,a_{k})$, and so $\K\not\models \q(a_{1},\ldots,a_{k})$, but $\I\models \q(a_{1},\ldots,a_{k})$, for all $\I\in \Mod$.
\end{proof}

Observe that, in the characterisation of Proposition~\ref{prop:char}, one cannot replace $n$-homomorphic embeddability by homomorphic embeddability as shown by the following example.
\begin{example}\label{ex:hom1}\em
Let $\K= \left(\{\top \sqsubseteq \exists R.\top\},\{A(a)\}\right)$. Then the class $\M$ of all interpretations that consist of a finite $R$-chain starting with $A(a)$ and followed by an $R$-cycle (of arbitrary length) is
complete for $\K$. However, there is no homomorphism from any member of $\M$ into the
model of $\K$ that consists of an infinite $R$-chain starting from $A(a)$.
\end{example}
We call an interpretation $\I$ a \emph{ditree interpretation}
if the directed graph $G_\I$ defined by taking
$$
G_{\I} = (\Delta^{\I}, \{(d,e) \mid (d,e)\in \bigcup_{R\in \NR} R^{\I}\})
$$
is a directed tree and $R^{\I}\cap S^{\I}=\emptyset$, for any
distinct role names $R$ and $S$.  $\I$ has \emph{outdegree} $n$ if $G_{\I}$ has
outdegree $n$.  A model $\I$ of $\K=(\T,\A)$ is \emph{forest-shaped} if
$\I$ is the disjoint union of ditree interpretations $\I_{a}$ with root $a$, for
$a\in \ind(\A)$, extended with all $R(a,b)\in \A$. In this case, the \emph{outdegree} of $\I$
is the maximum outdegree of the interpretations $\I_{a}$, for $a\in \ind(\A)$.
Denote by $\Mod^{\it bo}_\K$ the class of all forest-shaped $\sig(\K)$-models of $\K$ of outdegree $\le |\T|$. The following completeness result is well known \cite{Lutz-IJCAR08} (the first part is
shown in the proof of Proposition~\ref{regularcomplete}):
\begin{proposition}\label{forestcomplete}
$\Mod^{\it bo}_\K$ is complete for any $\mathcal{ALC}$ KB $\K$. If $\K$ is a \hALC
KB, then there is a single member $\I_{\K}$ of $\Mod^{\it bo}_\K$ that is complete for $\K$.
\end{proposition}
The model $\I_{\K}$ mentioned in Proposition~\ref{forestcomplete} is constructed using the standard chase procedure and called the
\emph{canonical model of $\K$}.
Proposition~\ref{forestcomplete} can be strengthened further. Call a subinterpretation $\I$ of a ditree interpretation
$\J$ a \emph{rooted subinterpretation of $\J$} if there exists $u\in \Delta^{\J}$ such that the domain $\Delta^{\I}$ of $\I$
is the set of all $u'\in \Delta^{\J}$ for which there is a path $u_{0},\ldots,u_{n}\in \Delta^{\J}$ with $u_{0}=u$, $u_{n}=u'$
and $(u_{i},u_{i+1})\in R_i^{\I}$ ($i<n$), for some role name $R_i$.
Call a ditree interpretation $\Imc$ \emph{regular}
if it has, up to isomorphism, only finitely many rooted subinterpretations. A forest-shaped model
$\Imc$ of a KB $\K$ is \emph{regular} if the ditree interpretations $\Imc_{a}$, $a\in {\sf ind}(\K)$, are regular.
Denote by $\Mod^{\it reg}_{\K}$ the class of all regular forest-shaped $\sig(\K)$-models of $\K = (\T,\A)$ of outdegree
bounded by $|\T|$.
\begin{proposition}\label{regularcomplete}
$\Mod^{\it reg}_\K$ is complete for any $\mathcal{ALC}$ KB $\K$.
%If $\K$ is a \hALC
%KB, then a single member $\I_{\K}$ of $\Mod^{\it fo}_\K$ is complete for $\K$.
\end{proposition}
\begin{proof}
Suppose $\K$ is an \ALC KB and $\K\not\models \q(\avec{a})$, for some UCQ $\q(\avec{x})$.
%It suffices
%to show that there exists an interpretation $\I\in \Mod^{\it reg}_\K$ with $\I\not\models \q(\avec{a})$.
As shown in \cite{Lutz-IJCAR08}, there exists a consistent KB
$\K'=(\T',\A')$ with $\T'\supseteq \T$, $\A' \supseteq \A$, and
$\ind(\Amc')=\ind(\Amc)$ such that $\I\not\models \q(\avec{a})$, for
every model $\I$ of $\K'$ (called a \emph{spoiler for $\q$ and} $\K$ in
\cite{Lutz-IJCAR08} and constructed by carefully analyzing all
possible homomorphism from $\q$ to models of \Kmc and `spoiling' all of
them by suitable KB extensions). We construct a regular model $\J'$
of $\K'$ as follows. Let $\I'$ be a model of $\K'$. We may assume that
$\T'$ does not use the constructor $\forall r.C$. Denote by
$\mathsf{cl}(\T')$ the set of subconcepts of concepts in $\T'$ closed
under single negation. For $d \in \Delta^{\I'}$, the \emph{$\T'$-type
  of $d$ in $\I'$}, denoted $\type_{\T'}^{\I'}(d)$, is defined as
$
\type_{\T'}^{\I'}(d)= \{ C\in \mathsf{cl}(\T') \mid d\in C^{\I'}\}.
$
A subset $\type \subseteq \mathsf{cl}(\T')$ is a \emph{$\T'$-type}
if $\type = \type_{\T'}^{\I}(d)$, for some model $\I$ of $\T'$ and $d\in \Delta^{\I}$.  We denote the set of all
$\T'$-types by $\types(\T')$.  Let $\type,\type' \in \types(\T')$.
% We write $\type \rightsquigarrow_R \type'$ if $\midsqcap \type \sqcap
% \SOME{R}{(\midsqcap \type')}$ is satisfiable w.r.t.\ $\T_2$.
For $\exists R.C \in \type$, we say that $\type'$ is an
\emph{$\exists R.C$-witness for $\type$} if $C \in \type'$ and the concept
$\midsqcap \type \sqcap \exists R.(\midsqcap \type')$ is satisfiable with respect to
$\T'$. Denote by $\succ_{\exists R.C}(\type)$ the set of all $\exists R.C$-witnesses for $\type$.
Now choose, for any $\T'$-type $\type$ and $\exists R.C$ such that $\succ_{\exists R.C}(\type)\not=\emptyset$, 
a single type $s_{\exists R.C}(\type)\in \succ_{\exists R.C}(\type)$. We construct the
model $\J'$ of $\K'$ as follows. The domain $\Delta^{\J'}$ is the set of words
$$
aR_{1}\type_{1}\cdots R_{n}\type_{n},
$$
where $a\in \ind(\K')$ and, for $\type_{0}=\type_{\T'}^{\I'}(a)$ and $i< n$,
$\type_{i+1}=s_{\exists R_{i+1}.C}(\type_{i})$ for some $\exists R_{i+1}.C\in \type_{i}$.
Set $aR_{1}\type_{1}\cdots R_{n}\type_{n} \in A^{\J'}$ if $n=0$ and
$A\in \type_{\T'}^{\I'}(a)$ or $n>0$ and $A\in \type_{n}$. Finally, set
$(aR_{1}\type_{1}\cdots R_{n}\type_{n},bS_{1}\type_{1}'\cdots S_{m}\type_{m}')\in R^{\J'}$
iff $n=m=0$ and $R(a,b)\in \A$ or $0<m=n+1$, $S_{m}=R$ and $aR_{1}\type_{1}\cdots \type_{n}=
bS_{1}\type_{1}'\cdots \type_{m-1}'$. One can easily show that $\J'$ is a regular model of $\K'$.
Hence $\J'\not\models \q(\avec{a})$. The outdegree of $\J'$ is bounded by $|\T'|$ but possibly not by $|\T|$, and so it 
remains to modify $\J'$ in such a way that its outdegree is bounded by $|\T|$. To this end, we remove from
$\J'$ all $R$-successors (together with the subtrees they root)
$aR_{1}\type_{1}\cdots R_{n}\type_{n}R\type$ of all $aR_{1}\type_{1}\cdots R_{n}\type_{n}\in \Delta^{\J'}$ such that $\type\not=s_{\exists R.C}(\type_{n})$
for any $\exists R.C\in \mathsf{cl}(\T)$. By the construction, the resulting interpretation $\J$ is still regular,
it is a model of $\K$ (since $\T'\supseteq \T$), its outdegree is bounded by $|\T|$, and $\J\not\models \q(\avec{a})$ since
$\J'\not\models \q(\avec{a})$. 
\end{proof}
\begin{example}\em
  Consider the KB $\K=(\T,\A)$ with
  $\T=\{A \sqcup B \sqsubseteq \exists R.(A\sqcup B)\}$ and $\A=\{A(a)\}$.  The
  following class of regular models $\I$ is complete for $\K$. The domain of
  $\I$ is the natural numbers with $a^{\I}=0 \in A^{\I}$, $(i,j)\in R^{\I}$ if $j=i+1$, for all natural
  numbers $i$ and $j$, and there are
  $k,n,m\geq 0$ such that $A^{\I}$ and $B^{\I}$ are mutually disjoint, cover
  the initial segment $\{1,\dots,k\}$ and, on the remainder $\{k+1,\dots\}$,
  they are interpreted by alternating between $n$ consecutive nodes in $A^{\I}$
  and $m$ consecutive nodes in $B^{\I}$. Then $\I$ is regular since the number of
  non-isomorphic rooted subinterpretations of $\I$ with root $r> k$ is $\leq n+m$ (the number
  of non-isomorphic rooted subinterpretations of $\I$ with root $r\leq k$ is clearly bounded by $k+1$).  
\end{example}
In the undecidability proofs of Section~\ref{sec:undecidability}, we do not use the full expressive power of $\ALC$
but work with a small fragment denoted \ELUf.
An \ELUf \emph{TBox} $\T$ consists of CIs of the form
\begin{itemize}\itemsep 0cm
\item[--] $A \sqsubseteq C$,
\item[--] $A \sqsubseteq C \sqcup D$,
%\item[--] $A \ISA \SOME{R}{C}$,
\end{itemize}
where $A$ is a concept name and $C,D$ are $\EL$-concepts. Given an \ELUf KB $\K = (\T,\A)$, we
construct by induction a (possibly infinite) labelled forest $\mathfrak{O}$
with a labelling function $\ell$. For each $a \in \ind(\A)$, $a$ is the root of
a tree in $\mathfrak{O}$ with $A \in \ell(a)$ iff $A(a) \in \A$.
Suppose now that $\sigma$ is a node in $\mathfrak{O}$ and $A \in
\ell(\sigma)$. If $A \sqsubseteq C$ is an axiom of $\T$ and $C \notin
\ell(\sigma)$, then we add $C$ to $\ell(\sigma)$. If $A \sqsubseteq C \sqcup D$ is
an axiom of $\T$ and neither $C \in \ell(\sigma)$ nor $D \in \ell(\sigma)$,
then we add to $\ell(\sigma)$ either $C$ or $D$ (but not both); in this case, we
call $\sigma$ an \emph{or-node}.
If $C \sqcap D \in \ell(\sigma)$, then we add both $C$ and $D$
to $\ell(\sigma)$ provided that they are not there yet. Finally, if $\exists R.C \in \ell(\sigma)$ and
the constructed part of the tree does not contain a node of the form $\sigma \cdot
w_{\exists R.C}$, then we add $\sigma \cdot w_{\exists R.C}$ as an
$R$-successor of $\sigma$ and set $\ell(\sigma \cdot w_{\exists R.C}) = \{C\}$.
Now we define a \emph{minimal model} $\I =
(\Delta^\I, \cdot^\I)$ of $\K$ by taking $\Delta^\I$ to be the set of nodes in
$\mathfrak{O}$, $a^\I = a$ for $a\in\ind(\A)$, $R^\I$ to be the
$R$-relation in $\mathfrak{O}$ together with $(a,b)$ such that $R(a,b) \in \A$,
and
$
A^\I = \{~\sigma \in \Delta^\I \mid A \in \ell(\sigma)~\},
$
for every concept name $A$. It follows from the construction that $\I$ is a model of $\K$.
\begin{lemma}\label{min-elu-complete}
For any \ELUf KB $\K$, the set $\Mod_\K$ of its minimal models is complete for $\K$.
\end{lemma}
\begin{proof}
By Proposition~\ref{prop:char}, it suffices to show that, 
for every model $\J$ of $\K$, there is a minimal model $\I$ that is homomorphically embeddable into $\J$ 
preserving $\ind(\K)$. 
Suppose a model $\J$ of $\K$ is given.  We can now inductively construct a set
$\Delta$, a labelling function $\ell$ defining a minimal model~$\I$, and a
homomorphism $h$ from $\I$ to $\J$ such that $h(\sigma) \in C^\J$, for each $C
\in \ell(\sigma)$ and $\sigma \in \Delta$. The model $\J$ is used as a guide.
For instance, let $\sigma \in \Delta$ such that $h(\sigma)$ is set.  Suppose
that $A \in \ell(\sigma)$, $A \sqsubseteq C \sqcup D$ is an axiom in $\T$, and
$C \notin \ell(\sigma)$, $D \notin \ell(\sigma)$. Since $\J$ is a model of
$\K$, it must be the case that $h(\sigma)^\J \in C^\J$ or $h(\sigma)^\J \in
D^\J$. In the former case, we add $C$ to $\ell(\sigma)$, in the latter case, we
add $D$ to $\ell(\sigma)$.
Suppose further that $\sigma \cdot w_{\exists R.C}$ is in $\Delta$ and
$h(\sigma \cdot w_{\exists R.C})$ is not set. Since $\J$ is a model of $\K$ and
by inductive assumption $h(\sigma) \in (\exists R.C)^\J$, there exists $d \in
\Delta^\J$ such that $(h(\sigma),d) \in R^\J$ and $d \in C^\J$. So we set
$h(\sigma \cdot w_{\exists R.C}) = d$.

Now we take the minimal model $\I = (\Delta, \cdot^\I)$, where $\cdot^\I$
is defined according to the labelling function $\ell$. By the construction of
$\Delta$ and the fact that $\I$ is minimal, we obtain that $h$ is indeed a
homomorphism from $\I$ to $\J$.
  %
  % (ii) Assume that $\M = \prod_{i=1}^N \I_i$, where $N \leq \omega$ and $\I_1, \dots,
  % \I_N$ are all minimal models of $\K$.
  % %
  % Let $\J$ be a model of $\K$. Then there exists an unravelling $\I_i$, for $1
  % \leq i\leq \omega$, such that $\I_i$ is homomorphically embeddable into $\J$.
  % We show that $\M$ is homomorphically embeddable into $\I_i$, from which the
  % claim follows.
  % %
  % Let $\pi_i : \dom[\M] \to \dom[\I_i]$ be the projection function defined as
  % $\pi_i\big( (d_1, \dots, d_\omega) \big) = d_i$. It is straightforward to
  % verify that $\pi_i$ is a homomorphism from $\M$ to $\I_i$.
\end{proof}

%%% Local Variables:
%%% mode: latex
%%% TeX-master: "aij-insep"
%%% TeX-PDF-mode: t
%%% save-place: t
%%% End:

% !TEX root = aij-insep.tex

\section{Model-Theoretic Criteria for Query Entailment and Inseparability between Knowledge Bases}\label{sec:criteria}

In this section, we first define the central notions of query entailment and inseparability between KBs for CQs and UCQs as well as
their restrictions to rooted queries. Then we give model-theoretic characterisations of these notions based on products of interpretations and (partial) homomorphisms.
\begin{definition}\em
Let $\K_{1}$ and $\K_{2}$ be consistent KBs, $\Sigma$ a signature, and $\mathcal{Q}$ one of CQ, \RCQ, UCQ or \RUCQ. We say that $\K_1$ \emph{$\Sigma$-$\mathcal{Q}$-entails} $\K_2$ if
$\K_2\models \q(\avec{a})$ implies $\avec{a} \subseteq \ind(\K_1)$ and \mbox{$\K_1\models \q(\avec{a})$}, for all $\Sigma$-$\mathcal{Q}$ $\q(\avec{x})$ and all tuples $\avec{a}$ in $\ind(\K_2)$.
We say that $\K_1$ and $\K_2$ are \emph{$\Sigma$-$\mathcal{Q}$ inseparable} if
they $\Sigma$-$\mathcal{Q}$ entail each other.
If $\Sigma$ is the set of all concept and role names, we say `\emph{full signature $\mathcal{Q}$-entails}' or `\emph{full signature $\mathcal{Q}$-inseparable}'\!.
\end{definition}
As larger classes of queries separate more KBs, $\Sigma$-UCQ inseparability implies all other inseparabilities and
$\Sigma$-CQ inseparability implies $\Sigma$-\RCQ inseparability. The following example 
shows that, in general, no other implications between the different notions of inseparability hold for $\ALC$.
\begin{example}\label{ex:ex1}\em
  Suppose $\T_{0}=\emptyset$, $\T_{0}'= \{ E \sqsubseteq A \sqcup B\}$ and
  $\Sigma_{0}= \{A,B,E\}$.  Let $\A_{0}=\{ E(a)\}$, $\K_{0}=(\T_{0},\A_{0})$,
  and $\K_{0}'=(\T_{0}',\A_{0})$. Then $\K_{0}$ and $\K_{0}'$ are
  $\Sigma_{0}$-CQ inseparable (and so also $\Sigma_{0}$-\RCQ inseparable) but not $\Sigma_{0}$-\RUCQ inseparable (and
  so also not $\Sigma_{0}$-UCQ inseparable).  
  The former claim can be proved
  using the model-theoretic criterion given in Theorem~\ref{crit:KB} below, and the latter one follows from
  $\K_{0}'\models \q(a)$ and $\K_{0}\not\models \q(a)$, for $\q(x)= A(x)\vee B(x)$.

  Now, let $\Sigma_{1}= \{ E,B\}$, $\T_{1}=\emptyset$, and $\T_{1}'= \{ E
  \sqsubseteq \exists R.B\}$.  Let $\A_{1}=\{ E(a)\}$, $\K_{1}=(\T_{1},\A_{1})$,
  and $\K_{1}'=(\T_{1}',\A_{1})$.  Then $\K_{1}$ and $\K_{1}'$ are
  $\Sigma_{1}$-rUCQ inseparable (and so also $\Sigma_{1}$-rCQ inseparable) but not $\Sigma_{1}$-CQ inseparable. 
  The former claim can be proved
  using the model-theoretic criterion of Theorem~\ref{crit:KB} and the latter one follows from
  the observation that $\K_{1}'\models \exists x B(x)$ but $\K_{1}\not\models \exists x B(x)$.
%It follows that $\T_{1}$ and $\T_{2}$ are not $\Sigma$-CQ inseparable.
%We shall see below that $\T_{1}$ and $\T_{2}$ are $\Sigma$-UCQ inseparable.
\end{example}
The situation changes for \hALC KBs. The following can be easily proved by observing (using Proposition~\ref{forestcomplete}) that the certain
answers to a UCQ over a \hALC KB $\K$ coincide with the certain answers to its disjuncts over $\K$:

\begin{proposition}\label{UCQtoCQ}
Let $\K_{1}$ be an $\ALC$ KB and $\K_{2}$ a \hALC KB. Then
$\K_{1}$ $\Sigma$-UCQ entails $\K_{2}$ iff $\K_{1}$ $\Sigma$-CQ entails $\K_{2}$. The same holds for \RUCQ and \RCQ.
\end{proposition}
Now we give model-theoretic criteria of $\Sigma$-query entailment between KBs.
As usual in model theory~\cite[page~405]{ChangKeisler90}, we define the \emph{product} $\prod \avec{\I}$ of a family $\avec{\I}= \{\I_{i} \mid i \in I\}$ of interpretations by taking
\begin{eqnarray*}
\Delta^{\prod \avec{\I}} & = &\{ f\colon I \rightarrow \bigcup_{i\in I}\Delta^{\I_{i}} \mid \forall i\in I \, f(i)\in \Delta^{\I_{i}}\},\\
A^{\prod \avec{\I}} & = &\{ f  \mid \forall i\in I\, f(i) \in A^{\I_{i}}\},\\
R^{\prod \avec{\I}} & = &\{ (f,g) \mid \forall i\in I\, (f(i),g(i)) \in R^{\I_{i}}\},\\
a^{\prod \avec{\I}} & = & f_{a}, \ \mbox{ where $f_{a}(i)=a^{\I_{i}}$ for all $i\in I$}.
\end{eqnarray*}
\begin{proposition}[\cite{ChangKeisler90}]\label{prop:products}
For any CQ $\avec{q}(\avec{x})$ and any tuple $\avec{a}$ of individual names,
$\prod \avec{\I} \models \avec{q}(\avec{a})$ iff $\I \models \avec{q}(\avec{a})$ for all $\I \in \avec{\I}$.
\end{proposition}
\begin{example}\em
The KB $\K = (\T_1,\A')$ from Example~\ref{ex:2} has two minimal models: $\I_1$ that agrees with $\A'$ on $a$, $b$, $d$ and has $c \in \textit{Undergraduate}^{\I_2}$, and $\I_2$ that also agrees with $\A'$ on $a$, $b$, $d$ but has $c \in \textit{Graduate}^{\I_1}$ (cf.~Example~\ref{andrea}). By Lemma~\ref{min-elu-complete}, the set $\avec{\I} = \{\I_1,\I_2\}$ is complete for $\K$. The picture below\footnote{As usual in model theory, we write $(b,c)$ for $f$ with $f \colon 1 \mapsto b$ and $f \colon 2 \mapsto c$, and similarly for $(c,b)$, $(c,d)$ and $(d,c)$.} shows the `interesting'  part of $\prod \avec{\I}$. Clearly, $\prod \avec{\I} \models \q'(a)$, 
where $\q'$ is the CQ from Example~\ref{ex:2}. It follows that $\K \models \q'(a)$. 
\begin{center}
  \begin{tikzpicture}[xscale=1.5, yscale=1.5]
  \foreach \name/\x/\y/\lab/\conc/\wh in {%
    a/-2/0.5/a/\mathit{Lecturer}/above,%
    b/-1/-0.3/b/\mathit{Undergraduate}/below, %
    c/0/0.5/c/\mathit{Graduate}/above, %
    d/1.5/0.5/d/\mathit{Graduate}/below%
  }{ \node[inner sep=1, outer sep=0, label=\wh:{\scriptsize$\conc$}] (\name)
    at (\x,\y) {\footnotesize $\lab$}; }

  \foreach \from/\to/\wh/\lab in {%
    a/b/below/\mathit{teaches}, a/c/above/\mathit{teaches}, b/c/below/\mathit{hasFriend},
    c/d/above/\mathit{hasFriend}%
  }{ \draw[role] (\from) -- node[\wh,sloped] {\scriptsize$\mathit{\lab}$} (\to); }

  \node[anchor=east] at (-2.5,0.2) {$\I_1$:};
  
   \begin{scope}[xshift=6cm]
   
   \foreach \name/\x/\y/\lab/\conc/\wh in {%
    a/-2/0.5/a/\mathit{Lecturer}/above,%
    b/-1/-0.3/b/\mathit{Undergraduate}/below, %
    c/0/0.5/c/\mathit{Undergraduate}/above, %
    d/1.5/0.5/d/\mathit{Graduate}/below%
  }{ \node[inner sep=1, outer sep=0, label=\wh:{\scriptsize$\conc$}] (\name)
    at (\x,\y) {\footnotesize $\lab$}; }

  \foreach \from/\to/\wh/\lab in {%
    a/b/below/\mathit{teaches}, a/c/above/\mathit{teaches}, b/c/below/\mathit{hasFriend},
    c/d/above/\mathit{\quad\quad hasFriend}%
  }{ \draw[role] (\from) -- node[\wh,sloped] {\scriptsize$\mathit{\lab}$} (\to); }

  \node[anchor=east] at (-2.5,0.2) {$\I_2$:};
   
   \end{scope}
\end{tikzpicture}
  
    \begin{tikzpicture}[xscale = 3, yscale = 1.2]
      \foreach \name/\x/\y/\lab/\conc/\wh in {%
        a/0/1/{f_a}/\mathit{Lecturer}/3,%
        b/-1.5/0/{f_b}/Undergraduate/left, %
        c/-0.5/0/{f_c}//above, %
        d/-0.5/-1/{f_d}/Graduate/left,%
        bc/0.5/0/{(c,b)}//above, %
        cd/0.5/-1/{(d,c)}//above, %
        cb/1.5/0/{(b,c)}/Undergraduate/right, %
        dc/1.5/-1/{(c,d)}/Graduate/right%
      }{ \node[inner sep=1, outer sep=0, label=\wh:{\footnotesize$\conc$}]
        (\name) at (\x,\y) {\small $\lab$}; }

      \foreach \from/\to/\wh/\lab in {%
        a/b/left/teaches\quad\quad, a/c/left/teaches\ \ , b/c/below/hasFriend, c/d/right/, a/bc/right/\ teaches, bc/cd/right/hasFriend, c/d/right/hasFriend%,
%        a/cb/right/\quad\quad teaches, cb/dc/right/hasFriend%
      }{ \draw[role] (\from) -- node[\wh] {\footnotesize$\lab$} (\to); }
        
        \draw[ultra thick, ->] (a) -- node[right] {\quad \quad \footnotesize$\mathit{teaches}$}  (cb);
        
          \draw[ultra thick, ->] (cb) -- node[right] {\footnotesize$\mathit{hasFriend}$} (dc);
        
      \node[anchor=east] at (-2.5,0) {$\prod \avec{\I}$:};

    \end{tikzpicture}
  \end{center}
\end{example}
We characterise $\Sigma$-query entailment in terms of products and
$n\Sigma$-homomorphic embeddability.  To also capture rooted queries, we first
introduce the corresponding refinement of $\Sigma$-homomorphic and, respectively, $n\Sigma$-homomorphic embeddability.
A \emph{$\Sigma$-path $\rho$ from $u$ to $v$} in an interpretation $\I$ is a sequence $u_{0},\ldots,u_{n}\in \Delta^{\I}$ such that
$u_{0}=u$, $u_{n}=v$, and there are $R_{0},\ldots,R_{n-1}\in \Sigma$ with $(u_{i},u_{i+1})\in R_{i}^\I$, for $0\leq i< n$.
For a KB $\K=(\T,\A)$ and model $\I$ of $\K$, we say that $u\in \Delta^{\I}$ is \emph{$\Sigma$-connected to $\A$ in $\I$} if there exist $a\in \ind(\K)$ and a $\Sigma$-path from
$a^{\I}$ to $u$ in $\I$. The subinterpretation $\I^{con}$ of $\I$ induced by the
set of all $u\in \Delta^{\I}$ that are $\Sigma$-connected to $\A$ in $\I$ is called the \emph{$\Sigma$-component of $\I$ with respect to $\K$}.
Let $\I_{1}$ be a model of $\K_{1}$ and $\I_{2}$ a model of $\K_{2}$.
We say that $\I_2$ is \emph{con-$\Sigma$-homomorphically embeddable into $\I_1$} if the $\Sigma$-component $\I_{2}^{con}$ of $\I_{2}$ with respect to $\K_{2}$
is $\Sigma$-homomorphically embeddable into $\I_1$; and we say that $\I_2$ is \emph{con-$n\Sigma$-homomorphically embeddable into $\I_1$} if the
$\Sigma$-component $\I_{2}^{con}$ of $\I_{2}$ with respect to $\K_{2}$ is $n\Sigma$-homomorphically embeddable into $\I_{1}$. 
\begin{theorem}\label{crit:KB}
Let $\K_1$ and $\K_2$ be \ALC KBs, $\Sigma$ a signature, and let $\Mod_{\!i}=\{ \I_{j} \mid j\in I_{i}\}$
be complete for $\K_i$, $i=1,2$.
\begin{description}\itemsep=0pt
\item[\rm (1)] $\K_{1}$ $\Sigma$-UCQ entails $\K_2$ iff, for any  $n>0$ and $\I_1\in \Mod_{\!1}$, there exists $\I_2 \in \Mod_{\!2}$ that is $n\Sigma$-homomorphically embeddable into $\I_1$ preserving $\ind(\K_{2})$.
	
\item[\rm (2)] $\K_{1}$ $\Sigma$-rUCQ entails $\K_2$ iff, for any  $n>0$ and $\I_1\in \Mod_{\!1}$, there exists $\I_2 \in \Mod_{\!2}$ that is con-$n\Sigma$-homomorphically embeddable into $\I_1$ preserving $\ind(\K_{2})$.
	
\item[\rm (3)] $\K_{1}$ $\Sigma$-CQ entails $\K_2$ iff $\prod\!\Mod_{\!2}$ is $n\Sigma$-homomorphically embeddable into $\prod\!\Mod_{\!1}$ preserving $\ind(\K_2)$ for any $n>0$.
	
\item[\rm (4)] $\K_{1}$ $\Sigma$-rCQ entails $\K_2$ iff $\prod\!\Mod_{\!2}$ is con-$n\Sigma$-homomorphically embeddable into $\prod\!\Mod_{\!1}$ preserving $\ind(\K_{2})$ for any $n>0$.
\end{description}
\end{theorem}
%
%\begin{lemma}\label{crit:KB}
%  Let $\K_1$ and $\K_2$ be \ALC KBs, $\Sigma$ a signature, and let ${\bf M}_{1}$
%	and ${\bf M}_{2}$ be complete for $\K_1$ and $\K_{2}$, respectively.
%
%	\begin{enumerate}
%	\item $\K_{1}$ $\Sigma$-UCQ entails $\K_2$ iff for each $\I_1\in {\bf M}_{1}$ there
%  exists $\I_2 \in {\bf M}_{2}$ such that $\I_2$ is finitely $\Sigma$-homomorphically embeddable
%	into $\I_1$.
%  \item $\K_{1}$ $\Sigma$-rUCQ entails $\K_2$ iff for each $\I_1\in {\bf M}_{1}^{\Sigma}$ there
%  exists $\I_2 \in {\bf M}_{2}^{\Sigma}$ such that $\I_2$ is finitely $\Sigma$-homomorphically embeddable
%	into $\I_1$.
%  \item $\K_{1}$ $\Sigma$-CQ entails $\K_2$ iff $\prod {\bf M}_{2}$ is finitely $\Sigma$-homomorphically
%	embeddable into $\prod {\bf M}_{1}$.
%  \item $\K_{1}$ $\Sigma$-rUCQ entails $\K_2$ iff $(\prod {\bf M}_{2})^{\Sigma}$ is finitely $\Sigma$-homomorphically
%	embeddable into $(\prod {\bf M}_{1})^{\Sigma}$.
%	\end{enumerate}
%\end{lemma}
%
\begin{proof}
  (1) Suppose $\K_{2}\models \q(\avec{a})$ but $\K_{1}\not\models \q(\avec{a})$,
  for a $\Sigma$-UCQ $\q$ and $\avec{a}$ in $\ind(\K_{1})$. Let $n$ be the
  number of variables in $\q$.  Take $\I_{1}\in \Mod_{\!1}$ such that
  $\I_{1}\not\models \q(\avec{a})$. Then no $\I_{2}\in \Mod_{\!2}$ is
  $n\Sigma$-homomorphically embeddable into $\I_{1}$ preserving $\ind(\K_{2})$
  since this would imply $\I_{2}\not\models \q(\avec{a})$.  Conversely, suppose
  $\I_{1}\in \Mod_{\!1}$ is such that, for some $n>0$, no $\I_{2}\in
  \Mod_{\!2}$ is $n\Sigma$-homomorphically embeddable into $\I_{1}$ preserving
  $\ind(\K_{2})$. Fix such an $n>0$ and take for every $\I_{2}\in \Mod_{\!2}$ a subinterpretation 
  $\I_{2}'$ of $\I_{2}$ with domain of size $\le n$ such that $\I_{2}'$ is not $\Sigma$-homomorphically
  embeddable into $\I_{1}$ preserving $\ind(K_{2})$. Recall from the proof of Proposition~\ref{prop:char} that we can 
  regard the $\Sigma$-reduct of any such $\I_{2}'$ as a $\Sigma$-CQ (with the answer variables corresponding to the 
  ABox individuals). The disjunction of all these CQs (up to
  isomorphisms) is entailed by $\K_{2}$ but not by $\K_{1}$. The proof of (2) is similar. 

  (3) Suppose $\K_{2}\models \q(\avec{a})$ but $\K_{1}\not\models \q(\avec{a})$, 
  for a $\Sigma$-CQ $\q$ and $\avec{a}$ in $\ind(\K_{1})$. By Proposition~\ref{prop:products},
  $\prod\!\Mod_{\!2}\models \q(\avec{a})$ but $\prod\!\Mod_{\!1}\not\models \q(\avec{a})$.
  Let $n$ be the number of variables in $\q$. Then $\prod\!\Mod_{\!2}$ is not $n\Sigma$-homomorphically embeddable into 
  $\prod\!\Mod_{\!1}$ preserving $\ind(\K_2)$ since this would imply $\prod\!\Mod_{\!1}\models \q(\avec{a})$.
   Conversely, suppose that, for some $n>0$, $\prod\!\Mod_{\!2}$ is not $n\Sigma$-homomorphically embeddable into 
   $\prod\!\Mod_{\!1}$ preserving $\ind(\K_2)$. Let $\I$ be the subinterpretation of $\prod\!\Mod_{\!2}$ with domain of size 
   $\leq n$ which cannot be $\Sigma$-homomorphically embedded in $\prod\!\Mod_{\!1}$ preserving 
   $\ind(\K_2)\cap \{ a \mid a^{\prod\!\Mod_{\!2}}\in \Delta^{\I}\}$. We can regard the $\Sigma$-reduct of
   $\I$ as a $\Sigma$-CQ which is entailed by $\K_{2}$ but not by $\K_{1}$ (by Proposition~\ref{prop:products}).
   The proof of (4) is similar.
\end{proof}
Example~\ref{ex:hom1} can be used to show that, in Theorem~\ref{crit:KB},  $n\Sigma$-homomorphic embeddability
cannot be replaced by $\Sigma$-homomorphic embeddability.
%For example, in~(1), let $\K_{1}=\K_{2}=(\{\top \sqsubseteq \exists R.\top\},\{A(a)\})$, $\Mod_{\!1}=\{\I_{1}\}$, where $\I_{1}$ %is the infinite $R$-chain starting with $a$, and let $\Mod_{\!2}$ contain arbitrary finite $R$-chains starting with $a$ followed %by an arbitrary long $R$-cycle. $\Mod_{\!1}$ and $\Mod_{\!2}$ are
%both complete for $\K$, but there is no $\Sigma$-homomorphism from any $\I_{2}\in \Mod_{\!2}$ to $\I_{1}$.
In Section~\ref{sec:ucq-kbs}, however, we show that in some cases
we \emph{can} find characterisations with full $\Sigma$-homomorphisms and use them to present decision procedures for entailment.

If both $\Mod_{\!i}$ are finite and contain only finite interpretations, then Theorem~\ref{crit:KB} provides a
decision procedure for KB entailment. This applies, for example, to KBs with acyclic classical
TBoxes~\cite{BCMNP03}, and to KBs for which the chase terminates \cite{GrauHKKMMW13}.

\section{Undecidability of (r)CQ-Entailment and Inseparability for \ALC KBs}\label{sec:undecidability}
The aim of this section is to show that CQ and rCQ-entailment and inseparability for \ALC KBs are undecidable. We begin by proving that it is undecidable whether an \EL KB
$\Sigma$-CQ entails an \ALC KB. A straightforward modification of the KBs
constructed in that proof is then used to prove that $\Sigma$-CQ inseparability
between $\mathcal{EL}$ and \ALC KBs is undecidable as well. It is to be noted that, as shown in Section~\ref{sec:ucq-kbs}, both $\Sigma$-UCQ and $\Sigma$-rUCQ entailments between \ALC KBs are decidable, which means, by Proposition~\ref{UCQtoCQ}, that checking whether an \ALC KB
$\Sigma$-(r)CQ entails an \EL KB is decidable.
We then consider rooted CQs and prove that $\Sigma$-rCQ entailment and inseparability
between $\mathcal{EL}$ and \ALC KBs are still undecidable. (In fact, the undecidability proof for rCQs implies the undecidability results for CQs, but is somewhat trickier.)  The
signature $\Sigma$ used in these undecidability proofs is a proper subset of the signatures of the KBs involved.
% In fact, the KBs constructed in the undecidability proofs trivially do not
% CQ-entail each other if arbitrary concept and role names can be used in
% witness CQs.
In the final part of this section, we prove that one can modify the KBs in such
a way that all the results stated above hold for full signature CQ and rCQ
entailment and inseparability.

\subsection{Undecidability of CQ-entailment and inseparability with respect to a signature $\Sigma$}

Our undecidability proofs are by reduction of the undecidable
\emph{rectangle tiling problem}: given a finite set $\mathfrak T$ of \emph{tile types} $T$ with four colours $\textit{up}(T)$, $\textit{down}(T)$, $\textit{left}(T)$ and $\textit{right}(T)$, a tile type \mbox{$I \in \mathfrak T$}, and two colours $W$ (for wall) and $C$ (for ceiling), decide whether there exist $N,M \in \mathbb N$ such that the $N \times M$ grid can be tiled using $\mathfrak T$ in such a way that $\textit{left}(T)=\textit{right}(T')$ if $(i,j)$ is covered by a tile of type $T$ and $(i+1,j)$ is covered by a tile of type $T'$, and $1\leq i <N$, $1\leq j \leq M$;
$\textit{up}(T)=\textit{down}(T')$ if $(i,j)$ is covered by a tile of type $T$ and $(i,j+1)$ is covered by a tile of type $T'$,
and $1\leq i \leq N$, $1\leq j < M$; $(1,1)$ is covered by a tile of type $I$; every $(N,i)$, for $i \le M$, is covered by a tile of type $T$ with $\textit{right}(T) = W$; and every $(i,M)$, for $i \le N$, is covered by a tile of type $T$ with $\textit{up}(T) = C$. (The reader can easily show that this problem is undecidable by reduction of the halting problem for Turing machines; cf.~\cite{lnpam187-Emd}.) If an instance $\mathfrak T$ of the rectangle tiling problem has a positive solution, we say that $\mathfrak T$ \emph{admits tiling}.

Given such an instance $\mathfrak{T}$, we construct
an $\mathcal{EL}$ TBox $\ToneCQ$, an $\mathcal{ALC}$ TBox $\TtwoCQ$, an ABox $\ACQ$, and a signature
$\SigmaCQ$ such that, for the KBs $\KoneCQ=(\ToneCQ,\ACQ)$ and $\KtwoCQ=(\TtwoCQ,\ACQ)$, the following conditions
are equivalent:
\begin{itemize}
\item[--] $\KoneCQ$ $\SigmaCQ$-CQ entails $\KtwoCQ$;

\item[--] the instance $\mathfrak{T}$ does not admit tiling.
\end{itemize}
The ABox $\ACQ$ does not depend on $\mathfrak{T}$ and is defined by setting $\ACQ=\{A(a)\}$.
The TBox $\TtwoCQ$ uses a role name $R$ to encode a grid by putting one row of the grid after the other
starting with the lower left corner of the grid. It also uses the following concept names:
\begin{itemize}
\item[--] $T^{\first}$, for each tile type $T\in \mathfrak{T}$, to encode the first row of a tiling;

\item[--] $T_{k}$, for $T\in \mathfrak{T}$ and $k=0,1,2$, to encode intermediate rows, with three copies of
each $T\in \mathfrak{T}$ needed to ensure the vertical matching conditions between rows;

\item[--] $T_{k}^{\halt}$, for $T\in \mathfrak{T}$ and $k=0,1,2$, to encode the last row;

\item[--] $\widehat{T}_{k}$, for $T\in \mathfrak{T}$ and $k=0,1,2$.
\end{itemize}
Of all these concept names, only the $\widehat{T}_{k}$ are in the signature $\SigmaCQ$ of the entailment problem we construct.
Thus, the $T^{\first}$, $T_{k}^{\halt}$, and $T_{k}$ are auxiliary concept names used to generate tilings, while the $\widehat{T}_{k}$ make the tilings `visible' to relevant CQs.

The TBox $\TtwoCQ$ uses the concept names $\Start$ and $\End$ as markers for the start and end of a tiling. Both concept names are in $\SigmaCQ$.
To mark the end of rows, $\TtwoCQ$ employs the concept names $\Row_{k}$ and $\Row_{k}^{\halt}$, for $k=0,1,2$,
where the $\Row_{k}^{\halt}$ indicate the last row. Similarly to the encoding of tile types above, the
concept names $\Row_{k}$ and $\Row_{k}^{\halt}$ are auxiliary concept names used to construct tilings. Three
copies are needed to ensure the vertical matching condition. In addition,  we
use a concept name $\Row\in \SigmaCQ$ that marks the end of rows and is visible to separating CQs.

The role name $R$ generating the grid is in $\SigmaCQ$. An additional concept name $A$
and role name $P$ link the individual $a$ in $\ACQ$ to the first row of the tiling.
The encoding does not depend on whether $A,P$ are in $\SigmaCQ$, but it will be useful later, when we
consider full signature CQ-entailment, to include them in $\SigmaCQ$.

Before writing up the axioms of $\TtwoCQ$, we explain how they generate all possible tilings.
We ensure that if a point $x$ in a model $\I$ of $\KtwoCQ$ is in $\widehat{T}_k$ and
$\textit{right}(T) = \textit{left}(S)$, then $x$ has an $R$-successor
in $\widehat{S}_k$. Thus, branches of $\I$ define (possibly infinite) horizontal rows of tilings with
$\mathfrak T$. If a branch contains a point $y \in \widehat{T}_k$ with $\textit{right}(T) = W$,
then this $y$ can be the last point in the row, which is indicated by an $R$-successor $z \in \Row$
of $y$. In turn, $z$ has $R$-successors in all $\widehat{T}_{(k+1)\, \text{mod}\, 3}$ that can be possible
beginnings of the next row of tiles. To coordinate the $\textit{up}$ and $\textit{down}$ colours
between the rows---which will be done by the CQs
separating $\KoneCQ$ and $\KtwoCQ$---we make every $x \in \widehat{T}_k$, starting from the
second row, an instance of all $\widehat{S}_{(k-1)\, \text{mod}\, 3}$ with
$\textit{down}(T) = \textit{up}(S)$. The row started by $z \in \Row$ can be the last one in the
tiling, in which case we require that each of its tiles $T$ has $\textit{up}(T) = C$. After the
point in $\Row$ indicating the end of the final row, we add an $R$-successor in $\End$ for the
end of tiling. The beginning of the first row is indicated by a $P$-successor in $\Start$ of the
ABox element $a$, after which we add an $R$-successor in $I^{\first}$ for the given initial tile type $I$.
%; see the lowest branch in Fig.~\ref{fig:k2}. To generate a tree with all possible branches described above, we onl%y require $\EL$ axioms of the form $E \sqsubseteq D$ and $E \sqsubseteq \exists S.D$.

The TBox $\TtwoCQ$ contains the following CIs, for $k = 0,1,2$:
\begin{align}
\label{initial}
& A \sqsubseteq \exists P.(\Start \sqcap \exists R.I^{\first}),\\
\label{horizontal-0}
& T^{\first} \sqsubseteq \exists R.S^{\first},\quad \text{ if $\textit{right}(T) =
  \textit{left}(S)$ and $T,S \in \mathfrak T$}, \\
\label{end-of-row-0}
& T^{\first} \sqsubseteq \exists R.(\Start \sqcap \Row_1),\quad \text{ if  $\textit{right}(T) = W$ and $T \in \mathfrak T$},\\
\label{first-row-0}
& T^{\first} \sqsubseteq \widehat{T}_0,\quad \text{ for $T \in \mathfrak T$},\\
%\end{align}
%
%\begin{align}
\label{start-row-k}
& \Row_k \sqsubseteq \exists R. T_k, \quad \text{ for $T \in \mathfrak T$},\\
\label{horizontal-k}
& T_k \sqsubseteq \exists R.S_k, \quad \text{ if $\textit{right}(T) =
  \textit{left}(S)$ and $T,S \in \mathfrak T$}, \\
\label{end-of-row-k}
& T_k \sqsubseteq \exists R.\Row_{(k+1) \,\text{mod}\, 3}, \quad \text{ if
$\textit{right}(T) = W$ and $T \in \mathfrak T$},\\
\label{start-last-row}
& T_k \sqsubseteq \exists R.\Row^{\halt}_{(k+1)\, \text{mod}\, 3},\quad
\text{ if $\textit{right}(T) = W$ and $T \in \mathfrak T$},\\
\label{row-k}
& \Row_k \sqsubseteq \Row, \\
\label{tile-hat}
& T_k \sqsubseteq \widehat{T}_k,\quad \text{ for $T \in \mathfrak T$},\\
\label{vertical-k}
& T_k \sqsubseteq \widehat{S}_{(k-1)\, \text{mod}\, 3}, \quad \text{ if $\textit{down}(T) = \textit{up}(S)$ and $T,S \in \mathfrak T$},\\
% \end{align}
%
% \begin{align}
\label{disjunction}
& \Row_k^{\halt} \sqsubseteq \exists R.{\End} ~\sqcup \hspace*{-1mm}
\bigsqcap_{\textit{up}(T) = C,\, T \in \mathfrak T} \exists R.T^{\halt}_k, \\
\label{horizontal-halt}
& T_k^{\halt} \sqsubseteq \exists R. S^{\halt}_k, \quad \text{ if $\textit{right}(T) = \textit{left}(S)$, $\textit{up}(S) = C$ and $T,S \in \mathfrak T$},\\
\label{end}
& T_k^{\halt} \sqsubseteq \exists R.(\Row \sqcap \exists R.{\End}), \quad
\text{ if $\textit{right}(T) = W$ and $T \in \mathfrak T$},\\
\label{row-halt}
& \Row^{\halt}_k \sqsubseteq \Row,\\
\label{vertical-halt}
& T_k^{\halt} \sqsubseteq \widehat{S}_{(k-1)\, \text{mod}\, 3}, \quad \text{ if $\textit{down}(T) = \textit{up}(S)$ and $T,S \in \mathfrak{T}$}.
\end{align}

\begin{figure}
  \centering
  \begin{tikzpicture}[yscale=1.2]
    \begin{scope}[yshift=1cm]
      \foreach \al/\x/\y/\type/\lab/\wh/\extra in {%
        a/0/0/point/A/right/constant, %
        start/0/-1/point/\Start/right/, %
        n0/0/-2/point/I^{\first}/80/,%
        n1/-0.8/-3/point/T^{\first}/left/,%
        n2/1.8/-3/point/T^{\first}/right/,%
        % st1/2/-3/subtree/{\begin{minipage}{0.6cm}\centering
        %     ~\\$\Start$\\$\Row_1$\end{minipage}}/70/,%
        n3/-1.8/-4/point/T^{\first}/left/,%
        n4/-0.9/-4/point/T^{\first}/right/,%
        st2/0.3/-4/subtree/{\qquad\Start}/90/%
      }{ \node[\type, \extra, label=\wh:{\footnotesize $\lab$}] (\al)
        at (\x,\y) {}; }

      \foreach \from/\to/\type/\wh in {%
        start/n0/role/left, %
        n0/n1/role/left, n0/n2/role/right, %n0/st1.north/subtreecolor, %
        n1/n3/role/left, n1/n4/role/left, n1/st2.north/subtreecolor/right%
      } {\draw[role, \type] (\from) -- node[\wh,black] {\scriptsize $R$} (\to);}

      \draw[role] (a) -- node[left] {\scriptsize $P$} (start);

      \foreach \from in {%
        n2, n3, n4%
      } {\draw[thick, dotted] (\from) -- +(0,-0.5);}

      \foreach \from/\to in {%
        n1/n2, n3/n4%
      } {\draw[thick, dotted] ($0.7*(\from)+0.3*(\to)+(0,0.1)$) --
        ($0.3*(\from)+0.7*(\to)+(0,0.1)$);}

      \foreach \wh in {%
        st2%
      }{ \node[anchor=south] at (\wh.south) {\small $\tau_1$}; }
    \end{scope}

    \begin{scope}[xshift=6cm]
      \foreach \al/\x/\y/\type/\lab/\wh/\extra in {%
        r0/0/0/point/{\Row_k}/above/rowind, %
        r1/-1.2/-1/point/T_k/left/, %
        r2/1.2/-1/point/T_k/right/, %
        r3/-2.2/-2/point/T_k/left/, %
        r4/-1.1/-2/point/T_k/right/, %
        nr2/0/-2/subtree//85/,%
        end2/1.8/-2/subtree//90/%
      }{ \node[\type, \extra, label=\wh:{\footnotesize $\lab$}] (\al) at (\x,\y) {}; }

      \foreach \from/\to/\type/\wh in {%
        r0/r1/role/left, r0/r2/role/right, %
        r1/r3/role/left, r1/r4/role/left, %
        r1/nr2.north/subtreecolor/right, r1/end2.north/subtreecolor/right%
      } {\draw[thick, role, \type] (\from) -- node[\wh,black] {\scriptsize $R$} (\to);}

      \foreach \from in {%
        r2, r3, r4%
      } {\draw[thick, dotted] (\from) -- +(0,-0.5);}

      \foreach \from/\to in {%
        r1/r2, r3/r4%
      } {\draw[thick, dotted] ($0.7*(\from)+0.3*(\to)$) --
        ($0.3*(\from)+0.7*(\to)$);}

      \foreach \wh/\lab in {%
        nr2/, end2/\halt%
      }{ \node[anchor=north] at (\wh.south) {\small $\tau_{_{(k+1)\, \text{mod}\, 3}}^{\lab}$}; }

      \node at (-2,1) {$\tau_k$};
    \end{scope}

    %%%%%%%%%%%%%%%%%%%%%%%%%%%%%%%%%%%
    \begin{scope}[xshift=11.2cm]
      \foreach \al/\x/\y/\lab/\wh/\extra in {%
        v0/-0.25/0/\Row_k^{\halt}/above/rowind, %
        u1/-1.8/-1/{\End}/left/,%
        v1/-0.5/-1/{T_k^{\halt}}/left/,%
        v2/1.8/-1/T_k^{\halt}/right/,%
        v3/-1.25/-2/T_k^{\halt}/left/,%
        v4/-0.5/-2/T_k^{\halt}/right/,%
        row1/0.75/-2/\Row/right/rowind,%
        end1/0.75/-3/\End/right/subtreecolor%
      }{ \node[point, \extra, label=\wh:{\footnotesize $\lab$}] (\al) at (\x,\y)
        {}; }

      %% the edges
      \foreach \from/\to/\type/\wh in {%
        v0/u1/role/left, v0/v1/role/right, v0/v2/role/right, v1/v3/role/left,
        v1/v4/role/right, v1/row1/subtreecolor/right, row1/end1/subtreecolor/right%
      } {\draw[thick, role, \type] (\from) -- %
        node[pos=0.4, outer sep=0, inner sep=0] (\from-\to) {} %
        node[\wh,black] {\scriptsize $R$} (\to);}

      \foreach \from in {%
        v2, v3, v4%
      } {\draw[thick, dotted] (\from) -- +(0,-0.5);}

      \foreach \from/\to in {%
        v1/v2, v3/v4%
      } {\draw[thick, dotted] ($0.7*(\from)+0.3*(\to)$) --
        ($0.3*(\from)+0.7*(\to)$);}

      \draw (v0-u1) to[bend right] node[pos=0.7, yshift=0.12cm] {\scriptsize $\lor$} (v0-v1);

      \node at (-2,1) {$\tau_k^{\halt}$};
    \end{scope}

    \node at (0,-4.2) {(a)};
    \node at (6,-4.2) {(b)};
    \node at (11.2,-4.2) {(c)};
  \end{tikzpicture}

  \caption{The paths in the minimal models generated by the axioms of $\TtwoCQ$.}
  \label{fig:t2-axioms}
\end{figure}

The KB $\TtwoCQ$ is an \ELUf KB, with \eqref{disjunction} being the only CIs with $\sqcup$. Throughout the proof, we work with the set $\Mod_{\KtwoCQ}$
of minimal models of $\KtwoCQ$ and use the notation introduced in the construction of minimal models.
In figures, $\lor$ indicates an \emph{or-node}. We now comment on the  role of the CIs in $\TtwoCQ$.
\begin{itemize}
\item[--] The CIs \eqref{initial}--\eqref{end-of-row-0} produce all possible first rows whose ends are indicated by points in $\Start$ and $\Row_1$; see Fig.~\ref{fig:t2-axioms}(a), where $\tau_1$ denotes trees described below. The CI \eqref{first-row-0} ensures that the tiling of the first row is visible in $\SigmaCQ$ using the concept names $\widehat{T}_{0}$. Note that $\Row$ is visible in $\SigmaCQ$ due to \eqref{row-k}.

\item[--] The CIs \eqref{start-row-k}--\eqref{start-last-row} produce all possible intermediate rows starting with points in $\Row_k$ and ending by points in $\Row_{(k+1)\, \text{mod}\, 3}$ or $\Row^{\halt}_{(k+1)\,  \text{mod}\, 3}$; see Fig.~\ref{fig:t2-axioms}(b), where $\tau_k$ is the tree with root in $\Row_k$ and $\tau_k^{\halt}$ the tree with root in $\Row^{\halt}_k$ as described below. The CIs \eqref{row-k}--\eqref{vertical-k} ensure that the tilings of the intermediate rows as well as $\Row$ are visible in $\SigmaCQ$. Note that, for each intermediate row,  there exists $k$ such that the current row is encoded using $\widehat{T}_{k}$ and the matching previous row using $\widehat{T}_{(k-1)\, \text{mod}\, 3}$.

\item[--] The CIs \eqref{disjunction}--\eqref{end} produce all possible final rows starting with points in $\Row^{\halt}_k$. The role of the disjunction is explained below; see Fig.~\ref{fig:t2-axioms}(c). Finally, the axioms \eqref{row-halt}--\eqref{vertical-halt} make $\Row$ and the matching previous row visible in $\SigmaCQ$. Note that the last row itself is not visible in $\SigmaCQ$.
\end{itemize}
The existence of a tiling of some $N \times M$ grid for the given instance $\mathfrak T$ can be checked by
Boolean CQs $\q_n$, for $n \ge 1$, that require an $R$-path from $\Start$ to $\End$ going through $\widehat{T}_k$- or $\Row$-points:
\begin{equation*}
\q_n ~=~ \exists \avec{x} \, \big( \Start(x_0) \land \bigwedge_{i=0}^n R(x_i, x_{i+1}) \land \bigwedge_{i=1}^n B_{i}(x_{i}) \land \End(x_{n+1}) \big),
\end{equation*}
where $B_{i} \in \{\Row\} \cup \{\widehat{T}_k \mid T \in \mathfrak{T}, k=0,1,2\}$. The $\q_{n}$ will
serve as the separating $\SigmaCQ$-CQs if $\mathfrak{T}$ admits a tiling (in fact, if $\mathfrak{T}$ admits
a tiling of some $N\times M$ grid, then $q_{n}$ is a separating $\SigmaCQ$-CQ for $n=(N+1)\times (M-1)$).
We illustrate the relationship
between $\Mod_{\KtwoCQ}$ and the CQs $\q_{n}$ in Fig.~\ref{fig:k2}: the lower part of the figure shows
two interpretations, $\I_{l}$ and $\I_{r}$, from $\Mod_{\KtwoCQ}$ (we only mention the extensions of
concept names in $\SigmaCQ$). The two interpretations coincide up to the $\Row$-point before the final
row of the tiling. Then, because of the axiom~\eqref{disjunction}, they realise \emph{two alternative}
continuations: one as described above, and the other one having just a
single $R$-successor in $\End$. In the picture, we show a situation where row $m$
coincides with the row depicted below row $m+1$ (that satisfies the vertical tiling conditions with row $m+1$).
For example, the first row $\widehat{I}_{0}\cdots\widehat{T}_{0}^{N1}$ coincides with the row depicted below
the second row (after the second $\Start$). This is no accident and is enforced by the query $\q_{n}$
that is depicted in the upper part of the figure. If $\KtwoCQ \models \q_n$, then $\q_n$ holds in
both $\I_{l}$ and $\I_{r}$, and so there are homomorphisms $h_l \colon \q_n \to \I_l$ and
$h_r \colon \q_n \to \I_r$. As  $h_l(x_{n-1})$ and $h_r(x_{n-1})$ are instances of $B_{n-1}$, we have
$B_{n-1} = \widehat{T}_1^{N M-1}$ in the figure, and so $\textit{up}(T^{N M-1}) = \textit{down}(T^{N M})$.
By repeating this argument until $x_0$, we see that the colours between horizontal rows match and the
rows are of the same length.
Note that for this to work, we have to make both the $P$-successor of $a$ and the first $\Row$-point
an instance of $\Start$. We now formalise the observations above by proving the following:

\begin{figure}
  \centering
  \scalebox{1}{% \documentclass[tikz, border=0]{standalone}

% \usepackage{amsmath,amsthm,amssymb,stmaryrd}

% \usetikzlibrary{shapes,decorations,shadows}
% \usetikzlibrary{decorations.pathmorphing,decorations.pathreplacing,decorations.markings}
% \usetikzlibrary{arrows}
% \usetikzlibrary{calc,fit}
% \usetikzlibrary{trees}
% \usetikzlibrary{backgrounds}
% \usetikzlibrary{patterns}

% \input{macros-def}
% \input{macros-tikz}

% % \newcommand{\Start}{\textit{Start}}
% % \newcommand{\Row}{\textit{Row}}
% % \newcommand{\End}{\textit{End}}
% % \newcommand{\first}{\textit{first}}
% % \newcommand{\halt}{\textit{halt}}

% \begin{document}

\begin{tikzpicture}[xscale=0.89, %
  point/.style={thick,circle,draw=black,fill=white, minimum
    size=1.3mm,inner sep=0pt}%
  ]
  \foreach \al/\x/\y/\lab/\wh/\extra in {%
    a/0.2/0/A/below/constant, %
    start/1/0/\Start/below/, %
    x11/2/0/{\widehat{I}_0}/below/, %
    xN1/4/0/{\widehat{T}_0^{N1}}/below/, %
    row1/5/0/\Start/below/rowind, %
    x12/6/0/{\widehat{T}_1^{12}}/below/, %
    xN2/8/0/{\widehat{T}_1^{N2}}/below/, %
    row2/9/0/{\Row}/above/rowind,%
    x1m/10.5/0/{\widehat{T}_1^{1M\text{-}1}}/below/, %
    xNm/12.5/0/{\widehat{T}_1^{NM\text{-}1}}/below/, %
    rowm/13.5/0/{\Row}/above/rowind,%
    endl/14.5/0.6/{~~\End}/below/,%
    x1M/14.5/-0.6/{\widehat{T}_2^{1M}}/below/, %
    xNM/16.5/-0.6/{\widehat{T}_2^{NM}}/below/, %
    rowM/17.5/-0.6/{\Row}/above/rowind,%
    endr/18.3/-0.6/{\End}/below/%
  }{ \node[point, \extra, label=\wh:{$\lab$}] (\al) at (\x,\y)
    {}; }

  \node[label={[inner ysep=1]above:{$\Row~~~~~$}}] at (row1) {};%

  \node[label={[inner xsep=1]above:{$a$}}] at (a) {};%

  \foreach \al/\lab/\wh in {%
    x12/{$\widehat{I}_0$\dots}/1, %
    xN2/{$\widehat{T}_0^{N1}$\dots}/1, %
    x1m/{$\widehat{T}_0^{1M\text{-}2}$\dots}/1, %
    xNm/{$\widehat{T}_0^{NM\text{-}2}$\dots}/1, %
    x1M/{$\widehat{T}_1^{1M\text{-}1}$\dots}/-1, %
    xNM/{$\widehat{T}_1^{NM\text{-}1}$\dots}/-1%
  }{ \node[label={[scale=0.8]center:{\lab}}, yshift=-0.85cm] at
    (\al) {}; }

  \foreach \from/\to/\type/\extra in {%
    start/x11/role, x11/xN1/dotted, xN1/row1/role,%
    row1/x12/role, x12/xN2/dotted, xN2/row2/role,%
    row2/x1m/dotted, x1m/xNm/dotted, xNm/rowm/role,%
    rowm/x1M/role, x1M/xNM/dotted, xNM/rowM/role,%
    rowm/endl/role, rowM/endr/role%
  } {\draw[thick, \type, \extra] (\from) -- (\to);}

  \draw[role] (a) -- node[above] {$P$} (start); %

  %%% the disjunction
  \foreach \from/\to/\type in {%
    rowm/endl, rowm/x1M% 
  } {\draw[draw=none] (\from) -- node[pos=0.5, inner sep=0] (\from-\to) {} (\to);}

  \draw[red] (rowm-endl) to[bend left] node[left] {\footnotesize $\lor$}
  (rowm-x1M);

  \node[yshift=0.1cm, xshift=-0.5cm] at (endl) {\Large$\I_{\!l}$};
  \node[yshift=-0.1cm, xshift=-0.6cm] at (x1M) {\Large$\I_{\!r}$};

  %% the gray ovals
  \begin{pgfonlayer}{background}
    \foreach \first/\last in {%
      x11/row1, x12/row2, x1m/rowm, x1M/rowM%
    }{ \node[fit=(\first)(\last), rounded corners, fill=gray!20] {}; }
  \end{pgfonlayer}

  %% the other branches
  \begin{pgfonlayer}{background}
    \def\drawsubtree[#1,#2]{%
      \draw[secondary, thick, dotted] (#1) -- ++(0.4,0.08) (#1) -- ++(0.4,-0.08);%
      % \draw[secondary] (#1) -- ++(0.5,-0.1) -- ++(-0.25,-0.25) -- (#1);%
    }
    
    \colorlet{lightgray}{gray}
    \begin{scope}[every node/.style={scale=0.7}, secondary/.style={minimum
        size=0.8mm, draw=lightgray, thin}]\small%\footnotesize%
      \coordinate (xk1) at (2.8,0);%
      \coordinate (xk2) at (7,0);%
      \coordinate (xkm) at (11.5,0);%
      \coordinate (xkM) at (15.5,-0.6);

      %%% other tiles
      \foreach \al/\tile in {%
        x11/T_0, xk1/T_0, xN1/T_0, row1/T_1, x12/T_1, xN2/T_1, xNm/T_1,
        xkM/T_2, xNM/T_2%
      }{ \node[point, secondary, xshift=0.6cm, yshift=0.8cm,
        label={[lightgray, inner sep=1]above:{$\widehat \tile$}}] (\al-z1) at (\al) {};

        \draw[role, lightgray, secondary] (\al) -- (\al-z1);

        \drawsubtree[\al-z1,1]
      }

      %% dots
      \foreach \al in {%
        row2, xk2, x1m, xkm, x1M%
      }{ \draw[dotted, lightgray, thick] (\al) -- +(0.5,0.6); }

      %%% other Row-halt branches
      \foreach \al/\tile/\ysh in {%
        xN2/T_2/1.6cm%
      }{ \node[point, secondary, fill=lightgray!50, xshift=0.5cm, yshift=\ysh,
        label={[lightgray, inner sep=1]left:{$\Row$}}] (\al-halt) at (\al) {};

        \node[point, secondary, xshift=0.7cm, yshift=0.3cm,
        label={[lightgray, inner sep=1]right:{$\End$}}] (\al-halt-end) at (\al-halt) {};

        \node[point, secondary, xshift=0.7cm, yshift=-0.3cm,
        label={[lightgray, inner sep=1]above:{$\widehat\tile$}}] (\al-halt-tile) at (\al-halt) {};

        \foreach \from/\to in {%
          \al/\al-halt, \al-halt/\al-halt-end, \al-halt/\al-halt-tile%
        }{ \draw[role, lightgray, secondary] (\from) -- (\to); }

        %% disjunction
        \foreach \from/\to in {%
          \al-halt/\al-halt-end, \al-halt/\al-halt-tile%
        }{ \draw[draw=none] (\from) -- node[inner sep=0, outer sep=0, pos=0.55]
          (\from-\to) {} (\to); }

        \draw[lightgray!50!red] (\al-halt-\al-halt-end) to[bend left]
        node[left, inner sep=2, scale=0.7] {$\lor$} (\al-halt-\al-halt-tile);

        \drawsubtree[\al-halt-tile,1]%
      }

      %%% other Row branches
      \foreach \al/\ysh in {%
        xNm/1.6cm%
      }{ \node[point, secondary, fill=lightgray!50, xshift=0.5cm, yshift=\ysh,
        label={[lightgray, inner sep=1]left:{$\Row$}}] (\al-row) at (\al) {};

        \draw[role, lightgray, secondary] (\al) -- (\al-row); 

        \drawsubtree[\al-row,1]%
      }

      %%% other Start branches
      \foreach \al/\ysh in {%
        xk1/1.6cm%
      }{ \node[point, secondary, fill=lightgray!50, xshift=0.6cm, yshift=\ysh,
        label={[lightgray,text width=0.65cm]left:{$\Row$~~\\$\Start$}}]
        (\al-start) at (\al) {};

        \draw[role, lightgray, secondary] (\al) -- (\al-start); 

        \drawsubtree[\al-start,1]%
      }

      %%% other End
      \foreach \al/\tile/\ysh in {%
        xkM/T_2/1.6cm%
      }{ \node[point, secondary, fill=lightgray!50, xshift=0.5cm, yshift=\ysh,
        label={[lightgray, inner sep=1]left:{$\Row$}}] (\al-end-row) at (\al) {};

        \node[point, secondary, xshift=0.7cm,
        label={[lightgray, inner sep=1]right:{$\End$}}] (\al-end) at (\al-end-row) {};

        \foreach \from/\to/\type in {%
          \al/\al-end-row, \al-end-row/\al-end%
        }{ \draw[role, lightgray, secondary] (\from) -- (\to); } %
      }
    \end{scope}    
  \end{pgfonlayer}

  %%% the query
  \begin{scope}[xshift=3cm, yshift=2.25cm]
    \foreach \al/\x/\y/\lab/\wh/\extra in {%
      y0/0/0/{\Start}/right/, %
      y1/1/0/{B_1}/right/, %
      yN/3/0/{B_N}/right/, %
      yN1/4/0/{B_{N+1}}/right/rowind, %
      y21/5/0/{}/right/, %
      y2N/7/0/{}/right/, %
      y2N1/8/0/{}/right/rowind, %
      yl1/9.5/0/{B_{n{-}N}}/right/, %
      ylN/11.5/0/{B_{n{-}1}}/right/, %
      yl/12.5/0/{B_n}/right/rowind, %
      yend/13.5/0/\End/right/%
    }{ \node[point, \extra, label=above:{$\lab$}] (\al) at
      (\x,\y) {}; }

    \foreach \from/\to/\type in {%
      y0/y1/role, y1/yN/dotted, yN/yN1/role, yN1/y21/role, y21/y2N/dotted,
      y2N/y2N1/role, y2N1/yl1/dotted, yl1/ylN/dotted, ylN/yl/role, yl/yend/role%
    } {\draw[thick, \type] (\from) -- (\to);}

    \begin{pgfonlayer}{background}
      \foreach \first/\last in {%
        y1/yN1, y21/y2N1, yl1/yl%
      }{ \node[fit=(\first)(\last), rounded corners, fill=gray!20] {}; }
    \end{pgfonlayer}

    \node[xshift=-1cm] at (y0) {\Large$\q_n$};
  \end{scope}

  %%% homomorphisms
  \begin{pgfonlayer}{background}

    \foreach \al in {y0,yend,start,row1,endl,endr} { %
      \node[fit=(\al), inner sep=3] (\al-a) {}; }

    %% left
    \foreach \from/\to/\out/\in in {%
      y0-a/start-a/250/70, yend-a/endl-a/250/30%, ylN/xNm/190/110%
    } {\draw[homomorphism, black!50] (\from) to[out=\out, in=\in]
      node[above] {$h_l$} (\to);}

    %% right
    \foreach \from/\to/\in in {%
      y0-a/row1-a/90, yend-a/endr-a/110%, ylN/xNM/90%
    } {\draw[homomorphism, black!75] (\from) to[out=290, in=\in]
      node[above] {$h_r$} (\to);}
  \end{pgfonlayer}
\end{tikzpicture}

%\end{document}

%%% Local Variables:
%%% mode: latex
%%% TeX-master: "aij-insep"
%%% TeX-PDF-mode: t
%%% save-place: t
%%% End:
}
  \caption{The structure of the models $\I_l$ and $\I_r$ of $\K_2$, and
    homomorphisms $h_l \colon \q_n \to \I_l$ and $h_r \colon \q_n \to \I_r$.}
  \label{fig:k2}
\end{figure}

\begin{lemma}\label{qn-tiling}
  The instance $\mathfrak{T}$ admits a rectangle tiling iff there
  exists $\q_n$ such that $\KtwoCQ \models \q_n$.
\end{lemma}
%
% \begin{lemma}\label{qn-tiling}
% There exists a CQ $\q_n$ such that $\prod \Mod_{\K_2} \models \q_n$ iff there exist $N,M \in \mathbb N$ for which $\mathfrak T$ tiles the $N \times M$ grid as described above.
% \end{lemma}
%
\begin{proof}
$(\Rightarrow)$ Suppose $\mathfrak T$ tiles the $N \times
M$ grid so that a tile of type $T^{ij} \in \mathfrak T$ covers $(i,j)$. Let
$$
\textit{block}_j = (\widehat T^{1,j}_k, \dots, \widehat T^{N,j}_k, \Row),
$$
for $j=1,\dots,M-1$ and $k =(j-1)\!\! \mod\! 3$. Let $\q_n$ be the CQ in which the $B_i$ follow the pattern
$$
\textit{block}_1, \ \textit{block}_2,  \dots, \ \textit{block}_{M-1}
$$
(thus, $n = (N+1) \times (M-1)$).  In view of Lemma~\ref{min-elu-complete}, we
only need to prove that $\I \models \q_n$, for each model $\I\in \Mod_{\KtwoCQ}$.  Take such an $\I$. We have to
show that there is an $R$-path $x_0,\dots,x_{n+1}$ in $\I$ such that $x_{0}\in \Start^{\I}$, $x_i \in
B_i^\I$ for $1\leq i \leq n$, and $x_{n+1} \in \End^\I$.

First, we construct an auxiliary $R$-path $y_0,\dots,y_n$.
We take $y_0 \in \Start^\I$ and $y_1 \in {I_0}^\I$ by~\eqref{initial} ($I = T^{1,1}$). Then  we take $y_2 \in (T^{2,1}_0)^\I, \dots, y_{N} \in (T^{N,1}_0)^\I$ by~\eqref{horizontal-0}. We now have $\textit{right}(T^{N,1}) = W$. By~\eqref{end-of-row-0}, we obtain $y_{N+1} \in \Row_1^\I \cap \Start^\I$. By~\eqref{row-k}, $y_{N+1} \in \Row_1^\I \subseteq \Row^\I$. We proceed in this way, starting with \eqref{start-row-k}, till the moment we construct $y_{n-1} \in (T^{N,M-1}_k)^\I$ with $\textit{right}(T^{N,M-1}) = W$, for which we use~\eqref{start-last-row} and \eqref{row-halt} to obtain $y_n \in \Row^{\halt}_k \subseteq \Row^\I$, for some $k$. Note that ${T_k}^\I \subseteq \widehat {T_k}^\I$ by~\eqref{tile-hat}, for a tile type $T$.

By~\eqref{disjunction}, two cases are possible now:

\emph{Case 1}: there is $y$ such that $(y_n,y) \in R^\I$ and $y
\in \End^\I$. Then we take $x_0 = y_0, \dots, x_n = y_n, x_{n+1} = y$.

\emph{Case 2}: there is $z_1$ such that $(y_n,z_1) \in R^\I$ and $z_1
\in (T^{\halt}_k)^\I$, where $T = T^{1,M}$ and $\textit{up}(T) =
C$. We then use~\eqref{horizontal-halt} and find a sequence
$z_2,\dots,z_N,u,v$ such that $z_i \in (T^{\halt}_k)^\I$, where $T = T^{i,M}$,
$u \in \Row^\I$ and $v \in \End^\I$. So we take $x_0 = y_{N+1},\dots, x_{n-N-1} =
y_n$, $x_{n-N} = z_1,\dots, x_{n-1}= z_N$, and $x_n = u, x_{n+1} = v$. Note
that, by~\eqref{vertical-k} and~\eqref{vertical-halt}, we have $(T^{i,j}_k)^\I
\subseteq (\widehat T^{i,j-1}_{(k-1)\!\!\mod\! 3})^\I$.

\smallskip%
$(\Leftarrow)$ Let $\q_n$ be such that $\KtwoCQ \models \q_n$.
Then $\I \models \q_n$, for each $\I \in \Mod_{\KtwoCQ}$. Consider all the
pairwise distinct pairs $(\I,h)$ such that $\I \in \Mod_{\KtwoCQ}$ and $h$ is a
homomorphism from $\q_n$ to $\I$.  Note that $h(\q_n)$ contains an or-node
$\sigma_h$ (which is an instance of $\Row^{\halt}_k$, for some $k$).  We call
$(\I,h)$ and $h$ \emph{left} if $h(x_{n+1}) = \sigma_h \cdot w_{\exists
  R.\End}$, and \emph{right} otherwise.
It is not hard to see that there exist a left $(\I_l,h_l)$ and a right
$(\I_r,h_r)$ with $\sigma_{h_l} = \sigma_{h_r}$ (if this is not the case, we
can construct $\I \in \Mod_{\KtwoCQ}$ with $\I\not\models\q_n$ by choosing at every or-node $\sigma$ the left (right) branch if there is no left (respectively, right) homomorphism $h$ from $\q_n$ such that $h(x_n) = \sigma$).

Take $(\I_l,h_l)$ and $(\I_r,h_r)$ such that $\sigma_{h_l} = \sigma_{h_r} =
\sigma$ and use them to construct the required tiling. Let $\sigma = a
w_0\cdots w_n$. We have $h_l(x_{n+1}) = \sigma \cdot w_{\exists R.\End}$ and
$h_l(x_{n}) = \sigma$. Let $h_r(x_{n+1}) = \sigma v_1\cdots v_{m+2}$, which
is an instance of $\End$ (see Fig.~\ref{fig:k2}). Then $h_r(x_n) = \sigma v_1\cdots v_{m+1}$, which is
an instance of $\Row$.

Suppose $v_{m} = w_{\exists R. T^{\halt}_2}$ (other $k$\,s are treated
analogously). By~\eqref{end}, $\textit{right}(T) = W$; by
\eqref{horizontal-halt}, $\textit{up}(T) = C$.
Suppose $w_{n-1} = w_{\exists R. S_k}$. Then
$k=1$. By~\eqref{start-last-row}, $\textit{right}(S) = W$. By considering
the atom $B_{n-1}(x_{n-1})$ in $\q_n$, we obtain that both $a w_0 \cdots w_{n-1}$ and
$\sigma v_1 \cdots v_{m}$ are instances of $B_{n-1}$. By \eqref{tile-hat} and
\eqref{vertical-halt}, $B_{n-1} = \widehat S_1$ and $\textit{down}(T) =
\textit{up}(S)$.

Suppose $v_{m-1} = w_{\exists R. U^{\halt}_2}$. By \eqref{horizontal-halt}, $\textit{right}(U) = \textit{left}(T)$ and $\textit{up}(U) = C$.
Suppose $w_{n-2} = w_{\exists R. Q_1}$. By \eqref{horizontal-k}, we have $\textit{right}(Q) = \textit{left}(S)$.
By considering $B_{n-2}(x_{n-2})$ in $\q_n$, we obtain that both $a w_0\cdots w_{n-2}$
and $\sigma v_1\cdots v_{m-1}$ are instances of $B_{n-2}$. By \eqref{tile-hat}
and \eqref{vertical-halt}, $B_{n-2} = \widehat Q_1$ and $\textit{down}(U) =
\textit{up}(Q)$.

We proceed in the same way until we reach $\sigma$ and $a w_0\cdots w_{n-N-1}$,
for $N = m$, both of which are instances of $B_{n-N-1} = \Row$. Thus, we  have tiled
the two last rows of the grid. We proceed further and tile the whole $N \times
M$ grid, where $M = n/(N+1) + 1$.
\end{proof}

% It is to be noted that to construct $\T_2$ with the properties described above one needs quite a few auxiliary concept names.
Next, we define the \EL-KB $\KoneCQ = (\ToneCQ,\ACQ)$.
Let
$
\Sigma_0 = \{\Row\} \cup \{\widehat{T}_k \mid T \in \mathfrak{T}, k = 0,1,2\},
$
and let $\ToneCQ$ contain the following CIs:
\begin{align}
  A \sqsubseteq &~ \exists P. D,\\
   D\sqsubseteq &~ \exists R.D ~\sqcap~ \exists R. \exists R. E ~\sqcap
  \bigsqcap_{X \in \Sigma_0} X \sqcap \Start,\\
  E \sqsubseteq &~ \exists R. E ~\sqcap~ \bigsqcap_{X \in \Sigma_0} X
  ~\sqcap~ \End.% ,\\
  % C \sqsubseteq&~  \exists R.C \sqcap \exists R.E \sqcap \exists R.(\Start \sqcap \exists R.N) \sqcap  \bigsqcap_{X \in \Sigma_0} X\\
  % N \sqsubseteq&~ \exists R.N ~\sqcap~ \bigsqcap_{X \in \Sigma_0} X
\end{align}
As $\KoneCQ$ is an $\EL$-KB, it has a canonical model $\I_{\KoneCQ}$:
\begin{center}
  \begin{tikzpicture}[xscale=2, yscale=1.2, %
    emptyind/.style={fill=gray!50, inner sep=1.7},
    rotatelabels/.style={rotate=-14},
    abovesloped/.style={above,rotatelabels}]
    \small
    
    \foreach \al/\x/\y/\lab/\wh/\extra/\rot in {%
      a/0.5/0/A/above/constant/, %
      x1/1.5/0/{\Start,\Sigma_0,D}/above//, %
      x2/1.5/-0.7/{}/below/emptyind/, %
      x3/2.2/-1.0/{\End,\Sigma_{0},E~~~}/below//rotatelabels, %
      x4/2.9/-1.3/{~~~\End,\Sigma_{0},E}/below//rotatelabels, %
      y1/3/0/{\Start,\Sigma_0,D}/above/,%
      y2/3/-0.7/{}/above/emptyind/,%
      y3/3.7/-1.0/{\End,\Sigma_{0},E~~~}/below//rotatelabels,%
      y4/4.4/-1.3/{~~~\End,\Sigma_{0},E}/below//rotatelabels,%
      z1/4.5/0/{\Start,\Sigma_0,D}/above/,%
      z2/4.5/-0.7/{}/above/emptyind,%
      z3/5.2/-1.0/{\End,\Sigma_{0},E~~~}/below//rotatelabels,%
      z4/5.9/-1.3/{~~~\End,\Sigma_{0},E}/below//rotatelabels%
    }{ \node[point, \extra, label={[inner sep=2,\rot]\wh:{
          $\lab$}}] (\al) at (\x,\y) {}; }

    \node[label={[inner xsep=1]left:{$a$}}] at (a) {};%

    \foreach \from/\to/\lab/\wh in {%
      a/x1/P/above, x1/x2/R/left, x2/x3/R/abovesloped, x3/x4/R/abovesloped, %
      x1/y1/R/above, y1/y2/R/left, y2/y3/R/abovesloped, y3/y4/R/abovesloped, %
      y1/z1/R/above, z1/z2/R/left, z2/z3/R/abovesloped, z3/z4/R/abovesloped%
    } {\draw[role] (\from) -- node[\wh] {$\lab$} (\to);}

    % \foreach \al/\lab/\wh in {%
    % x2/\sigma_{\End}/right, x0/\sigma_{\Start}/right%, w0/\sigma/left%
    % } {\node[label={[red, inner sep=1]\wh:{\footnotesize $\lab$}}] at (\al)
    %   {};}
    \foreach \from in {%
      x4, y4, z4%, w2, v11a, v00a, u0, u1a%
    } {\draw[dotted, thick] (\from) -- +(0.5,-0.2);}

    \draw[dotted, thick] (z1) -- +(1,0);
  \end{tikzpicture}
\end{center}
Note that the vertical $R$-successors of the $\Start$-points are not instances of any concept name, and so $\KoneCQ$ does not satisfy
any CQ $\q_n$. Now let $\SigmaCQ=\sig(\KoneCQ)$. We show that $\KtwoCQ \models \q$ implies $\KoneCQ \models \q$,
for every $\SigmaCQ$-CQ $\q$ without a subquery of the form $\q_n$.
\begin{lemma}\label{qn-rejected-by-k1}
$\prod\!\Mod_{\KtwoCQ}$ is $n\SigmaCQ$-homomorphically embeddable into $\I_{\KoneCQ}$ preserving $\{a\}$, for all $n\geq 1$,
iff $\KtwoCQ \not\models \q_{m}$, for all $m \ge 1$.
\end{lemma}
\begin{proof}
  $(\Rightarrow)$ Suppose $\KtwoCQ \models \q_{m}$ for some $m$. Then
  $\prod\!\Mod_{\KtwoCQ} \models \q_m$.  By assumption, $\prod\!\Mod_{\KtwoCQ}$
  is $m\SigmaCQ$-homomorphically embeddable into $\I_{\KoneCQ}$ preserving
  $\{a\}$, and so we have $\I_{\KoneCQ} \models \q_m$, which is clearly
  impossible because none of the paths of $\I_{\KoneCQ}$ contains the full
  sequence of symbols mentioned in $\q_m$.

$(\Leftarrow)$ Suppose $\KtwoCQ \not\models \q_{m}$ for all $m$. Then
$\prod \Mod_{\KtwoCQ} \not\models \q_m$ for all $m$. Take any subinterpretation
of $\prod\!\Mod_{\KtwoCQ}$ whose domain contains $n$ elements. Recall from the
proof of Proposition~\ref{prop:char} that we can regard the $\SigmaCQ$-reduct
of this subinterpretation as a Boolean $\SigmaCQ$-CQ, and so denote it by
$\q$. Without loss of generality we can assume that $\q$ is connected; clearly,
$\q$ is tree-shaped. We know that there is no $\SigmaCQ$-homomorphism from
$\q_m$ into $\q$ for any $m$; in particular, $\q$ does not have a subquery of
the form $\q_m$.  We have to show that $\I_{\KoneCQ} \models \q$.

If $\q$ contains $A$ or $P$, then they appear at the root of $\q$ or,
respectively, in the first edge of $\q$. By the structure of $\K_2$, the product $\prod\!\Mod_{\KtwoCQ}$ does not contain a path from $A$ to $\End$, so $\q$ does not contain $\End$ and, therefore, can be mapped into
$\I_{\KoneCQ}$.  In what follows, we assume that $\q$ does not contain $A$ and
$P$ (note that $D$ and $E$ also do not occur in $\q$).

If $\q$ does not contain $\Start$ atoms or $\q$ does not contain $\End$ atoms,
then clearly, $\I_{\KoneCQ} \models \q$.  % In the former case, $\q$ can  be mapped
% to $\pi_1$ by sending the root of $\q$ to $\sigma_{\End}$. In the latter
% case, $\q$ can be mapped to $\pi_\omega$ by sending the root of $\q$ to
% $\sigma_{\Start}$.

Suppose $\q$ contains both $\Start$ and $\End$ atoms. If there exists an $R$-path from a $\Start$ node to an $\End$ node in $\q$ then, by the structure
of $\KtwoCQ$, the $\End$ node is a leaf of $\q$ (as $\End$ nodes are always leaves in the models from $\Mod_{\KtwoCQ}$) and the $\Start$ node is the root of $\q$ (as there are minimal models $\I_l$ and $\I_r$ in Fig.~\ref{fig:k2}, in which the first $\Start$ node has no $R$-predecessor). 
%
%To see the latter, note that the path above should be present in every minimal model from $\Mod_{\KtwoCQ}$. Now, take the minimal model $\I_l$ such that in each or-node $\sigma$ corresponding to~\eqref{disjunction} we chose and add $\sigma \cdot w_{\exists R.\End}$. The end of our path is then $\sigma \cdot w_{\exists R.\End}$, for some $\sigma$. Consider the model $\I_r$, which is like $\I_l$ except that we chose to add the successor of $\sigma$ for the right disjunct of~\eqref{disjunction}. The end of our path in this model is some node $\sigma \sigma' w_{\exists R.\End}$, where $|\sigma'| > 0$. Observe that there are exactly two elements $\sigma_1$ and $\sigma_2$, predecessors of $\sigma$, in both $\I_l$ and $\I_r$ where $\Start$ holds true. It then holds that $\sigma_2 = \sigma_1 \sigma''$, where $|\sigma''| = |\sigma'|$ and $\sigma_1$ has no $R$-predecessors; see Figure~\ref{fig:t2-axioms}.) \nb{V: added for R3.23}
%
Since $\q$ does not contain a subquery of the form $\q_m$, this $R$-path should contain variables with the empty $\SigmaCQ$-concept label, in which case $\q$ can be mapped into
$\I_{\KoneCQ}$ by sending the root of $\q$ to the $P$-successor of $a$ and the
rest of the query so as to map a variable with the empty $\SigmaCQ$-concept
label to the vertical $R$-successor of a $\Start$ node.

\begin{figure}
  \centering
  \begin{tikzpicture}[xscale=1, yscale=0.7]
    \foreach \al/\y/\x/\type/\lab/\wh/\extra in {%
      n0/0/-1/point/T^{\first}/below/,%
      n1/0/0/point/T^{\first}/below/,%
      st0/-1/0/subtree/{}/90/gray,%
      n1a/0/1/point/T^{\first}/below/,%
      n2/0/2/point/T^{\first}/below/,%
      st2/-1/3/subtree/{}/90/gray,%
      n3/0/3/point/T^{\first}/below/,%
      row1/0/4/point/{\Start}/below/rowind,%
      x12/0/5/point/{T_1}/below/,%
      xN2/0/6/point/{T_1}/below/,%
      row2/0/7/point/{\Row}/above/rowind,%
      x1m/0/8.5/point/{T_1}/below/,%
      xNm/0/9.5/point/{T_1}/below/,%
      rowm/0/10.5/point/{\Row}/above/rowind,%
      endl/0.75/11.5/point/{\mathit{End}}/below/,%
      x1M/-0.75/11.5/point/{T_2}/below/,%
      xNM/-0.75/12.5/point/{T_2}/below/,%
      rowM/-0.75/13.5/point/{\Row}/below/rowind,%
      endr/-0.75/14.5/point/{\mathit{End}}/below/%
    }{ \node[\type, \extra, label=\wh:{\footnotesize $\lab$}] (\al) at (\x,\y)
      {}; }

    \node[anchor=south] at (row1.north) {\footnotesize $\Row$};

    \foreach \from/\to/\type in {%
      n0/n1/role, n1/n1a/dotted, n1a/n2/dotted, n2/n3/dotted,
      n3/row1/role, %
      row1/x12/role, x12/xN2/dotted, xN2/row2/role,%
      row2/x1m/dotted, x1m/xNm/dotted, xNm/rowm/role, rowm/endl/role, %
      rowm/x1M/role, x1M/xNM/dotted, xNM/rowM/role, rowM/endr/role%
    } {\draw[thick, \type] (\from) -- (\to);}

    \foreach \to/\al in {endl/left,x1M/right}{%
      \draw[draw=none] (rowm) -- node[pos=0.5, inner sep=0] (\al) {} (\to); }%
    \draw[red] (left) to[bend left] node[left] {\scriptsize $\lor$} (right);

    \begin{scope}[gray]
      \foreach \from/\to/\type in {%
        n0/st0.north/role, n2/st2.north/role%
      } {\draw[thick, \type] (\from) -- (\to);}

      \foreach \wh in {%
        st0, st2%
      }{ \node[anchor=west] at (\wh.east) {\scriptsize $\tau_1$}; %
        \node[anchor=west] at (\wh.north) {\tiny $\Start$};
      }
    \end{scope}

    \foreach \al/\lab/\wh in {%
      n0/\sigma_l/above, n1a/\sigma'~~/above, n2/\sigma_r~~~~/above%
    }{\node[label={[inner sep=2,red]\wh:{\footnotesize $\lab$}}, inner sep=0]
      at (\al) {}; }

    %% the gray ovals
    \begin{pgfonlayer}{background}
      \foreach \first/\last in {%
        n0/row1, x12/row2, x1m/rowm, x1M/rowM%
      }{ \node[fit=(\first)(\last), rounded corners, fill=gray!20] {}; }
    \end{pgfonlayer}

  \node[yshift=0.1cm, xshift=-0.5cm] at (endl) {\Large$\I_l$};
  \node[yshift=-0.1cm, xshift=-0.6cm] at (x1M) {\Large$\I_r$};

    \begin{scope}[yshift=2.5cm, xshift=1cm]
      \foreach \al/\y/\x/\type/\lab/\wh/\extra in {%
        y0/0/0/point/\widehat{T}_0/90/,%
        y1/0/1/point/\widehat{T}_0/right/,%
        y3/0/3/point/\widehat{T}_0/right/,%
        z0/0/4/point/\Row/right/rowind,%
        z1/0/5/point/\widehat{T}_1/right/,%
        z2/0/6/point/{\widehat{T}_1}/right/,%
        z10/0/7/point/{\Row}/right/rowind,%
        z11/0/8.5/point/{\widehat{T}_1}/right/,%
        z12/0/9.5/point/{\widehat{T}_1}/right/,%
        z20/0/10.5/point/{\Row}/right/rowind,%
        end/0/11.5/point/{\mathit{End}}/right/%
      }{ \node[\type, \extra, label=above:{\footnotesize $\lab$}] (\al) at
        (\x,\y) {}; }

      \foreach \from/\to/\type in {%
        y0/y1/role, y1/y3/dotted, y3/z0/role, %
        z0/z1/role, z1/z2/dotted, z2/z10/role,%
        z10/z11/dotted, z11/z12/dotted, z12/z20/role, z20/end/role%
      } {\draw[thick, \type] (\from) -- (\to);}

      \begin{scope}
        \node[subtree,rotate=180] (yst0) at (1,1) {};

        \foreach \from/\to/\type in {%
          y0/yst0.north/role%
        } {\draw[thick, \type] (\from) -- (\to);}

        \foreach \wh in {%
          yst0%
        }{ \node[anchor=west] at (\wh.east) {\scriptsize $\q_{\Start}$}; %
          \node[anchor=west] at (\wh.north) {\tiny $\Start$};
        }
      \end{scope}

      \foreach \al/\lab/\wh in {%
        y0/y_0/below left, z0/y_t/below right%
      }{\node[label={[inner sep=2,red]\wh:{\footnotesize $\lab$}}, inner sep=0]
        at (\al) {}; }

      \begin{pgfonlayer}{background}
        \foreach \first/\last in {%
          y0/z0, z1/z10, z11/z20%
        }{ \node[fit=(\first)(\last), rounded corners, fill=gray!20] {}; }
      \end{pgfonlayer}

      % \node at (0,-13) {$\q$};
    \end{scope}

    %% homomorphisms
    \begin{pgfonlayer}{background}
      \foreach \al in {y0,z0,z11,end,n0,n1a,x1m,endl,n2,x1M,endr} { %
        \node[fit=(\al), inner sep=3] (\al-a) {}; }

      %% left
      \foreach \from/\to/\out/\in in {%
        y0/n0/250/70, z0/n1a/250/60, end/endl/250/60%
      } {\draw[homomorphism, black!50] (\from-a) to[out=\out, in=\in]
        node[above] {$h_l$} (\to-a);}

      %% right
      \foreach \from/\to/\in in {%
        y0/n2/90, end/endr/110%
      } {\draw[homomorphism, black!75] (\from-a) to[out=290, in=\in]
        node[above] {$h_r$} (\to-a);}
    \end{pgfonlayer}
  \end{tikzpicture}\caption{A query that contains both $\Start$ and $\mathit{End}$ atoms must have variables with empty concept labels.}
\label{fig:empty-labels}
\end{figure}
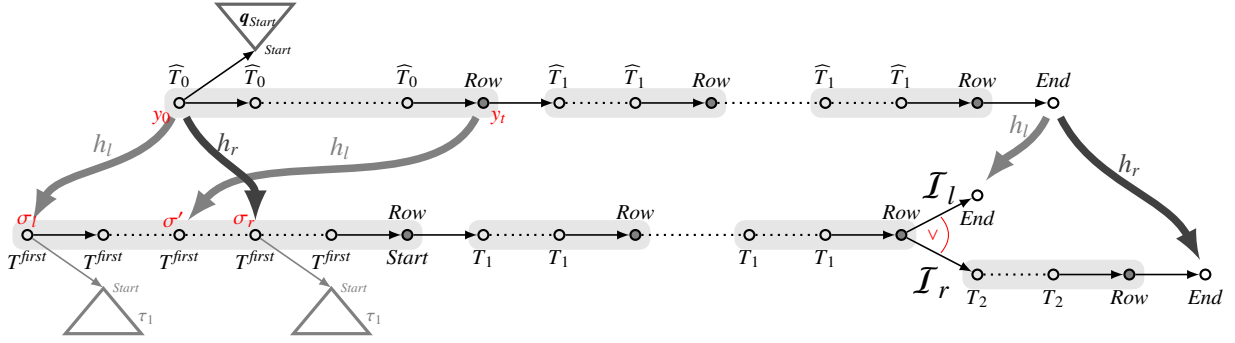

Now, suppose that $\q$ does not contain a (directed) path from a $\Start$
node to an $\End$ node. Then the $\Start$ node is not the root of $\q$. We denote by $\q_\Start$ the subtree of $\q$ generated by this node (see Fig.~\ref{fig:empty-labels}), and  by $\q_\End$ the path from the root $y_0$ of $\q$ to the $\End$ node.  
By the structure of $\KtwoCQ$ shown in Fig.~\ref{fig:t2-axioms}(a), the projection of $y_0$ onto every minimal model of $\KtwoCQ$ is of the form $\delta \cdot w_{\exists R.T^{\first}}$. 
We prove that $\q_\End$
must have at least one intermediate node with the empty $\SigmaCQ$-concept
label.
Indeed, suppose to the contrary that each intermediate variable $x$ in $\q_{\End}$ appears in
an atom of the form $B(x)$, for $B\in\{\widehat{T}_k\mid k=0,1,2\} \cup \{\Row\}$.
Since $\KtwoCQ\models \q_{\End}$, it follows that there is some $k$ such that
the distance between two neighbour $\Row$ nodes in $\q_{\End}$ is $k$.  Let
$\I_l$ and $\I_r$ be the minimal models that satisfy \eqref{disjunction} by
picking the first and the second disjunct, respectively, and identical,
otherwise (see Fig.~\ref{fig:empty-labels}). Suppose that $\I_l$ satisfies
$\q_{\End}$ by mapping $y_0$ to $\sigma_l$ of the form $\delta \cdot w_{\exists
  R.T^{\first}}$ and $\I_r$ satisfies $\q_{\End}$ by mapping $y_0$ to
$\sigma_r$ of the form $\sigma_l \cdots w_{\exists R.T^{\first}}$. Then the
distance between $\sigma_l$ and $\sigma_r$ is $k$.
Let $t$ be the distance from $y_0$ to the first $\Row$ node $y_t$. If $t \leq
k$, then $y_t$ should be mapped to $\sigma'$ that is a predecessor of
$\sigma_r$ in $\I_l$ or $\sigma_r$ itself. However, such a map is not possible as the $\SigmaCQ$-label of $\sigma'$ does not contain $\Row$
(only a concept of the form $\widehat{T}_0$), and we get a contradiction. In the case $t > k$, the argument is similar; one needs to observe that the structure of $\KtwoCQ$ (in particular, \eqref{first-row-0}, \eqref{end-of-row-k}, \eqref{tile-hat}) makes it impossible to map $y_0, \dots, y_t$ onto the common part of $I_l$ and $I_r$ in such a way that $h_r(y_i) = h_l(y_i) \sigma$ with $|\sigma| = k$. 
%
%Contradiction\nb{don't understand??? Since all labels are not empty, homomorphisms to $\I_l$ and $\I_r$ are not possible} with the assumption that in $\q_{\End}$ each intermediate
%variable is labeled by a $\SigmaCQ$ concept.
%
Thus, we conclude that $\q$ can be homomorphically mapped to $\I_{\KoneCQ}$ as follows: $y_0$ goes to $a w_{\exists P.D}$, $\q_{\Start}$ to the infinite path of $\Start$
nodes, and $\q_{\End}$ so as to map a variable with the empty
$\SigmaCQ$-concept label to the vertical successor of a $\Start$ node.
\end{proof}

As an immediate consequence of Lemmas~\ref{qn-tiling} and~\ref{qn-rejected-by-k1} and the characterisation of $\Sigma$-CQ-entailment given in Theorem~\ref{crit:KB} (3), we obtain:
\begin{theorem}\label{thm:undecidability-cq}
The problem whether an $\mathcal{EL}$ KB $\Sigma$-CQ entails an $\ALC$ KB is undecidable.
\end{theorem}
We now modify the KBs constructed in the proof of Theorem~\ref{thm:undecidability-cq} to
show undecidability of $\Sigma$-CQ-inseparability.
\begin{theorem}\label{thm:undecidability-cqinsep}
$\Sigma$-CQ inseparability between $\mathcal{EL}$ and \ALC KBs is undecidable.
\end{theorem}
\begin{proof}
We set $\K_2 = \KtwoCQ \cup \KoneCQ$ and show that the following conditions are equivalent:
\begin{enumerate}
\item[(1)] $\KoneCQ$ $\SigmaCQ$-CQ entails $\KtwoCQ$;

\item[(2)] $\KoneCQ$ and $\K_2$ are $\SigmaCQ$-CQ inseparable.
\end{enumerate}
Let $\I_{\KoneCQ}$ be the canonical model of $\KoneCQ$ and $\Mod_{\KtwoCQ}$
the set of minimal models of $\KtwoCQ$. One can easily show that the following set $\Mod_{\K_{2}}$
is complete for $\K_2$:
$$
\Mod_{\K_2} ~=~ \{~ \I \uplus \I_{\KoneCQ} \mid \I \in \Mod_{\KtwoCQ}~\},
$$
where $\I \uplus \I_{\KoneCQ}$ is the interpretation that results from merging the roots $a$ of
$\I$ and $\I_{\KoneCQ}$. Now, the implication $(2) \Rightarrow (1)$ is trivial. For the converse
direction, suppose $\KoneCQ$ $\SigmaCQ$-CQ entails $\KtwoCQ$. It follows that
$\K_2$ $\SigmaCQ$-CQ entails $\KoneCQ$. So it remains to show that $\KoneCQ$ $\SigmaCQ$-CQ entails $\K_2$.
Suppose this is not the case and there is a $\SigmaCQ$-CQ $\q$ such that
$\K_2 \models \q$ and $\KoneCQ \not\models \q$. We can assume  $\q$ to be a \emph{smallest connected} CQ with this property;
in particular, no proper sub-CQ of $\q$ separates $\KoneCQ$ and $\K_2$.
%It follows from the proof $(\Leftarrow)$ of Lemma~\ref{qn-rejected-by-k1} that we can assume $\q$ is a tree\nb{V: OK?}, let $x_0$ be its root
Now, we cannot have $\KtwoCQ \models \q$ because this would contradict the fact that $\KoneCQ$ $\SigmaCQ$-CQ entails $\KtwoCQ$.
Then $\KtwoCQ \not\models \q$, and so there is $\I \in \Mod_{\KtwoCQ}$ such that  $\I \not \models \q$. On the other hand, we have
$\I \uplus \I_{\KoneCQ} \models \q$. Take a homomorphism $h\colon \q \to \I \uplus \I_{\KoneCQ}$. As $\q$ is connected, $\I \not \models \q$ and
$\I_{\KoneCQ} \not\models \q$, there is a variable $x$ in $\q$ such that $h(x) = a$. For every  variable $x$ with $h(x) = a$, we remove
$\exists x$ from the prefix of $\q$ if any.
Denote by $\q'$ the maximal sub-CQ of $\q$ such that $h(\q') \subseteq \I$ (more precisely, $S(\avec{y})\in \q$
is in $\q'$ iff $h(\avec{y}) \subseteq \Delta^\I$). Clearly, $\q' \subsetneqq \q$ and $\K_2 \models \q'$. Denote by $\q''$ the complement of $\q'$ to $\q$. Obviously, $h(\q'') \subseteq \I_{\KoneCQ}$. Now, we either have $\KoneCQ \models \q'$ or $\KoneCQ \not\models \q'$. The latter case contradicts the choice of $\q$ because $\q'$ is a proper sub-CQ of $\q$.
Thus, $\KoneCQ \models \q'$, and so there is a homomorphism $h' \colon \q' \to \I_{\KoneCQ}$ with $h'(x) = a$, for
every free variable $x$. Define a map $g \colon \q \to \I_{\KoneCQ}$ by taking $g(y) = h'(y)$ if $y$ is in
$\q'$ and $g(y) = h(y)$ otherwise. The map $g$ is a
homomorphism because all the variables that occur in both $\q'$ and $\q''$ are free and must be mapped by
$g$ to $a$.
Therefore, $\I_{\KoneCQ} \models \q$, which is a contradiction.
%
% Finally, we prove~(\emph{iii}) by replacing non-$\Sigma$ symbols in $\K_{2}$ with complex \ALC-concepts that cannot be used in CQs and extending the TBoxes appropriately; cf.~\cite[Lemma~21]{Lutz2012non}.
\end{proof}
Observe that our undecidability proof does not work for UCQs as the UCQ composed of the two disjunctive branches shown in  Fig.~\ref{fig:k2} (for non-trivial instances) distinguishes between the KBs independently of the existence of a tiling.
In Section~\ref{sec:ucq-kbs}, we show that, for UCQs, entailment is decidable.

%************

\subsection{Undecidability of rCQ-entailment and inseparability with respect to a signature $\Sigma$}

It is not difficult to see that the KBs $\KoneCQ$ and $\KtwoCQ$ constructed
in the undecidability proof for CQ-entailment cannot be used to prove undecidability of rCQ-entailment.
In fact, $\KoneCQ$ $\SigmaCQ$-rCQ entails $\KtwoCQ$, for any instance of the rectangle tiling problem.
We now sketch how the KBs defined above can be modified to show that 
rCQ-entailment and inseparability are indeed undecidable. Detailed proofs are given in the appendix.
\begin{theorem}\label{thm:undecidability-rcq}
\textup{(}i\textup{)} The problem whether an $\mathcal{EL}$ KB $\Sigma$-rCQ entails an $\ALC$ KB is undecidable.

\textup{(}ii\textup{)} $\Sigma$-rCQ inseparability between $\mathcal{EL}$ and \ALC KBs is undecidable.
\end{theorem}
\begin{figure*}
  \centering
  \scalebox{1}{% \documentclass[tikz, border=0]{standalone}

% \usepackage{amsmath,amsthm,amssymb,stmaryrd}

% \usetikzlibrary{shapes,decorations,shadows}
% \usetikzlibrary{decorations.pathmorphing,decorations.pathreplacing,decorations.markings}
% \usetikzlibrary{arrows}
% \usetikzlibrary{calc,fit}
% \usetikzlibrary{trees}
% \usetikzlibrary{backgrounds}
% \usetikzlibrary{patterns}

% \input{macros-def}
% \input{macros-tikz}

% % \newcommand{\Start}{\textit{Start}}
% % \newcommand{\Row}{\textit{Row}}
% % \newcommand{\End}{\textit{End}}
% % \newcommand{\first}{\textit{first}}
% % \newcommand{\halt}{\textit{halt}}

% \begin{document}

\begin{tikzpicture}[xscale=0.82, %
  point/.style={thick,circle,draw=black,fill=white, minimum
    size=1.3mm,inner sep=0pt}%
  ]
  \foreach \al/\x/\y/\lab/\wh/\extra in {%
    a/0/0/{A,\Row,\widehat{T}_0~}/below/constant, %
    start/1/0/\qquad\Row/above/rowind, %
    x11/2/0/{\widehat{I}_0}/below/, %
    xN1/4/0/{\widehat{T}_0^{N1}}/below/, %
    row1/5/0/\Row\quad~/above/rowind, %
    x12/6/0/{\widehat{T}_1^{12}}/below/, %
    xN2/8/0/{\widehat{T}_1^{N2}}/below/, %
    row2/9/0/{\Row}/above/rowind,%
    x1m/10.5/0/{\widehat{T}_1^{1M\text{-}1}}/below/, %
    xNm/12.5/0/{\widehat{T}_1^{NM\text{-}1}}/below/, %
    rowm/13.5/0/{\Row}/above/rowind,%
    endl/14.5/0.6/{~~\End}/below/,%
    x1M/14.5/-0.6/{\widehat{T}_2^{1M}}/below/, %
    xNM/16.5/-0.6/{\widehat{T}_2^{NM}}/below/, %
    rowM/17.5/-0.6/{\Row}/above/rowind,%
    endr/18.3/-0.6/{\End}/below/%
  }{ \node[point, \extra, label=\wh:{$\lab$}] (\al) at (\x,\y)
    {}; }

  \node[label={[inner xsep=1]left:{$a$}}] at (a) {};%

  \foreach \al/\lab/\wh in {%
    x12/{$\widehat{I}_0$\dots}/1, %
    xN2/{$\widehat{T}_0^{N1}$\dots}/1, %
    x1m/{$\widehat{T}_0^{1M\text{-}2}$\dots}/1, %
    xNm/{$\widehat{T}_0^{NM\text{-}2}$\dots}/1, %
    x1M/{$\widehat{T}_1^{1M\text{-}1}$\dots}/-1, %
    xNM/{$\widehat{T}_1^{NM\text{-}1}$\dots}/-1%
  }{ \node[label={[scale=0.8]center:{\lab}}, yshift=-0.85cm] at
    (\al) {}; }

  \draw[role] (a) to[out=220,in=140,looseness=20] (a);

  \foreach \from/\to/\type/\extra in {%
    a/start/role, start/x11/role, x11/xN1/dotted, xN1/row1/role,%
    row1/x12/role, x12/xN2/dotted, xN2/row2/role,%
    row2/x1m/dotted, x1m/xNm/dotted, xNm/rowm/role,%
    rowm/x1M/role, x1M/xNM/dotted, xNM/rowM/role,%
    rowm/endl/role, rowM/endr/role%
  } {\draw[thick, \type, \extra] (\from) -- (\to);}

  %%% the disjunction
  \foreach \from/\to/\type in {%
    rowm/endl, rowm/x1M% 
  } {\draw[draw=none] (\from) -- node[pos=0.5, inner sep=0] (\from-\to) {} (\to);}

  \draw[red] (rowm-endl) to[bend left] node[left] { $\lor$}
  (rowm-x1M);

  \node[yshift=0.1cm, xshift=-0.5cm] at (endl) {\Large$\I_{\!l}$};
  \node[yshift=-0.1cm, xshift=-0.6cm] at (x1M) {\Large$\I_{\!r}$};

  %% the gray ovals
  \begin{pgfonlayer}{background}
    \foreach \first/\last in {%
      x11/row1, x12/row2, x1m/rowm, x1M/rowM%
    }{ \node[fit=(\first)(\last), rounded corners, fill=gray!20] {}; }
  \end{pgfonlayer}

  %% the other branches
  \begin{pgfonlayer}{background}
    \def\drawsubtree[#1,#2]{%
      \draw[secondary, thick, dotted] (#1) -- ++(0.4,0.08) (#1) -- ++(0.4,-0.08);%
      % \draw[secondary] (#1) -- ++(0.5,-0.1) -- ++(-0.25,-0.25) -- (#1);%
    }
    
    \colorlet{lightgray}{gray}
    \begin{scope}[every node/.style={scale=0.7}, secondary/.style={minimum
        size=0.8mm, draw=lightgray, thin}]\small%\footnotesize%
      \coordinate (xk1) at (2.8,0);%
      \coordinate (xk2) at (7,0);%
      \coordinate (xkm) at (11.5,0);%
      \coordinate (xkM) at (15.5,-0.6);

      %%% other tiles
      \foreach \al/\tile in {%
        x11/\widehat{T}_0, xk1/\widehat{T}_0, xN1/\widehat{T}_0,
        row1/\widehat{T}_1, x12/\widehat{T}_1, xN2/\widehat{T}_1,
        xNm/\widehat{T}_1, xkM/\widehat{T}_2, xNM/\widehat{T}_2%
      }{ \node[point, secondary, xshift=0.6cm, yshift=0.8cm,
        label={[lightgray, inner sep=1]above:{$\tile$}}] (\al-z1) at
        (\al) {};

        \draw[role, lightgray, secondary] (\al) -- (\al-z1);

        \drawsubtree[\al-z1,1]
      }

      %% dots
      \foreach \al in {%
        row2, xk2, x1m, xkm, x1M%
      }{ \draw[dotted, lightgray, thick] (\al) -- +(0.5,0.6); }

      %%% other Row-halt branches
      \foreach \al/\tile/\ysh in {%
        xN2/\widehat{T}_2/1.6cm%
      }{ \node[point, secondary, fill=lightgray!50, xshift=0.5cm, yshift=\ysh,
        label={[lightgray, inner sep=1]left:{$\Row$}}] (\al-halt) at (\al) {};

        \node[point, secondary, xshift=0.7cm, yshift=0.3cm,
        label={[lightgray, inner sep=1]right:{$\End$}}] (\al-halt-end) at (\al-halt) {};

        \node[point, secondary, xshift=0.7cm, yshift=-0.3cm,
        label={[lightgray, inner sep=1]above:{$\tile$}}] (\al-halt-tile) at (\al-halt) {};

        \foreach \from/\to in {%
          \al/\al-halt, \al-halt/\al-halt-end, \al-halt/\al-halt-tile%
        }{ \draw[role, lightgray, secondary] (\from) -- (\to); }

        %% disjunction
        \foreach \from/\to in {%
          \al-halt/\al-halt-end, \al-halt/\al-halt-tile%
        }{ \draw[draw=none] (\from) -- node[inner sep=0, outer sep=0, pos=0.55]
          (\from-\to) {} (\to); }

        \draw[lightgray!50!red] (\al-halt-\al-halt-end) to[bend left]
        node[left, inner sep=2, scale=0.7] {$\lor$} (\al-halt-\al-halt-tile);

        \drawsubtree[\al-halt-tile,1]%
      }

      %%% other Row branches
      \foreach \al/\ysh in {%
        xNm/1.6cm%
      }{ \node[point, secondary, fill=lightgray!50, xshift=0.5cm, yshift=\ysh,
        label={[lightgray, inner sep=1]left:{$\Row$}}] (\al-row) at (\al) {};

        \draw[role, lightgray, secondary] (\al) -- (\al-row); 

        \drawsubtree[\al-row,1]%
      }

      %%% other Start branches
      \foreach \al/\ysh in {%
        xk1/1.6cm%
      }{ \node[point, secondary, fill=lightgray!50, xshift=0.6cm, yshift=\ysh,
        label={[lightgray]left:{$\Row$}}]
        (\al-start) at (\al) {};

        \draw[role, lightgray, secondary] (\al) -- (\al-start); 

        \drawsubtree[\al-start,1]%
      }

      %%% other End
      \foreach \al/\tile/\ysh in {%
        xkM/T_2/1.6cm%
      }{ \node[point, secondary, fill=lightgray!50, xshift=0.5cm, yshift=\ysh,
        label={[lightgray, inner sep=1]left:{$\Row$}}] (\al-end-row) at (\al) {};

        \node[point, secondary, xshift=0.7cm,
        label={[lightgray, inner sep=1]right:{$\End$}}] (\al-end) at (\al-end-row) {};

        \foreach \from/\to/\type in {%
          \al/\al-end-row, \al-end-row/\al-end%
        }{ \draw[role, lightgray, secondary] (\from) -- (\to); } %
      }
    \end{scope}    
  \end{pgfonlayer}

  %%% the query
  \begin{scope}[xshift=1cm, yshift=2.25cm]
    \foreach \al/\x/\y/\lab/\wh/\extra in {%
      y/-1/0/{}/right/, %
      y0/0/0/{B_0}/right/rowind, %
      y1/1/0/{B_1}/right/, %
      yN/3/0/{B_N}/right/, %
      yN1/4/0/{B_{N+1}}/right/rowind, %
      y21/5/0/{}/right/, %
      y2N/7/0/{}/right/, %
      y2N1/8/0/{}/right/rowind, %
      y31/9/0/{}/right/, %
      y3N/11/0/{}/right/, %
      y3N1/12/0/{}/right/rowind, %
      yl1/13.5/0/{B_{n{-}N}}/right/, %
      ylN/15.5/0/{B_{n{-}1}}/right/, %
      yl/16.5/0/{B_n}/right/rowind, %
      yend/17.5/0/\End/right/%
    }{ \node[point, \extra, label=above:{ $\lab$}] (\al) at
      (\x,\y) {}; }

    \foreach \from/\to/\type in {%
      y/y0/role, y0/y1/role, y1/yN/dotted, yN/yN1/role, yN1/y21/role,
      y21/y2N/dotted, y2N/y2N1/role, y2N1/y31/role, y31/y3N/dotted,
      y3N/y3N1/role, y3N1/yl1/dotted, yl1/ylN/dotted, ylN/yl/role,
      yl/yend/role%
    } {\draw[thick, \type] (\from) -- (\to);}

    \begin{pgfonlayer}{background}
      \foreach \first/\last in {%
        y1/yN1, y21/y2N1, y31/y3N1, yl1/yl%
      }{ \node[fit=(\first)(\last), rounded corners, fill=gray!20] {}; }
    \end{pgfonlayer}

    \node[xshift=-0.5cm] at (y) {\Large$\q^{r}_n$};
  \end{scope}

  %%% homomorphisms
  \begin{pgfonlayer}{background}

    \foreach \al in {y,y0,yN1,yend,start,row1,endl,endr} { %
      \node[fit=(\al), inner sep=3] (\al-a) {}; }

    %% left
    \foreach \from/\to/\out/\in/\wh in {%
      y-a/a/255/105/left, yN1-a/start-a/210/90/above, yend-a/endl-a/250/30/above%, ylN/xNm/190/110%
    } {\draw[homomorphism, black!50] (\from) to[out=\out, in=\in]
      node[\wh] {$h_l$} (\to);}

    %% right
    \foreach \from/\to/\in/\wh in {%
      y-a/a/290/right, yN1-a/row1-a/90/left, yend-a/endr-a/110/left%, ylN/xNM/90%
    } {\draw[homomorphism, black!75] (\from) to[bend left=15]
      node[\wh] {$h_r$} (\to);}
  \end{pgfonlayer}
\end{tikzpicture}

%\end{document}

%%% Local Variables:
%%% mode: latex
%%% TeX-master: "aij-insep"
%%% TeX-PDF-mode: t
%%% save-place: t
%%% End:
}
  \caption{The structure of models $\I_l$ and $\I_r$ of $\K_2$, and
    homomorphisms $h_l \colon \q^{r}_n \to \I_l$ and $h_r \colon \q^{r}_n \to \I_r$.}
  \label{fig:k2-rooted}
\end{figure*}
\begin{proof}
For (\emph{i}), we do not use the role name
$P$ but add $R(a,a)$ and $\textit{Row}(a)$ to the ABox $\{A(a)\}$. The CQs $\q_{n}$ are modified by adding
a conjunct $R(y,x_0)$ with answer variable $y$ to $\q_n$.
% (The loop $R(a,a)$ plays roughly the same role as the path between two $\Start$-points in Fig.~\ref{fig:k2}.) 
%
In more detail, suppose that an instance $\Tmf$ of the rectangle tiling problem is given. Let
\begin{align}\label{r-abox}
\ArCQ = \{ R(a,a), \Row(a), A(a) \} \cup \{\widehat T_0(a) \mid T \in \mathfrak{T}\}, 
\end{align}
let $\TtworCQ$ contain the CIs \eqref{start-row-k}--\eqref{vertical-halt} of $\TtwoCQ$ as well as 
\begin{align}\label{r-initial}
& A \sqsubseteq \exists R. (\Row \sqcap \exists R. I_0),
\end{align}
and let $\KtworCQ = (\TtworCQ, \ArCQ)$. Note that the loop $R(a,a)$ in $\ArCQ$ plays roughly the same role as 
the path between 
two $\Start$-points in the previous construction (see Fig.~\ref{fig:k2}).
The existence of a tiling can now be checked by the rCQs
\begin{equation*}
\q^{r}_n(y) ~=~  \exists \avec{x} \, \big( R(y, x_0) \land \bigwedge_{i=0}^{n} \bigl( R(x_i,
  x_{i+1}) \land B_{i}(x_{i})\bigr) \land \End(x_{n+1}) \big),
\end{equation*}
where $B_{i} \in \{\Row\} \cup \{\widehat{T}_k \mid T \in \mathfrak{T}, k=0,1,2\}$, for which we have an analogue of Lemma~\ref{qn-tiling} for $\KtworCQ$. 
The structure of the two homomorphisms is shown in 
Fig.~\ref{fig:k2-rooted}. Note that the CQ encodes the first row two times.
Now, we take $\KonerCQ = (\TonerCQ,\ArCQ)$, where $\TonerCQ$ contains the following CIs (recall that we set $\Sigma_0 = \{\Row\} \cup \{\widehat{T}_k \mid T \in \mathfrak{T}, k = 0,1,2\}$):
\begin{align}
  A \sqsubseteq &~\exists R.D ~\sqcap~ \exists R.\exists R. E,\\
  D\sqsubseteq &~ \exists R.D ~\sqcap~ \exists R. \exists R. E ~\sqcap \bigsqcap_{X \in \Sigma_0} X,\\
  E \sqsubseteq &~ \exists R. E ~\sqcap~ \bigsqcap_{X \in \Sigma_0} X ~\sqcap~ \End.
\end{align}
The canonical model $\I_{\KonerCQ}$ of $\KonerCQ$ is shown below:
\begin{center}
  \begin{tikzpicture}[xscale=2, yscale=1.2, %
    emptyind/.style={fill=gray!50, inner sep=1.7},
    rotatelabels/.style={rotate=-14},
    abovesloped/.style={above,rotatelabels}]
    \small
    
    \foreach \al/\x/\y/\lab/\wh/\extra/\rot in {%
      a/0/0/{\qquad\qquad A,\Row,\widehat{T}_0}/above/constant/, %
      a2/0/-0.7/{}/below/emptyind/, %
      a3/0.7/-1/{\End,\Sigma_{0},E~~~}/below//rotatelabels, %
      a4/1.4/-1.3/{~~~\End,\Sigma_{0},E}/below//rotatelabels, %
      x1/1.5/0/{\Sigma_0,D}/above//, %
      x2/1.5/-0.7/{}/below/emptyind/, %
      x3/2.2/-1/{\End,\Sigma_{0},E~~~}/below//rotatelabels, %
      x4/2.9/-1.3/{~~~\End,\Sigma_{0},E}/below//rotatelabels, %
      y1/3/0/{\Sigma_0,D}/above//,%
      y2/3/-0.7/{}/above/emptyind/,%
      y3/3.7/-1/{\End,\Sigma_{0},E~~~}/below//rotatelabels,%
      y4/4.4/-1.3/{~~~\End,\Sigma_{0},E}/below//rotatelabels,%
      z1/4.5/0/{\Sigma_0,D}/above//,%
      z2/4.5/-0.7/{}/above/emptyind/,%
      z3/5.2/-1/{\End,\Sigma_{0},E~~~}/below//rotatelabels,%
      z4/5.9/-1.3/{~~~\End,\Sigma_{0},E}/below//rotatelabels%
    }{ \node[point, \extra, label={[\rot, inner sep=2]\wh:{
          $\lab$}}] (\al) at (\x,\y) {}; }

    \node[label={[inner xsep=1]left:{$a$}}] at (a) {};%

    \foreach \from/\to/\wh in {%
      a/a2/left, a2/a3/abovesloped, a3/a4/abovesloped, %
      a/x1/above, x1/x2/left, x2/x3/abovesloped, x3/x4/abovesloped, %
      x1/y1/above, y1/y2/left, y2/y3/abovesloped, y3/y4/abovesloped, %
      y1/z1/above, z1/z2/left, z2/z3/abovesloped, z3/z4/abovesloped%
    } {\draw[role] (\from) -- node[\wh] {$R$} (\to);}

    \draw[role] (a) to[out=230,in=130,looseness=12] node[left] {$R$} (a);

    % \foreach \al/\lab/\wh in {%
    % x2/\sigma_{\End}/right, x0/\sigma_{\Start}/right%, w0/\sigma/left%
    % } {\node[label={[red, inner sep=1]\wh:{\footnotesize $\lab$}}] at (\al)
    %   {};}
    \foreach \from in {%
      a4, x4, y4, z4%, w2, v11a, v00a, u0, u1a%
    } {\draw[dotted, thick] (\from) -- +(0.5,-0.2);}

    \draw[dotted, thick] (z1) -- +(1,0);
  \end{tikzpicture}
\end{center}

We set $\SigmarCQ = \sig(\KonerCQ)$. Again, one can show Lemma~\ref{qn-rejected-by-k1} for 
$\KonerCQ$ and $\KtworCQ$. The proof of~(\emph{ii}) is similar to the non-rooted case
and given in the appendix.
\end{proof}

\subsection{Undecidability of \textup{(}r\textup{)}CQ-entailment and inseparability for full signature}

The KBs used in the undecidability proofs above trivially do not
$\Sigma$-CQ-entail each other for the \emph{full signature}~$\Sigma$. For
example, the answer to the CQ $\exists y\exists z \, (P(a,y) \wedge R(y,z)
\wedge I^{\first}(z))$ is `yes' over $\KtwoCQ$ and `no' over $\KoneCQ$. To establish undecidability results for separating CQs with arbitrary symbols, we modify the KBs constructed above.  We follow~\cite{Lutz2012non} and replace the
non-$\Sigma$-symbols by complex \ALC-concepts that, in contrast to concept
names, cannot occur in CQs.
Let $\Gamma$ be a set of concept names. For any $B\in \Gamma$, let $Z_{B}$ be a
fresh concept name and let $R_{B}$ and $S_{B}$ be fresh role names. The
\emph{abstraction of $B$} is the $\mathcal{ALC}$-concept
$$
H_{B} ~=~ \forall R_{B}.\exists S_{B}.\neg Z_{B}.
$$
The \emph{$\Gamma$-abstraction} $\Cuup$ of a (possibly compound) concept $C$
is obtained from $C$ by replacing every $B\in \Gamma$ with~$H_{B}$. The
\emph{$\Gamma$-abstraction $\Tup$} of a TBox $\T$ is obtained from $\T$ by
replacing all concepts in $\T$ with their $\Gamma$-abstractions.  We associate
with $\Gamma$ an auxiliary TBox
$$
\T_{\Gamma}^{\exists} ~=~ \{~ \top \sqsubseteq \exists R_{B}.\top, \ \top \sqsubseteq \exists S_{B}.Z_{B}\mid B \in \Gamma ~\}
$$
and call $\Tup \cup \T_{\Gamma}^{\exists}$ the
\emph{enriched $\Gamma$-abstraction} of $\T$ for $\Gamma$. In what
follows, we are going to replace TBoxes $\T$ with their enriched $\Gamma$-abstractions. 
We say that a TBox $\T$ \emph{admits trivial models} if any interpretation $\I$ with $X^{\I}=\emptyset$,
for any concept or role name $X$, is a model of $\T$. The TBoxes used in the undecidability proofs above
admit trivial models.
\begin{theorem}\label{Thm:equivalence}
Suppose $\K_{1}=(\T_{1},\A)$ and $\K_{2}=(\T_{2},\A)$ are $\mathcal{ALC}$ KBs and $\Sigma$ a signature such that ${\sig}(\A) \subseteq \Sigma$, $\Gamma ={\sig}(\T_{1}\cup \T_{2})\setminus \Sigma$ contains no role names, and $\T_{1}$ and $\T_{2}$ admit trivial models. 
Let $\Kup_{i}=(\Tup_{i}\cup \T_{\Gamma}^{\exists},\A)$, for $i=1,2$.
Then the following conditions are equivalent:
\begin{enumerate}
\item[$(1)$] $\K_{1}$ $\Sigma$-\textup{(}r\textup{)}CQ entails $\K_{2}$;

\item[$(2)$] $\Kup_{1}$ full signature \textup{(}r\textup{)}CQ entails $\Kup_{2}$.
\end{enumerate}
\end{theorem}
\begin{proof}
We start by defining the \emph{$\Gamma$-abstraction} $\Iup$ and the \emph{$\Gamma$-instantiation} $\Idown$ of an interpretation $\I$. 
The latter is defined in the same way as $\I$ except that 
$B^{\Idown}={H_{B}}^{\I}$, for all $B\in \Gamma$. It is straightforward to show the following.

\smallskip
\noindent 
\emph{Fact} 1. For all $\mathcal{ALC}$ concepts $D$ over the signature ${\sf sig}(\K_{1}\cup \K_{2})$ and all 
$d\in \Delta^{\I}$, we have
$
d\in D^{\Idown}$ iff $d\in (D^{\uparrow\Gamma})^{\I}.
$
In particular, if $\I$ is a model of $\Kup_{i}$, then $\Idown$ is a model of $\K_{i}$, for $i=1,2$.

\smallskip

We now define the interpretation $\Iup$. The domain $\Delta^{\Iup}$ of $\Iup$ is the set of words $w= d
v_{1}\cdots v_{n}$ such that $d\in \Delta^{\I}$ and $v_{i}\in
\{R_{B},S_{B},\bar{S}_{B}\mid B \in \Gamma\}$, where $v_{i}\not=\bar{S}_{B}$ if either 
(\emph{i}) $i>2$ or (\emph{ii}) $i=2$ and $d\not\in B^{\Imc}$ or $v_{1}\not=R_{B}$. Then
\begin{align*}
A^{\Iup}  &  ~=~ A^{\I}, \text{ for all concept names $A\in {\sf sig}(\K_{1}\cup \K_{2})\setminus \Gamma$;}\\
B^{\Iup} &  ~=~ \emptyset, \text{ for all concept names $B\in \Gamma$;}\\
{Z_{B}}^{\Iup} & ~=~ {Z_{B}}^{\I} \cup \{w \mid {\sf tail}(w) = S_{B}\}, \text{ for all concept names $B\in \Gamma$;}\\ 
S^{\Iup} & ~=~ S^{\I}, \text{ for all role names $S\not\in \{R_{B},S_{B} \mid B\in \Gamma\}$;}\\
{R_{B}}^{\Iup} &~=~ {R_{B}}^{\I} \cup \{(w,wR_{B}) \mid wR_{B}\in \Delta^{\Iup}\}, 
\text{ for all concept names $B\in \Gamma$;}\\				
{S_{B}}^{\Iup} &~=~ {S_{B}}^{\I} \cup \{(w,wS_{B}) \mid wS_{B}\in \Delta^{\Iup}\} \cup 
\{(w,w\bar{S}_{B}) \mid w\bar{S}_{B}\in \Delta^{\Iup}\}, \text{ for all concept names $B\in \Gamma$.}
\end{align*}
By the construction of $\Iup$, we have ${H_{B}}^{\Iup}=B^{\I}$, for all concept names $B\in \Gamma$.  For the interpretation $\I$ below consisting of two elements $d_1$ and $d_2$ with $d_1 \in B^\I$ and $d_2 \in (\neg B)^\I$ and $\Gamma = \{B\}$, the $\Gamma$-abstraction $\Iup$ can be depicted as follows, where the grey points \tikz[baseline]\node[point,draw=gray,fill=gray!80,inner sep=2pt] at (0,0.1) {}; correspond to the words of the form $w\bar{S}_{B}$:
\begin{center}
\begin{tikzpicture}[notSb/.style={draw=gray, fill=gray!80, inner sep=2pt}]
  \node[point, constant, label=left:{$d_1$}, label=above:{$B$}] {};
  \node[point, constant, label=right:{$d_2$}, label=above:{$\neg B$}] at (1,0) {};

  \node at (0.5,1) {$\I$};

  \begin{scope}[shift={(6,0)}]
    \node[point, constant, label=above:{$d_1$}] (d1) {};

    \foreach \al/\x/\y/\conc/\wh/\extra in {%
      x/-1.6/-1//above/,%
      y/0/-1/{Z_B}/right/,%
      z/1.6/-1//right/notSb,%
      x1/-3.5/-2//above/,%
      x2/-2.5/-2/{\qquad~~Z_B}/center/,%
      x3/-1.5/-2//center/notSb,%
      y1/-0.5/-2//above/,%
      y2/0.5/-2/{\qquad~~Z_B}/center/,%
      z1/1.5/-2//above/,%
      z2/2.5/-2/{\qquad~~Z_B}/center/%
    }{ \node[point, \extra, label=\wh:{\scriptsize $\conc$}] (\al) at (\x,\y) {};}

    \foreach \from/\to/\lab/\wh in {%
      d1/x/R_B/left, d1/y/S_B/right, d1/z/S_B/right, %
      x/x1/R_B/left, x/x2/S_B/right, x/x3/S_B/right,%
      y/y1/R_B/left, y/y2/S_B/right, z/z1/R_B/left, z/z2/S_B/right%
    }{ \draw[role] (\from) -- node[\wh] {\scriptsize $\lab$} (\to); }

    \foreach \from in {%
      x1, x2, x3, y1, y2, z1, z2%
    }{ \draw[role] (\from) -- node[inner sep=0.2mm, left] {\scriptsize $R_B$} ++(-0.2,-0.9);%
      \draw[role] (\from) -- node[inner sep=0.2mm, right] {\scriptsize $S_B$} ++(0.2,-0.9); }

    \draw[loosely dotted, thick] (-3.7,-3.1) -- ++(6.5,0);
    
    \node at (3,1) {$\Iup$};
  \end{scope}

  \begin{scope}[shift={(12,0)}]
    \node[point, constant, label=above:{$d_2$}] (d2) {};

    \foreach \al/\x/\y/\conc/\wh/\extra in {%
      x/-1.6/-1//above/,%
      y/0/-1/{Z_B}/right/,%
      z/1.6/-1//right/notSb,%
      x1/-2.5/-2//above/,%
      x2/-1.5/-2/{\qquad~~Z_B}/center/,%
      y1/-0.5/-2//above/,%
      y2/0.5/-2/{\qquad~~Z_B}/center/,%
      z1/1.5/-2//above/,%
      z2/2.5/-2/{\qquad~~Z_B}/center/%
    }{ \node[point, \extra, label=\wh:{\scriptsize $\conc$}] (\al) at (\x,\y) {};}

    \foreach \from/\to/\lab/\wh in {%
      d2/x/R_B/left, d2/y/S_B/right, d2/z/S_B/right, x/x1/R_B/left, x/x2/S_B/right,%
      y/y1/R_B/left, y/y2/S_B/right, z/z1/R_B/left, z/z2/S_B/right%
    }{ \draw[role] (\from) -- node[\wh] {\scriptsize $\lab$} (\to); }

    \foreach \from in {%
      x1, x2, y1, y2, z1, z2%
    }{ \draw[role] (\from) -- node[inner sep=0.2mm, left] {\scriptsize $R_B$} ++(-0.2,-0.9);%
      \draw[role] (\from) -- node[inner sep=0.2mm, right] {\scriptsize $S_B$} ++(0.2,-0.9); }

    \draw[loosely dotted, thick] (-2.7,-3.1) -- ++(5.5,0);
  \end{scope}

\end{tikzpicture}
\end{center}

\smallskip
\noindent 
\emph{Fact} 2. For all $\ALC$ concepts $D$ over the signature ${\sf sig}(\K_{1}\cup \K_{2})$ and all $d\in \Delta^{\I}$, we have
$
d\in (\Dup)^{\Iup}$ iff $d\in D^{\I}.
$
Moreover, if $\I$ is a model of $\K_{i}$, then $\Iup$ is a model of $\Kup_{i}$, for $i=1,2$.

\smallskip
\noindent
\emph{Proof of Fact~$2$}. For the `moreover'-part, observe that, for 
$C\sqsubseteq D\in \T_{i}$ and $d\in \Delta^{\I}$, we have that $d\in
(\Cuup)^{\Iup}$ implies $d\in (\Dup)^{\Iup}$ by the first part of Fact~2. For 
$d\in \Delta^{\Iup}\setminus \Delta^{\I}$, observe that $d\not\in H_{B}^{\Iup}$ for any $B\in \Gamma$, 
$d\not\in A^{\Iup}$ and any concept name $A\in {\sf sig}(\K_{1}\cup \K_{2})$, and $(d,d')\not\in R^{\Iup}$ for any $d'$
and role name $R\in {\sf sig}(\K_{1}\cup \K_{2})$. Thus, if $C\sqsubseteq D\in \T_{i}$ and $d\in C^{\Iup}$ then it follows
from the condition that $\T_{i}$ admits trivial models that $d\in D^{\Iup}$.
Thus $\Iup$ is a model of $\Tup_{i}$. Since $\Iup$
is a model of $\T_{\Gamma}^{\exists}$ by construction, it follows that $\Iup$
is a model of $\Tup_{i}\cup \T_{\Gamma}^{\exists}$. 

\smallskip

We collect further basic properties of the interpretations $\Iup$ and $\Idown$. In the formulation
and proofs of Facts~3--6 below, the homomorphisms are always constructed in such a way that individual names are preserved. For simplicity, we do not state this explicitly.  

\smallskip
\noindent
\emph{Fact} 3. Let $\I,\J$ be interpretations and $n>0$. 
%ith $X^{\I}=X^{\J}=\emptyset$
%for all concept and role names $X\not\in \sig(\K_{1}\cup \K_{2})$. 
If $\I$ is $n$-homomorphically embeddable into $\J$, then $\Iup$ is $n$-homomorphically embeddable into $\Jup$.

\smallskip
\noindent
\emph{Proof of Fact~$3$.}
Suppose $n>0$ and $\I$ is $n$-homomorphically embeddable into $\J$.
Let $\I'$ be a subinterpretation of $\Iup$ with $|\Delta^{\I'}|\leq n$.  For
the subinterpretation $\I''$ of $\I$ induced by $\Delta_{0}= \Delta^{\I} \cap
\Delta^{\I'}$, there exists a homomorphism $h_{0}$ from $\I''$ to $\J$. We extend
$h_{0}$ to a homomorphism $h$ from $\I'$ to $\Jup$ inductively as follows. 
Suppose $d\in \Delta^{\I'} \setminus \Delta^{\I}$ and $h(d)$ has not yet been
defined, but there is no $R_{B}$ or $S_{B}$-predecessor of $d$ in $\Iup$ for
which $h(d)$ has not been defined.  We distinguish three cases (which are
mutually exclusive by the construction of $\Iup$). If (\emph{i}) $h(d')$ has been defined for an $R_{B}$-predecessor $d'$ of $d$ in $\I'$, then choose an $R_{B}$-successor $e$ of $h(d')$ in $\Jup$ and set $h(d)=e$. Observe that such an
$R_{B}$-successor exists by the construction of $\Jup$. If (\emph{ii}) $h(d')$ has been
defined for an $S_{B}$-predecessor $d'$ of $d$ in $\I'$, then choose an
$S_{B}$-successor $e$ of $h(d')$ in $\Jup$ such that $e\in {Z_{B}}^{\Jup}$ and
set $h(d)=e$. Again such an $S_{B}$-successor exists by the construction of
$\Jup$. (\emph{iii}) There does not exist any $R_{B}$ or $S_{B}$-predecessor of $d$ in
$\I'$ for which $h$ has been defined. In this case, choose $h(d)$ arbitrarily
in $\Jup$ such that if $d\in {Z_{B}}^{\Iup}$, then $h(d)\in {Z_{B}}^{\Jup}$. Such a $d$ exists since ${Z_{B}}^{\Jup}\not=\emptyset$. The resulting map is a
homomorphism from $\I'$ to $\Jup$.
% The following fact can be shown similarly:

\smallskip
\noindent
\emph{Fact} 4. Let $\I$ be a model of $\Kup$, for $\K \in \{\K_1,\K_2\}$. Then
$\up{(\Idown)}$ is homomorphically embeddable into $\I$.

\medskip
\noindent
\emph{Proof of Fact~4.} Let $h_0$ be the identity mapping from $\Idown$
to $\I$ (observe that $\Delta^{\Idown} = \Delta^{\I}$). One can now extend $h_0$ to a
homomorphism $h$ from $\up{(\Idown)}$ to $\I$ in the same way as in the construction of $h$
in the proof of Fact~3 above. 
% $d\in \Delta^{\up{(\Idown)}} \setminus \Delta^{\Idown}$ and $h(d)$ has not
% yet been defined but there is no $R_{B}$ or $S_{B}$-predecessor of $d$ in
% $\up{(\Idown)}$ for which $h(d)$ has not been defined.  We distinguish two
% cases (which are mutually exclusive by construction of $\up{(\Idown)}$): (i)
% $h(d')$ has been defined for the $R_{B}$-predecessor $d'$ of $d$ in
% $\up{(\Idown)}$.  Then choose an $R_{B}$-successor $e$ of $h(d')$ in $\I$ and
% set $h(d)=e$. Observe that such an $R_{B}$-successor exists since $\I$ is a
% model of $\T_{\Gamma}^{\exists}$. (ii) $h(d')$ has been defined for the
% $S_{B}$-predecessor $d'$ of $d$ in $\up{(\Idown)}$.  Then choose an
% $S_{B}$-successor $e$ of $h(d')$ in $\I$ such that $e\in {Z_{B}}^{\I}$ and
% set $h(d)=e$.  Again such an $R_{B}$-successor exists since $\I$ is a model
% of $\T_{\Gamma}^{\exists}$. The resulting mapping is a homomorphism from
% $\up{(\Idown)}$ to $\I$.

\smallskip
\noindent
\emph{Fact} 5. Let $\K\in \{\K_{1},\K_{2}\}$. If $\Mod$ is complete for $\K$, 
then $\{ \Iup \mid \I \in \Mod\}$ is complete for $\Kup$.

\smallskip
\noindent
\emph{Proof of Fact~$5$.}
Suppose $\J$ is a model of $\Kup$. By Proposition~\ref{prop:char}, it suffices to show that, for any $n>0$, there is
$\I\in \Mod$ such that $\Iup$ is $n$-homomorphically embeddable into $\J$. Fix $n>0$ and consider the interpretation $\Jdown$.
By Fact~1, $\Jdown$ is a model of $\K$ and so there exists a model $\I$ of $\K$ such that $\I$ is $n$-homomorphically
embeddable into $\Jdown$. But then, by Fact~3, $\Iup$ is $n$-homomorphically embeddable into $\up{(\Jdown)}$ which, by Fact~4,
itself is homomorphically embeddable into $\J$. Thus, $\Iup$ is $n$-homomorphically embeddable into
$\J$. By Fact~2, $\Iup$ is a model of $\Kup$.

\smallskip
\noindent
\emph{Fact} 6. Let $\Mod_{i}$ be families of interpretations with
$X^{\I}=\emptyset$, for all $\I\in \Mod_{i}$ and all concept and role names $X$ with $X\not\in
\sig(\K_{i})$, $i=1,2$. Then the following conditions
are equivalent:
\begin{itemize}
\item[(a)] $\prod \Mod_{2}$ is $n\Sigma$-homomorphically embeddable into $\prod \Mod_{1}$, for all $n>0$;

\item[(b)] $\prod \{ \Iup \mid \I \in \Mod_{2}\}$ is $n$-homomorphically embeddable into 
$\prod \{ \Iup \mid \I \in \Mod_{1}\}$, for all $n>0$.
\end{itemize}
\emph{Proof of Fact~$6$.}
Suppose $\Mod_{1}= \{\I_{i} \mid i \in I\}$ and $\Mod_{2}= \{ \J_{j} \mid j \in J\}$.  

Assume first that (a) holds and let $\J$ is a subinterpretation of 
$\prod \{ \Jup_{j} \mid j \in J\}$ with $|\Delta^{\J}|\leq n$.
We have to construct a homomorphism from $\J$ to $\prod \{ \Iup_{i} \mid i\in I\}$.
There is a $\Sigma$-homomorphism $h_{0}$ from the subinterpretation $\J'$ of $\prod \Mod_{2}$ induced by
$\Delta^{\J} \cap \Delta^{\prod \Mod_{2}}$ to $\prod \Mod_{1}$. By definition, $h_{0}$ is a homomorphism from the 
subinterpretation $\J''$ of $\prod \{ \Jup_{j} \mid j \in J\}$ induced by $\Delta^{\J} \cap \Delta^{\prod \Mod_{2}}$ 
to $\prod \{ \Iup_{i} \mid i\in I\}$ (the only difference between $\J'$ and $\J''$ is that $B^{\J''}=\emptyset$ for 
all $B\in \Gamma$). Following the proof of Fact~3, one can now expand $h_{0}$ to a homomorphism $h$ from 
$\J$ to $\prod \{ \Iup_{i} \mid i\in I\}$.

Conversely, assume that (b) holds and assume that $\J$ is a subinterpretation of $\prod \Mod_{2}$ with $|\Delta^{\J}|\leq n$.
We have to construct a $\Sigma$-homomorphism from $\J$ to $\prod \Mod_{1}$. 
There is a $\Sigma$-homomorphism $h_{0}$ from the subinterpretation $\J'$ of $\prod \{ \Jup_{j} \mid j \in J\}$
induced by $\Delta^{\J}$ to $\prod \{ \Iup_{i} \mid i\in I\}$. To obtain from $h_{0}$ the required $\Sigma$-homomorphism
$h$, we have to re-define $h_{0}(d)$ for any $d$ with $h_{0}(d)\in \Delta^{\prod \{ \Iup_{i} \,\mid\, i\in I\}}\setminus \Delta^{\prod\Mod_{1}}$. Consider such a $d$. Observe that $h_{0}(d)\not\in B^{\prod \{ \Iup_{i} \,\mid\, i\in I\}}$ for any concept name $B\in \Sigma$ and 
$h_{0}(d)$ is not in the range or domain of any $R^{\prod \{ \Iup_{i} \,\mid\, i\in I\}}$ for any role name $R\in \Sigma$. 
But then, since $h_{0}$ is a $\Sigma$-homomorphism, $d\not\in B^{\J}$ for any concept name $B\in \Sigma$ and $d$ is 
not in the range or domain of $R^{\J}$ for any role name $R\in \Sigma$. Thus, we can choose $h(d)$ arbitrarily in 
$\Delta^{\prod \Mod_{1}}$ and obtain the required $\Sigma$-homomorphism.

\smallskip
For CQs, Theorem~\ref{Thm:equivalence} now follows directly from  Theorem~\ref{crit:KB} (3) and 
Facts~5 and 6. Note that we can consider sets $\Mod_{i}$ of interpretations that are complete for $\K_{i}$ such that
$X^{\I}=\emptyset$, for all $\I\in \Mod_{i}$ and all concept and role names $X$ with $X\not\in \sig(K_{i})$, $i=1,2$.
For rCQs, we use Theorem~\ref{crit:KB} (4).
\end{proof}

Now, to prove undecidability of full signature (r)CQ entailment and inseparability, we apply Theorem~\ref{Thm:equivalence} to the KBs constructed in the proofs of Theorems~\ref{thm:undecidability-cq},
\ref{thm:undecidability-cqinsep} and~\ref{thm:undecidability-rcq}.
Note that the KBs $(\KoneCQ)^{\uparrow\Gamma}$ with $\Gamma=\sig(\KoneCQ\cup\KtwoCQ)\setminus\SigmaCQ$ and
$(\KonerCQ)^{\uparrow\Gamma}$ with $\Gamma=\sig(\KonerCQ \cup\KtworCQ)\setminus\SigmarCQ$ are still $\mathcal{EL}$-KBs
since $\SigmaCQ = \sig(\KoneCQ)$ and $\SigmarCQ= \sig(\KonerCQ)$.
\begin{theorem}\label{thm:undecidability-full}
\textup{(}i\textup{)} The problem whether an $\mathcal{EL}$ KB full signature-\textup{(}r\textup{)}CQ entails an $\ALC$ KB is undecidable.

\textup{(}ii\textup{)} Full signature-\textup{(}r\textup{)}CQ inseparability between $\mathcal{EL}$ and \ALC KBs is undecidable.
\end{theorem}

%\% !TEX root = aij-insep.tex

\section{Decidability of (r)UCQ-Entailment and Inseparability for \ALC KBs}\label{sec:ucq-kbs}

We show that, in sharp contrast to the case of (r)CQs,
$\Sigma$-(r)UCQ-entailment and inseparability of \ALC KBs are
decidable and 2\ExpTime-complete.  We start by proving a new model-theoretic criterion
for $\Sigma$-(r)UCQ entailment that replaces finite partial
$\Sigma$-homomorphisms by $\Sigma$-homomorphisms and uses the class of
regular forest-shaped models for the entailing KB $\K_1$ and the class
of forest-shaped models for the entailed KB $\K_2$.  We then encode
this characterisation into an emptiness problem for two-way
alternating parity automata on infinite trees (2APTAs) by constructing
a 2APTA that accepts (representations of) forest-shaped models of the
entailing KB into which there is no $\Sigma$-homomorphism from any
forest-shaped model of the entailed KB.  Rabin's result that such an
automaton accepts a regular model iff it accepts any model will then yield
the desired 2\ExpTime upper bound for (r)UCQ-entailment.
%We use the automate theoretic encoding also to obtain a homomorphism
%based characterization of $\Sigma$-(r)UCQ entailment for arbitrary forest-shaped models.
%We then modify the KBs used for the 2\ExpTime hardness result for $\Sigma$-(r)UCQ entailment to
%prove 2\ExpTime-hardness of $\Sigma$-(r)UCQ inseparability.
Matching lower bounds are proved by a reduction of the word problem for exponentially space bounded alternating Turing machines.
Finally, we show that the same tight complexity bounds still hold in the full signature case.
\subsection{Model-theoretic characterisation of \textup{(}r\textup{)}UCQ-entailment based on regular models}

We show that finite partial homomorphisms can be replaced by
homomorphisms in the characterisation of $\Sigma$-(r)UCQ entailment
between $\mathcal{ALC}$-KBs given in Theorem~\ref{crit:KB} if one
considers regular forest-shaped models of the entailing KB $\K_1$ and
forest-shaped models of the entailed KB $\K_2$. Recall that, by
Proposition~\ref{regularcomplete}, the class
$\Mod^{\it reg}_{\K}$ of regular forest-shaped models of outdegree
$\leq |\T|$ is complete for any $\ALC$-KB $\K=(\T,\A)$.  We also show that if $\Sigma$
contains all role names in the entailed KB, then $\Sigma$-rUCQ
entailment coincides with $\Sigma$-UCQ entailment. This allows us
to transfer our 2\ExpTime lower bound from the non-rooted to the
rooted case.
%\Mod^{\it reg}_{\K_{1}}$\Mod^{\it reg}_{\K_{1}}$.
%
%
\begin{theorem}\label{crit:KBUCQ1}
Let $\K_1$ and $\K_2$ be \ALC-KBs and $\Sigma$ a signature.
\begin{description}\itemsep=0pt
\item[\rm (1)] $\K_{1}$ $\Sigma$-UCQ entails $\K_2$ iff, for any $\I_1\in \Mod^{\it reg}_{\K_{1}}$, there exists $\I_2 \in \Mod^{\it bo}_{\K_{2}}$
that is $\Sigma$-homomorphically embeddable into $\I_1$ preserving $\ind(\K_{2})$.
	
\item[\rm (2)] $\K_{1}$ $\Sigma$-rUCQ entails $\K_2$ iff, for any $\I_1\in \Mod^{\it reg}_{\K_{1}}$, there exists
$\I_2 \in \Mod^{\it bo}_{\K_{2}}$ that is con-$\Sigma$-homomorphically embeddable into $\I_1$
preserving $\ind(\K_{2})$.
\end{description}
\end{theorem}
\begin{proof}
We only prove (1) as the proof of (2) is similar. The direction ($\Leftarrow$) follows from
Theorem~\ref{crit:KB} and the facts that $\Mod^{\it reg}_{\K_{1}}$ and $\Mod^{\it bo}_{\K_{2}}$ are complete for $\K_{1}$
and $\K_{2}$, respectively (Propositions~\ref{forestcomplete} and~\ref{regularcomplete}).
To show ($\Rightarrow$), suppose that $\K_{1}$ $\Sigma$-UCQ entails $\K_2$ and let $\I_1\in \Mod^{\it reg}_{\K_{1}}$.
We construct $\I_{2}\in \Mod^{\it bo}_{\K_{2}}$ and a $\Sigma$-homomorphism $h$ from $\I_{2}$ to $\I_{1}$ preserving $\ind(\K_{2})$.
By Theorem~\ref{crit:KB} (1) and Propositions~\ref{forestcomplete} and~\ref{regularcomplete}, we have

\smallskip
\noindent
  ($\ast$) for any $n>0$,
  there exists a model $\J \in \Mod^{\it bo}_{\K_{2}}$
  that is $n\Sigma$-homomorphically embeddable into $\I_1$ preserving $\ind(\K_{2})$.

  \smallskip
  \noindent
  Denote by $\J_{|\leq n}$ the subinterpretation of an interpretation $\J\in \Mod^{\it bo}_{\K_{2}}$ induced by the domain elements of $\J$ connected to ABox individuals in $\ind(\K_{2})$ by paths of role names (possibly not in $\Sigma$) of length $\le n$. A
  \emph{$(\Sigma,n)$-homomorphism $h$ from $\J$ to $\I_{1}$ preserving $\ind(\K_{2})$} is a $\Sigma$-homomorphism preserving $\ind(\K_{2})$
  whose domain is a finite subinterpretation of $\J$ that contains $\J_{|\leq n}$. Let $\Xi_{n}$ be the class of pairs $(\J,h)$ with $\J \in
  \Mod^{\it bo}_{\K_{2}}$ and $h$ a $(\Sigma,n)$-homomorphism from $\J$ to
  $\I_{1}$. By ($\ast$), all $\Xi_{n}$ are non-empty. We may assume that for
  $(\I,h),(\J,f)\in \bigcup_{n\geq 0}\Xi_{n}$, if $\I_{|\leq n}$ and $\J_{|\leq n}$ are isomorphic then $\I_{|\leq
    n}=\J_{|\leq n}$, for all $n\geq 0$. We
  define classes $\Theta_{n} \subseteq \bigcup_{m\geq n}\Xi_{m}$, $n\geq 0$, with $\Theta_{0}\supseteq \Theta_{1}\supseteq \cdots$
  such that the following conditions hold:
  \begin{itemize}
  \item[(a)] $\Theta_{n} \cap \Xi_{m} \not=\emptyset$ for all $m\geq n$;
  \item[(b)] $\I_{|\leq n} = \J_{| \leq n}$ and $h_{|\leq n} = f_{|\leq n}$ for
    all $(\I,h),(\J,f)\in \Theta_{n}$ (here and below, $h_{|\leq n}$ denotes the restriction of
    $h$ to $\I_{|\leq n}$).
  \end{itemize}
  Let $\Theta_{0}$ be the set of all pairs $(\J,h)$ such that $(\J,h)\in
  \Xi_{0}$. Our assumptions imply that $\Theta_{0}$ has the
  properties (a) and (b) because $h(a^\J)=a^{\I}$ holds for every $\Sigma$-homomorphism $h$ preserving $\ind(\K_{2})$
  and all ABox individuals $a\in \ind(\K_{2})$. Suppose now that $\Theta_{n}$ is defined and satisfies (a) and (b).
% For $\J\in \Mod^{\it fo}_{\K_{2}}$ let $\Delta_{\J}^{\it pre}$ contain those
%  $x\in \Delta^{\J_{|\leq n+1}}\setminus\Delta^{\J_{|\leq n}}$ that are $R$-successors of some $y\in \Delta^{\J_{|\leq n}}$ for
%  some role name $R\in \Sigma$.
% Let \Delta_{\J}^{1}  & = & \Delta^{\J_{|\leq n+1}}\setminus (\Delta^{\J_{|\leq n}} \cup \Delta_{\J}^{0})
%  \end{eqnarray*}
  Define an equivalence relation $\sim$ on $\Theta_{n} \cap (\bigcup_{m\geq n+1}\Xi_{m})$ by setting
  $(\I,h) \sim (\J,f)$ if
  $\I_{|\leq n+1} = \J_{| \leq n+1}$ and, for all $x\in \Delta^{\J_{|\leq n+1}}\setminus\Delta^{\J_{|\leq n}}$, the
  following holds: $h(x)$ and $f(x)$ are always roots of isomorphic ditree subinterpretations of
  $\I_{1}$ and if, in addition, either $h(x)\in \ind(\K_{1})$ or $f(x)\in \ind(\K_{1})$, or there is a
  $y\in \Delta^{\J_{|\leq n}}$
  such that $x$ is an $R$-successor of $y$ in $\J_{| \leq n+1}$, for some role name $R\in \Sigma$, then $h(x)=f(x)$.
  By the finite outdegree and regularity of $\I_{1}$, the properties (a) and
  (b) of $\Theta_{n}$, and the finite outdegree of all $\J$ such that
  $(\J,h)\in \Xi_{n}$, the number of equivalence classes for~$\sim$ is finite. Hence there
  exists an equivalence class $\Theta$ satisfying (a). Clearly, we can modify the $(\Sigma,n)$-homomorphisms $h,f$ in
  the pairs $(\I,h),(\J,f)\in \Theta$ in such a way that $h(x)=f(x)$ holds for all $x\in \Delta^{\J_{|\leq n+1}}\setminus\Delta^{\J_{|\leq n}}$
  while preserving the remaining properties of $\Theta$. The resulting set of pairs satisfies  (a) and (b).
	
  We define an interpretation $\I_{2}$ as the union of all $\J_{|\leq n}$ such that there exists $(\J,h)\in \Theta_{n}$,
  $n\geq 0$:
\begin{eqnarray*}
\Delta^{\I_{2}} & = &  \bigcup_{n\geq 0} \big\{\Delta^{\J_{|\leq n}} \mid \exists h\;(\J,h)\in \Theta_{n}~\big\};\\
A^{\I_{2}} & = & \bigcup_{n\geq 0} \big\{A^{\J_{|\leq n}} \mid \exists h\;(\J,h)\in \Theta_{n}~\big\}, \text{ for all concept names $A$};\\
R^{\I_{2}} & = & \bigcup_{n\geq 0} \big\{R^{\J_{|\leq n}} \mid \exists h\;(\J,h)\in \Theta_{n}~\big\}, \text{ for all role names $R$}.
\end{eqnarray*}
Using Conditions~(a) and (b) and the fact that the sequence $\Theta_{0},\Theta_{1},\cdots$ is decreasing, it is
straightforward to show that $\I_{2}\in \Mod^{\it bo}_{\K_{2}}$. Define a function $h$ from $\I_{2}$ to $\I_{1}$ by setting
$$
h = \bigcup_{n\geq 0} \big\{~h_{|\leq n} \mid \exists \J\;(\J,h)\in \Theta_{n}~\big\}.
$$
It follows from Condition~(b) that $h$ is well defined. It is readily checked that $h$ is a $\Sigma$-homomorphism
from $\I_{2}$ to $\I_{1}$ preserving $\ind(\K_{2})$.
\end{proof}
\begin{lemma}\label{lem:rUCQ-CQ}
Let $\K_1$ and $\K_2$ be \ALC-KBs and $\Sigma$ a signature containing all role names in $\sig(\K_{2})$.
Then $\K_{1}$ $\Sigma$-UCQ entails $\K_{2}$ iff $\K_{1}$ $\Sigma$-rUCQ entails $\K_{2}$.
\end{lemma}
\begin{proof}
Suppose $\K_{1}$ $\Sigma$-rUCQ entails $\K_{1}$. By Theorem~\ref{crit:KBUCQ1}, it suffices to prove that, for any
$\I_1\in \Mod^{\it reg}_{\K_{1}}$, there exists $\I_2 \in \Mod^{\it bo}_{\K_{2}}$ that is $\Sigma$-homomorphically embeddable
into $\I_1$ preserving $\ind(\K_{2})$. By Theorem~\ref{crit:KBUCQ1}, we know that, for any
$\I_1\in \Mod^{\it reg}_{\K_{1}}$, there exists $\I_2 \in \Mod^{\it bo}_{\K_{2}}$ that is con-$\Sigma$-homomorphically
embeddable into $\I_1$ preserving $\ind(\K_{2})$.
Moreover, as $\Sigma$ contains the role names in $\sig(\K_{2})$, we may assume that every $u\in \Delta^{\I_{2}}$
is $\Sigma$-connected to the ABox $\A_{2}$ of $\K_{2}$. But then $\I_{2}$ is con-$\Sigma$-homomorphically
embeddable into $\I_{1}$ preserving $\ind(\K_{2})$ iff it is $\Sigma$-homomorphically embeddable into $\I_{1}$
preserving $\ind(\K_{2})$, as required.
\end{proof}

\subsection{2\ExpTime upper bound for \textup{(}r\textup{)}UCQ-entailment with respect
  to signature $\Sigma$}
\label{sect:2expupperKB}
We use the model-theoretic criterion of
Theorem~\ref{crit:KBUCQ1} to prove a 2\ExpTime upper bound for
(r)UCQ-entailment between $\mathcal{ALC}$-KBs with respect to a
signature $\Sigma$. We focus on the non-rooted case and then discuss
the modifications required for the rooted one. Let $\K_1$, $\K_2$ be
\ALC-KBs and $\Sigma$ a signature.  We aim to check if there is a
model $\Imc_1 \in \Mod^{\it reg}_{\K_{1}}$ into which no model
$\Imc_2\in \Mod^{\it bo}_{\K_{1}}$ is $\Sigma$-homomorphically
embeddable.  In the following, we construct an automaton \Amf that
accepts (a suitable representation of) the desired models $\Imc_1$. It
then remains to check whether the language $\L(\Amf)$ accepted by \Amf
is non-empty. Note that $\L(\Amf)$ contains also non-regular models,
but a well-known result by Rabin~\cite{Rabin1972} implies that, if
$\L(\Amf)$ is non-empty, then it contains a regular model, which is
sufficient for our purposes.

We use two-way alternating parity automata on infinite trees (\TWAPAs)
and encode forest-shaped interpretations as labeled trees to make them
inputs to \TWAPAs. Let \Nbbm\xspace denote the \emph{positive}
integers. A \emph{tree} is a non-empty (possibly infinite) set
$T \subseteq \Nbbm^*$ closed under prefixes.
%  such that if $x \cdot i \in T$ with $x \in
% \Nbbm^+$ and $c \in \Nbbm$, then $x \cdot i' \in T$ whenever $0 < i' <
% i$.
The node $\varepsilon$ is the \emph{root} of $T$. As a convention, for $x \in \Nbbm^*$,  we
take $x \cdot 0 = x$ and $ (x \cdot i) \cdot -1 = x$. Note that
$\varepsilon \cdot -1$ is undefined.  We say that $T$ is
\emph{$m$-ary} if, for every $x \in T$, the set $\{ i \mid x \cdot i
\in T \}$ is of cardinality exactly $m$.  Without loss of generality, we assume that all
nodes in an $m$-ary tree are from $\{1,\dots,m\}^*$.

We use $[m]$ to denote the set $\{-1,0,\dots,m\}$ and, for any set $X$,
let $\Bmc^+(X)$ denote the set of all positive Boolean formulas over
$X$, i.e., formulas built using conjunction and disjunction over the
elements of $X$ used as propositional variables, and where the special
formulas $\mn{true}$ and $\mn{false}$ are allowed as well. For an
alphabet $\Gamma$, a \emph{$\Gamma$-labeled tree} is a pair $(T,L)$, where $T$ is a tree and $L:T \rightarrow \Gamma$ a node labelling function.
\begin{definition}\em
  A \emph{two-way alternating parity automaton \textup{(}\TWAPA\!\textup{)} on infinite
    $m$-ary trees} is a tuple $\Amf=(Q,\Gamma,\delta,q_0,c)$, where $Q$
  is a finite set of \emph{states}, $\Gamma$ a finite alphabet,
  $\delta\colon Q \times \Gamma \rightarrow \Bmc^+(\mn{tran}(\Amf))$ the
  \emph{transition function} with the
  set of \emph{transitions} $\mn{tran}(\Amf) = [m] \times Q$, $q_0 \in Q$ the \emph{initial state},
  and $c: Q \rightarrow \Nbbm$ is the \emph{acceptance condition}.
\end{definition}
Intuitively, a transition $(i,q)$ with $i>0$ means that a copy of the
automaton in state $q$ is sent to the $i$-th successor of the
current node. Similarly, $(0,q)$
means that the automaton stays at the current node and switches to
state $q$, and $(-1,q)$ indicates moving to the
predecessor of the current node.
\begin{definition}\em
  A \emph{run} of a \TWAPA $\Amf = (Q,\Gamma,\delta,q_0,c)$ on an infinite
  $\Gamma$-labeled tree $(T,L)$ is a $T \times Q$-labeled tree
  $(T_r,r)$ such that the following conditions are satisfied:
  \begin{itemize}\itemsep 0cm
  \item[--] $r(\varepsilon) = ( \varepsilon, q_0)$;

  \item[--] if $y \in T_r$, $r(y)=(x,q)$, and $\delta(q,L(x))=\vp$, then
    there is a (possibly empty) set $Q = \{ (c_1,q_1),\dots,(c_n,q_n)
    \} \subseteq \mn{tran}(\Amf)$ such that $Q$ satisfies $\vp$ and,
    for $1 \leq i \leq n$, $x \cdot c_i$ is a node
    in $T$, and there is a $y \cdot i \in T_r$ such that $r(y \cdot
    i)=(x \cdot c_i,q_i)$.
  \end{itemize}
  We say that $(T_r,r)$ is \emph{accepting} if in all infinite paths $
  y_1y_2\cdots$ of $T_r$,
  $
  \textsf{min}(\{c(q) \mid r(y_i)=(x,q)\text{ for infinitely many } i\})
  $
  is even.
  An infinite $\Gamma$-labeled tree $(T,L)$ is \emph{accepted} by \Amf if there
  is an accepting run of \Amf on $(T,L)$. We use $\L(\Amf)$ to denote the set of
  all infinite $\Gamma$-labeled trees accepted by \Amf.
\end{definition}
We require the following results from automata theory:
\begin{theorem}[\cite{Rabin1972,Vardi98}]\label{thm:autostuff1}~\\[-5mm]
  \begin{enumerate}\itemsep 0cm

  \item Given a \TWAPA \Amf, one can construct in polynomial time a
    \TWAPA \Bmf with $L(\Bmf)=\overline{L(\Amf)}$.

  \item Given a constant number of \TWAPAs $\Amf_1,\dots,\Amf_c$, one
    can construct in polynomial time a \TWAPA \Amf such that $L(\Amf)=
    L(\Amf_1) \cap \cdots \cap L(\Amf_c)$.

  \item Emptiness of \TWAPAs can be decided in time single exponential in the
    number of states.

  \item For any \TWAPA \Amf, $\L(\Amf)\neq\emptyset$ implies that
    $\L(\Amf)$ contains a regular tree.
  \end{enumerate}
\end{theorem}
Now, let $\Gamma$ be the alphabet with symbols from the set
$$
\{\mathit{root}, \mathit{empty}\}  \cup (\ind(\K_1) \times 2^{\mn{CN}(\Tmc_1)}) \cup
(\mathsf{RN}(\T_1) \times 2^{\mn{CN}(\Tmc_1)}),
$$
where $\mathsf{CN}(\T_i)$ (respectively, $\mathsf{RN}(\T_i)$) denotes the set of
concept (respectively, role) names in $\T_i$.
%A \emph{$\Gamma$-labeled tree} is a pair
%$(T,L)$ with $T$ a tree and $L \colon T \rightarrow \Gamma$ a node labeling
%function.
We represent forest-shaped models of $\Tmc_1$ as $m$-ary $\Gamma$-labeled trees,
with $m = \text{max}(|\T_1|, |\ind(\K_1)|)$. The root
node labeled with $\mathit{root}$ is not used in the representation. Each ABox
individual is represented by a successor of the root labeled with a symbol from
$\ind(\K_1) \times 2^{\mn{CN}(\T_1)}$; non-ABox elements are represented by
nodes deeper in the tree labeled with a symbol from
$\mathsf{RN}(\T_1) \times 2^{\mn{CN}(\T_1)}$. The label $\mathit{empty}$ is
used for padding to make sure that every tree node has exactly $m$ successors.

We call a $\Gamma$-labeled tree \emph{proper} if it satisfies the following
conditions:
\begin{itemize}\itemsep 0cm

\item[--] the root is labeled with $\mathit{root}$;

\item[--] for every $a \in \ind(\A_1)$, there is exactly one successor of
  the root that is labeled with a symbol from $\{a\} \times
  2^{\mn{CN}(\T_1)}$; all of the remaining successors of the root are labeled
  with $\mathit{empty}$;

\item[--] all other nodes are labeled with a symbol from $\mathsf{RN}(\T_1) \times
  2^{\mn{CN}(\T_1)}$ or with $\mathit{empty}$;

\item[--] if a node is labeled with $\mathit{empty}$, then so are all of its successors.
\end{itemize}
A proper $\Gamma$-labeled tree $(T,L)$ represents the following interpretation~$\I_{(T,L)}$:
$$
\begin{array}{@{}r@{~}c@{~}l}
  \Delta^{\I_{(T,L)}} &=& \mn{ind}(\Amc_1) \cup \{ x \in T \mid |x|>1
                          \text{ and } L(x) \neq \mathit{empty}\},\\[1mm]
  %\Int[\I_{(T,L)}]{a_i} &=& a_i\\
  A^{\I_{(T,L)}} &=& \{ a \mid \exists x \in T: L(x)=(a,\type) \text{ with
  } A \in \type \}  \cup
  \{ x \in T \mid L(x)=(R,\type) \text{ with } A \in \type\},
  \\[1mm]
  R^{\I_{(T,L)}} &=& \{(a,b) \mid R(a,b) \in \A_1\} \cup{}\\[1mm]
  &&\{(a,ij) \mid ij \in T,\ L(i)=(a,\type_1), \text{ and }
  L(ij)=(R,\type_2) \} \cup{} \\[1mm]
  &&\{(x,xi) \mid xi \in T,\  L(x)=(S,\type_1), \text{ and }
  L(xi)=(R,\type_2) \}.
\end{array}
$$
Note that $\I_{(T,L)}$ is a forest-shaped interpretation of outdegree
at most $|\T_{1}|$ that satisfies all required conditions to qualify
as a \fosh model of $\Tmc_1$ except that it need not satisfy $\Tmc_1$.
In addition, the interpretation $\I_{(T,L)}$ is regular iff the tree
$(T,L)$ is regular (has, up to isomorphisms, only finitely many rooted
subtrees). Conversely, every model $\I\in \Mod^{\it bo}_{\K_{1}}$ can
be represented as a proper $m$-ary $\Gamma$-labeled tree.
Note that the assertions from $\Amc_1$ are not explicitly represented
in $(T,L)$, but readded in the construction of $\I_{(T,L)}$.

\medskip
% Now we construct three \TWAPAs $\Amf_i$, for $i=0,1,2$. $\Amf_0$ ensures that
% the tree is labeled in a meaningful way, e.g.\ that the \emph{root} label only
% occurs at the root node; $\Amf_1$ accepts $\Gamma$-labeled trees that represent
% a model of $\K_1$, and $\Amf_2$ accepts $\Gamma$-labeled trees $(T,L)$ which
% represent an interpretation $\Imc_{(T,L)}$ such that some model of $\K_2$ is
% $\Sigma$-homomorphically embeddable into $\Imc_{(T,L)}$. The most
% interesting automaton is $\Amf_2$, which guesses a model of $\K_2$ along with a
% homomorphism to $\Imc_{(T,L)}$; in fact, both can be read off from a successful
% run of the automaton. The number of states of the $\Amf_i$ is
% exponential in $|\Kmc_1 \cup \Kmc_2|$. It then remains to combine these
% automata into a single \TWAPA \Amf such that $\L(\Amf)=\L(\Amf_0) \cap \L(\Amf_1)
% \cap \overline{\L(\Amf_2)}$, which is possible with only polynomial blowup, and
% to test (in time exponential in the number of states) whether
% $\L(\Amf)=\emptyset$.

The required \TWAPA \Amf is assembled from the following three automata:
\begin{itemize}

\item[--] a \TWAPA $\Amf_0$ that accepts an $m$-ary $\Gamma$-labeled tree iff it is proper;

\item[--] a \TWAPA $\Amf_1$ that accepts a proper $m$-ary $\Gamma$-labeled tree
  $(T,L)$ iff $\Imc_{(T,L)}$ is a model of $\Tmc_1$;

\item[--] a \TWAPA $\Amf_2$ that accepts a proper $m$-ary $\Gamma$-labeled tree
  $(T,L)$ iff there is a model $\Imc_2 \in \Mod^{\it bo}_{\K_{2}}$ that is
  $\Sigma$-homomorphically embeddable into $\Imc_{(T,L)}$ preserving $\ind(\K_{2})$.
\end{itemize}
The following result shows that we would achieve our goal once we have constructed
$\Amf_{0}$, $\Amf_{1}$, and $\Amf_{2}$ and then define \Amf in such a way that
$\L(\Amf) = \L(\Amf_0) \cap \L(\Amf_1) \cap \overline{\L(\Amf_2)}$.
\begin{lemma}\label{lem:equivalent} The following conditions are equivalent:
\begin{enumerate}
\item[\rm (1)] $\L(\Amf_0) \cap \L(\Amf_1) \cap \overline{\L(\Amf_2)} = \emptyset$,
\item[\rm (2)] for each model $\Imc_1 \in \Mod^{\it bo}_{\K_{1}}$, there exists a model
$\Imc_2 \in \Mod^{\it bo}_{\K_{2}}$ that is $\Sigma$-homomorphically embeddable into $\Imc_1$ preserving $\ind(\K_{2})$,
\item[\rm (3)] for each regular model $\Imc_1 \in \Mod^{\it bo}_{\K_{1}}$, there exists a model
$\Imc_2 \in \Mod^{\it bo}_{\K_{2}}$ that is $\Sigma$-homomorphically embeddable into $\Imc_1$ preserving $\ind(\K_{2})$,
\item[\rm (4)] $\Kmc_{1}$ $\Sigma$-UCQ-entails $\Kmc_{2}$.
\end{enumerate}
\end{lemma}
\begin{proof}
(1) $\Leftrightarrow$ (2) follows from the properties of $\Amf_{0}$, $\Amf_{1}$, $\Amf_{2}$;
(1) $\Leftrightarrow$ (3) follows from the properties of $\Amf_{0}$, $\Amf_{1}$, $\Amf_{2}$,
and Rabin's Theorem~\cite{Rabin1972}; and
(3) $\Leftrightarrow$ (4) is Theorem~\ref{crit:KBUCQ1}.
\end{proof}
%We thus define \Amf to be the intersection of $\Amf_0$, $\Amf_1$, and the complement of $\Amf_2$.
%
%\smallskip
%
The construction of $\Amf_0$ is trivial and left to the reader. The construction of $\Amf_1$ is quite standard \cite{CalvaneseEO07}. Let $C_{\T_1}$
be the negation normal form (NNF) of the concept $$\midsqcap_{C \sqsubseteq D
  \in \Tmc_1} (\neg C \sqcup D)$$ and let $\mathsf{cl}(C_{\T_1})$ denote the set of
subconcepts of $C_{\T_1}$, closed under single negation.
Now, the
\TWAPA $\Amf_{1} = ( Q, \Gamma, \delta, q_0 , c )$ is defined by
setting
$$
Q = \{q_0, q_1, q_\emptyset\} \cup \{q^{a,C}, q^{C}, q^{R}, q^{\neg R}
\mid a \in \mn{ind}(\Amc_1), \ C \in \mathsf{cl}(C_{\T_1}), \ R \in
\mathsf{RN}(\T_1)\}
$$
and defining the transition function $\delta$ as follows:
%, assuming
%$\ind(\K_1) = \{a_1, \dots, a_n\}$:
%
$$
\begin{array}{@{}r@{~~}r@{}}
\begin{array}{@{}r@{~}c@{~}l@{}}
  \delta(q_0,\mathit{root}) &=& %
  \displaystyle \bigwedge_{i=1}^m (i, q_1),
  \\
  \delta(q_1,\ell) &=&
  \displaystyle
  ((0, q_\emptyset) \lor (0,q^{C_{\T_1}})) \land \bigwedge_{i=1}^m (i, q_1),
  \\
  \delta(q^{\exists R.C}, (a,U)) &=&
  \displaystyle
  \bigvee_{i=1}^m ((i, q^{R}) \wedge (i, q^{C})) \lor
  \bigvee_{R(a,b) \in \A_1}(-1,q^{b,C}),
  \\[5mm]
  \delta(q^{\forall R.C}, (a,U)) &=&
  \displaystyle
  \bigwedge_{i=1}^m ((i,q_{\emptyset}) \lor (i, q^{\neg R}) \lor (i, q^{C})) \land
  \bigwedge_{R(a,b) \in \A_1}(-1,q^{b,C}),
  \\[5mm]
  \delta(q^{\exists R.C}, (S,U)) &=&
  \displaystyle
  \bigvee_{i=1}^m ((i, q^{R}) \wedge (i, q^{C})),
  \\[4mm]
  \delta(q^{\forall R.C}, (S,U)) &=&
  \displaystyle
  \bigwedge_{i=1}^m ((i,q_{\emptyset}) \lor (i, q^{\neg R}) \lor (i, q^{C})),
\end{array}
&
\begin{array}{@{}r@{~}c@{~}l@{}}
  \delta(q^{C \sqcap C'}, (x,U)) &=&
  (0, q^{C}) \land (0, q^{C'}), \\
  \delta(q^{C \sqcup C'}, (x,U)) &=&
  (0, q^{C}) \lor (0, q^{C'}), \\[1mm]
  \delta(q^{a,C},\mathit{root}) &=&
  \displaystyle
  \bigvee_{i=1}^m  (i, q^{a,C}),
  \\
  \delta(q^{a,C}, (a,U)) &=& (0,q^C),
  \\[1mm]
  \delta(q^{A}, (x,U)) &=& \mathsf{true}, \text{ if }A\in U,
  \\
  \delta(q^{\neg A}, (x,U)) &=& \mathsf{true}, \text{ if }A\notin U,
  \\
  \delta(q^{R}, (R,U)) &=& \mathsf{true},
  \\
  \delta(q^{\neg R}, (S,U)) &=& \mathsf{true},  \text{ if }R\neq S,
  \\
  \delta(q_\emptyset,\mathit{empty}) &=& \mn{true},
  \\[7mm]
  \multicolumn{3}{l}{\delta(q,\ell) = \mathsf{false}, %
  \quad \text{ for all other }q \in Q,~\ell\in \Gamma.}
\end{array}
\end{array}
$$
Here $x$ in the labels $(x,U)$ stands for an individual $a$ or for a role name
$S$, and $\ell$ in the second transition is any label from $\Gamma$. The acceptance condition $c$
is trivial ($c(q)=0$ for all $q\in Q$). It is standard to show that $\Amf_1$ accepts the desired tree language.

To construct $\Amf_2$, we use the notation introduced in the proof of Proposition~\ref{regularcomplete}.
Note that the set $\types(\T_2)$ of $\T_{2}$-types can be computed in time single exponential in $|\Kmc_2|$.
%   Let $\type,\type' \in \types(\T_2)$.
% % We write $\type \rightsquigarrow_R \type'$ if $\midsqcap \type \sqcap
% % \SOME{R}{(\midsqcap \type')}$ is satisfiable w.r.t.\ $\T_2$.
% For $\exists R.C \in \type$, we say that $\type'$ is an
% \emph{$\exists R.C$-witness for $\type$} if $C \in \type'$ and the concept
% $\midsqcap \type \sqcap \exists R.(\midsqcap \type')$ is satisfiable w.r.t.\
% $\T_2$. Denote by $\succ_{\exists R.C}(\type)$ the set of all
% $\exists R.C$-witnesses for $\type$.
A \emph{completion of} $\K_2$ is a function $\tau \colon \ind(\A_2) \to \types(\T_2)$ such that, for any
$a \in \ind(\A_2)$, the KB
$$
\big(\T_2 \cup \bigcup_{a \in \ind(\A_2), C \in \tau(a)} \hspace*{-3mm} \{A_a \sqsubseteq C\}, \ ~\A \cup \bigcup_{a \in \ind(\A_2)} \{A_a(a)\}\big)
$$
is consistent, where $A_a$ is a fresh concept name for each $a \in\ind(\A_2)$.
Denote by $\compl(\K_2)$ the set of all completions of~$\K_2$; it can be
computed in time single exponential in $|\K_2|$.

We now construct the \TWAPA $\Amf_2$. It is easy to see that if there
is an assertion $R(a,b) \in \Amc_2 \setminus \Amc_1$ with $R \in
\Sigma$, then no model of $\Kmc_2$ is $\Sigma$-homomorphically
embeddable into a \fosh model of $\Kmc_1$ preserving $\ind(\K_{2})$. In this case, we choose
$\Amf_2$ so that it accepts the empty language.
Suppose there is no such assertion. It is easy to see
that any model $\Imc_2$ of $\Kmc_2$ such that some
$a \in \ind(\K_2) \setminus \ind(\K_1)$ occurs in $S^{\Imc_2}$, for
some symbol $S\in\Sigma$, is not
$\Sigma$-homomorphically embeddable into a \fosh model of $\Kmc_1$
preserving $\ind(\K_{2})$.  For this reason, we should only consider
completions of $\Kmc_2$ such that, for all
$a \in \ind(\K_2) \setminus \ind(\K_1)$, $\tau(a)$ contains no
$\Sigma$-concept names and no existential restrictions $\exists R . C$
with $R \in \Sigma$. We use
$\compl_{\mn{ok}}(\Kmc_2)$ to denote the set of all such completions.
%
% For a completion $\tau$, we use $\mathsf{part}^{\A_2,\tau}_\Sigma$ to denote all
% individuals $a$ in $\A_2$ such that $\tau(a)$ contains a $\Sigma$-concept name,
% or an existential restriction $\exists R.C$ with $R \in \Sigma$, or $a$ is
% connected via some $\Sigma$-role to some individual in $\A_2$.
%
We define the \TWAPA
$\Amf_{2}= (Q, \Gamma, \delta, q_0, c)$ by setting
$$
\begin{array}{r@{~}c@{~}l}
Q &=& \{q_0\} \cup
\{q^{a,\type}, q^{R,\type}, f^{\type} \mid a \in \ind(\Amc_1),\
\type \in \types(\T_2),\ R \in \mathsf{RN}(\T_2) \cap \Sigma\}
\end{array}
$$
%
%$F = \{q^{(R,\type)} \in Q\} \cup \{q^{\type}_i \in Q\}$
%
and defining the transition function $\delta$ as follows:
$$
\begin{array}{r@{\,}c@{\,}l@{\,}l}
  % \delta(q_0,\mathit{root}) &=&
  % \displaystyle
  % \bigvee_{\tau \in \compl(\K_2)} %
  % \bigwedge_{a\in\mathsf{part}^{\A_2,\tau}_\Sigma} %
  % \bigvee_{i=1}^{m} (i, q^{a,\tau(a)}) \\[6mm]
  %
  \delta(q_0,\mathit{root}) &=&
  \displaystyle
  \bigvee_{\tau \in \compl_{\mn{ok}}(\K_2)} \ %
  \bigwedge_{a \in \ind(\A_2) \cap \ind(\A_1)} \ %
  \bigvee_{i=1}^{m} (i, q^{a,\tau(a)}), \\[6mm]
  \delta(q^{a,\type}, (a,U)) &~=~&
  \displaystyle
  \bigwedge_{\substack{\exists R.C \in \type\\R \in \Sigma}} \ %
  \bigvee_{\boldsymbol{s}\in \succ_{\exists R.C}(\type)}
  \Big(\bigvee_{i=1}^m
    (i, q^{R,\boldsymbol{s}}) \lor%
    \bigvee_{R(a,b) \in \Amc_1} (-1, q^{b,\boldsymbol{s}})
  \Big) ~
  \wedge\bigwedge_{\substack{\exists R.C \in \type\\R \notin \Sigma}} \ %
  \bigvee_{\boldsymbol{s}\in \succ_{\exists R.C}(\type)} (0, f^{\boldsymbol{s}}),
  \\[8mm]
  \delta(q^{S,\type}, (S,U)) &=& %\\[1mm]
%   &&
  \displaystyle
\bigwedge_{\substack{\exists R.C \in \type\\R \in \Sigma}} \ %
  \bigvee_{\boldsymbol{s}\in \succ_{\exists R.C}(\type)} \ %
  \bigvee_{i=1}^m (i, q^{R,\boldsymbol{s}}) ~%
  \wedge\bigwedge_{\substack{\exists R.C \in \type\\R \notin \Sigma}} \ %
  \bigvee_{\boldsymbol{s}\in \succ_{\exists R.C}(\type)}%
  (0, f^{\boldsymbol{s}}), %
\end{array}
$$
where the last two transitions are subject to the conditions that
every $\Sigma$-concept name in $\type$ is also in $U$,
\begin{align*}
  \delta(f^{\type}, (v,U)) &~=~ (0,q^{v,\type}) \lor \bigvee_{i=1}^m (i,f^{\type}) \lor (-1,f^{\type}),\\%
  \delta(f^{\type}, \mathit{root}) &~=~ \bigvee_{i=1}^m (i,f^{\type}), \\%
  \label{hom-anon}
  \delta(q^{a,\type}, \mathit{root}) &~=~ \bigvee_{i=1}^m (i,q^{a,\type}), \\
  \delta(q,\ell) &~=~ \mathsf{false},
  \quad \text{ for all other }q \in Q \text{ and } \ell\in \Gamma,
\end{align*}
where $v$ is an individual $a$ or  a role name $S$. Note that
the states $f^{\type}$ are used to find non-deterministically the
homomorphic image of $\Sigma$-disconnected successors in the tree. Finally, we
set %$F=Q \setminus\{ f^{\type} \mid \type \in \types(\T_2) \}$.
$c(q) = 0$ for $q \in \{q_0,q^{a,\type},q^{R,\type}\}$ and $c(f^{\type})=1$.
\begin{lemma}
  $(T,L) \in \L(\Amf_2)$ iff there is a model $\Imc_2\in  \Mod^{\it bo}_{\K_{2}}$ such that
  $\Imc_2$ is $\Sigma$-homomorphically embeddable into $\I_{(T,L)}$ preserving $\ind(\K_{2})$.
\end{lemma}
\begin{proof}
  ($\Rightarrow$) Given an accepting run $(T_r,r)$ for $(T,L)$, we can
  construct a model $\Imc_2\in  \Mod^{\it bo}_{\K_{2}}$ and a
  $\Sigma$-homomorphism $h$ from $\Imc_2$ to $\I_{(T,L)}$.
  %
  %\nb{E: rephrased and extended here}%
  Intuitively, the type $\type$ of $a$ in $\Imc_2$ is given by the
  child $y_a$ of $\varepsilon$ in $T_r$ with
  $r(y_a) = (x_a,q^{a,\type})$, and the tree-shaped part of $\Imc_2$
  is defined inductively as follows. If an element $d$ of $\Imc_2$ has
  type $\type$ and $y_d \in T_r$, then for each
  $\exists R.C \in \type$ such that $R \in \Sigma$, $d$ has an
  $R$-successor $d'$ whose type
  $\boldsymbol{s} \in\succ_{\exists R.C}(\type)$ is determined by a
  child $y_{d'}$ of $y_d$ in $T_r$ with
  $r(y_{d'}) = (x_{d'}, q^{v,\boldsymbol{s}})$, for some $v$. Moreover,
  for each $\exists R.C \in \type$ such that $R \notin \Sigma$, $d$
  has an $R$-successor $d'$ whose type
  $\boldsymbol{s} \in\succ_{\exists R.C}(\type)$ is determined by the
  descendants $y_1,\dots,y_{n},y_{d'}$ of $y_d$ in $T_r$, $n\geq 1$,
  with $r(y_i) = (x_i, f^{\boldsymbol{s}})$, $1 \leq i \leq n$, and
  $r(y_{d'}) = (x_{d'},q^{v,\boldsymbol{s}})$ for some $v$.
  % \nb{E: this sentence is not precise, we could have
  %   $r(y) = (x,q^{a,\type})$ because of the disjunct
  %   $(-1,q^{b,\boldsymbol{s}})$, or $r(y) = (x,q^{R,\type})$ after the
  %   state $f^{\boldsymbol{S}}$}%
  % Intuitively, each node $y \in T_r$ with $r(y) = (x,q^{a,\type})$
  % imposes that $a$ has type $\type$ in $\Imc_2$, and each node $y \in
  % T_r$ with $r(y) = (x,q^{R,\type})$ imposes that $\Imc_2$ contains an
  % element $y$ that belongs to a tree-shaped part of $\Imc_2$, is
  % connected to its predecessor via $R$, and has type $\type$.
  The homomorphism $h$ is defined by taking the identity on
  individual names, and setting $h(d)=a$ if
  $r(y_d) = (x_d,q^{a,\type})$, and $h(d)=x_d$ if
  $r(y_d) = (x_d,q^{R,\type})$.
  Observe that due to the accepting condition for which
  $c(f^{\boldsymbol{t}}) = 1$, the automaton cannot remain forever in
  the states $f^{\boldsymbol{t}}$, and so has to eventually find the
  homomorphic image of $\Sigma$-disconnected successors in the tree.

  ($\Leftarrow$) Suppose there is a model $\Imc_2 \in \Mod^{\it bo}_{\K_{2}}$ such
  that $\Imc_2$ is $\Sigma$-homomorphically embeddable into
  $\I_{(T,L)}$ preserving $\ind(\K_{2})$. It is straightforward to construct an
  accepting run for $(T,L)$ by using $\Imc_2$ as a guide.
\end{proof}
The constructed automaton \Amf has only single exponentially many states.
Thus, by Theorem~\ref{thm:autostuff1}, checking its emptiness can be done in 2\ExpTime.
\begin{theorem}\label{thm:2explower}
The problem whether an \ALC KB $\Sigma$-UCQ entails an \ALC KB is decidable in 2\ExpTime.
\end{theorem}
We now briefly discuss the modifications needed in the automata construction to
obtain the same upper bound for $\Sigma$-rUCQ entailment. In the rooted case,
we modify the automaton $\Amf_2$ in such way that it does not attempt to
construct a $\Sigma$-homomorphism when reaching $\Sigma$-disconnected
successors. Thus, the set $Q$ of states of $\Amf_2$ does not contain
$f^{\type}$, and the transition function is simplified accordingly. In
particular, in the definition of the transitions $\delta(q^{x,\type}, (x,U))$,
for $x\in \{a,S\}$, the second set of conjunctions for $\exists R.C \in \type$
and $R \notin \Sigma$ is omitted.
\begin{theorem}\label{thm:2explowerrooted}
The problem whether an \ALC KB $\Sigma$-rUCQ entails an \ALC KB is decidable in 2\ExpTime.
\end{theorem}
Our characterisation of $\Sigma$-(r)UCQ entailment using automata also allows
us to formulate Theorem~\ref{crit:KBUCQ1} without the restriction to regular
interpretations. For UCQs, this is a consequence of Lemma~\ref{lem:equivalent}
and, for rUCQs, one can prove an analogous lemma.
%(of which a straightforwardalso entails the following version
%of Theorem~\ref{crit:KBUCQ1} that does not refer to regular interpretations.
%
\begin{theorem}\label{UCQ-full-hom}
Let $\K_{1}$ and $\K_{2}$ be $\mathcal{ALC}$ KBs and $\Sigma$ a signature.
\begin{description}\itemsep=0pt
\item[\rm (1)] $\mathcal{K}_{1}$ $\Sigma$-UCQ entails $\K_{2}$ iff, for any $\Imc_{1}\in \Mod^{\it bo}_{\K_{1}}$,
there exists $\I_{2}\in \Mod^{\it bo}_{\K_{2}}$ that is $\Sigma$-homomorphically embeddable into $\Imc_{1}$ preserving $\ind(\K_{2})$.
\item[\rm (2)] $\mathcal{K}_{1}$ $\Sigma$-rUCQ entails $\K_{2}$ iff, for any $\Imc_{1}\in \Mod^{\it bo}_{\K_{1}}$,
there exists $\I_{2}\in \Mod^{\it bo}_{\K_{2}}$ that is con-$\Sigma$-homomorphically embeddable into $\Imc_{1}$ preserving $\ind(\K_{2})$.
\end{description}
\end{theorem}

%
%
%We now prove the matching lower bounds.
%
\subsection{2\ExpTime lower bound for \textup{(}r\textup{)}UCQ-entailment and
  inseparability with respect to a signature}
\label{sect:first2explower}
We first show a 2\ExpTime lower bound for $\Sigma$-UCQ entailment between \ALC KBs by giving a polynomial
reduction of the word problem for exponentially space bounded alternating Turing machines. Using
Lemma~\ref{lem:rUCQ-CQ}, we obtain the same lower bound for rUCQs. We then modify the KBs from the entailment case
to obtain 2\ExpTime lower bounds for $\Sigma$-(r)UCQ inseparability.

An \emph{alternating Turing machine} (ATM) is a quintuple of the form $M =
(Q,\Gamma_I,\Gamma,q_0,\Delta)$, where the set of \emph{states} $Q =
Q_\exists \uplus Q_\forall \uplus \{q_a\} \uplus \{q_r\}$ consists of
\emph{existential states} in $Q_\exists$, \emph{universal states} in
$Q_\forall$, an \emph{accepting state} $q_a$, and a \emph{rejecting
  state} $q_r$; $\Gamma_I$ is the \emph{input alphabet} and $\Gamma \supseteq \Gamma_I$
the \emph{work alphabet} containing a \emph{blank symbol} $\square$; $q_0 \in Q_\exists \cup
Q_\forall$ is the \emph{starting} state; and the \emph{transition
  relation} $\Delta$ is of the form
$$
  \Delta \; \subseteq \; (Q \setminus \{ q_a,q_r \}) \times \Gamma \times Q \times \Gamma
  \times \{ -1,+1 \}.
$$
We write $\Delta(q,\sigma)$ to denote $\{ (q',\sigma',m) \mid
(q,\sigma,q',\sigma',m) \in \Delta \}$ and assume without loss of generality that every
set $\Delta(q,\sigma)$ contains exactly two elements.
A \emph{configuration} of $M$ is a word $wqw'$ with $w,w' \in
\Gamma^*$ and $q \in Q$. The intended meaning is that the tape
contains the word $ww'$, the machine is in state $q$, and the head is
on the symbol just after $w$. The \emph{successor configurations} of a
configuration $wqw'$ are defined in the usual way in terms of the
transition relation~$\Delta$. A \emph{halting configuration} is of the
form $wqw'$ with $q \in \{ q_a, q_r \}$. A configuration $wqw'$ is
\emph{accepting} if it is a halting configuration and $q=q_a$ or $q
\in Q_\forall$ and all of its successor configurations are accepting or
$q \in Q_\exists$ and there is an accepting successor configuration.
$M$ \emph{accepts} input $w$ if the \emph{initial configuration}
$q_0w$ is accepting. There is an exponentially space bounded ATM $M$ whose word problem is
{\sc 2Exp\-Time}-hard.

% \bigskip
%
% {\color{blue} In your proof, you would then say something like ``For
%   each $q \in Q$ and $\sigma \in \Gamma$, we use $\delta_1(q,\sigma)$
%   to denote the first transition in $\Delta(q,\sigma)$ and
%   $\delta_1(q,\sigma)$ to denote the second one. And of course you
%   could add any other notation / assumptions as you see fit.}
%
% \bigskip
% \bigskip
% \bigskip

%
%
\begin{theorem}
\label{thm:ucq-lower-bound}
The problem whether an \ALC KB $\Kmc_1$ $\Sigma$-\textup{(}r\textup{)}UCQ entails an \ALC KB
$\Kmc_2$ is 2\ExpTime-hard.
\end{theorem}
\begin{proof}
We only consider the non-rooted case; the rooted case follows using
Lemma~\ref{lem:rUCQ-CQ} since the signature $\Sigma$ in our proof contains
all the role names used in the entailed KB $\K_{2}$. The proof is by reduction of the word problem for exponentially space bounded
ATMs. Let $M = (Q,\Gamma_I,\Gamma,q_0,\Delta)$ be such an ATM. We may assume
without loss of generality that
\begin{itemize}
\item[--] the length of every (path in a) computation of $M$ on $w \in {\Gamma_I}^n$ is bounded by $2^{2^{n}}$;
\item[--] all the configurations $wqw'$ in such computations satisfy $|ww'| \leq 2^{n}$, see \cite{ChandraKS81};
\item[--] $M$ never attempts to move left of the tape cell on which the head was located in the initial configuration;
\item[--] the two transitions contained in $\Delta(q,\sigma)$ are ordered and use $\delta_0(q,\sigma)$
and $\delta_1(q,\sigma)$ to denote the first and second transition in $\Delta(q,\sigma)$, respectively;
\item[--] the existential and universal states strictly alternate: any
transition from an existential state leads to a universal state, and vice
versa;
\item[--] $q_0 \in Q_\exists$;
\item[--] any run of $M$ on every input stops either in $q_{a}$ or $q_{r}$.
\end{itemize}
% Let $2^n$ be the space used by $M$.  In what follows, by a \emph{configuration}
% of $M$ we mean a sequence of elements from $\{\sigma, (q,\sigma) \mid \sigma
% \in \Gamma, q\in Q\}$ of length $2^n$. A \emph{valid configuration} is a
% configuration with at most\nb{exactly?} one element of the form $(q,\sigma)$.

Let $w \in {\Gamma_I}^n$ be an input to $M$.  We construct \ALC TBoxes
$\Tmc_1$ and $\Tmc_2$ and a signature $\Sigma$ such that $M$ accepts $w$ iff
there is a model $\I_1$ of $\K_1 = (\Tmc_1,\{A(a)\})$ such that no model of
$\K_2 = (\Tmc_2,\{A(a)\})$ is $\Sigma$-homomorphically embeddable into $\Imc_1$.
In our construction, the models of $\K_1$ encode all possible sequences of
configurations of $M$ starting from the initial one and forming a full binary
tree.  Hence, most of the models do not correspond to correct runs of $M$. The
branches of the models stop at the accepting and rejecting states.
On the other hand, the models of $\K_2$ encode all possible local defects (such
as invalid configurations or incorrect executions of the transition function),
after the first step of the machine, or after the second step, and so on, or
detect valid (hence without local defects) but rejecting runs.
Then, if there exists a finite model $\I_1$ of $\K_1$ such that no model of
$\K_2$ is $\Sigma$-homomorphically embeddable into $\I_1$ preserving $\{a\}$,
we have that $\I_1$ represents a valid accepting computation of $M$.

The signature $\Sigma$ contains the following symbols:
\begin{itemize}\itemsep 0cm

\item[--] the concept name $A$;

\item[--] the concept names $A_0,\dots,A_{n-1},\overline{A}_0,\dots,\overline{A}_{n-1}$ that serve as bits in the binary
  representation of a number between 0 and $2^{n}-1$, identifying the position
  of tape cells inside configurations ($A_0$, $\overline{A}_0$ represent the
  lowest bit);

\item[--] the concept names $A_\sigma$, for $\sigma \in \Gamma$;

\item[--] the concept names $A_{q,\sigma}$, for $\sigma \in \Gamma$ and $q \in
  Q$;

\item[--] the concept names $X_0,X_1$ to distinguish the two successor configurations;

\item[--] the role names $R$, $S$; $R$ is used to connect the successor configurations, whereas $S$ is used to connect the root of each configuration with symbols that occur in the cells of it.
\end{itemize}
Also, we make use of the following auxiliary symbols that are not in $\Sigma$:
\begin{itemize}\itemsep 0cm
\item[--] $B_i$, $\overline{B}_i$, $B_\sigma$, $B_{q,\sigma}$; $G_i$, $\overline{G}_i$, $G_\sigma$, $G_{q,\sigma}$; $C_\sigma$, $C_{q,\sigma}$, for $\sigma \in \Gamma$, $q \in Q$, and $0 \leq i \leq n-1$,

\item[--] $L_i^\ell$, $D_{\textit{trans}}^\ell$, for $\ell\in\{0,1\}$  and $0 \leq i \leq n-1$,

\item[--] $K_0$, $K$, $\textit{Stop}$, $Y$, $D$, $\overline{D}$, $D_{\textit{conf}}$, $D_{\textit{trans}}$,  $D_{\textit{rej}}$, $D_{\textit{rej}}^\exists$, $D_{\textit{rej}}^\forall$, $\textit{Counter}_m$ for $m\in\{-1, 0, +1\}$, $E_B$, $E_G$.
\end{itemize}

\smallskip\noindent\textbf{TBox $\T_1$.} 
Each model of $\K_1$ encodes a binary tree of configurations of $M$. Thus,
$\T_1$ contains the axioms:
\[
\begin{array}{r@{~}l}
  A \sqsubseteq & \exists R. (X_0 \sqcap K) \sqcap \exists R. (X_1 \sqcap K),\\
  (X_0 \sqcup X_1) \sqcap \neg \textit{Stop} \sqsubseteq& \exists R. (X_0 \sqcap K) \sqcap \exists R. (X_1 \sqcap K),\\
  K \sqsubseteq& \exists S. (L_0^0 \sqcap \overline{A}_0) \sqcap \exists S. (L_0^1 \sqcap A_0),\\
  L_i^\ell \sqsubseteq& \exists S. (L_{i+1}^0 \sqcap \overline{A}_{i+1}) \sqcap \exists S. (L_{i+1}^1 \sqcap A_{i+1}),\quad \text{ for }0\leq i\leq n-2,\ \ell\in\{0,1\},\\
  L_{n-1}^\ell \sqsubseteq& \bigsqcup_{\sigma \in \Gamma} (A_\sigma \sqcup \bigsqcup_{q \in Q} A_{q,\sigma}),\\
  A_{\sigma_1} \sqcap A_{\sigma_2} \sqsubseteq& \bot, \quad \text{ for }\sigma_1\neq\sigma_2,\\
  A_{\sigma_1} \sqcap A_{q_2,\sigma_2} \sqsubseteq& \bot,\\
  A_{q_1,\sigma_1} \sqcap A_{q_2,\sigma_2} \sqsubseteq& \bot, \quad \text{ for }(q_1,\sigma_1)\neq(q_2,\sigma_2),\\
  A_i \sqsubseteq& \forall S. A_i, \qquad
  \overline{A}_i \sqsubseteq \forall S. \overline{A}_i,\\
  \exists S^n. A_{q_a,\sigma} \sqsubseteq& \textit{Stop}, \qquad
  \exists S^n. A_{q_r,\sigma} \sqsubseteq \textit{Stop},
\end{array}
\]
where $\exists S^n.A$ is an abbreviation for the concept $\exists S. \exists S. \dots \exists S.A$ ($S$ occurs $n$ times). The models of $\K_1$ look as in Fig.~\ref{fig:models-of-k1},
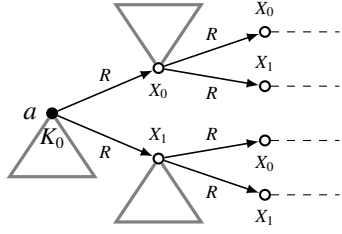
\begin{figure}
\begin{center}
  \begin{tikzpicture}[xscale=1.4, yscale=1.2]
    \node[point, constant, label=left:{$a$}, label=below:{\small$K_0$}] (a) at
    (0,0) {};

    \foreach \al/\x/\y/\conc/\wh in {%
      x0/1/0.5/X_0/below,%
      x1/1/-0.5/X_1/above,%
      x00/2/0.9/X_0/above,%
      x01/2/0.3/X_1/above,%
      x10/2/-0.3/X_0/below,%
      x11/2/-0.9/X_1/below%
    }{ \node[point, label=\wh:{\scriptsize$\conc$}] (\al) at (\x,\y) {}; }

    \foreach \al/\down in {%
      a/1, x0/-1, x1/1%
    }{ \draw[gray, very thick] (\al) -- ++(0.4,\down*-0.7) -- ++(-0.8,0) -- (\al); }

    \foreach \from/\to/\wh in {%
      a/x0/above, a/x1/below, x0/x00/above, x0/x01/below, x1/x10/above,
      x1/x11/below%
    }{ \draw[role] (\from) -- node[\wh] {\scriptsize $R$} (\to); }

    \foreach \from in {%
      x00, x01, x10, x11%
    }{ \draw[dashed] (\from) -- ++(0.7,0); }
  \end{tikzpicture}
\end{center}
\caption{The structure of the models of $\K_1$.}
\label{fig:models-of-k1}
\end{figure}
where the grey triangles are the trees encoding configurations rooted at $K$
except for the initial configuration. These trees are binary trees of depth
$n$, where each leaf represents a tape cell. For $w = \sigma_1\cdots \sigma_n$,
the initial configuration is encoded at $a$ by the following $\T_1$-axioms:
\[
\begin{array}{r@{~}l}
  A \sqsubseteq &  \exists S. (L_0^0 \sqcap \overline{A}_0 \sqcap K_0) \sqcap \exists S. (L_0^1 \sqcap A_0 \sqcap K_0),\\
  K_0 \sqsubseteq& \forall S. K_0,\\
  K_0 \sqcap (\textsf{val}_A=0) \sqsubseteq& A_{q_0,\sigma_1},\\
  K_0 \sqcap (\textsf{val}_A=i) \sqsubseteq& A_{\sigma_{i+1}}, \text{ for }1 \leq i \leq n-1,\\
  K_0 \sqcap (\textsf{val}_A\geq n) \sqsubseteq& A_{\square},\\
\end{array}
\]
where $(\textsf{val}_A=j)$ is the conjunction over $A_i,\overline{A}_i$
expressing the fact that the value of the $A$-counter is $j$, for $j \leq 2^n-1$.

\medskip\noindent\textbf{TBox $\T_2$.}
Each model of $\K_2$ encodes (at least) one of four possible defects:
\begin{itemize}\itemsep 0cm
\item[--] invalid configuration defect $D_{\textit{conf}}$;
\item[--] transition defect $D_{\textit{trans}}$ encoding errors in
  executing the transition function;
\item[--] copying defect $D_{\textit{copy}}$ encoding errors in
  copying a symbol not under the head;
\item[--] a rejecting run defect $D_{\textit{rej}}$.
\end{itemize}
The first three defects are used to filter out sequences of configurations that
do not correspond to valid runs of $M$. These defects are `local'\!, and so
they are connected to $a$ via paths. Instead, the last defect is used to detect
valid rejecting runs of $M$, so it is `global' and is represented by a tree.
Thus, $\T_2$ contains the following axioms:
\[
\begin{array}{r@{~}l}
  A \sqsubseteq &  \exists R. (X_0 \sqcap Y) \sqcup \exists R. (X_1 \sqcap Y) \sqcup D_\textit{rej}^\exists,\\
  Y \sqcap \overline{D} \sqsubseteq &  \exists R. (X_0 \sqcap Y) \sqcup \exists R. (X_1 \sqcap Y),\\
\end{array}\qquad
\begin{array}{r@{~}l}
  Y \sqsubseteq& D \sqcup \overline{D}, \qquad D \sqcap \overline{D} \sqsubseteq \bot,\\
  D\sqsubseteq& D_{\textit{conf}}  \sqcup D_{\textit{trans}} \sqcup D_{\textit{copy}}.\\
\end{array}
\]
We now describe each of the defects separately, using the following
abbreviations:
\[
  \begin{array}{l@{\qquad}l@{\qquad}l}
    \textsf{pos}^B=(\overline{B}_0 \sqcup B_0) \sqcap \cdots \sqcap (\overline{B}_{n-1} \sqcup B_{n-1}),
    & \textsf{symbol}^B=\bigsqcup_{\sigma \in \Gamma} B_{\sigma},
    &\textsf{state}^B= \bigsqcup_{q\in Q,~ \sigma \in \Gamma} B_{q,\sigma}.
\end{array}\]
The abbreviations $\textsf{pos}^G$, $\textsf{symbol}^{G}$ and
$\textsf{state}^G$ are defined analogously using concept names $G_i$,
$\overline{G}_i$, $G_{q,\sigma}$, and $G_\sigma$.

\smallskip\noindent\textbf{Invalid configuration defect.} %
$D_{\textit{conf}}$ is the simplest `local' defect that encodes incorrect
configurations, that is, configurations with at least two heads on the tape. %
It guesses the first position of the head, the symbol under it and the state by
means of the concepts $\textsf{pos}^B$ and $\textsf{state}^B$, and similarly,
it guesses the second position using the corresponding concepts with the
superscript $G$. This information is stored in the symbols transparent to
$\Sigma$ ($B_x$, $\overline{B}_x$ and $G_x$, $\overline{G}_x$).
\[
  D_{\textit{conf}} \sqsubseteq \textsf{pos}^B \sqcap
                                 \textsf{state}^B
                                 \sqcap \exists S^n.E_B \sqcap{} 
                               \textsf{pos}^G \sqcap
          \textsf{state}^G
          \sqcap \exists S^n.E_G \sqcap (\textsf{val}_B \neq \textsf{val}_G),
\]
where $(\textsf{val}_B \neq \textsf{val}_G)$ stands for
$(B_0 \sqcap \overline{G}_0) \sqcup (G_0 \sqcap \overline{B}_0) \sqcup \cdots
\sqcup (B_{n-1} \sqcap \overline{G}_{n-1}) \sqcup (G_{n-1} \sqcap
\overline{B}_{n-1})$ and ensures that the position encoded using $B$-symbols
is different from the position encoded using $G$-symbols.

All the symbols $B_x$ and $\overline{B}_x$, and $G_x$ and $\overline{G}_x$ are
propagated down the $S$-successors, and at the concepts $E_B$ and $E_G$ they
are copied into the $\Sigma$-symbols $A_x$ and $\overline{A}_x$:
\begin{align}
&  B_x \sqsubseteq \forall S. B_x, \ 
  G_x \sqsubseteq \forall S. G_x, \ 
  E_B \sqcap B_x \sqsubseteq A_x, \ 
  E_G \sqcap G_x \sqsubseteq A_x, \ 
  \text{ for }x \in \{0, \dots, n-1\} \cup \{ (q,\sigma), \sigma \mid q \in Q, \sigma \in \Gamma \}, \nonumber \\
&  \overline{B}_i \sqsubseteq \forall S. \overline{B}_i, \ 
  \overline{G}_i \sqsubseteq \forall S. \overline{G}_i, \
  E_B \sqcap \overline{B}_i \sqsubseteq \overline{A}_i, \ 
  E_G \sqcap \overline{G}_i \sqsubseteq \overline{A}_i, \
  \text{ for }i \in \{0, \dots, n-1\}.
  \label{eq:prop-down-S-copy-at-E}
\end{align}
A (partial) model of an invalid configuration defect is shown in
Fig.~\ref{fig:defects}(a), for $n=3$.

\smallskip\noindent\textbf{Transition defect.} %
Given a (correct) configuration, $D_{\textit{trans}}$ encodes defects in a
following configuration coming from an incorrect execution of the transition
function. It is also a `local' defect, but it operates on two consecutive
configurations. It guesses the position of the head, the symbol under it and
the state by means of the concepts $\textsf{pos}^B$ and $\textsf{state}^B$, and
also guesses which of the two transitions is violated:
\[
D_{\textit{trans}} \sqsubseteq \textsf{pos}^B
                                  \sqcap \textsf{state}^B
                                  \sqcap \exists S^n.E_B
                                  \sqcap (D_{\textit{trans}}^0 \sqcup D_{\textit{trans}}^1 ).
\]
Then, given the current state and the symbol under the head, the transition
defect guesses the result of an incorrect execution of the transition
function. The defective value at the successor configuration is stored in
symbols $C_x$, while the relative position of the defect is stored in
$\textit{Counter}_m$, for $m \in \{-1,0,+1\}$. Thus, for
$\delta_\ell(q,\sigma)=(q_\ell,\sigma_\ell,m_\ell)$, $\ell\in \{0,1\}$,
$m_\ell \in \{-1,+1\}$, we have
\[
  \begin{array}{r@{~}l}
    D_{\textit{trans}}^\ell \sqsubseteq & \exists R.(X^\ell \sqcap \exists S^n.E_B),
    \\
    B_{q,\sigma} \sqcap D_{\textit{trans}}^\ell \sqsubseteq & (\textit{Counter}_{0} \sqcap \bigsqcup_{\sigma' \in \Gamma \setminus \{ \sigma_\ell\}} C_{\sigma'}) \sqcup
      (\textit{Counter}_{m_\ell} \sqcap \bigsqcup_{\sigma' \in\Gamma} (C_{\sigma'} \sqcup
      \bigsqcup_{q' \in Q \setminus \{q_\ell\}}C_{q',\sigma'})).
\end{array}
\]
The position of the defect is passed/updated along the $R$-successor as follows:
\begin{equation}
\begin{array}{r@{~}l}
  \textit{Counter}_{+1} \sqcap \overline{B}_k \sqcap B_{k-1} \sqcap \cdots \sqcap B_0
  \sqsubseteq& \forall R.(B_k \sqcap \overline{B}_{k-1} \sqcap \cdots \sqcap \overline{B}_0), \quad \text{ for } n > k \geq 0,
  \\
  \textit{Counter}_{+1} \sqcap B \sqcap \overline{B}_{k}
  \sqsubseteq& \forall R.B, \quad \text{ for }B \in\{B_j,\overline{B}_j \mid n > j > k\},
  \\
  \textit{Counter}_{-1} \sqcap B_k \sqcap \overline{B}_{k-1} \sqcap \cdots \sqcap \overline{B}_0
  \sqsubseteq& \forall R.(\overline{B}_k \sqcap B_{k-1} \sqcap \cdots \sqcap B_0), \quad \text{ for } n > k \geq 0,
  \\
  \textit{Counter}_{-1} \sqcap B \sqcap B_{k}
  \sqsubseteq& \forall R.B, \quad \text{ for }B \in \{B_j,\overline{B}_j \mid n > j > k\},
  \\
  \textit{Counter}_{0} \sqcap  B \sqsubseteq& \forall R.B, \quad \text{ for }B \in \{B_i, \overline{B}_i \mid 0 \leq i \leq n-1\}.
\end{array}
\label{eq:pass-along-R}
\end{equation}
The defect is copied via $R$ as follows:
\begin{equation}
 C_x \sqsubseteq \forall R. B_x, \quad x \in \{ (q,\sigma),\ \sigma \mid q \in Q, \sigma \in \Gamma \}.
\label{eq:pass-defect-along-R}
\end{equation}
Then the symbols $B_x$ and $\overline{B}_x$ that have been copied via $R$ are
propagated down the $S$-successors, and copied at $E_B$ into the
$\Sigma$-symbols $A_x$ and $\overline{A}_x$ using
\eqref{eq:prop-down-S-copy-at-E}.
A model of a transition defect is shown in Fig.~\ref{fig:defects}(b),
for $n=3$ and  $\delta_1(q_1,\sigma_1) = (q_2, \sigma_2, +1)$.

\smallskip\noindent\textbf{Copying defect.}  %
Similarly to the transition defect, the copying defect concerns two consecutive
configurations and encodes errors in copying symbols that are not under the
head. So it guesses a position of the head, a symbol under it, and a state
by means of the concepts $\textsf{pos}^G$ and $\textsf{state}^G$, and a
position different from the position of the head and a symbol at this
position by means of the concepts $\textsf{pos}^B$ and
$\textsf{symbol}^B$:
\[
  D_{\textit{copy}} \sqsubseteq \textsf{pos}^G
                                 \sqcap \textsf{state}^G
                                 \sqcap \exists S^n.E_G
  \sqcap{} \textsf{pos}^B \sqcap \textsf{symbol}^B
                                 \sqcap \exists S^n.E_B \sqcap \exists R. \exists S^n.E_B
                                 \sqcap (\textsf{val}_B \neq \textsf{val}_G).
\]
Then it chooses a new (incorrect) symbol (possibly with a state) at
the $B$-position in the subsequent configuration:
\[
\begin{array}{r@{~}l}
  B_{\sigma} \sqcap D_{\textit{copy}} \sqsubseteq &
  \textit{Counter}_{0}
  \sqcap \bigsqcup_{\sigma'\in\Gamma,~ \sigma'\neq\sigma} (C_{\sigma'} \sqcup
  \bigsqcup_{q \in Q}C_{q,\sigma'}).
\end{array}
\]
Using \eqref{eq:pass-along-R} and \eqref{eq:pass-defect-along-R}, the incorrect
value and its position are copied via $R$, and then propagated via the
$S$-successors and copied at $E_B$ to $A$-symbols using
\eqref{eq:prop-down-S-copy-at-E}.
%
% In the case of existential states, when both defects are copying defects, it
% only remains to make sure that they do not conflict each other:
% %
% \[D_{\textit{copy}}^0 \sqcap D_{\textit{copy}}^1 \sqsubseteq
%   (\textsf{val}_{F^0} = \textsf{val}_{F^1}) \to (\textsf{symb}_{F^0} =
%   \textsf{symb}_{F^1})\]
% %
% where $(\textsf{val}_{F^0} = \textsf{val}_{F^1})$ stands for
% $\bigsqcap_{i=0}^{n-1} ((F_i^0 \leftrightarrow F_i^1) \sqcap (\overline{F}_i^0
% \leftrightarrow \overline{F}_i^1))$,
% %
% $(\textsf{symb}_{F^0} = \textsf{symb}_{F^1})$ for
% $\bigsqcap_{\sigma\in\Gamma} (F_\sigma^0 \leftrightarrow F_\sigma^1)$, and
% $X \leftrightarrow Y$ is the abbreviation for
% $(\neg X \sqcup Y) \sqcap (\neg Y \sqcup X)$.

%
%

\begin{figure}
  \centering
  \begin{tikzpicture}[xscale=1.4, yscale=0.8,
    config/.style={draw=gray, very thick},
    cell/.style={gray, fill=gray!50},
    defect/.style={fill=gray}]

    \begin{scope}
      \foreach \al/\x/\y/\conc/\wh/\extra in {%
        x0/0/1/D_{\textit{conf}}/above/defect,%
        x1/-0.1/0//below/config,%
        x2/-0.2/-1//above/config,%
        x3/-0.3/-2/{B_2,\overline{B}_1, B_0, B_{q_1,\sigma_1}\qquad\qquad~~}/below/cell,%
        xx1/0.1/0//below/config,%
        xx2/0.2/-1//above/config,%
        xx3/0.3/-2/{~~\qquad\qquad \overline{G}_2,G_1,G_0, G_{q_2,\sigma_2}}/below/cell%
      }{ \node[point, \extra, label=\wh:{\scriptsize$\conc$}] (\al) at (\x,\y) {}; }

      \foreach \from/\to/\wh in {%
        x0/x1/left, x1/x2/left, x2/x3/left, %
        x0/xx1/right, xx1/xx2/right, xx2/xx3/right%
      }{ \draw[role, config] (\from) -- node[\wh] {\scriptsize $S$} (\to); }

      \node at (0, -3.5) {\small (a) An invalid configuration defect.};
    \end{scope}

    \begin{scope}[xshift=2.75cm]
      \foreach \al/\x/\y/\conc/\wh/\extra in {%
        x0/0/1/{D_{\textit{trans}},D^1_{\textit{trans}}}/above/defect,%
        x1/0/0//below/config,%
        x2/0/-1//above/config,%
        x3/0/-2/{B_2,\overline{B}_1, B_0, B_{q_1,\sigma_1}}/below/cell,%
        y0/1.5/1/X_1/above/,%
        y1/1.5/0//above/config,%
        y2/1.5/-1//below/config,%
        y3/1.5/-2/{B_2,B_1, \overline{B}_0, B_{q_3,\sigma_3}}/below/cell%
      }{ \node[point, \extra, label=\wh:{\scriptsize$\conc$}] (\al) at (\x,\y) {}; }

      \foreach \from/\to/\wh in {%
        x0/x1/left, x1/x2/left, x2/x3/left, %
        y0/y1/left, y1/y2/left, y2/y3/left%
      }{ \draw[role, config] (\from) -- node[\wh] {\scriptsize $S$} (\to); }

      \draw[role] (x0) -- node[below] {\scriptsize $R$} (y0); %

      \node at (0.75, -3.5) {\small (b) A transition defect.};
    \end{scope}

    \begin{scope}[xshift=5.5cm, yshift=1cm]
      \node[point, constant, label=above:{\scriptsize$D_{\textit{rej}}^\exists$}] (a) at
      (0,0) {};

      \foreach \al/\x/\y/\conc/\wh in {%
        x0/1/0.5/{X_0,D_{\textit{rej}}^\forall}/above,%
        x1/1/-0.5/{X_1,D_{\textit{rej}}^\forall}/below,%
        x20/2/0.5/{D_{\textit{rej}}^\exists}/above,%
        x21/2/-0.5/{D_{\textit{rej}}}/right,%
        x30/3.6/1/{X_0,D_{\textit{rej}}}/above,%
        x31/3.2/0/{X_1,D_{\textit{rej}}}/above%
      }{ \node[point, label=\wh:{\scriptsize$\conc$}] (\al) at (\x,\y) {}; }

      \foreach \from/\to/\wh in {%
        a/x0/above, a/x1/below, x0/x20/above, x1/x21/below, x20/x30/above,
        x20/x31/below%
      }{ \draw[role] (\from) -- node[\wh] {\scriptsize $R$} (\to); }

      \foreach \at/\i/\wh in {%
        x21/2/left, x30/3/right, x31/1/left%
      }{ %
        \node[point, config, yshift=-0.8cm] (s1) at (\at) {};
        \node[point, config, yshift=-1.6cm] (s2) at (\at) {};
        \node[point, cell, yshift=-2.4cm, label=below:{\scriptsize$A_{q_r,\sigma_\i}$}] (s3) at (\at) {};
        \draw[role, config] (\at) -- node[\wh] {\scriptsize $S$} (s1);
        \draw[role, config] (s1) -- node[\wh] {\scriptsize $S$} (s2);
        \draw[role, config] (s2) -- node[\wh] {\scriptsize $S$} (s3);
      }

      \node at (2, -4.5) {\small (c) A rejecting run defect.};
    \end{scope}
  \end{tikzpicture}
  \caption{Models of defects.}
  \label{fig:defects}
\end{figure}
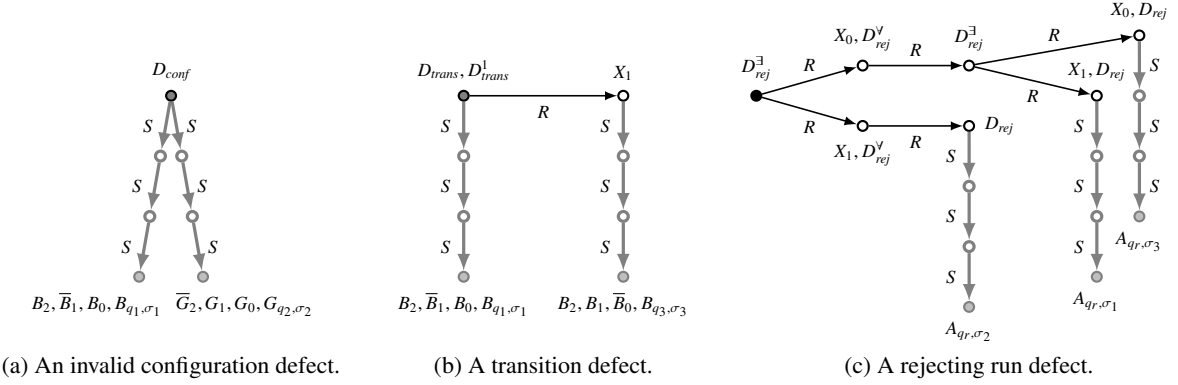

\smallskip\noindent\textbf{Rejecting run defect.} %
The rejecting run defect detects when $M$ does not accept $w$.  It is done by
checking the negation of the accepting condition. So this defect is a tree
starting at $A$ where every node at even distance from the root
($D_\textit{rej}^\exists$) has two successors (recall that $q_0\in Q_\exists$),
every node at odd distance from the root ($D_\textit{rej}^\forall$) has one
successor, and the leaves are `labelled' by rejecting states:
\[
\begin{array}{r@{~}l}
  D_{\textit{rej}}^\exists \sqsubseteq &
  \displaystyle\bigsqcap_{\ell\in\{0,1\}} \exists R. (X_\ell \sqcap (D_\textit{rej}^\forall \sqcup D_{\textit{rej}})),\\
  D_{\textit{rej}}^\forall \sqsubseteq &  \exists R.(D_\textit{rej}^\exists \sqcup D_\textit{rej}),\\
  D_{\textit{rej}} \sqsubseteq & \bigsqcup_{\sigma\in\Gamma}\exists S^n.A_{q_r, \sigma}.
\end{array}
\]
A (partial) model of a rejecting defect is shown in Fig.~\ref{fig:defects}(c).

Now we sketch a proof that $M$ accepts $w$ iff $\K_1$ does not
$\Sigma$-UCQ-entail $\K_2$.

($\Rightarrow$) Suppose $M$ accepts $w$. Then there is a model $\I_1$ of $\K_1$
such that
\begin{itemize}
\item[--] it has no local defects, that is, it has only valid configurations, and
  at each step the transition function is executed correctly and all symbols
  not affected by the head are copied correctly;
\item[--] it contains a subtree representing an accepting computation of $M$ on $w$.
\end{itemize}
Note that the former means that $\I_1$ is finite as we assumed that any run of
$M$ on every input stops either in $q_a$ or $q_r$. So the models of $\K_2$ that
are infinite paths or trees not `realising' any defect (such models never
actually pick $D$ or $D_{\textit{rej}}$ to satisfy disjunction) will not be
$\Sigma$-homomorphically embeddable into $\I_1$.  Moreover, the latter implies
that the models of $\K_2$ encoding rejecting run defect will not be
$\Sigma$-homomorphically embeddable into $\I_1$ either. So no model of $\K_2$
is $\Sigma$-homomorphically embeddable into $\I_1$, and hence $\K_1$ does not
$\Sigma$-UCQ entail $\K_2$.

($\Leftarrow$) Suppose $\K_1$ does not $\Sigma$-UCQ entail $\K_2$. Then there
exists a model $\I_1$ of $\K_1$ such that no model $\I_2$ of $\K_2$ is
$\Sigma$-homomorphically embeddable into $\I_1$. It follows that:
\begin{itemize}
\item[--] parts of $\I_1$ in grey triangles (see Fig.~\ref{fig:models-of-k1})
  represent configurations with at most one head, because of the models $\I_2$
  of $\K_2$ that detect invalid configurations;
\item[--] for every non-final configuration in $\I_1$ as explained above and for
  each of its two successor configurations, there are neither transition nor
  copying defects, because of the models of $\I_2$ that detect such defects;
\item[--] it is not the case that the tree of configurations represented by $\I_1$
  witnesses that $M$ does not accept $w$, because of the models $\I_2$ that
  detect such cases.
\end{itemize}
We thus conclude that $\I_1$ contains a valid accepting computation.
\end{proof}
%
%Thus we obtain the main result of this section.
%
%\begin{theorem}
%\label{thm:ucq}
%The problem whether an \ALC KB $\Kmc_1$ $\Sigma$-UCQ entails an \ALC KB
%$\Kmc_2$ is 2\ExpTime-complete.
%\end{theorem}
%
%
We now modify the KBs in the proof above to obtain the following:
\begin{theorem}\label{kb:2explower}
$\Sigma$-\textup{(}r\textup{)}UCQ inseparability between \ALC KBs is 2\ExpTime-hard.
\end{theorem}
\begin{proof}
  We only deal with the non-rooted case; the rooted case follows using Lemma~\ref{lem:rUCQ-CQ}.
  Consider the KBs $\K_i$, $i=1,2$, and the signature $\Sigma$ from the proof
  of Theorem~\ref{thm:ucq-lower-bound}. We
  construct (in \LogSpace) a KB $\K_{2}''$ such that $\K_1$ $\Sigma$-UCQ entails
  $\K_2$ iff $\K_1$ and $\K_{2}''$ are $\Sigma$-UCQ inseparable. This provides us
  with the desired lower bound for $\Sigma$-UCQ inseparability.
  % of the form $\K_i = (\T_i, \A)$ where $\A = \{A(a)\}$ and $\T_i$ admit
  % trivial models. First, we show that for such KBs the $\Sigma$-(r)UCQ
  % entailment problem can be polynomially reduced to the $\Sigma$-(r)UCQ
  % inseparability problem.  Then, since the KBs used to show the 2\ExpTime
  % lower bound of the $\Sigma$-(r)UCQ entailment problem in
  % Theorem~\ref{thm:ucq-lower-bound} are of such a form, we obtain the desired
  % lower bound of $\Sigma$-(r)UCQ inseparability. The upper bound clearly
  % reduces to two entailment checks, hence the tight complexity bound.
  Let $\T_i^i$ be a copy of $\T_i$ in which all concept names $X\in \sig(\T_{i})\setminus\{A\}$ are replaced
  by fresh symbols $X^i$, and let $\T_i'$ be the extension of $\T_i^i$ with
  $X^i \sqsubseteq X$, for all concept names $X \in \Sigma\setminus \{A\}$.
  We set $\K_{i}' = (\T_{i}',\{A(a)\})$, $i=1,2$, and let $\K_{2}'' = (\T_1'\cup\T_2',\{A(a)\})$.
  Observe that $\K_{i}'$ and $\K_{i}$ are $\Sigma$-UCQ inseparable, for $i=1,2$.
  We prove that $\K_{1}$ $\Sigma$-UCQ entails $\K_{2}$ iff $\K_{1}'$ and $\K_{2}''$ are $\Sigma$-UCQ inseparable. The implication $(\Leftarrow)$ is straightforward.

  Conversely, suppose $\K_{1}$ $\Sigma$-UCQ entails $\K_{2}$.
  Clearly, $\K_{2}''$ $\Sigma$-UCQ entails $\K_{1}'$, and thus it remains to prove that
  $\K_{1}'$ $\Sigma$-UCQ entails $\K_{2}''$. For $i=1,2$, we consider the class $\Mod_{i}$ of models
  $\I\in \Mod_{\K_{i}'}^{\it bo}$ such that $A^{\I}=\{a\}$, if $a\in X^{\I}$ for a concept name $X$, then
  $X\in \{{D_\textit{rej}^{0}}',A\}$, and $X^{\I}=\emptyset$, for all concept names $X\not\in \sig(\K_{i}')$.
It follows from the construction of $\K_{i}$ that $\Mod_{i}$ is complete for $\K_{i}'$.
Let
$$
\Mod = \{\I_1 \uplus \I_2 \mid \I_i\in \Mod_{i}, i =1,2\},
$$
where $\I_1 \uplus \I_2$ is the interpretation that results from merging the root $a$ of $\I_{1}$ and $\I_{2}$.
We first show that $\Mod$ is complete for $\K_{2}''$. The interpretations $\I\in \Mod$ are models of $\K_{2}''$
since, for all axioms $C \sqsubseteq D\in \T_{i}'$, either $C^{\I}\subseteq \Delta^{\I_i}\setminus \{a\}$ or
$C \in \{{D_\textit{rej}^{0}}',A,\exists S^{n}.A_{q_{a},\sigma},\exists S^{n}.A_{q_{r},\sigma}\}$ and $D$ is either a concept name
or of the form $\exists R.C'$ or $\exists S.C'$.
To see that $\Mod$ is complete for $\K_{2}''$, let $\J$ be a model of $\K_{2}''$ and $n\geq 1$.
It suffices to show that there exists $\I \in \Mod$ that is $n$-homomorphically embeddable into
$\J$ preserving $\{a\}$ (Proposition~\ref{prop:char}). But since $\J$ is a model of $\K_{i}'$, there are models $\I_{i}\in \Mod_{i}$
such that $\I_{i}$ is $n$-homomorphically embeddable into $\J$ preserving $\{a\}$, $i=1,2$ (Proposition~\ref{prop:char}).
By taking the union of the two partial witness homomorphisms from $\I_{1}$ and $\I_{2}$,  one can show that
$\I_{1}\uplus \I_{2}$ is $n$-homomorphically embeddable into $\J$ preserving $\{a\}$, as required.

We now use Theorem~\ref{crit:KB} (1) to prove that $\K_{1}'$ $\Sigma$-UCQ entails $\K_{2}''$.
Let $\I_{1}\in \Mod_{1}$ and $n\geq 1$. It suffices to find $\J\in \Mod$ that is $n\Sigma$-homomorphically embeddable into $\I_{1}$
preserving $\{a\}$. But since $\K_{1}'$ $\Sigma$-UCQ-entails $\K_{2}'$, there exists
$\I_2\in \Mod_{2}$ such that $\I_2$ is $n\Sigma$-homomorphically embeddable into $\I_1$ preserving $\{a\}$.
By combining $n\Sigma$-homomorphisms from $\I_{2}$ with the identity mapping from $\I_{1}$, it is now straightforward to
show that the model $\I_{1}\uplus \I_{2}\in \Mod$ is $n\Sigma$-homomorphically
embeddable into $\I_1$ preserving $\{a\}$, as required.
\end{proof}
The following theorem summarises the results obtained so far.
%Summarizing the results obtained so far we have thus proved the following.
\begin{theorem}
$\Sigma$-\textup{(}r\textup{)}UCQ inseparability and $\Sigma$-\textup{(}r\textup{)}UCQ-entailment between \ALC KBs are both 2\ExpTime-complete.
\end{theorem}

\subsection{\textup{(}r\textup{)}UCQ-entailment and inseparability with full signature}
We extend the 2\ExpTime lower bound from $\Sigma$-(r)UCQ entailment and inseparability to
full signature (r)UCQ entailment and inseparability.
To this end we prove a UCQ-variant of Theorem~\ref{Thm:equivalence}:
\begin{theorem}\label{Thm:equivalenceUCQ}
Let $\K_{1}=(\T_{1},\A)$ and $\K_{2}=(\T_{2},\A)$ be $\mathcal{ALC}$ KBs and $\Sigma$ a signature such that ${\sig}(\A) \subseteq \Sigma$ and $\Gamma ={\sig}(\T_{1}\cup \T_{2})\setminus \Sigma$ contains no role names. Suppose $\T_{1}$ and $\T_{2}$ admit trivial models.
Let $\Kup_{i}=(\Tup_{i}\cup \T_{\Gamma}^{\exists},\A)$, for $i=1,2$.
Then the following conditions are equivalent:
\begin{enumerate}
\item[\rm (1)] $\K_{1}$ $\Sigma$-(r)UCQ entails $\K_{2}$;
\item[\rm (2)] $\Kup_{1}$ full signature (r)UCQ entails $\Kup_{2}$.
\end{enumerate}
\end{theorem}
\begin{proof}
We use and modify the proof of Theorem~\ref{Thm:equivalence}.
%We focus on Point~(1). The proof of Point~(2) is similar and omitted.
Let $\Mod_{i}$ be complete for $\K_{i}$, $i=1,2$. We may assume that $X^{\I}=\emptyset$ for all concept and
role names $X\not\in \sig(\K_{i})$ and $\I\in \Mod_{i}$, $i=1,2$. By Fact~5 of the proof of
Theorem~\ref{Thm:equivalence}, $\{\Iup \mid \I\in \Mod_{i}\}$ is complete for $\Kup_{i}$.
Thus, by Theorem~\ref{crit:KB}, it suffices to prove that $\I_{2}$ is $n\Sigma$-homomorphically embeddable into $\I_{1}$ preserving
$\ind(\K_{2})$ iff $\Iup_{2}$ is $n$-homomorphically embeddable into $\Iup_{1}$ preserving $\ind(\K_{2})$, for any $n>0$, $\I_{1}\in \Mod_{1}$ and
$\I_{2}\in \Mod_{2}$.
This can be done in the same way as in the proof of Fact~6.
\end{proof}
The following complexity result now follows from the observation that
the KBs and signature $\Sigma$ used in the proof of Theorem~\ref{kb:2explower} satisfy the conditions of
Theorem~\ref{Thm:equivalenceUCQ}: $\Sigma$ contains the signature of the ABox
and all role names of the KBs, and the TBoxes admit trivial models.
\begin{theorem}
Full signature \textup{(}r\textup{)}UCQ inseparability and entailment between $\ALC$ KBs are
both 2\ExpTime-complete.
\end{theorem}

% !TEX root = aij-insep.tex

\section{Query Entailment and Inseparability for \ALC TBoxes}

In this section, we introduce query entailment and inseparability between TBoxes.
Two TBoxes $\T_{1}$ and $\T_{2}$ are query inseparable for a class $\mathcal{Q}$ of queries 
if, for all ABoxes $\A$ that are consistent with $\T_{1}$ and $\T_{2}$, queries from $\mathcal{Q}$ have the same certain answers 
over the KBs $(\T_{1},\A)$ and $(\T_{2},\A)$. The TBox $\T_{1}$ $\mathcal{Q}$-entails $\T_{2}$ if, for any such $\A$, the certain
answers to queries from $\mathcal{Q}$ over $(\T_{2},\A)$ are contained in the certain answers over $(\T_{1},\A)$. As in the KB case,
we consider the restriction of CQs and UCQs to a signature $\Sigma$ of relevant symbols and their restrictions to rooted queries. In applications, it is
also natural to restrict the signature of the ABox which might be different from the signature of the relevant queries.
\begin{definition}\em
\label{def:tboxinsep}
Let $\T_{1}$ and $\T_{2}$ be TBoxes, $\mathcal{Q}$ one of CQ, \RCQ, UCQ or \RUCQ, and let $\Theta = (\Sigma_1,\Sigma_2)$ be a pair of signatures. We say that $\T_{1}$ \emph{$\Theta$-$\mathcal{Q}$ entails} $\T_{2}$ if, for every $\Sigma_1$-ABox $\A$ that is consistent with both $\T_1$ and $\T_2$, the KB $(\T_1,\A)$ $\Sigma_2$-$\mathcal{Q}$ entails the KB $(\T_2,\A)$. $\T_1$ and $\T_2$ are \emph{$\Theta$-$\mathcal{Q}$ inseparable} if they $\Theta$-$\mathcal{Q}$ entail each other. If $\Sigma_{1}$ is the set of all concept and role names, we say `\emph{full ABox signature $\Sigma_{2}$-$\mathcal{Q}$ entails}' or `\emph{full ABox signature $\Sigma_{2}$-$\mathcal{Q}$ inseparable}'\!.
\end{definition}
In the definition of $\Theta$-$\mathcal{Q}$ entailment, we only consider ABoxes that are consistent with both TBoxes.
The reason is that the complexity of the problem of deciding whether every $\Sigma$-ABox
consistent with a TBox $\T_{1}$ is also consistent with a TBox $\T_{2}$ is already well understood and is dominated
by the $\Theta$-$\mathcal{Q}$-entailment problem as defined above. More precisely, we say that a TBox $\T_{1}$
\emph{$\Sigma_{\bot}$-entails} a TBox $\T_{2}$ if all $\Sigma$-ABoxes $\A$ consistent with $\T_{1}$ are consistent with
$\T_{2}$. $\Sigma_{\bot}$-entailment is closely related to the containment problem between ontology-mediated queries,
which we define next~\cite{BienvenuLW12,BienvenuCLW14,HLSW-IJCAI16}. For a query $\q$, 
TBoxes $\T_{1}$ and $\T_{2}$, and a signature $\Sigma$, we say that \emph{$(\T_{1},\q)$
is contained in $(\T_{2},\q)$ for $\Sigma$} and write $(\T_{1},\q)\subseteq_{\Sigma} (T_{2},\q)$ if, for every
$\Sigma$-ABox $\A$, the certain answers to $\q$ over $(\T_{1},\A)$ are contained in the certain answers to $\q$
over $(\T_{2},\A)$. We note that the authors of~\cite{BienvenuLW12,HLSW-IJCAI16}
demand that the $\Sigma$-ABoxes considered in the definition of containment are consistent with both TBoxes, 
but the complexity results for deciding containment do not depend on this condition. 
The \emph{containment problem} for a description 
logic $\Lmc$ relative to a class $\mathcal{Q}$ of queries is to decide, for TBoxes $\T_{1}$ and 
$\T_{2}$ in $\Lmc$, signature $\Sigma$, and query $\q\in \mathcal{Q}$, whether $(\T_{1},\q)\subseteq_{\Sigma} (T_{2},\q)$.
Thus, in contrast to $\Theta$-$\mathcal{Q}$-entailment, an instance of the containment problem does not 
quantify over all $\q\in \mathcal{Q}$ but takes the queries $\q\in \mathcal{Q}$ as inputs to the decision problem. 
It is known~\cite{BienvenuLW12,BienvenuCLW14,HLSW-IJCAI16}
that the containment problem is
\begin{itemize}
\item[--] \NExpTime-complete for $\mathcal{ALC}$ TBoxes and CQs of the form $\exists x A(x)$;
\item[--] \ExpTime-complete for $\hALC$ TBoxes and CQs of the form $\exists x A(x)$. 
\end{itemize}
It is straightforward to show that the containment problem for a DL $\Lmc$ and CQs of the form $\exists x A(x)$ is 
mutually polynomially reducible with the problem to decide $\Sigma_{\bot}$-entailment between $\Lmc$ TBoxes. For a polynomial
reduction of $\Sigma_{\bot}$-entailment to containment, observe that $\T_{1}$ $\Sigma_{\bot}$-entails $\T_{2}$ iff $(\T_{2},\exists x A(x))\subseteq_{\Sigma} (\T_{1},\exists x A(x))$ for $A\not\in \sig(\T_{1}\cup \T_{2})$. For a polynomial reduction of containment to $\Sigma_{\bot}$-entailment, assume that $\T_{1},\T_{2},\Sigma$, and $A$ are given. 
Let $\T_{i}'=\T_{i}\cup \{A\sqsubseteq \bot\}$. Then $(\T_{1},\exists x A(x))\subseteq_{\Sigma} (\T_{2},\exists x A(x))$ iff
$\T_{2}'$ $\Sigma_{\bot}$-entails $\T_{1}'$. We obtain the following result.
\begin{theorem}
The problem whether an $\mathcal{ALC}$ TBox $\Sigma_{\bot}$-entails an $\mathcal{ALC}$ TBox is \NExpTime-complete.
For $\hALC$ TBoxes $\T_{1}$ and $\T_{2}$, this problem is \ExpTime-complete.
\end{theorem}

It follows, in particular, that our complexity upper bounds for $\Theta$-CQ-entailment still hold if one 
admits ABoxes that are not consistent with the TBoxes. 
% containment problem} for
% ontology-mediated queries with CQs of the form $\exists x A(x)$, which is \NExpTime-complete for \ALC and \ExpTime-complete 
% for \hALC. Here we say that $(\T,\exists x A(x))Note that $\Theta$-CQ-entailment in the restricted case with 
% $\Theta=(\Sigma,\Sigma)$ has been investigated for $\mathcal{EL}$ TBoxes in \cite{LutzW10,KonevL0W12}.

As in the KB case, $\Theta$-UCQ inseparability of $\ALC$ TBoxes implies all other types of inseparability, and Example~\ref{ex:ex1} can be
used to show that no other implications hold in general. The situation is different for \hALC TBoxes. In fact, the following result follows
 directly from Proposition~\ref{UCQtoCQ}:
\begin{proposition}\label{TBoxUCQtoCQ}
For any $\ALC$ TBox $\T_{1}$ and \hALC TBox $\T_{2}$, 
$\T_{1}$ $\Theta$-\textup{(}r\textup{)}UCQ entails $\T_{2}$ iff $\T_{1}$ $\Theta$-\textup{(}r\textup{)}CQ entails $\T_{2}$.
\end{proposition}

We now show that $\Theta$-(r)CQ entailment and inseparability are undecidable for $\mathcal{ALC}$ TBoxes.
In fact, we show that $\Theta$-(r)CQ inseparability is undecidable even if one of the TBoxes is given in $\mathcal{EL}$ 
and that $\Theta$-(r)CQ entailment is undecidable even if the entailing TBox $\T_{1}$ is in $\mathcal{EL}$.
The proofs re-use the TBoxes constructed in the undecidability proofs for KBs in Theorems~\ref{thm:undecidability-cq} and~\ref{thm:undecidability-rcq}.
We also show that, for CQs, these problems are still undecidable in the full ABox signature case or
if one assumes that the signatures for the ABoxes and CQs coincide. It remains open whether rCQ-entailment 
or inseparability are still undecidable in those cases.
\begin{theorem}\label{thm:TBoxbasicundecidability}
\textup{(}i\textup{)} The problem whether an $\mathcal{EL}$ TBox $\Theta$-$\mathcal{Q}$ entails an $\ALC$ TBox is undecidable for $\mathcal{Q} \in \{\text{CQ},\text{\RCQ}\}$.

\textup{(}ii\textup{)} $\Theta$-$\mathcal{Q}$ inseparability between $\mathcal{EL}$ and \ALC TBoxes is undecidable for $\mathcal{Q} \in \{\text{CQ},\text{\RCQ}\}$.

\textup{(}iii\textup{)} For CQs, \textup{(}i\textup{)} and \textup{(}ii\textup{)} hold for full ABox signatures and for $\Theta = (\Sigma_1,\Sigma_2)$ with $\Sigma_{1}=\Sigma_{2}$.
\end{theorem}
\begin{proof}
Here, we focus on the CQs; the proofs for rCQs are given in the appendix. 
We use the KBs
$\KoneCQ =(\ToneCQ,\ACQ)$ and $\KtwoCQ =(\TtwoCQ,\ACQ)$ and the signature $\SigmaCQ=\sig(\KoneCQ)$ from the proof of 
Theorem~\ref{thm:undecidability-cq}. Recall that it is undecidable whether $\KoneCQ$ $\SigmaCQ$-CQ entails $\KtwoCQ$.
Also recall that, for $\K_{2}=(\T_{2},\ACQ)$ with $\T_{2}=\ToneCQ\cup \TtwoCQ$, it is undecidable whether $\KoneCQ$ and $\K_{2}$ 
are $\SigmaCQ$-CQ inseparable (Theorem~\ref{thm:undecidability-cqinsep}).

(\emph{i}) Let $\Sigma_{1}=\{A\}$, $\Sigma_{2}=\SigmaCQ$, and $\Theta=(\Sigma_{1},\Sigma_{2})$.
We show that $\ToneCQ$ $\Theta$-CQ-entails $\TtwoCQ$ iff $\KoneCQ$ $\SigmaCQ$-CQ-entails $\KtwoCQ$. 
Recall that $\ACQ=\{A(a)\}$. Thus, if $\KoneCQ$ does not $\SigmaCQ$-CQ entail $\KtwoCQ$, 
then we have found a $\Sigma_{1}$-ABox witnessing that $\ToneCQ$ does not $\Theta$-CQ entail $\TtwoCQ$. 
Conversely, observe that $\Sigma_{1}$-ABoxes $\A$ are sets of the form $\{A(b)\mid b\in I\}$, with $I$ a finite
set of individual names. Thus, if there exists a $\Sigma_{1}$-ABox $\A$ such that 
$(\ToneCQ,\A)$ does not $\SigmaCQ$-CQ entail $(\TtwoCQ,\A)$, then $(\ToneCQ,\ACQ)$ does not 
$\SigmaCQ$-CQ entail $(\TtwoCQ,\ACQ)$ either.

(\emph{ii}) Set again $\Theta=(\Sigma_{1},\Sigma_{2})$, for $\Sigma_{1}=\{A\}$ and $\Sigma_{2}=\SigmaCQ$. In exactly the 
same way as in (\emph{i}) one can show that $\KoneCQ$ and $\K_{2}$ are $\SigmaCQ$-inseparable iff $\ToneCQ$ and $\T_{2}$ are 
$\Theta$-CQ inseparable. 

(\emph{iii}) We first show undecidability of full ABox signature $\Sigma$-CQ inseparability. The
undecidability of full ABox signature $\Sigma$-CQ entailment follows directly from our proof.
%We take again the KBs $\KoneCQ =(\ToneCQ,\ACQ)$ and $\KtwoCQ =(\TtwoCQ,\ACQ)$ and the signature $\SigmaCQ$ from the proof of Theorem~\ref{thm:undecidability-cq}.
%Let, as in the proof of Theorem~\ref{thm:undecidability-cqinsep}, $\K_{2}=(\T_{2},\ACQ)$, where $\T_{2}=\ToneCQ \cup \TtwoCQ$. 
We employ the abstraction technique from Theorem~\ref{Thm:equivalence} for $\Gamma=\sig(\T_{2})\setminus \SigmaCQ$.
Let $\T_{1}'= \ToneCQ\cup \T_{\Gamma}^{\exists}$, $\T_{2}' = \T_{2}^{\uparrow\Gamma}\cup \T_{\Gamma}^{\exists}$ and 
$\Sigma=\SigmaCQ\setminus\{P\}$. We aim to prove that the 
following conditions are equivalent:
\begin{enumerate}
\item[(1)] $\KoneCQ$ and $\K_{2}$ are $\Sigma$-CQ inseparable;
\item[(2)] $\T_{1}'$ and $\T_{2}'$ are full ABox signature $\Sigma$-CQ inseparable.
\end{enumerate}
Observe that undecidability of full ABox signature CQ-inseparability of TBoxes of the form
$\T_{1}'$ and $\T_{2}'$ follows since the proof of Theorems~\ref{thm:undecidability-cq} 
and~\ref{thm:undecidability-cqinsep}
shows that the role name $P$ is not needed to CQ-separate the KBs $\KoneCQ$ and $\K_{2}$ (if they
are $\SigmaCQ$-CQ separable). Thus,
it is undecidable whether $\KoneCQ$ and $\K_{2}$ are $\Sigma$-CQ inseparable. 

The implication $(2) \Rightarrow (1)$ is straightforward: if $\KoneCQ$ and $\K_{2}$ are not $\Sigma$-CQ inseparable, 
then the ABox $\ACQ$ witnesses that $\T_{1}'$ and $\T_{2}'$ are not full ABox signature $\Sigma$-CQ inseparable.  
Conversely, suppose $\T_{1}'$ and $\T_{2}'$ are not full ABox signature $\Sigma$-CQ inseparable.
Then there exists an ABox $\A$ such that $(\T_{1}',\A)$ and $(\T_{2}',\A)$ are not $\Sigma$-CQ inseparable. 
The canonical model $\I_{1}$ of the $\mathcal{EL}$ KB $(\T_{1}',\A)$ can be constructed as follows:
\begin{itemize}
\item[--] for any $A(b)\in \A$, take a copy of the canonical model $\I_{\KoneCQ}$ and hook it to $b$ by identifying 
$a$ in $\I_{\KoneCQ}$ with $b$;
\item[--] for any $D(b)\in \A$, take a copy of the subinterpretation of the canonical model $\I_{\KoneCQ}$ rooted at the
$P$-successor of $a$ and hook it to $b$ by identifying the $P$-successor of $a$ with $b$;
\item[--] for any $E(b)\in \A$, take a copy of the (unique up to isomorphism) subinterpretation of the canonical model 
$\I_{\KoneCQ}$ rooted at an $E$-node and hook it to $b$ by identifying the $E$-node with $b$.
\item[--] to satisfy $\T_{\Gamma}^{\exists}$, let $\J$ be the singleton interpretation with $X^{\J}=\emptyset$ for all concept and role names $X$; we hook to any element $u$ of the interpretation constructed so far a copy of $\J^{\uparrow\Gamma}$ by identifying the root of
$\J^{\uparrow\Gamma}$ with $u$ (see the proof of Theorem~\ref{Thm:equivalence} for the construction and properties of $\J^{\uparrow\Gamma}$).
\end{itemize}
Let $\Mod$ be the class of interpretations obtained from $\I_{1}$
by adding to any $b$ with $A(b)\in \A$ a $P$-successor $b'$ to which one hooks the subinterpretation rooted in the $P$-successor of $a$ in an
interpretation from $\{\I^{\uparrow\Gamma}\mid \I \in \Mod_{\KtwoCQ}\}$. One can show that $\Mod$ is complete for the
KB $(\T_{2}',\A)$. To this end, first recall from the proof of Theorem~\ref{thm:undecidability-cqinsep}
that for the canonical model $\I_{\KoneCQ}$ of $\KoneCQ$, the set 
$\Mod_{\K_{2}}= \{~ \I \uplus \I_{\KoneCQ} \mid \I \in \Mod_{\KtwoCQ}~\}$ (where
$\I \uplus \I_{\KoneCQ}$ is the interpretation that results from merging the roots $a$ of
$\I$ and $\I_{\KoneCQ}$) is complete for $\K_{2}$. By Theorem~\ref{Thm:equivalence} (Fact~5),
$\{\I^{\uparrow\Gamma} \mid \I \in \Mod_{\K_{2}}\}$ is complete for $\K_{2}^{\uparrow\Gamma}$. Now completeness
of $\Mod$ for $(\T_{2}',\A)$ follows directly from the fact that every $\I\in \Mod$ is a model of $(\T_{2}',\A)$.
Next, observe that $P\not\in \Sigma$ and that two KBs are $\Sigma$-CQ inseparable 
iff they are $\Sigma$-CQ inseparable for connected $\Sigma$-CQs. Thus, the only $\Sigma$-components of interpretations in
$\Mod$ that could distinguish $\Sigma$-CQs true in $\Mod$ from $\Sigma$-CQs true in $\I_{1}$ are the interpretations $\{\I^{\uparrow\Gamma}\mid \I \in \Mod_{\KtwoCQ}\}$.
It follows that if $(\T_{1}',\A)$ and $(\T_{2}',\A)$ are not $\Sigma$-CQ inseparable, then $(\KoneCQ)^{\uparrow\Gamma}$ and $\K_{2}^{\uparrow\Gamma}$
are not $\Sigma$-CQ inseparable either. But then, by the proof of Theorem~\ref{thm:undecidability-full}, $\KoneCQ$ and $\K_{2}$ are not 
$\Sigma$-CQ inseparable, as required.

To show undecidability of $\Theta$-CQ inseparability and entailment for $\Theta=(\Sigma_{1},\Sigma_{2})$ with $\Sigma_{1}=\Sigma_{2}$, 
we re-use the undecidability proof for the full ABox signature case. Set $\Theta=(\Sigma,\Sigma)$. Then the proof above shows that
$\T_{1}'$ and $\T_{2}'$ are $\Theta$-CQ inseparable iff they are full ABox signature $\Sigma$-CQ inseparable since one can always choose
the ABox $\ACQ$ as a witness for CQ-inseparability if $\T_{1}'$ and $\T_{2}'$ are full ABox signature $\Sigma$-CQ inseparable.
\end{proof}

% !TEX root = aij-insep.tex

\section{Model-Theoretic Criteria for Query Entailment of \hALC TBoxes by \ALC TBoxes}

We have seen that $\Theta$-(r)CQ entailment of an $\ALC$ TBox $\T_{2}$ by an $\mathcal{EL}$ TBox $\T_{1}$ is undecidable.
We now investigate the converse direction, with drastically different results (which even hold if $\mathcal{EL}$ TBoxes are replaced by \hALC TBoxes).
Thus, in this section, we give model-theoretic criteria for $\Theta$-(r)CQ entailment of a \hALC TBox $\T_{2}$ by an
\ALC TBox $\T_{1}$. In the next section, we use these criteria to prove tight complexity bounds for deciding
$\Theta$-(r)CQ entailment and inseparability. Recall that, by Proposition~\ref{TBoxUCQtoCQ}, our model-theoretic criteria and complexity
results also apply to $\Theta$-(r)UCQ entailment.

We assume that \hALC TBoxes are given in \emph{normal form} where concept inclusions look as follows:
$$
A\sqsubseteq B, \quad A_{1} \sqcap A_{2} \sqsubseteq B, \quad \exists R.A \sqsubseteq B,\quad 
A \sqsubseteq \bot, \quad \top \sqsubseteq B, \quad A \sqsubseteq \exists R.B, \quad A \sqsubseteq \forall R.B
$$
and $A,B$ are concept names. It is standard (see, e.g.,~\cite[Proposition~28]{BaaderBL16})
to show the following reduction of $\Theta$-(r)CQ entailment for arbitrary \hALC TBoxes to \hALC TBoxes in normal form.
\begin{proposition}
For any \hALC TBox $\T_{2}$ and any pair $\Theta$ of signatures, one can construct in polynomial time
a \hALC TBox $\T_{2}'$ in normal form such that an \ALC TBox $\T_{1}$ $\Theta$-(r)CQ entails $\T_{2}$ iff $\T_{1}$ $\Theta$-(r)CQ entails $\T_{2}'$.
\end{proposition}

%In what follows we always assume that \hALC TBoxes have been transformed into
%normal form.  

Our model-theoretic criteria are based on two crucial
observations. First, to characterise $\Theta$-(r)CQ entailment between \hALC
TBoxes and \ALC TBoxes, it suffices to consider a very restricted class of
acyclic (r)CQs that corresponds exactly to queries constructed using
$\mathcal{EL}$ concepts. Second, it suffices to consider ABoxes that are
tree-shaped rather than arbitrary ABoxes when searching for witnesses for
non-$\Theta$-(r)CQ entailment.  We begin by introducing the relevant classes of
CQs and rCQs. A \emph{rooted $\mathcal{EL}$ query} takes the form $C(x)$, where
$C$ is an $\mathcal{EL}$ concept. The set of rooted $\mathcal{EL}$ queries is
denoted by rELQ.  Given a KB $\K$, $a\in \ind(\K)$, and an rELQ $C(x)$ we say that
\emph{$a$ is a certain answer to $C(x)$ over $\K$} if $a^{\I}\in C^{\I}$, for
every model $\I$ of $\K$. Note that rELQs can be regarded as acyclic CQs with
one answer variable. A \emph{Boolean $\mathcal{EL}$ query} takes the form
$\exists x C(x)$, where $C$ is an $\mathcal{EL}$ concept. The set of rooted and
Boolean $\mathcal{EL}$ queries is denoted by ELQ.  Given a KB $\K$ and a Boolean
$\mathcal{EL}$ query $\exists x C(x)$, we say that \emph{$\K$ entails $\exists x
  C(x)$} if $C^{\I}\not=\emptyset$, for every model $\I$ of $\K$. Boolean
$\mathcal{EL}$ queries can be regarded as Boolean acyclic CQs. In what follows
we use the same notation for (r)ELQs as for (r)CQs. For TBoxes $\T_{1}$ and
$\T_{2}$ and a pair $\Theta=(\Sigma_{1},\Sigma_{2})$ of signatures, we say that
$\T_{1}$ $\Theta$-(r)ELQ \emph{entails} $\T_{2}$ if, for every $\Sigma_{1}$ ABox $\A$
that is consistent with both $\T_{1}$ and $\T_{2}$, and every
$\Sigma_{2}$-(r)ELQ $\q(a)$ with $a\in \ind(\A)$, whenever $(\T_{2},\A)
\models \q(a)$ then $(\T_{1},\A)\models \q(a)$.
\begin{proposition}\label{Elqs}
Let $\T_{1}$ be an $\mathcal{ALC}$ TBox, $\T_{2}$ a Horn$\mathcal{ALC}$ TBox,
and $\Theta=(\Sigma_{1},\Sigma_{2})$ a pair of signatures.
Then $\T_{1}$ $\Theta$-\textup{(}r\textup{)}CQ entails $\T_{2}$ iff $\T_{1}$ $\Theta$-(r)ELQ entails $\T_{2}$.
\end{proposition}
\begin{proof}
  Suppose $\Amc$ is a $\Sigma_{1}$-ABox and $(\T_{2},\A)\models \q(\avec{a})$
  but $(\T_{1},\A)\not\models \q(\avec{a})$ for a $\Sigma_{2}$-CQ
  $\q$. As $(\T_{2},\A)\models \q(\avec{a})$, there is a homomorphism $h \colon \q \rightarrow \I_{(\T_{2},\A)}$. Let
  $\I$ be the $\Sigma_{2}$-reduct of the
  subinterpretation of $\I_{(\T_{2},\A)}$ induced by the image of $\q$ under
  $h$. Then $\I$ is the disjoint union of
\begin{itemize}
\item[--] ditree interpretations $\I_{a}$ attached to $a\in \ind(\A)\cap
  \Delta^{\I}$ such that $\ind(\A)\cap \Delta^{\I_{a}}=\{a\}$, and
\item[--] ditree interpretations $\J$ with $\ind(\A) \cap \Delta^{\J}=\emptyset$
  (there exists no such $\J$ if $\q$ is an rCQ),
\end{itemize}
and, additionally, pairs $(a,b)$ in $R^{\I}$ for $a,b\in \ind(\A)\cap
\Delta^{\I}$, $R\in \Sigma_{1}$, and $R(a,b)\in \A$.  Thus, if $\q$ is an rCQ
then there exists $\I_{a}$ such that the canonical CQ $\q_{\I_{a}}(x)$
determined by $\I_{a}$ is an rELQ (see the proof of Proposition~\ref{prop:char})
and $(\T_{2},\A)\models \q_{\I_{a}}(a)$ but $(\T_{1},\A)\not\models
\q_{\I_{a}}(a)$, as required. If $\q$ is not an rCQ and no such $\I_{a}$
exists, then there exists $\J$ such that the canonical CQ $\q_{\J}$ determined
by $\J$ is a Boolean $\mathcal{EL}$ query and $(\T_{2},\A)\models \q_{\J}$ but
$(\T_{1},\A)\not\models \q_{\J}$.
\end{proof}

An ABox $\A$ is called a \emph{tree ABox} if the undirected graph
$$
G_{\A} ~=~ \left(\ind(\A), \big\{\{a,b\} \mid R(a,b)\in \A\big\}\right)
$$
is an undirected tree and $R(a,b)\in \A$ implies $R(b,a)\not\in\A$ and $S(a,b) \notin \A$, for $S\not=R$.
The \emph{outdegree} of $\A$ is defined as the outdegree of $G_{\A}$.
\begin{theorem}\label{thm:homcharhorn}
Let $\T_{1}$ be an \ALC TBox, $\T_{2}$ a \hALC TBox, and $\Theta=(\Sigma_{1},\Sigma_{2})$.
Then
\begin{enumerate}
\item[\rm (1)] $\T_{1}$ $\Theta$-rCQ-entails $\T_{2}$ iff, for any tree
  $\Sigma_{1}$-ABox $\A$ of outdegree bounded by $|\T_{2}|$ and consistent
  with $\T_{1}$ and $\T_{2}$, and any model $\I_{1}$ of $(\T_{1},\A)$, 
  $\I_{(\T_{2},\A)}$ is con-$\Sigma_{2}$-homomorphically embeddable into
  $\I_{1}$ preserving $\ind(\A)$.
\item[\rm (2)] $\T_{1}$ $\Theta$-CQ-entails $\T_{2}$ iff, for any tree
  $\Sigma_{1}$-ABox $\A$ of outdegree bounded by $|\T_{2}|$ and consistent
  with $\T_{1}$ and $\T_{2}$, and any model $\I_{1}$ of $(\T_{1},\A)$, 
  $\I_{(\T_{2},\A)}$ is $\Sigma_{2}$-homomorphically embeddable into $\I_{1}$
  preserving $\ind(\A)$.
\end{enumerate}
\end{theorem}
\begin{proof}
(1) The direction from left to right follows from Theorem~\ref{UCQ-full-hom} and Proposition~\ref{UCQtoCQ}.  Conversely, suppose $\T_{1}$ does not $\Theta$-rCQ-entail $\T_{2}$.  By Proposition~\ref{Elqs}, there are a $\Sigma_{1}$-ABox $\A$ consistent with $\T_{1}$ and $\T_{2}$, a $\Sigma_{2}$-rELQ $C(x)$, and $a\in \ind(\A)$ such that $(\T_{2},\A)\models C(a)$ and $(\T_{1},\A)\not\models C(a)$.  It is shown in \cite{BaaderBL16} (proof of Proposition~30)\footnote{The proof of Proposition~30 in~\cite{BaaderBL16} shows this for $\mathcal{ELIF}_{\bot}$ TBoxes. Observe that we can regard every \hALC TBox in normal form as an $\mathcal{ELI}_{\bot}$ TBox by replacing $A\sqsubseteq \forall R.B$ by $\exists R^{-}.A \sqsubseteq B$.} that there exist a tree $\Sigma_{1}$-ABox $\A'$ with outdegree bounded by $|\T_{2}|$ and $(\T_{2},\A')\models C(a)$, and an ABox homomorphism\footnote{ABox homomorphisms are defined before Proposition~\ref{prop:ABoxhom}
in the appendix.} $h$ from $\A'$ to $\A$ with $h(a)=a$. 
%\end{itemize}
It follows from Proposition~\ref{prop:ABoxhom} in the appendix that $\A'$ is consistent with $\T_{1}$ and $\T_{2}$ and that $(\T_{1},\A')\not\models C(a)$.
Let $\I_{1}$ be a model of $(\T_{1},\A')$ such that $\I_{1}\not\models C(a)$. We know that $\I_{(\T_{2},\A')}\models C(a)$. Thus,
$\I_{(\T_{2},\A')}$ is not con-$\Sigma_{2}$-homomorphically embeddable into $\I_{1}$ preserving $\ind(\A')$, as required. 
(2) is proved similarly using ELQs instead of rELQs and $\Sigma_{2}$-homomorphisms instead of con-$\Sigma_{2}$-homomorphisms.
\end{proof}
The notion of (con-)$\Sigma$-CQ homomorphic embeddability used in Theorem~\ref{thm:homcharhorn} is slightly
unwieldy to use in the subsequent definitions and automata constructions. We therefore resort to simulations
whose advantage is that they are compositional (they can be partial and are closed under unions).
\newcommand{\Sigmaabox}{\Sigma_1} \newcommand{\Sigmaquery}{\Sigma_2}
Let $\Imc_1,\Imc_2$ be interpretations and $\Sigma$ a signature.  A relation
$\mathcal{S}\subseteq \Delta^{\Imc_1}\times \Delta^{\Imc_2}$ is a
\emph{$\Sigma$-simulation} from $\Imc_{1}$ to $\Imc_{2}$ if (\emph{i}) $d\in
A^{\Imc_{1}}$ and $(d,d')\in \mathcal{S}$ imply $d'\in A^{\Imc_{2}}$ for all
$\Sigma$-concept names $A$, and (\emph{ii}) if $(d,e)\in R^{\Imc_{1}}$ and
$(d,d')\in \mathcal{S}$ then there is a $(d',e')\in R^{\Imc_{2}}$ with
$(e,e')\in \mathcal{S}$ for all $\Sigma$-role names~$R$.  Let $d_i \in
\Delta^{\Imc_i}$, $i \in \{1,2\}$.  $(\Imc_{1},d_{1})$ is
\emph{$\Sigma$-simulated} by $(\Imc_{2},d_{2})$, in symbols $(\Imc_{1},d_{1})
\leq_{\Sigma}(\Imc_{2},d_{2})$, if there exists a $\Sigma$-simulation
$\mathcal{S}$ with $(d_{1},d_{2})\in \mathcal{S}$.  Observe that every
$\Sigma$-homomorphism from $\I_{1}$ to $\I_{2}$ is a
$\Sigma$-simulation. Conversely, if $\I_{1}$ is a ditree interpretation and
$(\I_{1},d_{1}) \leq_{\Sigma} (\I_{2},d_{2})$, then one can construct a
$\Sigma$-homomorphism $h$ from $\I_{1}$ to $\I_{2}$ with $h(d_{1})=d_{2}$.
%
%
% \begin{lemma}
% \label{lem:homtosim}
% Let $\Sigma_{1}$ and $\Sigma_{2}$ be signatures, \Amc a
% $\Sigmaabox$-ABox, and $\Imc_1$ a model of $(\Tmc_1,\Amc)$. Then
% $\Imc_{\Tmc_2,\Amc}$ is not con-$\Sigmaquery$-homomorphically
% embeddable into $\Imc_1$ iff one of the following holds:
%   %
%   \begin{enumerate}\itemsep 0cm

%   \item[\rm (1)] there is a $\Sigmaquery$-concept name $A$ with
%     $a \in A^{\Imc_{\Tmc_2,\Amc}} \setminus A^{\Imc_1}$;

%   \item[\rm (2)] there is an $R$-successor $d$ of $a$ in
%     $\Imc_{\Tmc_2,\Amc}$, for some $\Sigmaquery$-role name $R$, such
%     that $d \notin \mn{ind}(\Amc)$ and, for all $R$-successors $e$
%     of $a$ in $\Imc_1$, we have  $(\Imc_{\Tmc_2,\Amc},d)
%     \not\leq_{\Sigmaquery}(\Imc_1,e)$.

%   \end{enumerate}
% $\Imc_{\Tmc_2,\Amc}$ is not $\Sigmaquery$-homomorphically embeddable into
% $\Imc_1$ if there is $a\in \ind(\A)$ such that (1) or (2) or (3) holds, where
% \begin{enumerate}
% \item[\rm (3)] there is an element $d$ in the subinterpretation of
%       $\Imc_{\Tmc_2,\Amc}$ rooted at $a$ (with possibly $d=a$) and $d$ has an
%     $R_0$-successor $d_0$, for some role name $R_0 \notin
%     \Sigmaquery$, such that for all elements $e$ of $\Imc_1$, we have
%     $(\Imc_{\Tmc_2,\Amc},d_0) \not\leq_{\Sigmaquery}(\Imc_1,e)$.
% \end{enumerate}
% \end{lemma}
\begin{lemma}
\label{lem:homtosim}
Let $\Sigma_{1}$ and $\Sigma_{2}$ be signatures, \Amc a
$\Sigmaabox$-ABox, and $\Imc_1$ a model of $(\Tmc_1,\Amc)$. Then

\smallskip
\noindent
(i)~$\Imc_{\Tmc_2,\Amc}$ is not con-$\Sigmaquery$-homomorphically embeddable into
$\Imc_1$ iff there is $a \in \mn{ind}(\Amc)$ such that one of the following
holds:
  \begin{enumerate}\itemsep 0cm

  \item[\rm (1)] there is a $\Sigmaquery$-concept name $A$ with
    $a \in A^{\Imc_{\Tmc_2,\Amc}} \setminus A^{\Imc_1}$;

  \item[\rm (2)] there is an $R$-successor $d$ of $a$ in
    $\Imc_{\Tmc_2,\Amc}$, for some $\Sigmaquery$-role name $R$, such
    that $d \notin \mn{ind}(\Amc)$ and, for all $R$-successors $e$
    of $a$ in $\Imc_1$, we have  $(\Imc_{\Tmc_2,\Amc},d)
    \not\leq_{\Sigmaquery}(\Imc_1,e)$.

  \end{enumerate}
$(ii)$ $\Imc_{\Tmc_2,\Amc}$ is not $\Sigmaquery$-homomorphically embeddable into
$\Imc_1$ if there is $a\in \ind(\A)$ such that {\rm (1)} or {\rm (2)} or {\rm (3)} holds, where
\begin{enumerate}
\item[\rm (3)] there is an element $d$ in the subinterpretation of
      $\Imc_{\Tmc_2,\Amc}$ rooted at $a$ \textup{(}with possibly $d=a$\textup{)} and $d$ has an
    $R_0$-successor $d_0$, for some role name $R_0 \notin
    \Sigmaquery$, such that 
    $(\Imc_{\Tmc_2,\Amc},d_0) \not\leq_{\Sigmaquery}(\Imc_1,e)$, for all elements $e$ of $\Imc_1$.
\end{enumerate}
\end{lemma}
\begin{proof} We only prove $(ii)$ as $(i)$ is a direct consequence
of our proof. Clearly, if there exists $a\in \ind(\A)$ such that (1) or (2) or (3) holds for $a$,
then there does not exist a $\Sigma$-homomorphism from $\I_{1}$ to $\Imc_{\T_{2},\A}$ preserving $\{a\}\subseteq \ind(\A)$.

Conversely, suppose none of (1), (2) or (3) holds for any $a \in
\mn{ind}(\Amc)$. Then, for any $a \in \mn{ind}(\Amc)$, $R$-successor $d$ of
$a$ in $\Imc_{\Tmc_2,\Amc}$ with $R \in \Sigmaquery$ and $d \notin
\mn{ind}(\Amc)$, there is an $R$-successor $d'$ of $a$ in $\Imc_1$ and a
$\Sigma_{2}$-simulation $\S_d$ from $\Imc_{\Tmc_2,\Amc}$ to $\Imc_1$ such that
$(d,d') \in \S_d$. As the subinterpretation of $\Imc_{\Tmc_2,\Amc}$ rooted at
$d$ is a ditree interpretation, we can assume that $\S_d$ is a partial
function. Also, for every $d_{0}$ in $\Imc_{\Tmc_{2},\A}$ with $d_{0}\not\in
\ind(\A)$ that has an $R_{0}$-predecessor in $\Imc_{\Tmc_{2},\A}$ with
$R_{0}\not\in \Sigma_{2}$, we find an $e$ in $\Imc_{1}$ such that there is a
$\Sigmaquery$-simulation $\mathcal{S}_{d_{0}}$ between $\Imc_{\Tmc_{2},\A}$ and
$\Imc_{1}$ with $(d_{0},e)\in \mathcal{S}_{d_{0}}$.  As the
subinterpretation of $\Imc_{\Tmc_2,\Amc}$ rooted at $d_{0}$ is ditree
interpretation, we can assume that $\S_{d_{0}}$ is a partial function.  Now
consider the function $h$ defined by setting $h(a)=a$, for all $a \in
\mn{ind}(\Amc)$, and then taking the union with all the simulations $\S_d$ and
$\S_{d_{0}}$. It can be verified that $h$ is a $\Sigmaquery$-homomorphism from
$\Imc_{\Tmc_2,\Amc}$ to $\Imc_1$.
\end{proof}

\section{Decidability of Query Entailment of \hALC TBoxes by \ALC TBoxes}

We show that the problem whether an \ALC TBox $\Theta$-CQ entails a
\hALC TBox is in 2\ExpTime, and that the complexity drops to \ExpTime in the case of rooted CQs.
Using the fact that satisfiability of \hALC TBoxes is \ExpTime-hard, it is straightforward to prove
a matching \ExpTime lower bound even for the full ABox signature case and $(\Sigma,\Sigma)$-rCQ entailment and inseparability between \hALC TBoxes.
Proving a matching lower bound for the non-rooted case is more involved. Using a reduction of exponentially space
bounded alternating Turing machines, 
we show that $(\Sigma,\Sigma)$-CQ inseparability between the empty TBox and \hALC TBoxes is 2\ExpTime-hard. It follows that both
$(\Sigma,\Sigma)$-CQ inseparability and $(\Sigma,\Sigma)$-CQ entailment between \hALC TBoxes are 2\ExpTime-hard. The problem whether the
2\ExpTime upper bound is tight in the full ABox signature case remains open.
%As Thus, how that $\ThetaAentailmIt is straightforward to show that
%We also show matching lower bound by provingthat $\Theta$-CQ inseparability between
%\hALWe start with establishing the
%upper bounds, first for rooted CQs and then for unrestricted CQs.

\subsection{\ExpTime upper bound for $\Theta$-rCQ-entailment of \hALC TBoxes by \ALC TBoxes}

% {\color{blue}
% We define the canonical model $\I_{\T,\A}$ of a consistent \hALC KB $\K=(\T,\A)$ with $\T$ in normal form in the standard way using
% a chase procedure. Consider the following rules that are applied to ABox $\Amc$:
% \begin{enumerate}
% \item if $A(a)\in \A$ and $A\sqsubseteq B\in\T$, then add $B(a)$ to $\A$;
% \item if $A_{1}(a)\in A$ and $A_{2}(a)\in \A$ and $A_{1} \sqcap A_{2} \sqsubseteq B\in \T$, then add $B(a)$ to $\A$;
% \item if $R(a,b)\in A$ and $A(b)\in \A$ and $\exists R.A \sqsubseteq B \in \T$, then add $B(a)$ to $\A$;
% \item if $a\in {\sf ind}(\A)$ and $\top \sqsubseteq B\in \T$, then add $B(a)$ to $\A$;
% \item if $A(a)\in \A$ and $A \sqsubseteq \exists R.B\in \T$ and there are no $R(a,b),B(b)\in \A$, then
% add assertions $R(a,b),B(b)$ to $\A$ for a fresh $b$;
% \item if $A(a)\in \A$ and $R(a,b)\in \A$ and $A \sqsubseteq \forall R.B\in \T$, then add $B(b)$ to $\A$.
% \end{enumerate}
% Denote by $\A^{c}$ the (possibly infinite) ABox resulting from $\A$ by applying these rules
% exhaustively to $\A$. Then the canonical model $\I_{\T,\A}$ is the interpretation defined by $\A^c$.
% }
%
%\bigskip
%
Our aim is to establish the following:
\begin{theorem}
\label{thm:exptbox}
$\Theta$-\RCQ inseparability between \hALC TBoxes and $\Theta$-\RCQ entailment of a \hALC TBox by an \ALC TBox are both \ExpTime-complete.
The \ExpTime lower bound holds already for $\Theta$ of the form $(\Sigma,\Sigma)$ and the full ABox signature case.
\end{theorem}
The lower bounds can be proved in a straightforward way using the fact that satisfiability of \hALC TBoxes is \ExpTime-hard.
Note that \ExpTime-hardness of $(\Sigma,\Sigma)$-rCQ inseparability is also inherited from \cite{LutzW10}, where
this bound is shown for $\mathcal{EL}$ TBoxes. It thus remains to prove the upper bound.

We use a mix of two-way alternating B\"uchi automata (\TWABAs) and
non-deterministic top-down tree automata (NTAs), both on \emph{finite}
trees (in contrast to Section~\ref{sect:2expupperKB}).  A finite tree
$T$ is \emph{$m$-ary} if, for any $x \in T$, the set
$\{ i \mid x \cdot i \in T \}$ is of cardinality zero or exactly~$m$.
\TWABAs on finite trees are defined exactly like \TWAPAs on infinite
trees except that
\begin{itemize}

\item[--] the acceptance condition now takes the form $F \subseteq Q$ and
  a run is accepting if, for every infinite path $y_1y_2\cdots$, the set
  $\{ i \mid r(y_i)=(x,q) \text{ with } q \in F \}$ is infinite;

\item[--] we allow a special transition $\mn{leaf}$ and add to
the definition of a run $r$ the condition that, for any node $y$ of
the input tree $T$, $r(y)=(x,\mn{leaf})$ implies that $x$ is a leaf in
$T$.

\end{itemize}
Note that runs can still be infinite.
\begin{definition}\em 
  A \emph{nondeterministic top-down tree automaton} (NTA) on finite
  $m$-ary trees is a tuple $\Amf=(Q, \Gamma, Q_0, \delta, F)$ where
  $Q$ is a finite set of \emph{states}, $\Gamma$ a finite alphabet,
  $Q_0 \subseteq Q$ a set of \emph{initial states},
  $\delta\colon Q \times \Gamma \rightarrow 2^{Q^m}$ a \emph{transition
    function}, and $F \subseteq Q$ is a set of \emph{final states}.
  Let $(T,L)$ be a $\Gamma$-labeled $m$-ary tree.  A \emph{run} of
  \Amf on $(T,L)$ is a $Q$-labeled $m$-ary tree $(T,r)$ such that
  $r(\varepsilon) \in Q_0$ and
  $\langle r(x \cdot 1), \ldots ,r(x \cdot m)\rangle \in \delta( r(x),
  L(x))$, for each node $x \in T$.
  The run is \emph{accepting} if $r(x) \in F$, for every leaf $x$ of
  $T$. The set of trees accepted by $\mathcal A$ is denoted by
  $L(\mathcal A)$.
\end{definition}
We use the following results from automata theory
\cite{Vardi98,MullerS87,ThomasHandbook97}.
\begin{theorem}~\\[-5mm]
  \begin{enumerate}

  \item Every \TWABA $\Amf=(Q,\Gamma,\delta,q_0,F)$ can be converted
    into an equivalent NTA $\Amf'$ whose number of states is \textup{(}single\textup{)}
    exponential in $|Q|$; the conversion needs time polynomial in the
    size of $\Amf'$;

  \item Given a constant number of \TWABAs \textup{(}respectively, NTAs\textup{)}
    $\Amf_1,\dots,\Amf_c$, one can construct in polynomial time a \TWABA
    \textup{(}respectively, an NTA\textup{)} $\Amf$ such that $L(\Amf)=L(\Amf_1) \cap \cdots
    \cap L(\Amf_c)$;

  \item Emptiness of NTAs $\Amf=(Q, \Gamma, Q_0, \delta,F)$ can be
    decided in polynomial time. %  when the number of priorities that occur
    % in $c$ is bounded by a constant (which shall always be the case).

  \end{enumerate}
\end{theorem}
Before proceeding further, we give a concrete
definition of the canonical model for \hALC KBs that was mentioned in Proposition~\ref{forestcomplete}, tailored towards
the constructions used in the rest of this section.  Let $\K =
(\T,\A)$ be a \hALC KB with \Tmc in normal form. We use $\mn{CN}(\T)$
to denote the set of concept names in $\T$. For any $a \in
\mn{ind}(\A)$, we use $\mn{tp}_\K(a)$ to denote the set $\{A \in
\mn{CN}(\Tmc) \mid \K \models A(a) \}$. For $t \subseteq
\mn{CN}(\Tmc)$, set $\mn{cl}_{\Tmc}(t) = \{ A \in \mn{CN}(\Tmc) \mid
\Tmc \models \midsqcap t \sqsubseteq A \}$.  A set $S = \{ \exists R
. A, \forall R . B_1, \dots, \forall R . B_n \}$ is a \emph{successor
  set for} $t$ if there is a concept name $A' \in t$ such that $A'
\sqsubseteq \exists R . A \in \Tmc$ and $\forall R . B_1, \dots,
\forall R . B_n$ is the set of all concepts of this form such that,
for some $B \in t$, we have $B \sqsubseteq \forall R . B_i \in
\Tmc$. Later on, we shall call $S$ a \emph{$\Sigmaquery$-successor
  set} if $R \in \Sigmaquery$. We use $S^\downarrow$ to denote the set
$\{A,B_1,\dots,B_n\}$.  A \emph{path for \Kmc} is a sequence $aS_1
\cdots S_n$ such that $a \in \mn{ind}(\Amc)$, $S_1$ is a successor set
for $\mn{tp}_\K(a)$, and $S_{i+1}$ is a successor set for
$\mn{cl}_{\Tmc}(S_i^\downarrow)$, for $1 \leq i < n$. Now, the
\emph{canonical
model $\I_{\K}$ of \Kmc} is defined as follows:
$$
\begin{array}{rcl}
  \Delta^{\I_\K} &=& \mn{ind}(\A) \cup \{ aS_1 \cdots S_n \mid
  aS_1 \cdots S_n \text{ path for } \K \}, \\
  A^{\I_{\K}} &=& \{ a \mid A \in \mn{tp}_{\Kmc}(a) \}
 \cup \{ aS_1 \cdots S_n \mid n \geq 1 \text{ and } A \in \mn{cl}_\Tmc(S_n^\downarrow) \},
 \\
 R^{\I_{\K}} &=&  \{ (a,b) \mid R(a,b) \in \A \} \cup
 % \{ (a, aS_1) \mid aS_1 \text{$R$ is the role name in } S_1 \} \; \cup
 % \\[1mm]
 % &&
 \{ (aS_1 \cdots S_{n-1},aS_1 \cdots S_n) \mid \text{$R$ is the role name in } S_n \}.
\end{array}
$$
The following result is standard:
\begin{lemma}
\label{lem:canmodprops}
  Let $\K=(\T,\A)$ be a \hALC KB in normal form. Then
  % %
  % \begin{enumerate}
%
  % \item $\{ A \mid a \in A^{\I_\K}\}=\mn{tp}_\K(a)$ for all $a \in
  %   \mn{ind}(\Amc)$;
  %
  % \item
    $\I_\K$ is a model of \Kmc iff \Kmc is consistent iff there is no
    $a \in \mn{ind}(\Amc)$ with $\Tmc \models \mn{tp}_{\K}(a) \sqsubseteq
    \bot$. 
%
%  \end{enumerate}
\end{lemma}
We now establish the upper bound in Theorem~\ref{thm:exptbox}.
Let $\Tmc_1$ be an \ALC TBox, $\Tmc_2$ a \hALC TBox, and
$\Sigmaabox,\Sigmaquery$ signatures.  Set $m = |\Tmc_2|$.  We aim to
construct an NTA \Amf such that a tree is accepted by \Amf iff this tree encodes a tree $\Sigmaabox$-ABox \Amc of outdegree at
most $m$ that is consistent with both $\Tmc_1$ and $\Tmc_2$ and a
(part of a) model $\Imc_1$ of $(\Tmc_1,\Amc)$ such that the canonical
model $\Imc_{\Tmc_2,\Amc}$ of $(\Tmc_2,\Amc)$ is not
con-$\Sigmaquery$-homomorphically embeddable into $\Imc_1$.  By
Theorem~\ref{thm:homcharhorn}, this means that \Amf accepts the empty
language iff $\Tmc_2$ is $(\Sigma_1,\Sigma_2)$-rCQ entailed
by $\Tmc_1$. To ensure that $\Imc_{\Tmc_2,\Amc}$ is not
con-$\Sigmaquery$-homomorphically embeddable into $\Imc_1$, we use
the characterisation provided by Lemma~\ref{lem:homtosim}.
We first make precise which trees should be accepted by
the NTA \Amf and then show how to construct \Amf.

We assume that $\Tmc_1$ takes the form $\top \sqsubseteq C_{\T_1}$
with $C_{\T_1}$ in NNF and use $\mathsf{cl}(C_{\T_1})$ to denote the
set of subconcepts of $C_{\T_1}$, closed under single negation.  We
also assume that $\Tmc_2$ is in normal form and use $\mn{sub}(\Tmc_2)$
for the set of subconcepts of (concepts in) $\Tmc_2$.  Let $\Gamma_0$
denote the set of all subsets of $\Sigmaabox \cup \{ R^- \mid R \in
\Sigmaabox \}$ that contain at most one role, where a \emph{role} is a
role name $R$ or its \emph{inverse} $R^-$.  Automata will run on $m$-ary
$\Gamma$-labeled trees where
$$
\Gamma ~=~ \Gamma_0 \times 2^{\mn{cl}(\Tmc_1)} \times 2^{\mn{CN}(\Tmc_2)}
\times \{ 0,1 \} \times 2^{\mn{sub}(\Tmc_2)}.
$$
For a $\Gamma$-labeled tree $(T,L)$ and a node $x$
from $T$, we write $L_i(x)$ to denote the $i+1$st component of $L(x)$, for each
$i \in \{0,\dots,4\}$. Informally, the projection of a $\Gamma$-labeled tree
to the
\begin{itemize}

\item[--] $L_0$-components %of its $\Gamma$-labels
represents the
  tree $\Sigmaabox$-ABox \Amc that witnesses non-$\Sigmaquery$-query
  entailment of $\Tmc_2$ by $\Tmc_1$;

\item[--] $L_1$-components (partially) represents a model $\Imc_1$ of
  $(\Tmc_1,\Amc)$;

\item[--] $L_2$-components (partially) represents the
  canonical model $\Imc_{\Tmc_2,\Amc}$ of $(\Tmc_2,\Amc)$;

\item[--] $L_3$-components mark the individual $a$ in \Amc from
  Lemma~\ref{lem:homtosim};

\item[--] $L_4$-components contains bookkeeping information that helps to
  ensure that the individual marked by the $L_3$-component indeed
  satisfies one of the two conditions from Lemma~\ref{lem:homtosim}.

\end{itemize}
By `partial' we mean that the restriction of the respective model
to individuals in \Amc is represented whereas its `anonymous' part
is not.  We now make these intuitions more precise by defining certain
properness conditions for $\Gamma$-labeled trees, one for each
component in the labels, which make sure that each component can
indeed be meaningfully interpreted to represent what it is supposed
to. A $\Gamma$-labeled tree $(T,L)$ is \emph{0-proper} if
%$L_0(\varepsilon)$ contains no role.
it satisfies the following conditions:
\begin{enumerate}

\item for the root $\varepsilon$ of $T$, $L_0(\varepsilon)$ contains no role;

\item for every non-root node $x$ of $T$, $L_0(x)$ contains a role.

\end{enumerate}
Every 0-proper $\Gamma$-labeled tree $(T, L)$ represents the tree
$\Sigmaabox$-ABox
$$
\Amc_{( T, L)} =  \{ A(x) \mid A \in L_0(x) \}
   \cup \{ R(x,y) \mid R \in L_0(y), y \text{ is a child of } x \}
   \cup \{ R(y,x) \mid  R^{-} \in L_0(y), y \text{ is a child of } x \}.
$$
A $\Gamma$-labeled tree $(T, L)$ is \emph{1-proper} if it
satisfies the following conditions, for all $x,y \in T$:
\begin{enumerate}

\item there is a model \Imc of $\Tmc_1$ and a $d \in \Delta^\Imc$ such
  that $d \in C^\Imc$ iff $C \in L_1(x)$ for all $C \in
  \mn{cl}(\Tmc_1)$;

%\item for every node $x$, $L_1(x)=\mn{cl}_\Tmc(L_1(x))$;

\item $A \in L_0(x)$ implies $A \in L_1(x)$;

\item if $y$ is a child of $x$ and $R \in L_{0}(y)$, then $\forall R . C
  \in L_1(x)$ implies $C \in L_1(y)$ for all $\forall R . C \in \mn{cl}(\Tmc_1)$;

\item if $y$ is a child of $x$ and $R^- \in L_{0}(y)$, then $\forall R . C
  \in L_1(y)$ implies $C \in L_1(x)$ for all $\forall R . C \in \mn{cl}(\Tmc_1)$.

\end{enumerate}
A $\Gamma$-labeled tree $(T, L)$ is \emph{2-proper} if, for every node
$x \in T$,
\begin{enumerate}

\item  $L_2(x)=\mn{tp}_{\T_{2},\Amc_{(T,L)}}(x)$;

\item $\Tmc_2 \not\models \bigsqcap L_2(x) \sqsubseteq \bot$.

% \item \emph{3-proper} if there is exactly one node $x$ with
%   $L_3(x)=1$.

\end{enumerate}
%
%A $\Gamma$-labeled tree $(T, L)$
It is \emph{3-proper} if there is
exactly one node $x$ with $L_3(x)=1$.

The canonical model $\Imc_{\Tmc_2,\Smc}$ of $\Tmc_2$ \emph{and a
  finite set} $\Smc \subseteq \mn{sub}(\Tmc_2)$ is the interpretation
obtained from the canonical model of the KB that consists of the TBox
$\Tmc_{2} \cup \{ A_C \sqsubseteq C \mid C \in \Smc \}$ and the ABox $
\{ A_C(a_\varepsilon) \mid C \in \Smc \}$, with all fresh concept
names $A_C$ removed.  A $\Gamma$-labeled tree $(T,L)$ is
\emph{4-proper} if % for all nodes $x$, $L_4(x)$ is either empty or of
% the form $\{\exists r . A, \forall r .B_1,\dots,\forall r.B_n\}$, and
the following
%additional
conditions hold, for $x_1,x_2 \in T$:
\begin{enumerate}

\item if $L_3(x_1)=1$, then there is a $\Sigmaquery$-concept name in $L_2(x_1)
  \setminus L_1(x_1)$ or $L_4(x_1)$ is
% a set
%   %
%   $$
%     \{ \exists r . C, \forall r . D_1, \cdots \forall r . D_n \}
%   $$
%   %
%   such that {\color{blue} $\Tmc \models L_2(x) \sqsubseteq \exists u . \midsqcap
%   L_4(x)$;}
a $\Sigmaquery$-successor set for~$L_2(x_1)$;

\item yif $L_4(x_1) = \{\exists R . A, \forall R.B_1,\dots,\forall
  R.B_n\}$, then there is a model \Imc of $\Tmc_1$ and a $d \in
  \Delta^\Imc$ such that $d \in C^\Imc$ iff $C \in L_1(x_1)$ for all
  $C \in \mn{cl}(\Tmc_1)$ and $(\Imc_{\Tmc_2,\{A,B_1,\dots,B_n\}},a_\varepsilon)
  \not\leq_{\Sigmaquery} (\Imc,e)$ for all $(d,e) \in R^\Imc$;

% \item if $L_3(x) \neq \emptyset$, then for all successor sets $S$ from
%   $L_1(x)$, there is {\color{blue}no $\Sigma$-simulation from
%     $\Imc_{\Tmc_2,L_3(x)}$ to
%     $\Imc_{\Tmc_1,S}$;\footnote{\color{blue}here and elsewhere: fix
%       notation}} {\color{red}CANONICAL!!!}

\item if $x_2$ is a child of $x_1$, $L_0(x_2)$ contains the role name
  $R$, and $L_4(x_1)= \{ \exists R . A, \forall R . B_1, \dots,\forall
  R . B_n \}$, then
  %
%  \begin{enumerate}
%  \item
there is a $\Sigmaquery$-concept name in
  $\mn{cl}_{\Tmc_2}(\{A,B_1,\dots,B_n\}) \setminus L_1(x_2)$ or
%  \item
$L_4(x_2)$ is
  a $\Sigmaquery$-successor set for $\mn{cl}_{\Tmc_2}(\{A,B_1,\dots,B_n\})$;
%  \end{enumerate}

\item if $x_2$ is a child of $x_1$, $L_0(x_2)$ contains the role
  $R^-$, and $L_4(x_2)= \{ \exists R . A, \forall R . B_1,
  \dots,\forall R . B_n \}$, then
  %
%  \begin{enumerate}
%
%  \item
there is a $\Sigmaquery$-concept name in
  $\mn{cl}_{\Tmc_2}(\{A,B_1,\dots,B_n\}) \setminus L_1(x_1)$ or
%
 % \item
$L_4(x_1)$ is
  a $\Sigmaquery$-successor set for $\mn{cl}_{\Tmc_2}(\{A,B_1,\dots,B_n\})$.
%  \end{enumerate}
%

% \item if $L_4(x_1)=\{ \exists r . A, \forall r . B_1,\dots, \forall r
%   . B_n \}$, then $x_1$ has an $r$-successor in $\Amc_{(T,L)}$.

\end{enumerate}
For $L_4(x) = \{\exists R . A, \forall R .B_1,\dots,\forall
  R.B_n\}$, this expresses the obligation that
$(\Imc_{\Tmc_2,\{A,B_1,\dots,B_n\}},a_\varepsilon) \not\leq_{\Sigmaquery}
(\Imc,e)$, for $(d,e) \in R^\Imc$, where \Imc is the interpretation that is (partly) represented by the
$L_1$-components of the labels in $(T,L)$; see the proof of
Lemma~\ref{lem:gammattrees} for a precise definition of~$\Imc$.  With
this in mind, note how 4-properness addresses (1) and (2) of
Lemma~\ref{lem:homtosim}. In fact, Condition~1 of 4-properness decides
whether (1) or (2) is satisfied. If (2) is satisfied,
which says that there is an $R$-successor $d$ of $x_1$ in
$\Imc_{\Tmc_2,\Amc}$, for some $\Sigmaquery$-role name $R$, such that
$d \notin \mn{ind}(\Amc)$ and, for all $R$-successors $e$ of $x_1$ in
$\Imc$, we have
$(\Imc_{\Tmc_2,\Amc},d) \not\leq_{\Sigmaquery}(\Imc,e)$, then the
role name $R$ and the element $d$ are represented by the successor set
stored in $L_4(x_1)$. In fact, that element is $d=x_1L_4(x_1)$,
see the definition of canonical models. The remaining conditions of
4-properness implement the obligations represented by the
$L_4$-components of node labels.
\begin{lemma}
\label{lem:gammattrees}
  % $\Tmc_1$ does not $\Sigmaquery$-CQ entail $\Tmc_2$ iff there is an
  % $m$-ary $\Gamma$-labeled tree that is $i$-proper for all $i \in
  % \{0,\dots,4\}$.
There is an $m$-ary $\Gamma$-labeled tree that is $i$-proper for all
$i \in \{0,\dots,4\}$ iff there are a tree $\Sigmaabox$-ABox
$\Amc$ of outdegree at most $m$ that is consistent with $\Tmc_1$
and $\Tmc_2$ and a model $\Imc_1$ of
$(\Tmc_1,\Amc)$ such that the canonical model $\Imc_{\Tmc_2,\Amc}$ of
$(\Tmc_2,\Amc)$ is not con-$\Sigmaquery$-homomorphically embeddable
into $\Imc_1$.
\end{lemma}
\begin{proof}
$(\Rightarrow)$ Let $(T,L)$ be an $m$-ary $\Gamma$-labeled tree that is
  $i$-proper for all $i \in \{0,\dots,4\}$. Then $\Amc_{(T,L)}$ is a
  tree $\Sigmaabox$-ABox of outdegree at most $m$.  Moreover,
  $\Amc_{(T,L)}$ is consistent with $\Tmc_2$, by 2-properness and
  Lemma~\ref{lem:canmodprops}.

  Since $(T,L)$ is $3$-proper, there is exactly one $x_0 \in T$ with
  $L_3(x_0)=1$. By construction, $x_0$ is also an individual name in
  $\Amc_{(T,L)}$.  To finish this direction of the proof, it suffices
  to construct a model $\Imc_1$ of $(\Tmc_1,\Amc_{(T,L)})$ such that
  $(\Imc_{\Tmc_2,\Amc},x_0) \not\leq_{\Sigmaquery} (\Imc_1,x_0)$.  In
  fact, such an $\Imc_1$ witnesses consistency of $\Amc_{(T,L)}$ with
  $\Tmc_1$ and, moreover, by the definition of simulations, $\Imc_1$ must
  satisfy one of (1) or (2) of Lemma~\ref{lem:homtosim} with $a$
  replaced by $x_0$. Consequently, by that lemma, $\Imc_{\Tmc_2,\Amc}$
  is not con-$\Sigmaquery$-homomorphically embeddable into $\Imc_1$.

  We start with the interpretation $\Imc_0$ defined as follows:
  $$
  \begin{array}{r@{~}c@{~}l}
    \Delta^{\Imc_0} &=& T, \\
    A^{\Imc_0} &=& \{ x \in T \mid A \in L_1(x) \}, \\ 
    R^{\Imc_0} &=& \{ (x_1,x_2) \mid x_2 \text{ child of } x_1 \text{
                   and } R \in L_0(x_2) \} \, \cup{} %\\[1mm]
    %&& 
    \{ (x_2,x_1) \mid x_2 \text{ child of } x_1 \text{
                   and } R^- \in L_0(x_2) \}.
  \end{array}
  $$
  Then take, for each $x \in T$, a model $\Imc_x$ of $\Tmc_{1}$ such
  that $x \in C^{\Imc_x}$ iff $C \in L_1(x)$ for all $C \in
  \mn{cl}(\Tmc_1)$, which exists by Condition~1 of 1-properness.
  Moreover, if $L_4(x) = \{\exists R . A, \forall R .B_1,\dots,\forall
  R.B_n\}$, then choose $\Imc_x$ such that
  $(\Imc_{\Tmc_2,\{A,B_1,\dots,B_n\}},a_\varepsilon)
  \not\leq_{\Sigmaquery} (\Imc_x,y)$ for all $(x,y) \in R^{\Imc_x}$,
  which is possible by Condition~2 of 4-properness. Further, suppose 
  $\Delta^{\Imc_0}$ and $\Delta^{\Imc_x}$ share only the element
  $x$. Then $\Imc_1$ is the union of $\Imc_0$ and all chosen
  interpretations $\Imc_x$.  It is straightforward to prove that
  $\Imc_1$ is indeed a model of $(\Tmc_1,\Amc_{(T,L)})$.

  We show that $(\Imc_{\Tmc_2,\Amc_{(T,L)}},x_0)
  \not\leq_{\Sigmaquery} (\Imc_1,x_0)$.
  By Condition~1 of 4-properness, there is a $\Sigmaquery$-concept name
  $A$ in $L_2(x_0) \setminus L_1(x_0)$ or $L_4(x_0)$ is a
  $\Sigmaquery$-successor set for $L_2(x_0)$. In the former case, $x_0
  \in A^{\Imc_{\Tmc_2,\Amc_{(T,L)}}} \setminus A^{\Imc_1}$, and so we
  are done. In the latter case, it suffices to show the following.
  \\[2mm]
  {\em Claim.} For all $x \in T$, if $L_4(x) = \{ \exists R . A,
  \forall R . B_1, \dots, \forall R . B_n \}$, then
    $(\Imc_{\Tmc_2,\{A,B_1,\dots,B_n\}},a_\varepsilon)
    \not\leq_{\Sigmaquery}(\Imc_1,y)$ for all $(x,y) \in R^{\Imc_1}$.
  \\[2mm]
  The proof of the claim is by induction on the co-depth of $x$ in
  $\Amc_{(T,L)}$, which is the length $n$ of the longest sequence of
  role assertions $R_1(x,x_1),\dots,R_n(x_{n-1},x_n)$ in
  $\Amc_{(T,L)}$. It uses Conditions~2 to~4 of 4-properness.
  % \nb{not clear why we need to prove this claim if we already chose
  % $\I_{x_0}$ properly; clu: because there are additional successors
  % of $x_0$ that come from $I_0$.}

  $(\Leftarrow)$ Let \Amc be a tree $\Sigmaabox$-ABox of outdegree at most $m$ that is
  consistent with $\Tmc_1$ and $\Tmc_2$, and $\Imc_1$ a model of
  $(\Tmc_1,\Amc)$ such that $\Imc_{\Tmc_2,\Amc}$ is not
  con-$\Sigmaquery$-homomorphically embeddable into $\Imc_1$. By duplicating
  successors, we can make sure that every non-leaf in \Amc has exactly $m$
  successors. We can further assume without loss of generality that $\mn{ind}(\Amc)$ is a
  prefix-closed subset of $\mathbbm{N}^*$ that reflects the tree-shape of \Amc,
  that is, $R(a,b) \in \Amc$ implies $b=a\cdot c$ or $a = b \cdot c$, for some
  $c \in \mathbbm{N}$.  By Lemma~\ref{lem:homtosim}, there is an $a_0 \in
  \mn{ind}(\Amc)$ such that one of the following holds:
  \begin{enumerate}

  \item[(1)] there is a $\Sigmaquery$-concept name $A$ with
    $a_0 \in A^{\Imc_{\Tmc_2,\Amc}} \setminus A^{\Imc_1}$;

  \item[(2)] there is an $R_0$-successor $d_0$ of $a_0$ in
    $\Imc_{\Tmc_2,\Amc}$, for some $\Sigmaquery$-role name $R_0$, such
    that $d_0 \notin \mn{ind}(\Amc)$ and, for all $R_0$-successors $d$
    of $a_0$ in $\Imc_1$, we have  $(\Imc_{\Tmc_2,\Amc},d_0)
    \not\leq_{\Sigmaquery}(\Imc_1,d)$.

  \end{enumerate}
  We now show how to construct from \Amc a $\Gamma$-labeled tree
  $(T,L)$ that is $i$-proper for all $i \in \{0,\dots,4 \}$.  For each
  $a \in \mn{ind}(\Amc)$, set $\Rsf(a) = \emptyset$ if $a=\varepsilon$,
  and otherwise set $\Rsf(a) = \{ R\}$ if $R(b,a) \in \Amc$ and $a= b
  \cdot c$, for some $c \in \mathbbm{N}$, and $\Rsf(a) = \{ R^- \}$ if
  $R(a,b) \in \Amc$ and $a= b \cdot c$, for some $c \in \mathbbm{N}$.
  Now set
  $$
  \begin{array}{rcl}
    T &=& \mn{ind}(\Amc) ,\\
    L_0(x) &=& \{ A \mid A(x) \in \Amc \} \cup \{\Rsf(x) \},
    \\
    L_1(x) &=& \{ C \in \mn{cl}(\Tmc_1) \mid x \in C^{\Imc_1} \},
    \\
    L_2(x) &=& \mn{tp}_{\Tmc_{2},\Amc}(x), \\
    L_3(x) &=& \left \{
               \begin{array}{rl}
                 1 & \text{ if } x=a_0, \\
                 0 & \text{ otherwise.}
               \end{array}
               \right .
  \end{array}
  $$
  It remains to define $L_4$.  Start with setting $L_4(x)=\emptyset$ for
  all~$x$. If (1) above holds, we are done. If (2) holds, then
  there is a $\Sigmaquery$-successor set $\Smc =\{ \exists R_0 . A, \forall R_0
  . B_1,\dots,\forall R_0 . B_n\}$ for $L_2(a_0)$ such that the restriction of
  $\Imc_{\Tmc_2,\Amc}$ to the subtree-interpretation rooted at $d_0$ is the
  canonical model $\Imc_{\Tmc_{2},\{A,B_1,\dots,B_n\}}$. Set $L_4(a_0)=\Smc$. We
  continue to modify $L_4$, proceeding in rounds. To keep track of the
  modifications that we have already done, we use a set
  $$\Omega \subseteq
  \mn{ind}(\Amc) \times (\NR \cap \Sigmaquery) \times
  \Delta^{\Imc_{\Tmc_2,\Amc}}
   $$
   such that the following conditions are satisfied:
   \begin{itemize}

   \item[(i)] if $(a,R,d) \in \Omega$, then $L_4(a)$ has the form $\{ \exists R
     . A, \forall R . B_1,\dots,\forall R . B_n\}$ and the restriction of
     $\Imc_{\Tmc_2,\Amc}$ to the subtree-interpretation rooted at $d$ is the
     canonical model $\Imc_{\Tmc_{2},\{A,B_1,\dots,B_n\}}$;

   \item[(ii)] if $(a,R,d) \in \Omega$ and $d'$ is an $R$-successor of
     $a$ in $\Imc_1$, then $(\Imc_{\Tmc_2,\Amc},d)
     \not\leq_{\Sigmaquery} (\Imc_1,d')$.

   % \item[($*$)] if $(a,r,d) \in \Omega$ and $b$ is an $r$-successor of
   %   $a$ in \Amc, then $(\Imc_{\Tmc_2,\Amc},d)$ is not
   %   {\color{blue}$\Sigma$-simulated by} $(\Imc_1,b)$.

   \end{itemize}
   Initially, set $\Omega = \{ (a_0,R_0,d_0) \}$.  In each round of
   the modification of $L_4$, iterate over all elements $(a,R,d) \in
   \Omega$ that have not been processed in previous rounds. Let
   $L_4(a)= \{ \exists R . A, \forall R . B_1,\dots,\forall R . B_n\}$
   and iterate over all $R$-successors $b$ of $a$ in \Amc. By (ii),
   $(\Imc_{\Tmc_2,\Amc},d) \not\leq_{\Sigmaquery} (\Imc_1,b)$. By (i),
   there is thus a top-level $\Sigmaquery$-concept name $A'$ in
   $\mn{cl}_{\Tmc_2}(\{A,B_1,\dots,B_n\})$ such that $b \notin
   {A'}^{\Imc_1}$ or there is an $R'$-successor $d'$ of $d$ in
   $\Imc_{\Tmc_2,\Amc}$, $R'$ a $\Sigmaquery$-role name, such that for
   all $R'$-successors $d''$ of $b$ in $\Imc_1$,
   $(\Imc_{\Tmc_2,\Amc},d') \not\leq_{\Sigmaquery} (\Imc_1,d'')$. In
   the former case, we do nothing. In the latter case, there is a
   $\Sigmaquery$-successor set $S'=\{ \exists R' . A', \forall R'
   . B'_1,\dots,\forall R' . B'_{n'}\}$ for
   $\mn{cl}_{\Tmc_2}(\{A,B_1,\dots,B_n\})$ such that the restriction
   of $\Imc_{\Tmc_2,\Amc}$ to the subtree-interpretation rooted at
   $d'$ is the canonical model
   $\Imc_{\Tmc_{2},\{A',B'_1,\dots,B'_{n'}\}}$. Set $L_4(b)=S'$ and add
   $(b,R',d')$ to $\Omega$.

   Since we are only following role names (but not inverse roles)
   during the modification of $L_4$ and since \Amc is tree-shaped, we
   shall never process tuples $(a_1,R_1,d_1), (a_2,R_2,d_2)$ from
   $\Omega$ such that $a_1=a_2$. For any $x$, we might thus only
   redefine $L_4(x)$ from the empty set to a non-empty set, but never
   from one non-empty set to another. For the same reason, the
   definition of $L_4$ finishes after finitely many rounds.

   It can be verified that the $\Gamma$-labeled tree $(T,L)$ just
   constructed is $i$-proper for all $i \in \{0,\dots,4\}$. The
   most interesting point is 4-properness, which consists of four
   conditions. Condition~1 is satisfied by the construction of $L_4$.
   Condition~2 is satisfied by (ii), and Conditions~3 and~4 again
   by the construction of $L_4$.
\end{proof}

By Theorem~\ref{thm:homcharhorn} and Lemma~\ref{lem:gammattrees}, we
can decide whether $\Tmc_1$ does $(\Sigma_1,\Sigma_2)$-rCQ entail
$\Tmc_2$ by checking whether there is no $\Gamma$-labeled tree that is
$i$-proper for each $i \in \{0,\dots,4\}$. We do this by constructing
automata $\Amc_0,\dots,\Amc_4$ such that each $\Amc_i$ accepts exactly
the $\Gamma$-labeled trees that are $i$-proper, then intersecting the
automata and finally testing for emptiness. Some of the constructed
automata are \TWABAs while others are NTAs. Before intersecting, all
\TWABAs are converted into equivalent NTAs (which involves an
exponential blowup). To achieve \ExpTime overall complexity, the
constructed \TWABAs should thus have at most polynomially many states, 
while the NTAs can have at most (single) exponentially many
states. % \footnote{\color{blue}This is all a bit
  % sloppy; but do we really want to give full details on the sizes of
  % transition relations etc? just looks messy, but is not very
  % insightful...}
It is straightforward to construct
\begin{enumerate}

\item[--] an NTA $\Amf_0$ that checks 0-properness and has constantly many
  states;

\item[--] a \TWABA $\Amf_1$ that checks 1-properness and whose number of
  states is polynomial in $|\Tmc_1|$ (note that Conditions~1 and~2 of
  1-properness are in a sense trivial as they could also be guaranteed
  by removing undesired symbols from the alphabet $\Gamma$;
  %
  % one state for each concept \forall r. C. Alternatively, directly build
  % an NTA with one state for each set of such concepts

\item[--] an NTA $\Amf_3$ that checks 3-properness and has constantly many
  states.

\end{enumerate}
%
%Details are omitted. 
It thus remains to construct
\begin{enumerate}

\item[--] a \TWABA $\Amf_2$ that checks 2-properness and whose number
  of states is polynomial in $|\Tmc_2|$;

\item[--] an NTA $\Amf_4$ that checks 4-properness and whose number of
states is (single) exponential in $|\Tmc_2|$.

\end{enumerate}
In fact, the reason for mixing \TWABAs and NTAs is that while $\Amf_2$
is easier to construct as a \TWABA, there is no obvious way to
construct $\Amf_4$ as a \TWABA with only polynomially many states: it
seems that one state is needed for every possible value of the
$L_4$-components in $\Gamma$-labels.
The \TWABA $\Amf_2$ is actually the intersection of two \TWABAs
$\Amf_{2,1}$ and $\Amf_{2,2}$. The \TWABA $\Amf_{2,1}$ ensures one
direction of Condition~1 of 2-properness as well as Condition~2,
that is:
\begin{itemize}

\item[(i)] %if $\Amc_{(T,L)}$ is consistent w.r.t.\ $\Tmc_2$, then
 $(\Tmc_{2},\Amc_{(T,L)}) \models A(x)$ implies $A \in L_2(x)$ for all $x
  \in T$ and $A \in \mn{CN}(\Tmc_2)$;

\item[(ii)] $\Tmc_{2} \not\models \bigsqcap L_2(x) \sqsubseteq \bot$.

\end{itemize}
Note that, by Lemma~\ref{lem:canmodprops}, (i) and (ii) imply that
$\Amc_{(T,L)}$ is consistent with $\Tmc_2$. It is easy for a \TWABA
to verify~(ii), alternatively one can simply refine $\Gamma$.  To
achieve~(i), it suffices to guarantee the following conditions, for 
$x_1,x_2 \in T$:
\begin{enumerate}

\item[--] $A \in L_0(x_1)$ implies $A \in L_2(x_1)$;

\item[--] if $A_1,\dots,A_n \in L_2(x_1)$ and $\Tmc_2 \models A_1
  \sqcap \cdots \sqcap A_n \sqsubseteq A$, then $A \in L_2(x_1)$;

\item[--] if $A \in L_2(x_1)$, $x_2$ is a successor of $x_1$,
  $R \in L_0(x_2)$, and $A \sqsubseteq \forall R . B \in \Tmc_2$,
  then $B \in L_2(x_2)$;

\item[--] if $A \in L_2(x_2)$, $x_2$ is a successor of $x_1$,
  $R^- \in L_0(x_2)$, and $A \sqsubseteq \forall R . B \in \Tmc_2$,
  then $B \in L_2(x_1)$;

\item[--] if $A \in L_2(x_2)$, $x_2$ is a successor of $x_1$,
  $R \in L_0(x_2)$, and $\exists R . A \sqsubseteq B \in \Tmc_2$,
  then $B \in L_2(x_1)$;

\item[--] if $A \in L_2(x_1)$, $x_2$ is a successor of $x_1$,
  $R^- \in L_0(x_2)$, and $\exists R . A \sqsubseteq B \in \Tmc_2$,
  then $B \in L_2(x_2)$,

\end{enumerate}
all of which are easily verified with a \TWABA. Note that Conditions~1 and~2 can again be ensured by
refining $\Gamma$.

The purpose of $\Amf_{2,2}$ is to ensure the converse of (i). Before
constructing it, it is convenient to characterise the entailment
of concept names at ABox individuals in terms of derivation trees.
%  and let $\Tmc_\exists$ be the restriction of
% $\Tmc_2$ to CIs of the form $\exists r . A \sqsubseteq B$, with $r$ a
% (potentially inverse) role.  {\color{red}???}
A \emph{$\Tmc_2$-derivation tree} for an assertion $A_0(a_0)$ in \Amc
with $A_0 \in \mn{CN}(\Tmc_2)$ is a finite $\mn{ind}(\Amc) \times
\mn{CN}(\Tmc_2)$-labeled tree $(T,V)$ that satisfies
the following conditions:
\begin{itemize}

\item[--] $V(\varepsilon)=(a_0,A_0)$;

%\item if %$x \neq \varepsilon$ and
%$V(x)=(a,\bot)$, then $x=\varepsilon$;

% \item if $V(x)=(a,A)$ and $A(a) \in \Amc$ or $\top \sqsubseteq A \in
%   \Tmc$, then $x$ is a leaf;

% \item if $V(x)=(a,A)$, then $x$ has one successor $y$ with $V(y)=(a,\Gamma)$
%   for some $\Gamma$ with $\Tmc \models \Gamma \sqsubseteq A$;

\item[--] if $V(x)=(a,A)$ and neither $A(a) \in \Amc$ nor \mbox{$\top
  \sqsubseteq A \in \Tmc_2$}, then one of the following holds:
  \begin{itemize}

  \item $x$ has successors $y_1,\dots,y_n$ with
    $V(y_i)=(a,A_i)$, for $1 \leq i \leq n$, and $\Tmc_2 \models A_1
    \sqcap \cdots \sqcap A_n \sqsubseteq A$;
    % (then $x$ is of  \emph{local type});

  \item $x$ has a single successor $y$ with $V(y)=(b,B)$ and there is
    an $\exists R . B \sqsubseteq A \in \Tmc_2$ such that
    $R(a,b) \in \Amc$;
    % (then $x$ is of \emph{existential type}).

  \item $x$ has a single successor $y$ with $V(y)=(b,B)$ and
    there is a $B \sqsubseteq \forall R . A \in \Tmc_2$ such
    that $R(b,a) \in \Amc$.

  \end{itemize}

\end{itemize}
%
% Note that there might be more than one derivation tree for the same
% assertion $A_0(a_0)$. Also note that derivations trees are essentially
% the same thing as proof trees for datalog.
%We call a TBox \Tmc \emph{satisfiable} if it has a model.
%The main property of derivation trees is the following.
%
\begin{lemma}
\label{lem:derivationtrees}
%~\\[-4mm]
  %
  %\begin{enumerate}
  %
  % \item $\Amc,\Tmc_2 \models A(a)$ iff \Amc is inconsistent w.r.t.\
  %   $\Tmc_2$ or there is a derivation tree for $A(a)$ in \Amc, for all
  %   assertions $A(a)$ with $A \in \NC$ and $a \in \mn{ind}(\Amc)$;
  %
  % \item \Amc is inconsistent w.r.t.\ $\Tmc_2$ iff $\Tmc_2$ is
  %   unsatisfiable or there is a derivation tree for $\bot(a)$ in \Amc,
  %   for some $a \in \mn{ind}(\Amc)$.
  %
  %\item
If $(\Tmc_2 , \Amc) \models A(a)$ and \Amc is consistent with
    $\Tmc_2$, then there is a derivation tree for $A(a)$ in \Amc, for
    all assertions $A(a)$ with $A \in \mn{CN}(\Tmc_2)$ and $a \in \mn{ind}(\Amc)$.
%
  % \item \Amc is inconsistent w.r.t.\ $\Tmc_2$ iff $\Tmc_2$ is
  %   unsatisfiable or  there is a derivation tree for $\bot(a)$ in \Amc,
  %   for some $a \in \mn{ind}(\Amc)$.
  %
  %
  % \end{enumerate}
\end{lemma}

(A proof of Lemma~\ref{lem:derivationtrees} is based on the chase
procedure, details can be found in \cite{BienvenuLW-ijcai-13}.) 
We are now ready to construct the \TWABA~$\Amf_{2,2}$. Since $\Amf_{2,1}$ ensures
that $\Amc_{(T,L)}$ is consistent with $\Tmc_2$, by
Lemma~\ref{lem:derivationtrees} it is enough for
$\Amf_{2,2}$ to verify that, for each node $x \in T$ and each concept
name $A \in L_2(x)$, there is a $\Tmc_2$-derivation tree for $A(x)$ in
$\Amc_{(T,L)}$.

For readability, we use $\Gamma^- = \Gamma_0 \times \mn{CN}(\Tmc_2)$ as
the alphabet instead of $\Gamma$ since transitions of $\Amf_{2,2}$
only depend on the $L_0$- and $L_2$-components of $\Gamma$-labels. Let
$\mn{rol}(\Tmc_2)$ be the set of all roles $R,R^-$ such that the role
name $R$ occurs in $\Tmc_2$.  Set $\Amf_2=(Q,\Gamma^-,\delta,q_0,F)$, where 
  $
  %\begin{array}{rcl}
  Q = \{
  q_0 \} \uplus \{ q_A \mid A \in \mn{CN}(\Tmc_2) \} %\cup \{ \bot \} \}
   \uplus \{ q_{A,R}, q_R \mid A \in \mn{CN}(\Tmc_2), %\cup \{ \bot \},
     R \in \mn{rol}(\Tmc_2) \}
  %\end{array}
  $
  and $F=\emptyset$ (i.e., exactly the finite runs are accepting).
% For all $(\sigma_0,\sigma_2) \in \Gamma^-$ such that $C \in \sigma_2$
%   and $D \notin \sigma_2$ for some $C \sqsubseteq D \in \Tmc_2$, set
%   %
%   $$
%     \delta(q_0,(\sigma_0,\sigma_2)) = \bot.
%   $$
  %
  For all $(\sigma_0,\sigma_2) \in \Gamma^-$, set
  $$
  \begin{array}{rcll}
    \delta(q_0,(\sigma_0,\sigma_2)) &=& \displaystyle \bigwedge_{A
      \in \sigma_2} (0,q_A) \wedge (\mn{leaf} \vee \bigwedge_{i \in
                                        1..m} (i,q_0)), \\
    \delta(q_A,(\sigma_0,\sigma_2))&=&\mn{true}, & \text{whenever } A
                                                  \in \sigma_0 \text{
                                                  or } \top
  \sqsubseteq A \in \Tmc_2, \\ 
\delta(q_A,(\sigma_0,\sigma_2))&=&
\displaystyle
    \bigvee_{\Tmc_2 \models A_1 \sqcap \cdots \sqcap A_n \sqsubseteq A}
    %\!\!\!\!\!\!\!\!
    ((0,q_{A_1}) \wedge \cdots \wedge (0,q_{A_n})) \vee{}
                                    &\text{whenever } A \notin
                                      \sigma_0 \text{ and } \top \sqsubseteq A \notin \Tmc_2,
\\
&&\displaystyle
    \bigvee_{\exists R . B \sqsubseteq A \in \Tmc, \ R \in \Sigmaabox}
    (((0,q_{R^-}) \wedge (-1,q_{B})) \vee
\displaystyle
    \bigvee_{i \in 1..m} (i,q_{B,R})) \vee{}\\[5mm]
&&\displaystyle
    \bigvee_{B \sqsubseteq \forall R . A \in \Tmc, \ R \in \Sigmaabox}
    ((0,q_{R}) \wedge (-1,q_{B})) \vee \bigvee_{i \in 1..m} (i,q_{B,{R^-}})),
\\
\delta(q_{A,R},(\sigma_0,\sigma_2))&=& (0,q_A), &\text{whenever } R \in
                                       \sigma_0, \\
\delta(q_{A,R},(\sigma_0,\sigma_2))&=& \mn{false}, &\text{whenever } R
                                       \notin \sigma_0, \\
\delta(q_{R},(\sigma_0,\sigma_2))&=& \mn{true},  &\text{whenever } R \in
                                     \sigma_0 ,\\
\delta(q_{R},(\sigma_0,\sigma_2))&=&\mn{false}, & \text{whenever } R \notin \sigma_0.
\end{array}
$$
Note that the finiteness of runs ensures that $\Tmc_2$-derivation
  trees are also finite, as required.

We next discuss the construction of the NTA $\Amf_4$, omitting most of the details because the construction is not difficult. 
Conditions~1 and~2 of 4-properness can be enforced by making sure that
certain symbols from $\Gamma$ do not occur. However, in the case of
Condition~2, we have to decide during the automaton construction
whether, for given sets $S_1 \subseteq \mn{cl}(\Tmc_1)$ and
$S_2=\{\exists R_0 . A, \forall R_0 .B_1,\dots,\forall R_0.B_n\} \subseteq \mn{sub}(\Tmc_2)$, there is a model $\Imc$ of $\Tmc_1$
and a $d \in \Delta^\Imc$ such that
\begin{enumerate}

\item[(a)] $d \in C^\Imc$ iff $C \in S_1$ for all $C \in \mn{cl}(\Tmc_1)$
  and

\item[(b)] $(\Imc_{\Tmc_2,S_2^\downarrow},a_\varepsilon)
  \not\leq_{\Sigma_2} (\Imc,e)$ for all $(d,e) \in R_0^\Imc$.

\end{enumerate}
We have to show that this check can be done in \ExpTime.  We
give a sketch of a decision procedure based on nondeterministic
B\"uchi automata on infinite trees that borrows ideas from
the above constructions, but is much simpler.
\begin{definition}
\label{def:NBA}
\em 
  A \emph{nondeterministic B\"uchi tree automaton} (NBA) on infinite $m$-ary
  trees is a tuple $\Amf=(Q, \Gamma, Q_0, \delta,F)$ where $Q$ is a
  finite set of \emph{states}, $\Gamma$  a finite alphabet,
  $Q_0 \subseteq Q$ a set of \emph{initial states},
  $\delta\colon Q \times \Gamma \rightarrow 2^{Q^m}$ a \emph{transition
    function}, and % $c: Q \to \Nbbm$ is a function that assigns priorities
  % to the states.
  $F \subseteq Q$ is an \emph{acceptance condition}.
  Let $(T,L)$ be a $\Gamma$-labeled $m$-ary tree.  A \emph{run} of
  \Amf on $(T,L)$ is a $Q$-labeled $m$-ary tree $(T,r)$ such that
  $r(\varepsilon) \in Q_0$ and
  $\langle r(x \cdot 1), \ldots ,r(x \cdot m)\rangle \in \delta( r(x),
  L(x))$, for each $x \in T$.
  We say that $(T,r)$ is \emph{accepting} if in all infinite paths
  $ y_1y_2\cdots$ of $T$, the set $\{ i \mid r(y_i) \in F \}$ is
  infinite.  An infinite $\Gamma$-labeled tree $(T,L)$ is
  \emph{accepted} by \Amf if there is an accepting run of \Amf on
  $(T,L)$. We use $\L(\Amf)$ to denote the set of all infinite
  $\Gamma$-labeled trees accepted by \Amf.
\end{definition}
The emptiness problem for NBAs can be solved in polynomial time.  Our
aim is to build an NBA \Bmf such that the labeled trees accepted by
\Bmf represent tree interpretations \Imc that satisfy Conditions~(a)
and~(b). We make precise which trees should be accepted by \Bmf. Let
$\Gamma'_0$ be the set of all subsets of $\mn{cl}(\Tmc_1) \cup \{
R \in \NR \mid R \text{ occurs in } \Tmc_1 \}$ that contain at most
one role name and let $\Gamma'=(\Gamma'_0 \times 2^{\mn{sub}(\Tmc_2)}) \cup
\{\mathit{empty}\}$. For a $\Gamma'$-labeled tree $(T,L)$ and a
node $x$ in $T$ with $L(x) \neq \mathit{empty}$, we write $L_i(x)$
to denote the $i+1$st component of $L(x)$, for  $i \in
\{0,1\}$. Informally, the projection of a $\Gamma'$-labeled tree to
the $L_0$-components %of its $\Gamma$-labels
represents \Imc and the projection to the $L_1$-components contains
bookkeeping information that helps to ensure Condition~(b).  A
$\Gamma'$-labeled tree is \emph{proper} if the following
conditions hold, for $x_1,x_2 \in T$:
\begin{enumerate}

\item[--] $L(\varepsilon)=(S_1,S_2)$;

% \item[--] for all successors $x$ of $\varepsilon$ in $T$ with $R_0 \in
%   L_0(x)$, there is a concept name $A \in \mn{cl}_{\Tmc_2}(S_2^\downarrow)
%   \setminus L_0(x)$ or $L_1(x)$ is a
%   $\Sigma_2$-successor set of $S_2^\downarrow$;

\item[--] if $L(x_1) \neq \mathit{empty}$, then $L_0(x_1)$ is satisfiable with $\Tmc_1$;

\item[--] if $x_2$ is a child of $x_1$ and $R \in L_0(x_2)$, then $\forall R . C
  \in L_0(x_1)$ implies $C \in L_0(x_2)$ for all $\forall R . C \in \mn{cl}(\Tmc_1)$;

% \item if $x_2$ is a child of $x_1$ and $R^- \in L_0(x_2)$, then $\forall R . C
%   \in L_0(x_2)$ implies $C \in L_0(x_1)$ for all $\forall R . C \in \mn{cl}(\Tmc_1)$;

\item[--] if $\exists R . C \in L_0(x_1)$, then there is a child $x_2$ of
  $x_1$ such that $\{R,C\} \subseteq L_0(x_2)$;

\item[--] if $x_2$ is a child of $x_1$ and $L(x_1)=\mathit{empty}$,
  then $L(x_2)=\mathit{empty}$;

\item[--] if $x_2$ is a child of $x_1$, $L_0(x_2)$ contains the role name
  $R$, and $L_1(x_1)= \{ \exists R . A, \forall R . B_1, \dots,\forall
  R . B_n \}$, then
  %
%  \begin{enumerate}
%  \item
there is a $\Sigmaquery$-concept name in
  $\mn{cl}_{\Tmc_2}(\{A,B_1,\dots,B_n\}) \setminus L_0(x_2)$ or
%  \item
$L_1(x_2)$ is
  a $\Sigmaquery$-successor set for $\mn{cl}_{\Tmc_2}(\{A,B_1,\dots,B_n\})$;
%  \end{enumerate}

\item[--] there are only finitely many nodes $x$ with $L_1(x) \neq \emptyset$.

\end{enumerate}
In the conditions above, we assume that whenever a condition is posed
on a component of the label of a node $x$, then
$L(x) \neq \mathit{empty}$. Note that the $L_1$-component of a node
label plays the same role as the $L_4$-component in the previous construction.  Every proper $\Gamma'$-labeled tree $(T, L)$
represents the following tree interpretation $\Imc_{(T,L)}$:
$$
\begin{array}{rcl}
   \Delta^{\Imc_{(T,L)}} &=& \{ x \in T \mid
                             L(x) \neq \mathit{empty} \}, \\
  A^{\Imc_{(T,L)}} &=& \{ x \mid A \in L_0(x) \}, \\
  R^{\Imc_{(T,L)}} &=& \{ (x_1,x_2) \mid \text{$x_2$ child of $x_1$
                       and } R \in L_0(x_2) \}.
\end{array}
$$
Set $m'=|\Tmc_1|$. The proof of the following lemma is similar to that
of Lemma~\ref{lem:gammattrees}, but simpler. 
\begin{lemma}
  There is an $m'$-ary proper $\Gamma'$-labeled tree $(T,L)$ iff there
  is a model \Imc of $\Tmc_1$ and a $d \in \Delta^\Imc$ that satisfy
  Conditions~$(a)$ and~$(b)$ from before Definition~\ref{def:NBA}; in
  fact, $\Imc_{(T,L)}$ is such a model.
\end{lemma}
It is now straightforward to construct an NBA \Bmf whose number of
states is polynomial in $|\Tmc_1|$ and exponential in $|\Tmc_2|$ and
which accepts exactly the $m'$-ary proper $\Gamma'$-labeled trees.
Details are left to the reader.

%
%   The remaining automaton $\Amf_{2,3}$ has to ensure that
%   $\Amc_{(T,L)}$ is consistent w.r.t.\ $\Tmc_2$. As already mentioned,
%   it is rather similar to $\Amf_{2,2}$ and we only sketch the
%   differences. We add a state $q_\bot$ and replace the initial
%   transition with
%   %
%   $$
%   \delta(q_0,(\sigma_0,\sigma_2)) = \displaystyle q_\bot \vee \bigvee_{i \in 1..m} (i,q_0).
%   $$
% {\color{red}stop at leaves}
%   %
% The resulting automaton ensures that there is some node $x \in T$ such
% that there is a $\Tmc_2$-derivation tree of $\bot(a)$ in
% $\Amc_{(T,L)}$. As $\Amf_{2,3}$, we use its complementation.
% Note that we can assume w.l.o.g.\ that $\Tmc_2$ is satisfiable
% since this can be checked in \ExpTime and, otherwise, $\Tmc_1$
% trivially $(\Sigma_1,\Sigma_2)$-entails $\Tmc_2$. Consequently,
% by Point~2 of Lemma~\ref{lem:derivationtrees} the automaton
% $\Amf_{2,3}$ verified consistency of $\Amc_{(T,L)}$ with $\Tmc_2$,
% as desired.

\subsection{2\ExpTime upper bound for $\Theta$-CQ-entailment of \hALC TBoxes by \ALC TBoxes}

We now consider the case of non-rooted CQs. Our aim is to prove the following 2\ExpTime upper bound:
\begin{theorem}
\label{thm:2exptbox:upper}
$\Theta$-CQ entailment of \hALC TBoxes by \ALC TBoxes is in 2\ExpTime.
%-complete to decide whether an \ALC TBox $\Theta$-CQ entails a Horn-\ALC TBox.
%The lower bound
%holds for $\Theta = (\Sigma,\Sigma)$.
%
%It is 2\ExpTime-complete to decide whether a Horn-\ALC TBox
%$(\Sigma_1,\Sigma_2)$-CQ entails a Horn-\ALC TBox.  The lower bound
%already holds when $\Sigma_1 = \Sigma_2$.
\end{theorem}
%
%We prove the upper bound by modifying the construction given in the
%previous section while deferring the lower bound to the subsequent
%section.
The proof again builds on the characterisations provided by
Theorem~\ref{thm:homcharhorn}. Since we are now working with CQs rather
than rCQs, we have to consider $\Sigma_2$-homomorphic embeddability
instead of con-$\Sigma_2$-homomorphic embeddability. Note that
Lemma~\ref{lem:homtosim} also provides a characterisation in terms of
simulations in that case, adding a third condition. We modify the
previous construction to accommodate this additional condition.

Condition (2) of Lemma~\ref{lem:homtosim} tells us to avoid certain
simulations. In the previous construction, we were able to do that by
storing a single successor set in the $L_4$-component of each
$\Gamma$-label, that is, it was sufficient to avoid at most one
simulation into each individual of the ABox $\Amc_{(T,L)}$. In the
current construction, this is no longer the case. We thus let the
$L_4$-component of $\Gamma$-labels range over
$2^{2^{\mn{sub}(\Tmc_2)}}$ rather than $2^{\mn{sub}(\Tmc_2)}$ and
use it to store \emph{sets of} successor sets. To address (3) in
Lemma~\ref{lem:homtosim}, we add an $L_5$-component to
$\Gamma$-labels, which also ranges over
$2^{2^{\mn{sub}(\Tmc_2)}}$. The purpose of this component is to
represent elements of the canonical model $\Imc_{\Tmc_2,\Amc}$ from
which we have to avoid a simulation into \emph{any} individual in
$\Amc_{(T,L)}$ and, in fact, into any element of the interpretation
(partially) represented by the $L_2$-components of node labels.  The
notion of $i$-properness remains the same for $i \in \{0,1,2,3\}$. We
adapt the notion of 4-properness and add a notion of 5-properness.

As a preliminary, we define a notion of $\Sigmaquery$-descendant
set. While a $\Sigmaquery$-successor set for $t \subseteq
\mn{CN}(\Tmc_2)$ represents a $\Sigma_2$-successor of an element $d$
in a canonical model $\Imc_{\Tmc_2,\Amc}$ that satisfies $d \in
A^{\Imc_{\Tmc_2,\Amc}}$ for all $A \in t$, a $\Sigmaquery$-descendent
set represents a \emph{descendent} of such a $d$ that is attached to
its predecessor via a role name that is \emph{not} in $\Sigmaquery$,
as in (3) of Lemma~\ref{lem:homtosim}.  Formally, for $t \subseteq
\mn{CN}(\Tmc_2)$, we define $\Gamma_t$ to be the smallest set such
that $t \in \Gamma_t$ and if $t' \in \Gamma_t$ and $S$ is a successor
set for $\mn{cl}_{\Tmc_2}(t')$, then $S^\downarrow \in \Gamma_t$.  A
set $s \subseteq \mn{CN}(\Tmc_2)$ is a \emph{$\Sigmaquery$-descendant
  set} for $t$ if there is a $t' \in \Gamma_t$ and successor set
$S=\{\exists R . A, \forall R . B_1, \dots, \forall R . B_n \}$ for
$\mn{cl}_{\Tmc_2}(t')$ with $R\not\in\Sigma_{2}$ such that $s=S^\downarrow$.

A $\Gamma$-labeled tree $(T,L)$ is \emph{4-proper} if % all for all
% nodes $x$, the element of $L_4(x)$ are of the form $\{\exists r . A,
% \forall r .B_1,\dots,\forall r.B_n\}$, and
the following % additional
conditions are satisfied for all $x_1,x_2 \in T$:
\begin{enumerate}

\item[--] if $L_3(x_1)=1$, then one of the following holds:
  \begin{itemize}
  \item[--]
  there is a $\Sigmaquery$-concept name in $L_2(x_1)
  \setminus L_1(x_1)$;

  \item[--] $L_4(x_1)$ contains
% a set
%   %
%   $$
%     \{ \exists r . C, \forall r . D_1, \cdots \forall r . D_n \}
%   $$
%   %
%   such that {\color{blue} $\Tmc \models L_2(x) \sqsubseteq \exists u . \midsqcap
%   L_4(x)$;}
  a $\Sigma_2$-successor set for~$L_2(x_1)$;

\item[--] $L_5(x_1)$ contains a $\Sigma_2$-descendant set for~$L_2(x_1)$;

  \end{itemize}

\item[--] there is a model \Imc of $\Tmc_1$ and a $d \in \Delta^\Imc$ such
  that the following hold:
  \begin{itemize}

  \item[--] $d \in C^\Imc$ iff $C \in L_1(x_1)$, for all $C \in
    \mn{cl}(\Tmc_1)$;

  \item[--] if $\{\exists R . A, \forall R .B_1,\dots,\forall
  R.B_n\} \in
    L_4(x_1)$ and $(d,e) \in R^\Imc$, then
    $(\Imc_{\Tmc_2,\{A,B_1,\dots,B_n\}},a_\varepsilon) \not\leq_{\Sigmaquery} (\Imc,e)$;

  \item[--] if $s \in L_5(x_1)$ and $e \in \Delta^\Imc$, then
    $(\Imc_{\Tmc_2,s},a_\varepsilon)
    \not\leq_{\Sigmaquery} (\Imc,e)$;

  \end{itemize}

% \item if $L_3(x) \neq \emptyset$, then for all successor sets $S$ from
%   $L_1(x)$, there is {\color{blue}no $\Sigma$-simulation from
%     $\Imc_{\Tmc_2,L_3(x)}$ to
%     $\Imc_{\Tmc_1,S}$;\footnote{\color{blue}here and elsewhere: fix
%       notation}} {\color{red}CANONICAL!!!}

\item[--] if $x_2$ is a child of $x_1$, $L_0(x_2)$ contains the role name
  $R$, and $L_4(x_1) \ni \{ \exists R . A, \forall R . B_1, \dots,\forall
  R . B_n \}$, then
  %
%  \begin{enumerate}
%  \item
there is a $\Sigmaquery$-concept name in
  $\mn{cl}_{\Tmc_2}(\{A,B_1,\dots,B_n\}) \setminus L_1(x_2)$ or
%  \item
$L_4(x_2)$ contains
  a $\Sigmaquery$-successor set for $\mn{cl}_{\Tmc_2}(\{A,B_1,\dots,B_n\})$;
%  \end{enumerate}

\item [--]if $x_2$ is a child of $x_1$, $L_0(x_2)$ contains the role
  $R^-$, and $L_4(x_2) \ni \{ \exists R . A, \forall R . B_1,
  \dots,\forall R . B_n \}$, then
  %
%  \begin{enumerate}
%
%  \item
there is a $\Sigmaquery$-concept name in
  $\mn{cl}_{\Tmc_2}(\{A,B_1,\dots,B_n\}) \setminus L_1(x_1)$ or
%
 % \item
$L_4(x_1)$ contains
  a $\Sigmaquery$-successor set for $\mn{cl}_{\Tmc_2}(\{A,B_1,\dots,B_n\})$.
%  \end{enumerate}
%

% \item if $L_4(x_1)=\{ \exists r . A, \forall r . B_1,\dots, \forall r
%   . B_n \}$, then $x_1$ has an $r$-successor in $\Amc_{(T,L)}$.

\end{enumerate}
A $\Gamma$-labeled tree $(T,L)$ is \emph{5-proper} if the following
conditions are satisfied for all $x_1 \in T$:
\begin{enumerate}

\item[--] all $x \in T$ agree regarding their $L_5$-label;

\item[--] if $s \in L_5(x_1)$, then one of the following holds:
  \begin{itemize}
  \item[--]
  there is a $\Sigmaquery$-concept name in $s
  \setminus L_1(x_1)$;

  \item[--] $L_4(x_1)$ contains
% a set
%   %
%   $$
%     \{ \exists r . C, \forall r . D_1, \cdots \forall r . D_n \}
%   $$
%   %
%   such that {\color{blue} $\Tmc \models L_2(x) \sqsubseteq \exists u . \midsqcap
%   L_4(x)$;}
  a $\Sigma_2$-successor set for~$s$.

\end{itemize}

\end{enumerate}
Note that 4-properness and 5-properness together implement (2) and (3)
of Lemma~\ref{lem:homtosim}; in particular, Point~(3) from
Lemma~\ref{lem:homtosim} requires that $(\Imc_{\Tmc_2,\Amc},d_0)
\not\leq_{\Sigmaquery}(\Imc_1,e)$ for any element $e$ of
$\mathcal{I}_1$ which can be broken down into the two cases
above.

The proof of the following lemma is
similar to that of Lemma~\ref{lem:derivationtrees}: 
\begin{lemma}
  % $\Tmc_1$ does not $\Sigmaquery$-CQ entail $\Tmc_2$ iff there is an
  % $m$-ary $\Gamma$-labeled tree that is $i$-proper for all $i \in
  % \{0,\dots,4\}$.
There is an $m$-ary $\Gamma$-labeled tree that is $i$-proper for all
$i \in \{0,\dots,5\}$ iff there is a tree $\Sigmaabox$-ABox
$\Amc$ of outdegree at most $m$ that is consistent with $\Tmc_1$
and $\Tmc_2$ and a model $\Imc_1$ of
$(\Tmc_1,\Amc)$ such that the canonical model $\Imc_{\Tmc_2,\Amc}$ of
$(\Tmc_2,\Amc)$ is not $\Sigmaquery$-homomorphically embeddable
into $\Imc_1$.
\end{lemma}
We can now adapt the automata construction presented in the
previous section. It is straightforward to construct an NTA $\Amf_5$
with double exponentially many states that verifies 5-properness. Also, 
the NTA $\Amf_4$ for 4-properness will now have double exponentially
many states because $L_4$- and $L_5$-components are sets of sets of
concepts rather than sets of concepts. In fact, we could dispense with NTAs
altogether and use a \TWABA that has exponentially many states, both
for $\Amf_4$ and $\Amf_5$. The construction of $\Amf_4$ needs to
decide whether, for given sets $S_1 \subseteq \mn{cl}(\Tmc_1)$ and
$S_2,S_3 \subseteq 2^{\mn{CN}(\Tmc_2)}$, there is a model $\Imc$ of
$\Tmc_1$ and a $d \in \Delta^\Imc$ such that
\begin{enumerate}

\item[(a)] $d \in C^\Imc$ iff $C \in S_1$, for all $C \in \mn{cl}(\Tmc_1)$; 

\item[(b)] $(\Imc_{\Tmc_2,S},a_\varepsilon) \not\leq_{\Sigma_2}
  (\Imc,d)$ for all $S \in S_2$;

\item[(c)] $(\Imc_{\Tmc_2,S},a_\varepsilon) \not\leq_{\Sigma_2}
  (\Imc,e)$ for all $S \in S_3$ and $e \in \Delta^\Imc$;

\end{enumerate}
This check can be implemented in 2\ExpTime using
a decision procedure based on NBAs, mixing ideas from the
corresponding construction in the previous section and the
construction above.  Overall, we obtain the 2\ExpTime upper bound stated
in Theorem~\ref{thm:2exptbox:upper}.

\subsection{2\ExpTime lower bound for $\Theta$-CQ-inseparability between \hALC TBoxes}

We prove a matching lower bound for the 2\ExpTime upper bound established in Theorem~\ref{thm:2exptbox:upper} using a
reduction of the word problem of exponentially space bounded ATMs (see Section~\ref{sect:first2explower}). More precisely, we show the
following:
\begin{theorem}\label{thm:lower2exp}
%
% A \emph{computation tree} of an ATM $M$ on input $w$ is a tree whose
% nodes are labeled with configurations of $M$ on $w$, such that the
% descendants of any non-leaf labeled by a universal (resp.\
% existential) configuration include all (resp. one) of the successors
% of that configuration. A computation tree is \emph{accepting} if the
% root is labeled with the \emph{initial configuration} $q_0w$ for $w$
% and all leaves with accepting configurations. An ATM $M$ accepts
% input $w$ if there is a computation tree of $M$ on $w$.
%
$(\Sigma,\Sigma)$-CQ inseparability between the empty TBox and \hALC TBoxes is 2\ExpTime-hard.
\end{theorem}
Note that we obtain a 2\ExpTime lower bound for $\Theta$-CQ entailment
as well since, clearly, the empty TBox $(\Sigma,\Sigma)$-CQ-entails a
TBox $\Tmc$ iff the empty TBox and $\Tmc$ are
$(\Sigma,\Sigma)$-CQ-inseparable.  Let
$M=(Q,\Gamma_{I},\Gamma,q_{0},\Delta)$ be an exponentially space
bounded ATM whose word problem is {\sc 2Exp\-Time}-hard, where $Q$ is
the finite set of \emph{states}, $\Gamma_I$ the \emph{input alphabet},
$\Gamma \supseteq \Gamma_I$ the \emph{tape alphabet} with \emph{blank
  symbol} $\Box \in \Gamma \setminus \Gamma_I$, $q_0 \in Q$ the
\emph{initial state}, and $\Delta \subseteq Q \times \Gamma \times Q
\times \Gamma \times \{L,R\}$ the \emph{transition relation}. We use
$\Delta(q,\sigma)$ to denote the set of \emph{transitions}
$(q',\sigma',D) \in Q \times \Gamma \times \{L,R\}$ possible when $M$
is in state $q$ and reads $\sigma$, that is, $(q,\sigma,q',\sigma',D)
\in \Delta$. We may assume that the length of every computation path of
$M$ on $w \in \Sigma^n$ is bounded by $2^{2^{n}}$, and all the
configurations $wqw'$ in such computation paths satisfy $|ww'| \leq
2^{n}$ (see \cite{ChandraKS81}). % We may
% also assume w.l.o.g.\ that $M$ never attempts to move left
% on the left-most tape cell.
To simplify the reduction, we may also assume without loss of
generality that $M$ makes at least
one step on every input, that it never reaches the last tape cell, and
that every universal configuration has exactly two successor
configurations. 

Note that when $M$ accepts an input $w$, this is
witnessed by an \emph{accepting computation tree} whose nodes are
labeled with configurations such that the root is labeled with the
initial configuration of $M$ on $w$, the descendants of any non-leaf
labeled with a universal (respectively, existential) configuration
include all (respectively, one) of the successors of that
configuration, and all leafs are labeled with accepting
configurations.

Let $w$ be an input to $M$.  We aim to construct a \hALC
TBox $\Tmc$ and a signature $\Sigma$ such that $M$ accepts~$w$ iff there is a tree $\Sigma$-ABox \Amc\xspace such that
  \begin{itemize}
  \item[(a)] \Amc is consistent with $\Tmc$ and

  \item[(b)] $\Imc_{\Tmc,\Amc}$ is
    not $\Sigma$-homomorphically embeddable into
    $\Imc_{\Tmc_\emptyset,\Amc}$,

  \end{itemize}
where $\Tmc_\emptyset = \emptyset$.  Note that this is
equivalent to $(\Sigma,\Sigma)$-CQ-entailment of $\Tmc$ by
$\Tmc_\emptyset$ due to Theorem~\ref{thm:homcharhorn}~(2);
that theorem additionally imposes a restriction on the outdegree of
\Amc, but it is easy to go through the proofs and verify that the
characterisation holds also without that restriction.
We are going to construct $\Tmc$ and $\Sigma$ such that \Amc
represents an accepting computation tree of $M$ on $w$.

When dealing with an input $w$ of length $n$, in \Amc we represent
configurations of $M$ by a sequence of $2^n$ elements linked by the
role name $R$, from now on called \emph{configuration
  sequences}. These sequences are then interconnected to form a
representation of the computation tree of $M$ on $w$. This is
illustrated in Fig.~\ref{fig:red3}, which shows three configuration
sequences, enclosed by dashed boxes. The topmost configuration is
universal, and it has two successor configurations. All solid arrows
denote $R$-edges. We shall see at the very end of the reduction why
successor configurations are separated by two consecutive edges
instead of a single one.

\begin{figure}[tb]
  \centering
    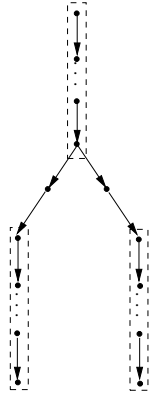
    \caption{Configuration tree (partial).}
    \label{fig:red3}
\end{figure}

The above description is an oversimplification. In fact, every
configuration sequence stores two configurations instead of only one:
the current configuration and the previous configuration in the
computation. We will later use the homomorphism condition~(b) above to
ensure that
\begin{itemize}

\item[($*$)] the previous configuration stored in a configuration
sequence is identical to the current configuration stored in its
predecessor configuration sequence.

\end{itemize}
The actual transitions of $M$ are
then enforced locally inside configuration sequences.

The signature $\Sigma$ consists of the following
symbols:
\begin{enumerate}

\item[--] the concept names $A_0,\dots,A_{n-1},\overline{A}_0,\dots,\overline{A}_{n-1}$ that serve as bits in the
  binary representation of a number between 0 and $2^{n}-1$,
  identifying the position of tape cells inside configuration
  sequences ($A_0$, $\overline{A}_0$ are the lowest bit);

\item[--] the concept names $A'_0,\dots,A'_{m-1}$ and
  $\overline{A}'_0,\dots,\overline{A}'_{m-1}$, where $m = \lceil
  \mn{log}(2^n+2) \rceil$, that serve as bits of another counter which
  is able to count from 0 to $2^n+2$ and whose purpose will be explained
  later;

\item[--] the concept names $A_\sigma$, $A'_\sigma$, $\overline{A}_\sigma$, 
  for each $\sigma \in \Gamma$;

\item[--] the concept names $A_{q,\sigma}$, $A'_{q,\sigma}$, $\overline{A}_{q,\sigma}$,  for
  each $\sigma \in \Gamma$ and $q \in Q$;

% \item the concept names $\overline{H}$, $W$, and $\overline{W}$
%  which stand for ``cell without head'', ``cell being written to reach
% current configuration'', and ``cell not being written to reach current
% configuration'';

% \item concept names $C$ and $P$ which stand for ``current''
%    and ``previous'';

% \item the concept name $A_{q,\sigma,M}$ for each $q \in Q$, $\sigma \in
%    \Gamma$, and $M \in \{L,R\}$ to describe transitions of $M$;

%\item a concept name $I$ that marks the initial configuration;

\item[--] the concept names $X_1,X_2$ that mark the first and second successor
  configuration;

\item[--] the role name $R$.

\end{enumerate}
From the above list, the concept names $A_\sigma$ and $A_{q,\sigma}$ are
used to represent the current configuration and $A'_\sigma$ and
$A'_{q,\sigma}$ for the previous configuration. The role of the
concept names $\overline{A}_\sigma$ and $\overline{A}_{q,\sigma}$ will
be explained later.

It thus remains to construct the TBox $\Tmc$, which is the most
laborious part of the reduction. We use \Tmc to verify the existence
of a computation tree of $M$ on input $w$ in the ABox. For the time
being, we are going to assume that ($*$) holds and, in a second step,
we will demonstrate how to actually achieve that.
\setcounter{equation}{0}
We start with verifying halting configurations, which must all be
accepting in an accepting computation tree, in a bottom-up manner:
%
%$$
%\begin{array}l@{}r@{\,}c@{\,}l}
\begin{eqnarray}
 A_0 \sqcap \cdots \sqcap A_{n-1} \sqcap A_\sigma \sqcap A'_{\sigma}
  &\sqsubseteq& V ,\\%[1mm]
  % A_0 \sqcap \cdots \sqcap A_{n-1} \sqcap A_\sigma \sqcap A'_{q,\sigma'}
  % &\sqsubseteq& V_{L,\sigma} \\%[1mm]
  % A_0 \sqcap \cdots \sqcap A_{n-1} \sqcap A_{q_\mn{acc},\sigma} \sqcap A'_\alpha
  % &\sqsubseteq& V_{R,q_{\mn{acc}}} \\%[1mm]
  A_i \sqcap \exists R . A_i \sqcap \midsqcup_{j < i} \exists R
  . A_j &\sqsubseteq& \mn{ok}_i , \\%[1mm]
  \overline{A}_i \sqcap \exists R . \overline{A}_i \sqcap \midsqcup_{j < i} \exists R
  . A_j &\sqsubseteq& \mn{ok}_i  ,\\%[1mm]
  A_i \sqcap \exists R . \overline{A}_i \sqcap \midsqcap_{j < i} \exists R
  . \overline{A}_j &\sqsubseteq& \mn{ok}_i , \\%[1mm]
  \overline{A}_i \sqcap \exists R . A_i \sqcap \midsqcap_{j < i} \exists R
  . \overline{A}_j &\sqsubseteq& \mn{ok}_i , \\%[1mm]
  \mn{ok}_0 \sqcap  \cdots \sqcap \mn{ok}_{n-1} \sqcap \overline{A}_i \sqcap \exists R . V
  \sqcap A_\sigma \sqcap A'_{\sigma} &\sqsubseteq& V ,\\%[1mm]
  \mn{ok}_0 \sqcap  \cdots \sqcap \mn{ok}_{n-1} \sqcap \overline{A}_i \sqcap \exists R . V
  \sqcap A_\sigma \sqcap A'_{q,\sigma'} &\sqsubseteq& V_{L,\sigma}, \\%[1mm]
  \mn{ok}_0 \sqcap  \cdots \sqcap \mn{ok}_{n-1} \sqcap \overline{A}_i \sqcap \exists R . V
  \sqcap A_{q_a,\sigma} \sqcap A'_{\sigma} &\sqsubseteq& V_{R,q_a} ,\\%[1mm]
  \mn{ok}_0 \sqcap  \cdots \sqcap \mn{ok}_{n-1} \sqcap \overline{A}_i \sqcap \exists R . V_{L,\sigma}
  \sqcap A_{q_a,\sigma'} \sqcap A'_{\sigma'} &\sqsubseteq& V_{L,q_a,\sigma}, \\%[1mm]
  \mn{ok}_0 \sqcap  \cdots \sqcap \mn{ok}_{n-1} \sqcap \overline{A}_i \sqcap \exists R . V_{R,q_a}
  \sqcap A_{\sigma} \sqcap A'_{q,\sigma'} &\sqsubseteq& V_{R,q_a,\sigma} ,\\%[1mm]
\hspace*{2cm}
  \mn{ok}_0 \sqcap  \cdots \sqcap \mn{ok}_{n-1} \sqcap \overline{A}_i \sqcap \exists R . V_{D,q_a,\sigma}
  \sqcap A_{\sigma'} \sqcap A'_{\sigma'} &\sqsubseteq& V_{D,q_a,\sigma} ,\\%[1mm]
  \exists R . A_i \sqcap \exists R . \overline{A}_i &\sqsubseteq& \bot,
\end{eqnarray}
%\end{array}
%$$
%
where $\sigma,\sigma'$ range over $\Gamma$, $q$ over $Q$, $i$ over
$0,\dots, n-1$, and $D$ over $\{L,R\}$.  The first line starts the
verification at the last tape cell, ensuring that at least one concept
name $A_\sigma$ and one concept name $A'_\sigma$ is true (it also
verifies that the symbol is identical in the current and previous
configuration, assuming ($*$); it is here that the assumption that $M$
never reaches the last tape cell makes the construction easier). The
following lines implement the verification of the remaining tape cells
of the configuration. Lines (2)--(5) implement decrementation of a
binary counter and the conjunct $\overline{A}_i$ in lines (6)--(11)
prevents the counter from wrapping around once it has reached 0. We
use several kinds of verification markers:
\begin{itemize}

\item[--] with $V$, we indicate that we have not yet seen the head of the
  ATM;

\item[--] $V_{L,\sigma}$ indicates that the ATM made a step to the left to
  reach the current configuration, writing $\sigma$;

\item[--] $V_{R,q}$ indicates that the ATM made a step to the right to
  reach the current configuration, switching to state $q$;

\item[--] $V_{D,q,\sigma}$ indicates that the ATM moved in direction $D$
to reach the current configuration, switching to state $q$ and writing
$\sigma$.

\end{itemize}
In the remaining reduction, we expect that a marker 
$V_{D,q,\sigma}$ has been derived at the first (thus top-most) cell of
the configuration.  This makes sure that there is exactly one head in
the current and previous configuration, and that the head moved
exactly one step between the previous and current position. Also
note that the above CIs ensure that the tape content does not change
for cells that were not under the head in the previous configuration,
assuming ($*$). % We exploit that $M$ never moves its head to the
% right-most tape cell, simply ignoring this case in the CIs above.
Note
that it is not immediately clear that lines (2)--(11) work as
intended since they can speak about different $R$-successors for
different bits. The last line fixes this problem.
We also ensure that relevant concept names are mutually exclusive:
%
%$$
%\begin{array}{rcll}
\begin{eqnarray}
  A_i \sqcap \overline{A}_i &\sqsubseteq& \bot, \\%[1mm]
  A_{\sigma_1} \sqcap A_{\sigma_2} &\sqsubseteq& \bot, \qquad \text{ if }
  \sigma_1 \neq \sigma_2, \\%[1mm]
  A_{\sigma_1} \sqcap A_{q_2,\sigma_2} &\sqsubseteq& \bot, \\%[1mm]
\hspace*{4cm}
  A_{q_1,\sigma_1} \sqcap A_{q_2,\sigma_2} &\sqsubseteq& \bot, \qquad \text{ if }
  (q_1,\sigma_1) \neq (q_2,\sigma_2),
\end{eqnarray}
%\end{array}
%$$
%
where $i$ ranges over $0,\dots,n-1$, $\sigma_1,\sigma_2$ over $\Gamma$,
and $q_1,q_2$ over~$Q$. We also add the same CIs for
the primed versions of these concept names.  The next step is to
verify non-halting configurations:
%
% $$
% \begin{array}{@{}r@{\,}c@{\,}l}
\begin{eqnarray}
\hspace*{2.5cm}
  \exists R. \exists R . (X_1 \sqcap \overline{A}_0 \sqcap \cdots \sqcap
  \overline{A}_{n-1} \sqcap (V_{D,q,\sigma} \sqcup V'_{D,q,\sigma})) &\sqsubseteq& L\mn{ok}, \\%[1mm]
  \exists R . \exists R . (X_2  \sqcap \overline{A}_0 \sqcap \cdots \sqcap
  \overline{A}_{n-1}\sqcap (V_{D,q,\sigma} \sqcup V'_{D,q,\sigma})) &\sqsubseteq& R\mn{ok} ,\\%[1mm]
  A_0 \sqcap \cdots \sqcap A_{n-1} \sqcap A_\sigma \sqcap A'_{\sigma}
  \sqcap L\mn{ok} \sqcap R\mn{ok}   &\sqsubseteq& V', \\%[1mm]
  \mn{ok}_0 \sqcap  \cdots \sqcap \mn{ok}_{n-1} \sqcap \overline{A}_i \sqcap \exists R . V'
  \sqcap A_\sigma \sqcap A'_{\sigma} &\sqsubseteq& V', \\%[1mm]
  \mn{ok}_0 \sqcap  \cdots \sqcap \mn{ok}_{n-1} \sqcap \overline{A}_i \sqcap \exists R . V'
  \sqcap A_\sigma \sqcap A'_{q,\sigma'} &\sqsubseteq& V'_{L,\sigma} ,\\%[1mm]
  % \mn{ok}_0 \sqcap  \cdots \sqcap \mn{ok}_{n-1} \sqcap \overline{A}_i \sqcap \exists R . V'
  % \sqcap A_{q,\sigma} \sqcap A'_{\sigma'} &\sqsubseteq& V'_{R,q} \\%[1mm]
  % \mn{ok}_0 \sqcap  \cdots \sqcap \mn{ok}_{n-1} \sqcap \overline{A}_i \sqcap \exists R . V_{L,\sigma}
  % \sqcap A_{q},\sigma'} \sqcap A'_{\sigma''} &\sqsubseteq& V'_{L,q,\sigma} \\%[1mm]
  \mn{ok}_0 \sqcap  \cdots \sqcap \mn{ok}_{n-1} \sqcap \overline{A}_i \sqcap \exists R . V'_{R,q}
  \sqcap A_{\sigma} \sqcap A'_{q',\sigma'} &\sqsubseteq& V'_{R,q,\sigma}, \\%[1mm]
  \mn{ok}_0 \sqcap  \cdots \sqcap \mn{ok}_{n-1} \sqcap \overline{A}_i \sqcap \exists R . V'_{D,q,\sigma}
  \sqcap A_{\sigma'} \sqcap A'_{\sigma'} &\sqsubseteq& V'_{D,q,\sigma},
\end{eqnarray}
% \end{array}
% $$
where $\sigma,\sigma',\sigma''$ range over $\Gamma$, $q$ and $q'$ over
$Q$, $i$ over $0,\dots,n-1$, and $D$ over $\{L,R\}$. We switch to
different verification markers $V'$, $V'_{L,\sigma}$, $V'_{R,q}$,
$V'_{D,q,\sigma}$ to distinguish between halting and non-halting
configurations. Note that the first verification step is different for
non-halting configurations: we expect to see one successor marked with
$X_1$ and one with $X_2$, both the first cell of an already verified
(halting or non-halting) configuration. For easier construction, we
require two successors also for existential configurations; they can
simply be identical.  The above CIs do not yet deal with cells where
the head is currently located. We need some prerequisites because when
verifying these cells, we want to (locally) verify the transition
relation. For this purpose, we carry the transitions implemented
locally at a configuration up to its predecessor configuration:
%
% $$
% \begin{array}{r@{~}c@{~}l}
\begin{eqnarray}
  % A_{q_1,\sigma_1} \sqcap \exists r . (A'_{q_2,\sigma_2} \sqcap A_{\sigma_3}) &\sqsubseteq&
  % S_{q_1,\sigma_3,L} \\%[1mm]
  % A'_{q_1,\sigma_1}  \sqcap A_{\sigma_2} \sqcap \exists r . A_{q_2,\sigma_3} &\sqsubseteq&
  % S_{q_2,\sigma_2,R} \\%[1mm]
  % \overline{A}_i \sqcap \exists r . S_{q,\sigma,M} &\sqsubseteq& S_{q,\sigma,M}
  % \\%[1mm]
\hspace*{3cm}
  \exists R . \exists R . (X_t \sqcap \overline{A}_0 \sqcap \cdots \sqcap
  \overline{A}_{n-1} \sqcap  V_{q,\sigma,D'}) & \sqsubseteq&
  S^t_{q,\sigma,D'} ,\\%[1mm]
  \exists R . \exists R . (X_t \sqcap \overline{A}_0 \sqcap \cdots \sqcap
  \overline{A}_{n-1} \sqcap  V'_{q,\sigma,D'}) & \sqsubseteq&
  S^t_{q,\sigma,D'} ,\\%[1mm]
  \exists R . (A_{\sigma} \sqcap S^t_{q,\sigma',D}) &\sqsubseteq&
  S^t_{q,\sigma',D}, %\\%[1mm]
\end{eqnarray}
% \end{array}
% $$
%
where $q$ ranges over $Q$, $\sigma$ and $\sigma'$ over $\Gamma$, $t$
over $\{1,2\}$, and $i$ over $0,\dots,n-1$.  Note that markers are
propagated up exactly to the head position. One issue with the above
is that additional $S^t_{q,\sigma,D}$-markers could be propagated up not
from the successors that we have verified, but from surplus
(unverified) successors. To prevent such undesired markers, we add
the CIs
%
%$$
\begin{equation}
\hspace*{5cm}
 S^t_{q_1,\sigma_1,D_1} \sqcap S^t_{q_2,\sigma_2,D_2} \sqsubseteq \bot
\end{equation}
%$$
%
for all $t \in \{1,2\}$ and all distinct
$(q_1,\sigma_1,D_1),(q_2,\sigma_2,D_2) \in Q \times \Gamma \times
\{L,R\}$.  We can now implement the verification of the cells under the
head in non-halting configurations. We take
%
%$$
%\begin{array}{rcl}
\begin{eqnarray}
\hspace*{1cm}
  \mn{ok}_0 \sqcap  \cdots \sqcap \mn{ok}_{n-1} \sqcap \overline{A}_i \sqcap \exists R . V'
  \sqcap
  A_{q_1,\sigma_1} \sqcap A'_{\sigma_1} \sqcap S^1_{q_2,\sigma_2,D_2}
  \sqcap S^2_{q_3,\sigma_3,D_3}&\sqsubseteq& V'_{R,q_1} ,\\%[2mm]
  \mn{ok}_0 \sqcap  \cdots \sqcap \mn{ok}_{n-1} \sqcap \overline{A}_i \sqcap \exists R . V'_{L,\sigma}
  \sqcap
  A_{q_1,\sigma_1} \sqcap A'_{\sigma_1} \sqcap S^1_{q_2,\sigma_2,D_2}
  \sqcap S^2_{q_3,\sigma_3,D_3}&\sqsubseteq& V'_{L,q_1,\sigma},
\end{eqnarray}
% \end{array}
% $$
for all $(q_1,\sigma_1) \in Q \times \Gamma$ with $q_1$ a universal
state and
$\Delta(q_1,\sigma_1)=\{(q_2,\sigma_2,D_2),(q_3,\sigma_3,D_3)\}$, $i$
from $0, \dots,n-1$, and $\sigma$ from $\Gamma$; moreover, we take
%
% $$
% \begin{array}{rcl}
\begin{eqnarray}
\hspace*{1cm}
   \mn{ok}_0 \sqcap  \cdots \sqcap \mn{ok}_{n-1} \sqcap \overline{A}_i \sqcap \exists R . V'
   \sqcap
   A_{q_1,\sigma_1} \sqcap A'_{\sigma_1} \sqcap S^1_{q_2,\sigma_2,D_2}
   \sqcap S^2_{q_2,\sigma_2,D_2}&\sqsubseteq& V'_{R,q_1} ,\\%[2mm]
   \mn{ok}_0 \sqcap  \cdots \sqcap \mn{ok}_{n-1} \sqcap \overline{A}_i \sqcap \exists R . V'_{L,\sigma}
   \sqcap
   A_{q_1,\sigma_1} \sqcap A'_{\sigma_1} \sqcap S^1_{q_2,\sigma_2,D_2}
   \sqcap S^2_{q_2,\sigma_2,D_2}&\sqsubseteq& V'_{L,q_1,\sigma} ,
\end{eqnarray}
% \end{array}
% $$
%
for all $(q_1,\sigma_1) \in Q \times \Gamma$ with $q_1$ an existential
state, for all $(q_2,\sigma_2,D_2) \in \Delta(q_1,\sigma_1)$, all $i$
from $0,\dots,n-1$, and all $\sigma$ from~$\Gamma$.  It remains to verify
the initial configuration. Let $w=\sigma_0 \cdots \sigma_{n-1}$, let
$(C=j)$ be the conjunction over the concept names $A_i$,
$\overline{A}_i$ that expresses $j$ in binary, for $0 \leq j < n$, and
let $(C \geq n)$ be the Boolean concept over the concept names $A_i$,
$\overline{A}_i$ expressing that the counter value is at least
$n$. Then we take
%
% $$
% \begin{array}{r@{\,}c@{\,}l}
\begin{eqnarray}
\hspace*{4cm}
  A_0 \sqcap \cdots \sqcap A_{n-1} \sqcap A_\Box %\sqcap A'_\sigma
  \sqcap L\mn{ok} \sqcap R\mn{ok}   &\sqsubseteq& V^I ,\\%[1mm]
\mn{ok}_0 \sqcap  \cdots \sqcap \mn{ok}_{n-1}
\sqcap (C \geq n) \sqcap \exists R . V^I
  \sqcap A_\Box %\sqcap A'_\sigma
&\sqsubseteq& V^I , \\%[1mm]
 \mn{ok}_0 \sqcap  \cdots \sqcap \mn{ok}_{n-1} \sqcap (C=i) \sqcap \exists R . V^I
  \sqcap A_{\sigma_i} %\sqcap A'_\sigma
&\sqsubseteq& V^I,% \\%[1mm]
% \mn{ok}_0 \sqcap  \cdots \sqcap \mn{ok}_{n-1} \sqcap (C=1) \sqcap \exists R . V^I
 % \sqcap A_{\sigma_1} %\sqcap A'_{q,\sigma'}
% &\sqsubseteq& V^I_{R,q}
\end{eqnarray}
% \end{array}
% $$
%
where $i$ ranges over $1,\dots,n-1$ and $\sigma,\sigma'$ over $\Gamma$.  This
verifies the initial conditions except for the left-most cell, where the head
must be located (in initial state $q_0$) and where we must verify the
transition, as in all other configurations.  Recall that we assume
$q_0$ to be an existential state. We can thus add
%
% $$
% \begin{array}{rcl}
\begin{eqnarray}
\hspace*{2cm}
   \mn{ok}_0 \sqcap  \cdots \sqcap \mn{ok}_{n-1} \sqcap (C=0) \sqcap \exists R . V^I
   \sqcap
   A_{q_0,\sigma_0} %\sqcap A'_\sigma
\sqcap S^1_{q,\sigma,D}
   \sqcap S^2_{q,\sigma,D}&\sqsubseteq& I
\end{eqnarray}
% \end{array}
% $$
%
for all $(q,\sigma,D) \in \Delta(q_0,\sigma_0)$.

At this point, we have finished the verification of the computation
tree, except that we have assumed but not yet established ($*$).
Achieving ($*$) consists of two parts. In the first part, we use the
concept names $B_i$, $\overline{B}_i$, $i < m$ (recall that $m=
\lceil \mn{log}(2^n+2) \rceil$) to implement an additional counter
that serves the purpose of generating a path whose length is $2^n+2$,
the distance between two corresponding tape cells in consecutive
configurations. Let $\alpha_0,\dots,\alpha_{k-1}$ be the elements of
$Q \cup (Q \times \Gamma)$.  We add the following to $\Tmc$:
%
% $$
% \begin{array}{rcl}
\begin{eqnarray}
\hspace*{5cm}
  \exists R . I & \sqsubseteq & \exists S .  \midsqcap_{\ell <k}
  \exists R . (A_{\alpha_\ell} \sqcap B_{\alpha_\ell} \sqcap (C_B = 0)) \\%[1mm]
  B_{\alpha_\ell} &\sqsubseteq& \exists R . \top ,\\%[1mm]
  B_i \sqcap \midsqcap_{j < i} B_j & \sqsubseteq& \forall R
  . \overline{B}_i ,\\%[1mm]
  \overline{B}_i \sqcap \midsqcap_{j < i} B_j & \sqsubseteq& \forall R
  . B_i ,\\%[1mm]
%\end{eqnarray}
%\begin{eqnarray}
  B_i \sqcap \midsqcup_{j < i} \overline{B}_j & \sqsubseteq& \forall R
  . B_i ,\\%[1mm]
  \overline{B}_i \sqcap \midsqcup_{j < i} \overline{B}_j & \sqsubseteq& \forall R
  . \overline{B}_i ,\\%[1mm]
  (C_B < 2^n+1)  \sqcap B_{\alpha_\ell}& \sqsubseteq & \forall R
                                                     . B_{\alpha_\ell}, \\%[1mm]
  (C_B = 2^n+1) \sqcap B_{\alpha_\ell} & \sqsubseteq & \forall R . \overline{A}_{\alpha_\ell},
\end{eqnarray}
% \end{array}
% $$
%
where  $\ell$ ranges over $0,\dots,k-1$,
$i$ ranges over $0,\dots,m$, and $(C_B = j)$ (respectively, $(C_B < j)$) denotes a Boolean
concept expressing that the value of the $B_i$/$\overline{B}_i$-counter is $j$
(respectively, smaller than $j$).  We will explain shortly why we need to
travel one more $R$-step (in the first line) after seeing~$I$.

The above CIs generate, after the verification of the
computation tree has ended successfully, a tree in the canonical model
of the input ABox and of $\Tmc$ as shown in Fig.~\ref{fig:mod}.
\begin{figure}[t!]
  \begin{center}
    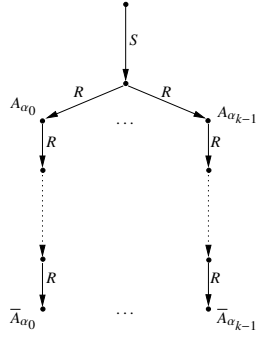
    \caption{Tree gadget.}
    \label{fig:mod}
  \end{center}
\end{figure}
\begin{figure}[t!]
  \begin{center}
    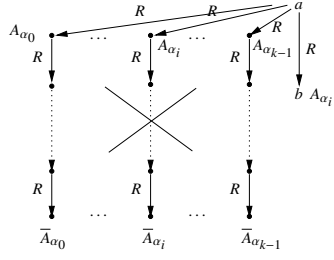
    \caption{Additional paths attached to computation tree. In the
     sequence of paths on the left, the path for $A_{\alpha_i}$ is missing.}
    \label{fig:add}
  \end{center}
\end{figure}
Note that the topmost edge is labeled with the role name $S$, which is
\emph{not} in~$\Sigma$. To satisfy Condition~(b) above, we must thus
not (homomorphically) find the subtree rooted at the node with the
incoming $S$-edge \emph{anywhere} in the canonical model of the ABox
and $\Tmc_\emptyset$ (which is just a different presentation of
\Amc). We use this effect to ensure that ($*$) is satisfied
\emph{everywhere}. Note that the $R$-paths in Fig.~\ref{fig:mod} have
length $2^n+2$ and that we do not display the labelling with the
concept names $B_i$, $\overline{B}_i$, $B_\alpha$.  These concept
names are not in $\Sigma$ and only serve the purpose of achieving the
intended path length and of memorising $\alpha$. Informally, every
$R$-path in the tree represents one possible \emph{copying
  defect}. The concept names of the form $\overline{A}_\alpha$ stand
for the disjunction over all $A'_\beta$ with $\beta \neq
\alpha$.  Although we have not done it so far, we can easily modify
$\Tmc$ to achieve that they are indeed used this way in the input
ABox. For example, we can add the conjunct $\bigsqcap_{\sigma' \in
  \Gamma \setminus \{ \sigma\}} \overline{A}_{\sigma'}$ to the left-hand
side of the concept inclusion in (1), and likewise for (6), (7), and
so on.

If there is a copying defect somewhere in the ABox, then one of the
$R$-paths in Fig.~\ref{fig:mod} can be homomorphically embedded.  We
have to ensure that the other paths can be embedded, too. The first
step is to add the following CIs:
\begin{eqnarray}
\hspace*{6cm}
  (C'=2^n+2) \sqcap \overline{A}_{\alpha_\ell}
  &\sqsubseteq& V'_\ell ,\\%[1mm]
  A'_i \sqcap \exists R . A'_i \sqcap \midsqcup_{j < i} \exists R
  . A'_j &\sqsubseteq& \mn{ok}'_i , \\%[1mm]
  \overline{A}'_i \sqcap \exists R . \overline{A}'_i \sqcap \midsqcup_{j < i} \exists R
  . A'_j &\sqsubseteq& \mn{ok}'_i  ,\\%[1mm]
  A'_i \sqcap \exists R . \overline{A}'_i \sqcap \midsqcap_{j < i} \exists R
  . \overline{A}'_j &\sqsubseteq& \mn{ok}'_i , \\%[1mm]
  \overline{A}'_i \sqcap \exists R . A'_i \sqcap \midsqcap_{j < i} \exists R
  . \overline{A}'_j &\sqsubseteq& \mn{ok}'_i  ,\\%[1mm]
  \mn{ok}'_0 \sqcap  \cdots \sqcap \mn{ok}'_{n-1} \sqcap \overline{A}'_i \sqcap \exists R . V'_\ell
  \sqcap A_\sigma \sqcap A'_{\sigma} &\sqsubseteq& V'_\ell ,\\
  \exists R . ((C'=0)\sqcap V'_\ell \sqcap A_{\alpha_\ell}) & \sqsubseteq & V_\ell,
%[1mm]
%   \exists R . A_{\alpha_i} & \sqsubseteq& \midsqcap_{\ell \in \{0,\dots,k-1\}
%     \setminus \{i\}} \exists R . (B_{\alpha_\ell} \sqcap
% (C=0)) \\%[1mm]
%                B_{\alpha_\ell}&\sqsubseteq& \exists R . \top \\%[1mm]
%   B_i \sqcap \midsqcap_{j < i} B_j & \sqsubseteq& \forall R
%   . \overline{B}_i \\%[1mm]
%   \overline{B}_i \sqcap \midsqcap_{j < i} B_j & \sqsubseteq& \forall R
%   . B_i \\%[1mm]
%   B_i \sqcap \midsqcup_{j < i} \overline{B}_j & \sqsubseteq& \forall R
%   . B_i \\%[1mm]
%   \overline{B}_i \sqcap \midsqcup_{j < i} \overline{B}_j & \sqsubseteq& \forall R
%   . \overline{B}_i \\%[1mm]
%   (C < 2^n) \sqcap B_{\alpha_\ell} & \sqsubseteq & \forall R . B_{\alpha_\ell}\\%[1mm]
%   (C = 2^n) \sqcap B_{\alpha_\ell}& \sqsubseteq & \forall R . \overline{A}_{\alpha_\ell}
\end{eqnarray}
% \end{array}
% $$
%
where $\ell$ ranges over $0,\dots,k-1$, $i$ ranges over $0,\dots,m$,
and $(C' = j)$ denotes a Boolean concept which expresses that the
value of the $A'_i$/$\overline{A}'_i$-counter is $j$; recall that the
concept names implementing this counter are in $\Sigma$.  The purpose
of the above CIs is to set the verification marker $V_\ell$ at an
individual $a$ whenever we find in the ABox an $R$-path with root $a$
that is isomorphic to the $R$-path labeled with
$A_{\alpha_\ell}$/$\overline{A}_{\alpha_\ell}$ in Fig.~\ref{fig:mod}
(and additionally is decorated in an appropriate way with the concept
names used by the $A'_i$/$\overline{A}'_i$-counter).

As the second step, it remains to add the verification markers
$V_\ell$ to the left-hand side of the CIs in $\Tmc$ in such a way that
\begin{itemize}

\item[($**$)] whenever an ABox individual $a$ that is part of the
  computation tree has an $R$-successor in that tree which is labeled
  with $A_{\alpha_\ell}$, then all verification markers $V_j$ with $j
  \in \{0,\dots,\ell-1,\ell+1,\dots,k-1 \}$ must be present at $a$.

\end{itemize}
Informally, ($**$) achieves the presence of additional paths attached
to nodes of the computation tree, as displayed in
Fig.~\ref{fig:add}. There, $a$ and $b$ are nodes in the computation
tree proper and since $A_{\alpha_i}$ holds at $b$, we attach to $a$
all paths from Fig.~\ref{fig:mod} except the one for $A_{\alpha_i}$.
By what was achieved in the first step, we can thus homomorphically
embed the $R$-tree in Fig.~\ref{fig:mod} at $a$ iff there is a copying
defect at the successor of $a$.

We next describe the modifications required to achieve
($**$). Line~(20) needs to be extended by adding to the left-hand side
the conjunct
$\bigsqcap_{j \in \{0,\dots,\ell-1,\ell+1,\dots,k-1 \}} V_{j} \sqcap
\exists R . \alpha_\ell$
where $\ell$ ranges over $0,\dots,k-1$. Here, we want
$\exists R . \alpha_\ell$ to refer to the same $R$-successor whose
existence is verified by the existing concept $\exists R . V'$ on the
left-hand side of~(20), or at least to a successor that has the same
$\alpha_{\ell}$-label. This can be achieved by adding the CIs
\begin{eqnarray}
\hspace*{6cm}
\exists R . \alpha_\ell \sqcap \exists R . \alpha_{\ell'} \sqsubseteq \bot
\end{eqnarray}
where $\ell$ and $\ell'$ are distinct, ranging over $0,\dots,k-1$.

The same conjunct needs to be added to the left-hand sides of Lines
(21)--(23), (28)--(31), and (33)--(35). We also need to add the
conjunct into the scope of the outermost (but not innermost!)
existential quantifier in~(17) and~(18) and to (36), outside the scope
of the existential quantifier. Note that we indeed need to travel one
more $R$-step after seeing $I$ (the explanation of this was deferred
until now): we always consider copying defects at $R$-successor of
some individual name and thus also the root of our configuration tree
should be the $R$-successor of some individual.  Also note that we
indeed need to separate successor configurations by two $R$-steps (the
remaining deferred explanation).  If we used only one $R$-step, then
the branching ABox individual would \emph{always} allow the $R$-tree
from Fig.~\ref{fig:mod} to be homomorphically embedded, no matter
whether there is a copying defect or not.
%
% instead of only by one. If we dIn fact, the mid point needs not make true any
% of the concept names $A_\alpha$ and thus we are not forced to violate
% the above constraint when branching at the end of configuration
% sequences. Also note that copying the content of the first cell of the
% initial configuration requires traveling one more $R$-step after
% seeing $I$, as implemented in the first line of the final block of CIs
% in $\Tmc_2$ above.
%
% We need to avoid that the inclusions in $\Tmc_1$ enable a homomorphism
% from the tree in Figure~\ref{fig:mod} due to an ABox where some node
% has two $R$-successors labeled with different concepts $A_\alpha$,
% $A_\beta$:
% %
% %$$
% \begin{equation}
%   \exists R. A_\alpha \sqcap \exists R . A_\beta \sqsubseteq \bot.
% \end{equation}
% %$$
% %
% This explains why we need to separate successor configurations by two
% $R$-steps. In fact, the mid point needs not make true any of the
% concept names $A_\alpha$ and thus we are not forced to violate the
% above constraint when branching at the end of configuration
% sequences. Also note that copying the content of the first cell of the
% initial configuration requires traveling one more $R$-step after
% seeing $I$, as implemented in the first line of the final block of CIs
% in $\Tmc_2$ above.
%
\begin{lemma}
\label{lem:2expcorr}
  The following conditions are equivalent:
\begin{enumerate}
\item[$(1)$] there is a tree $\Sigma$-ABox \Amc\xspace such that
  $(a)$ \Amc is consistent with $\Tmc$ and $(b)$ $\Imc_{\Tmc,\Amc}$ is not
    $\Sigma$-homomorphically embeddable into
    $\Imc_{\Tmc_\emptyset,\Amc}$;

\item[$(2)$] $M$ accepts $w$.

\end{enumerate}
\end{lemma}
\begin{proof}(sketch) For $(2) \Rightarrow (1)$, suppose $M$ accepts $w$. The accepting computation tree of $M$ on $w$
  can be represented as a $\Sigma$-ABox as detailed above alongside
  the construction of the TBox $\Tmc$. The representation only uses
  the role name $R$ and the concept names $A_i$,
  $\overline{A}_i$,$A'_i$, $\overline{A}'_i$, $A_\sigma$,
  $A_{q,\sigma}$, $A'_\sigma$, $A'_{q,\sigma}$, $\overline{A}_\sigma$,
  $\overline{A}_{\q,\sigma}$, $X_1$, and $X_2$. As explained above, we
  need to duplicate the successor configurations of existential
  configurations to ensure that there is binary branching after each
  configuration. Also, we need to add one additional incoming $R$-edge
  to the root of the tree. The resulting ABox \Amc is consistent with
  $\Tmc$. Moreover, since there are no copying defects, there is no
  homomorphism from $\Imc_{\Tmc,\Amc}$ to $\Imc_{\Tmc_\emptyset,\Amc}$.

  For $(1) \Rightarrow (2)$, suppose there is a tree
  $\Sigma$-ABox \Amc that satisfies~(a) and~(b).  Because
  of~(b), $I$ must be true somewhere in
  $\Imc_{\Tmc,\Amc}$: otherwise, $\Imc_{\Tmc,\Amc}$ does not
  contain anonymous elements and the identity is a homomorphism from
  $\Imc_{\Tmc,\Amc}$ to $\Imc_{\Tmc_\emptyset,\Amc}$, contradicting
  (b). Since $I$ is true somewhere in $\Imc_{\Tmc,\Amc}$ and by the 
  construction of $\Tmc$, the ABox must contain the representation
  of an accepting computation tree of $M$ on $w$, except satisfaction
  of ($*$).  For the same reason, $\Imc_{\Tmc,\Amc}$ must contain a
  tree as shown in Fig.~\ref{fig:mod}.  As already been argued
  during the construction of $\Tmc$, however, condition
  ($*$) follows from the existence of such a tree in
  $\Imc_{\Tmc,\Amc}$ together with (b).
\end{proof}
We remark that the above reduction also yields 2\ExpTime hardness for
$(\Sigma,\Sigma)$-CQ entailment in the DL 
$\mathcal{ELI}$ extending $\mathcal{EL}$ with inverse roles. In fact,
CIs $D\sqsubseteq \forall r.C$ can be replaced by
$\exists r^{-}.D \sqsubseteq C$ and disjunctions on the left-hand side
can be removed with only a polynomial blowup. It thus remains to
eliminate $\bot$, which only occurs non-nested on the right-hand side
of CIs.  With the exception of the CIs in~(27),
 % $$
 % S^t_{q_1,\sigma_1,M_1} \sqcap S^t_{q_2,\sigma_2,M_2} \sqsubseteq \bot,
 % $$
this can be done as follows: replace $\Tmc_\emptyset$ with a non-empty
TBox $\Tmc_1$ and rename $\Tmc$ to $\Tmc_2$ for uniformity; include
all CIs with $\bot$ on the right-hand side in $\Tmc_1$ instead of in
$\Tmc_2$; then replace $\bot$ with a fresh concept name $D$ and
further extend $\Tmc_1$ with CIs which make sure that
$\Imc_{\Tmc_1,\Amc}$ contains an $R$-tree as in Fig.~\ref{fig:mod}
whenever $D$ is non-empty, which is straightforward. As a consequence,
any ABox that satisfies the left-hand side of a $\bot$-CI in the
original TBox $\Tmc$ cannot satisfy (b) from Lemma~\ref{lem:2expcorr}
and does not have to be considered.

For the excluded CIs, a different approach needs
to be taken since these CIs rely on many CIs in $\Tmc_2$ that
are not included in $\Tmc_1$. We only sketch the required modifiction:
instead of introducing the concept names $S^t_{q_1,\sigma_1,D_1}$, one
would propagate transitions inside the $V'$-markers. Thus,
$S^{1}_{q_1,\sigma_1,D_1}$, $S^2_{q_2,\sigma_2,D_2}$, and $V'$ would
be integrated into a single marker
$V'_{q_1,\sigma_1,D_1,q_2,\sigma_2,D_2}$, and likewise for
$V_{L,q}$. The excluded CIs can then simply be
dropped.

\begin{theorem}
\label{thm:2exptboxeli}
It is 2\ExpTime-hard to decide whether an \ELI TBox
$(\Sigma,\Sigma)$-CQ entails an \ELI TBox.
%
%It is 2\ExpTime-complete to decide whether a Horn-\ALC TBox
%$(\Sigma_1,\Sigma_2)$-CQ entails a Horn-\ALC TBox.  The lower bound
%already holds when $\Sigma_1 = \Sigma_2$.
\end{theorem}
A corresponding upper bound has recently been established in
\cite{JungLMS17}.

% !TEX root = aij-insep.tex

\section{Related Work}

The comparison of logical theories has been an active research area almost since the invention 
of formal logic. Important concepts include Tarski's notion of \emph{interpretability}~\cite{Alfred} of one theory into
another and the notion of \emph{conservative extension}, which has been employed extensively
in mathematical logic, in particular to compare theories of sets and numbers~\cite{Rautenberg}.
Conservative extensions have also been used to formalise modular software specification~\cite{GoguenB92,Maibaum1,Goguen93}
and to enable modular ontology development~\cite{GhiLuWo-06,DBLP:series/lncs/5445,KutzML10}.
Query entailment can be regarded as a generalisation of conservative extension where we do not require that one of the theories under consideration is included in the other and where conservativity depends on 
database queries in a signature of interest instead of formulas in the signature of the smaller theory. 
In an independent but closely related research field, various notions of equivalence 
between (extended) datalog programs have been proposed and investigated~\cite{Woltran10}, often focusing on 
answer set programming~\cite{Woltran10,LifschitzPV01,EiterFW07,HarrisonLPV17}.

The state of the art in the research of inseparability between
description logic ontologies has recently been presented in great
detail in \cite{BotoevaKLRWZ16}. This survey contains,
in particular, a discussion of the relationships between
concept-based, model-based, and query-based inseparability. In the
first approach, one compares the concept inclusions entailed by the
two versions of an ontology. In the second approach, one compares the
models of the two versions.  In contrast, in the query-based approach
underpinning the present investigation, one compares the certain
answers to database queries.  It turns out that the three approaches
exhibit rather different properties and require different
model-theoretic and algorithmic techniques. While various forms of
bisimulations and corresponding bisimulation-invariant tree automata
are required to investigate concept-based inseparability, query-based
inseparability relies on understanding homomorphisms between
interpretations and products, which are then reflected in the games or
automata required to design algorithms; we refer the reader
to~\cite{BotoevaKLRWZ16} for an in-depth
discussion. Important notions that are closely related to query
inseparability, such as knowledge exchange and entailment between OBDA
specifications, are discussed in~\cite{BotoevaKRWZ16}. 

In what follows, we focus on summarising what is known about query
inseparability between description logic ontologies, discussing both
the KB and the TBox cases. All existing results are about Horn-DLs as
the present paper is the first one to study query-based inseparability
for expressive non Horn-DLs. As discussed in this paper, for Horn-DLs, there 
is no difference between CQ- and UCQ-inseparability, so we do not explicitly 
distinguish between them below.

We start with the KB case. In \cite{BotoevaKRWZ16},
CQ-inseparability between KBs is investigated for Horn-DLs ranging
from the lightweight $\mathcal{EL}$ and \DLc to
$\textsl{Horn}\mathcal{ALCHI}$.  The authors develop model-theoretic
and game-theoretic characterisations of query inseparability. In
contrast to the present investigation, the main complexity results,
summarised in~Table~\ref{table:kb-related}, are then obtained using
the game-theoretic characterisations instead of reductions to the
emptiness problem of tree-automata.  It is also proved that rootedness
does not affect the worst-case complexity of query entailment. Observe
that the addition of the inverse role constructor leads to an
exponential increase of the complexity of checking query
inseparability.

\begin{table}[h]
\centering
\caption{KB query inseparability~\cite{BotoevaKRWZ16}.}
\label{table:kb-related}
\begin{tabular}{|c|c|c|c|}
\hline
 DL &  complexity & DL & complexity \\
\hline
$\mathcal{EL(H}^{dr}_\bot)$ &  \PTime   & -  & - \\
\hline
\DLc & \PTime & \DLcH & \ExpTime\\
\hline
$\textsl{Horn}\mathcal{ALC(H)}$ & \ExpTime & $\textsl{Horn}\mathcal{ALC(H)I}$ & 2\ExpTime\\
\hline
\end{tabular}
\end{table}

CQ-inseparability between TBoxes has been investigated for
$\mathcal{EL}$ terminologies (a restricted form of TBox) extended with
role inclusions and domain and range restrictions
\cite{KonevL0W12,KonevLW12-cex25}, for (unrestricted TBoxes
in) the description logic $\mathcal{EL}$
\cite{LutzW10}, and for variants of DL-Lite
\cite{BotoevaKLRWZ16,BotoevaKRWZ16}. The
algorithms presented in \cite{KonevL0W12} are based
on both model-theoretic and proof-theoretic methods. The authors focus
not only on deciding inseparability but also on presenting the logical
difference between TBoxes to the user.  A versioning and
modularisation system for acyclic $\mathcal{EL}$ TBoxes based on
CQ-inseparability is presented and evaluated in \cite{KonevLW12-cex25}. The
system makes intense use of the fact that, in this case, query
inseparability can be decided in polynomial time.  This is in contrast
to general $\mathcal{EL}$ TBoxes for which \ExpTime completeness of
deciding CQ-inseparability is shown in
\cite{LutzW10}. The method is purely model-theoretic
and based on the close relationship between concept and query
inseparability for $\mathcal{EL}$.  More recently, CQ inseparability
has been investigated for Horn$\mathcal{ALCHI}$ and shown to be
2\ExpTime-complete, using a subtle approach that combines a mosaic
technique with automata~\cite{JungLMS17}. The mentioned results are
summarised in Table~\ref{table:TBox-related}.
%
%\end{figure}
%
%\begin{figure}[h]
\begin{table}[h]
\centering
\caption{TBox query inseparability.}%
\label{table:TBox-related}
\begin{tabular}{|c|c|c|c|}
\hline
DL &  complexity & DL & complexity \\
\hline
 $\mathcal{EL}$&  \ExpTime~\cite{LutzW10}  &  $\textsl{Horn}\mathcal{ALC(H)I}$ & 2\ExpTime~\cite{JungLMS17} \\
\hline
 \DLc & in \PTime~\cite{BotoevaKLRWZ16} &  \DLcH & \ExpTime~\cite{BotoevaKRWZ16} \\
\hline
\end{tabular}
\end{table}

% !TEX root = aij-insep.tex

\section{Conclusion and Future Work}
We have made significant steps towards understanding query entailment and
inseparability for KBs and TBoxes in expressive DLs. Our main---and rather unexpected---results are as follows:
\begin{itemize}
\item[--] for $\mathcal{ALC}$-KBs, $\Sigma$-(r)UCQ inseparability is decidable and (r)CQ-inseparability
is undecidable (even without restrictions on the signature);
\item[--] for \hALC-TBoxes, $\Theta$-rCQ inseparability is \ExpTime complete and $\Theta$-CQ inseparability is
2\ExpTime complete.
\end{itemize}
The first result reflects a fundamental difference between the model-theoretic characterisations of inseparability
for CQs and UCQs: while UCQ-inseparability can be characterised using (partial) homomorphisms between models of the
respective KBs, CQ-inseparability requires the construction of products of the models of the respective KBs, a result
which is at the core of our undecidability proof. The second result reflects a fundamental difference between homomorphisms whose
domain is connected to ABox individuals (as required for rooted CQs) and those whose domain is not necessarily reachable from the ABox. Searching for the latter turns out to be much harder. Both results have important practical
implications. The first one indicates that one should approximate CQ-inseparability using UCQ-inseparability when designing practical algorithms.
Observe that this is a sound approximation as no two ontologies that are UCQ-inseparable can be separated by CQs.
The second one indicates that it is worth focusing on rooted (U)CQs rather than all (U)CQs when designing practical algorithms
for inseparability. The latter are likely to cover the vast majority of queries used in practice. We believe that
our model-theoretic characterisations provide a good foundation for developing practical (approximation) algorithms.

Many problems remain open. The main one, which can be directly inferred from the tables presenting our
results,  is the decidability of UCQ-inseparability for $\mathcal{ALC}$ TBoxes. We conjecture that this problem is
undecidable but have found no way of proving this. Another family of interesting open problems concerns the role of
the signatures $\Sigma$ and $\Theta$ in our investigation of the decidability/complexity of inseparability
between KBs and TBoxes, respectively. Observe that admitting more symbols in $\Sigma$ or $\Theta$ leads to sound approximations
of the original inseparability problem: for example, if TBoxes are $\Theta'$-CQ inseparable for a pair of signatures $\Theta'\supseteq \Theta$,
then they are $\Theta$-CQ inseparable as well. It would, therefore, be of great interest to understand the complexity
of inseparability if $\Sigma$ and $\Theta$ consist of \emph{all} concept and role names (the `full signature' case).
We have been able to prove undecidability of full signature (r)CQ-inseparability for $\mathcal{ALC}$ KBs, but the complexity of full signature (r)UCQ-inseparability
between $\mathcal{ALC}$ KBs remains open. Similarly, the decidability of full signature (r)CQ-inseparability and (r)UCQ-inseparability
between \ALC TBoxes remains open. The `hiding technique' discussed in this paper might be a good starting point to attack
those problems.
Finally, it would be of interest to consider extensions of $\mathcal{ALC}$ with inverse roles, qualified number restrictions, nominals,
and role inclusions. We conjecture that extensions of our results to DLs with qualified number restrictions and role inclusions are
rather straightforward (though proofs might become significantly less transparent). The addition of inverse roles, however, might
lead to non-trivial modifications of the model-theoretic criteria, see
also~\cite{JungLMS17}. 

\section*{Acknowledgements} 
We thank the anonymous reviewers for their very thorough and useful comments.
This research was supported by the DFG
grant LU 1417/2-1 (C.~Lutz), the ERC consolidator grant CODA 647289,
and the EPSRC joint grants EP/M012646/1 and EP/M012670/1 `iTract:
Islands of Tractability in Ontology-Based Data Access' (F.~Wolter and
M.~Zakharyaschev).

% ranility (r)Our results include
% the unexpected Main results includeMany problems remain to be
% addressed. From a theoretical viewpoint, it would be of interest to solve the
% open problems in Figures~\ref{table:kb} and \ref{table:tbox}, and also consider
% other expressive DLs such as \textsl{DL-Lite}$_{\it bool}^\mathcal{H}$~\cite{ACKZ09} or
% $\mathcal{ALCI}$.
%For example, if Theorem~\ref{rUCQ-full-hom} could be
%generalized to UCQs (and $\Sigma$-homomorphisms), we would obtain a 2\ExpTime
%upper bound for UCQ-entailment between $\mathcal{ALC}$ KBs using the same
%technique as for rUCQs.
% Also, our undecidability proof goes through for
% \textsl{DL-Lite}$_{\it bool}^\mathcal{H}$, but the other cases remain open. From
% a practical viewpoint, our model-theoretic criteria for query entailment are a
% good starting point for developing algorithms for approximations of query
% entailment based on simulations. Our undecidability and complexity results also
% indicate that rUCQ-entailment is more amenable to practical algorithms than,
% say, CQ-entailment and can be used as an approximation of the latter.
%
%%% Local Variables:
%%% mode: latex
%%% TeX-master: "aij-insep"
%%% TeX-PDF-mode: t
%%% save-place: t
%%% End:

\newenvironment{theoremnum}[1]{\smallskip\noindent\textbf{Theorem~#1.}\hspace*{0.3em}\em}{\par\smallskip}

\newenvironment{lemmanum}[1]{\smallskip\noindent\textbf{Lemma~#1.}\hspace*{0.3em}\em}{\par\smallskip}
\appendix

% !TEX root = aij-insep.tex

%**********************

%\section{Proof of Theorem~\ref{UCQtoCQ}}
%
%\noindent
%%
%{\bf Theorem 4} \emph{
%Let $\K_{1}$ be an $\ALC$ KB and $\K_{2}$ a \hALC KB. Then
%$\K_{1}$ $\Sigma$-UCQ entails $\K_{2}$ iff $\K_{1}$ $\Sigma$-CQ entails $\K_{2}$. The same holds for rUCQ and rCQ, and for TBox entailment.}
%%
%\begin{proof}
%TBD
%\end{proof}
%
%
%%**************
%
%\section{Proof of Theorem~\ref{crit:KB}}
%
%\noindent
%%
%{\bf Theorem 5} {\em
%Let $\K_1$ and $\K_2$ be \ALC KBs, $\Sigma$ a signature, and let $\Mod_{i}$ be complete for $\K_i$, $i=1,2$.
%%
%\begin{description}\itemsep=0pt
%\item[\rm (1)] $\K_{1}$ $\Sigma$-UCQ entails $\K_2$ iff, for any  $n>0$ and $\I_1\in \Mod_{1}$, there exists $\I_2 \in \Mod_{2}$ that is $n\Sigma$-homomorphically embeddable into $\I_1$.
%	
%\item[\rm (2)] $\K_{1}$ $\Sigma$-rUCQ entails $\K_2$ iff, for any  $n>0$ and $\I_1\in \Mod_{1}$, there exists $\I_2 \in \Mod_{2}$ that is con-$n\Sigma$-homomorphically embeddable into $\I_1$.
%	
%\item[\rm (3)] $\K_{1}$ $\Sigma$-CQ entails $\K_2$ iff $\prod \Mod_{2}$ is $n\Sigma$-homomorphically embeddable into $\prod \Mod_{1}$ for any $n>0$.
%	
%\item[\rm (4)] $\K_{1}$ $\Sigma$-rCQ entails $\K_2$ iff $\prod \Mod_{2}$ is con-$n\Sigma$-homo\-mor\-phically embeddable into $\prod \Mod_{1}$ for any $n>0$.
%\end{description}}
%%
%\begin{proof}
%TBD
%\end{proof}

%**************

\section{Proof of Theorem~\ref{thm:undecidability-rcq}}

For the proof of Theorem~\ref{thm:undecidability-rcq}~(\emph{i}), suppose that an instance $\mathfrak{T}$
of the rectangle tiling problem is given. Consider the KBs $\KonerCQ = (\TonerCQ,\ArCQ)$ and
$\KtworCQ = (\TtworCQ,\ArCQ)$ given in the proof sketch for Theorem~\ref{thm:undecidability-rcq}~(\emph{i}).
It suffices to prove Lemmas~\ref{qn-tiling} and~\ref{qn-rejected-by-k1} for the new KBs, the rCQs $q_{n}^{r}(y)$, and
the signature $\SigmarCQ$.
\begin{lemma}
%\label{qn-tiling-loop}
The instance $\mathfrak{T}$ admits a rectangle tiling iff there exists $\q_n^{r}(a)$ such that
$\KtworCQ \models \q_n^{r}(a)$.
%An instance of the rectangle tiling problem admits a tiling iff there exists $\q_n$ such that
%$\KtwoCQ \models \q_n$.
%
%There exists an rCQ $\q_n(y)$ such that $\prod \Mod_{\KtworCQ} \models \q_n(a)$ iff there exist $N,M \in \mathbb N$% for which $\mathfrak T$ tiles the $N \times M$ grid as described above.
\end{lemma}
\begin{proof} $(\Rightarrow)$ Suppose $\mathfrak T$ tiles
the $N \times M$ grid so that a tile of type $T^{ij} \in \mathfrak T$
covers $(i,j)$. Let
$$
\textit{block}_j = (\widehat T^{1,j}_k, \dots, \widehat T^{N,j}_k, \Row),
$$
for $j=1,\dots,M-1$ and $k =(j-1) \!\! \mod \! 3$. Let $\q^{r}_n$ be the CQ in which the $B_i$ follow the pattern
$$
\Row, \ \textit{block}_1, \ \textit{block}_1, \ \textit{block}_2,  \dots, \ \textit{block}_{M-1}
$$
(thus, $n = (N+1) \times M +1$).  In view of Lemma~\ref{min-elu-complete},  we
only need to prove $\I \models \q^{r}_n(a)$ for each minimal model $\I \in
\Mod_{\KtworCQ}$. Take such an $\I$. We have to show that there is an $R$-path
$a,x_0,\dots,x_{n+1}$ in $\I$ such that $x_i \in B_i^\I$ and $x_{n+1}
\in \End^\I$.

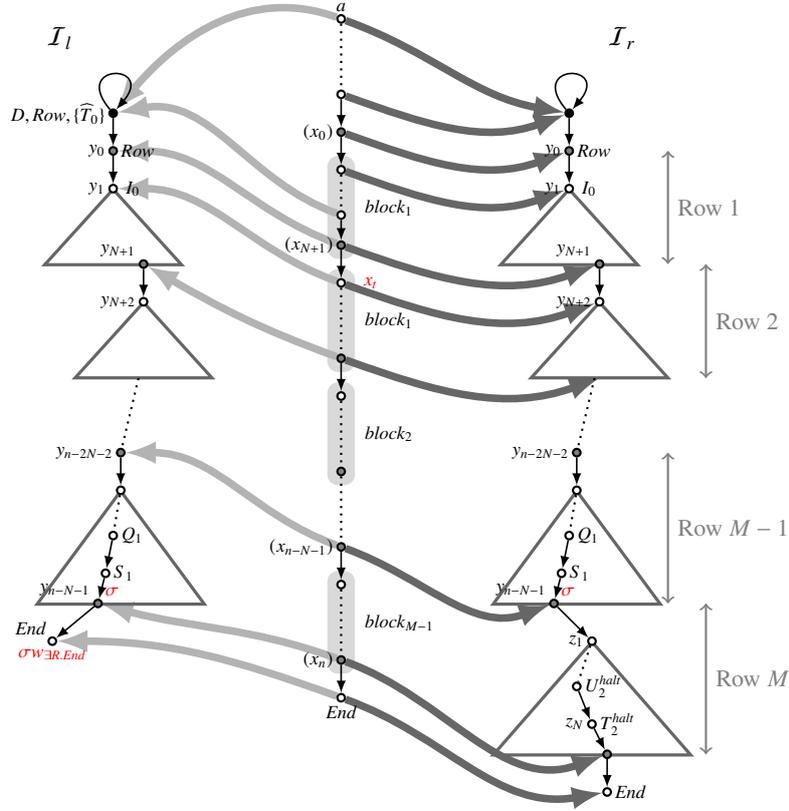
\begin{figure}[h]
  \centering
  \scalebox{1}{\begin{tikzpicture}[yscale=1, %
    point/.style={thick,circle,draw=black,fill=white, minimum size=1mm,inner
      sep=0pt}%
    ]

    %%% the query
    \begin{scope}[yshift=-0.75cm]
      \foreach \al/\x/\y/\lab/\wh/\extra in {%
        y0/0/3/{a}/above/, %
        yy0a/0/2//right/, %
        yy0/0/1.5/(x_0)/left/rowind, %
        yy1/0/1//right/, %
        yyNa/0/0.4//right/, %
        yyN/0/0/(x_{N+1})/left/rowind, %
        y1/0/-0.5/{~~~\textcolor{red}{x_t}}/right/, %
        yN/0/-1.5/{}/right/rowind, %
        yN1/0/-2//right/, %
        yN2/0/-3//right/rowind, %
        ylN2/0/-4/{(x_{n-N-1})}/left/rowind, %
        ylN1/0/-4.5//right/, %
        % yl1/0/-5.5/C_{n-1}/right/, %
        yl/0/-5.5/{(x_n)}/left/rowind, %
        yend/0/-6/\End/below/%
      }{ \node[point, \extra, label={[inner sep=1]\wh:{\scriptsize $\lab$}}] (\al) at (\x,\y)
        {}; }

      \foreach \from/\to/\type in {%
        y0/yy0a/dotted, yy0a/yy0/role, yy0/yy1/role, yy1/yyNa/dotted,
        yyNa/yyN/role, yyN/y1/role, y1/yN/dotted, yN/yN1/role, yN1/yN2/dotted,
        yN2/ylN2/dotted, ylN2/ylN1/role, ylN1/yl/dotted, %yl1/yl/role,
        yl/yend/role%
      } {\draw[thick, \type] (\from) -- (\to);}

      \begin{pgfonlayer}{background}
        \foreach \first/\last/\ind in {%
          yy1/yyN/1, y1/yN/1, yN1/yN2/2, ylN1/yl/{M-1}%
        }{ \node[fit=(\first)(\last), rounded corners, fill=gray!30,
          label=right:{\scriptsize \textit{block}$_{\ind}$}] {}; }
      \end{pgfonlayer}
    \end{scope}

    %%%% left model
    \begin{scope}[xshift=-3cm]
      \foreach \al/\x/\y/\hei/\wid in {%
        t1/0/0/1cm/1.8cm, %
        t2/0.4/-1.5/1cm/1.8cm, %
        t3/0.1/-4/1.5cm/2.2cm%
      }{ \node[subtree, minimum height=\hei, minimum width=\wid] (tree-\al) at
        (\x,\y) {};%

        \node[point] (\al) at (\x,\y) {};%
      }

      \foreach \al/\x/\y/\lab/\wh/\extra in {%
        a/0/1/{D,\Row,\{\widehat{T}_0\}}/left/constant,%
        start1/0/0.5/\Row/right/rowind,%
        start2/0.4/-1//right/rowind,%
        start3/0.1/-3.5//right/rowind,%
        startm/-0.2/-5.5//right/rowind,%
        end1/-0.8/-6/\End/above left/,%
        s1/0/-4.6/Q_1/right/,%
        s2/-0.1/-5.1/S_1/right/%
      }{ \node[point, \extra, label={[inner sep=1]\wh:{\scriptsize $\lab$}}] (\al) at
        (\x,\y) {}; }

      %% the edges
      \foreach \from/\to/\type in {%
        a/start1/role, %
        start1/t1/role, start2/t2/role, start3/t3/role, %
        s1/s2/role, s2/startm/role, startm/end1/role, %
        tree-t2.260/start3/dotted, t3/s1/dotted%
      } {\draw[thick, \type] (\from) -- (\to);}

    \draw[role] (a) to[out=130, in=50, looseness=25] (a);

      %% additional labels
      \foreach \al/\lab/\wh in {%
        start1/y_0/left, t1/y_1/left, start2/y_{N+1}/above left,
        t2/y_{N+2}/left, start3/y_{n-2N-2}/left, startm/y_{n-N-1}/above left,
        t1/~I_0/right%
      }{ \node[label={[inner sep=2]\wh:{\scriptsize $\lab$}}, inner sep=0] at
        (\al) {}; }

      \foreach \al/\lab/\wh in {%
        startm/\sigma/above right, end1/\sigma w_{\exists R. \End}/272%
      }{\node[label={[inner sep=2,red]\wh:{\scriptsize $\lab$}}, inner sep=0] at (\al)
        {}; }

      % \foreach \y/\i in {%
      % -0.5/1, -1.8/2, -4.5/M-1, -7/M%
      % }{ \node[anchor=east] at (-1.5, \y) {\small $t_{\i}$:}; }

      \node at (-0.7,2) {$\I_l$};

    \begin{pgfonlayer}{background}
      %%% left homomorphism
      \foreach \from/\to/\in in {%
        y0/a/50, yyNa/a/10,%
        yyN/start1/0, y1/t1/0, yN/start2/-30, ylN2/start3/0, yl/startm/-15,
        yend/end1/0%
      } {\draw[homomorphism, black!30] (\from) to[out=160, in=\in] (\to);}
    \end{pgfonlayer}
    \end{scope}

    %%%%% right model
    \begin{scope}[xshift=3cm, yshift=-0cm]
      \foreach \al/\x/\y/\hei/\wid in {%
        t1/0/0/1cm/1.8cm, %
        t2/0.4/-1.5/1cm/1.8cm, %
        t3/0.1/-4/1.5cm/2.2cm, %
        tm/0.3/-6/1.5cm/2.6cm%
      }{ \node[subtree, minimum height=\hei, minimum width=\wid] (tree-\al) at
        (\x,\y) {};%
        \node[point] (\al) at (\x,\y) {};%
      }

      \foreach \al/\x/\y/\lab/\wh/\extra in {%
        a/0/1//right/constant,%
        start1/0/0.5/\Row/right/rowind,%
        start2/0.4/-1//right/rowind,%
        start3/0.1/-3.5//right/rowind,%
        startm/-0.2/-5.5//right/rowind,%
        s1/0/-4.6/Q_1/right/,%
        s2/-0.1/-5.1/S_1/right/,%
        vN1/0.1/-6.6/U_2^{\halt}/right/,%
        vN/0.3/-7.1/T_2^{\halt}/right/,%
        row/0.5/-7.5//above right/rowind,%
        end2/0.5/-8/\End/right/%
      }{ \node[point, \extra, label={[inner sep=1]\wh:{\scriptsize $\lab$}}] (\al) at
        (\x,\y) {}; }

      %% the edges
      \foreach \from/\to/\type in {%
        a/start1/role, %
        start1/t1/role, start2/t2/role, start3/t3/role, startm/tm/role, %
        s1/s2/role, s2/startm/role, vN1/vN/role, vN/row/role, row/end2/role,
        tree-t2.260/start3/dotted, t3/s1/dotted, tm/vN1/dotted%
      } {\draw[thick, \type] (\from) -- (\to);}

    \draw[role] (a) to[out=130, in=50, looseness=25] (a);

      %% additional labels
      \foreach \al/\lab/\wh in {%
        start1/y_0/left, t1/y_1/left, start2/y_{N+1}/above left,
        t2/y_{N+2}/left, start3/y_{n-2N-2}/left, startm/y_{n-N-1}/above left,%
        t1/~I_0/right, tm/z_1/left, vN/z_{N}/left% , row/u/below left,
        % end2/v/left%
      }{ \node[label={[inner sep=2]\wh:{\scriptsize $\lab$}}, inner sep=0] at
        (\al) {}; }

      \foreach \al/\lab/\wh in {%
        startm/\sigma/above right%
      }{\node[label={[inner sep=2,red]\wh:{\scriptsize $\lab$}}, inner sep=0] at (\al)
        {}; }

      %%% right homomorphism
      \begin{pgfonlayer}{background}
        \foreach \from/\to/\in in {%
          y0/a/170, yy0a/a/200, yy0/start1/200, yy1/t1/200, %
          yyN/start2/200, y1/t2/200, yN/tree-t2.south/200, ylN2/startm/200,
          yl/row/200, yend/end2/190%
        } {\draw[homomorphism, black!60] (\from) to[out=-20, in=\in] (\to);}
      \end{pgfonlayer}

      %%%% Numbered rows
      \begin{pgfonlayer}{background}
        \foreach \start/\finish/\x/\i in {%
          start1/start2/1.3/1, %
          start2/tree-t2.south/1.8/2,%
          start3/startm/1.3/M-1, %
          startm/row/1.8/M%
        } { \draw[gray, thick, <->] ({(\x,0)}|-\start) -- %
          node[right] {\small Row $\i$} ({(\x,0)}|-\finish); }
      \end{pgfonlayer}

      \node at (0.7,2) {$\I_r$};
    \end{scope}
\end{tikzpicture}}

\caption{Two homomorphisms to minimal models.}
\label{two-homomorphisms-rooted}
\end{figure}

%%% Local Variables:
%%% mode: latex
%%% TeX-master: "aij-insep"
%%% TeX-PDF-mode: t
%%% fill-column: 79
%%% save-place: t
%%% End: 

First, we construct an auxiliary $R$-path $y_0,\dots,y_{n-N-1}$.
We take $y_0 \in \Row^\I$, the successor of $a$ in $\I$, and
$y_1 \in {I_0}^\I$, the successor of $y_0$ in $\I$, by~\eqref{r-initial}
($I_0 = T^{1,1}$). Then we take
$y_2 \in (T^{2,1})^\I, \dots, y_{N} \in (T^{N,1})^\I$
by~\eqref{horizontal-k}. We now have $\textit{right}(T^{N,1}) =
W$. By~\eqref{end-of-row-k}, we obtain $y_{N+1} \in {\Row_{1}}^{\I}$.
By~\eqref{row-k}, $y_{N+1} \in {\Row_1}^\I \subseteq \Row^\I$. We proceed in
this way, starting with \eqref{start-row-k}, till the moment we construct
$y_{n-1} \in (T^{N,M-1})^\I$, for which we use~\eqref{start-last-row} and
\eqref{row-halt} to obtain $y_n \in (\Row^{\halt}_k)^\I \subseteq \Row^\I$, for
some $k$.  Note that $T^\I \subseteq \widehat T^\I$ by~\eqref{tile-hat}.

By~\eqref{disjunction}, two cases are possible now.

\emph{Case 1}: there is $y$ such that $(y_n,y) \in R^\I$ and $y \in \End^\I$. Then we take $x_0 = \dots = x_{N} = a$, $x_{N+1} = y_0, \dots, x_{n} = y_{n-N-1}, x_{n+1} = y$.

\emph{Case 2}: there is $z_1$ such that $(y_n,z_1) \in R^\I$ and $z_1 \in (T^{\it halt}_k)^\I$, where $T = T^{1,M}$ and $\textit{up}(T) = C$. We then use~\eqref{horizontal-halt} and find $z_2,\dots,z_N,u,v$ such that $z_i \in (T^{\it halt}_k)^\I$, where $T = T^{i,M}$, $u \in \Row^\I$ and $v \in \End^\I$.
We take $x_0 = y_0,\dots, x_{n-N-1} = y_{n-N-1}, x_{n-N} = z_1,\dots, x_{n-1}= z_N, x_n = u, x_{n+1} = v$. Note that, by~\eqref{vertical-k} and~\eqref{vertical-halt}, we have $(T^{i,j})^\I \subseteq (\widehat T^{i,j-1})^\I$.

$(\Leftarrow)$ Suppose $\KtworCQ\models \q^{r}_{n}(a)$ for some $n>0$.
%Let $\q_n(y)$ be such that $\prod \Mod_{\KtworCQ} \models \q_n(a)$,
%by Proposition 5 it follows $\M \models \q_n$ for each $\M \in
%\Mod_{\KtworCQ}$.
Consider all the pairwise distinct pairs $(\I,h)$ such that $\I
\in \Mod_{\KtworCQ}$ and $h$ is a homomorphism from $\q^{r}_{n}(a)$ to $\I$.
%Call an object $x$ in $\M$ \emph{special} if $x \in (\Row^{\it halt}_k)^\M$, for some $k$. Denote by $\Pi^\M(x)$ the path from $a$ to $x$ in $\M$.
Note that $h(\q^{r}_{n})$ contains an or-node $\sigma_h$ (which is an instance of
$\Row^{\halt}_k$, for some $k$).  We call $(\I,h)$ and $h$ \emph{left} if
$h(x_{n+1}) = \sigma_h \cdot w_{\exists R.\End}$, and \emph{right} otherwise.
It is not hard to see that there exist a left $(\I_l,h_l)$ and a right $(\I_r,h_r)$ with $\sigma_{h_l} = \sigma_{h_r}$ (if this is not the case, we can construct $\I\in\Mod_{\KtworCQ}$ such that $\I\not\models\q^{r}_{n}(a)$).

Take $(\I_l,h_l)$ and $(\I_r,h_r)$ such that $\sigma_{h_l} = \sigma_{h_r} = \sigma$ and use them to construct the required tiling. Let $\sigma = a w_0\cdots w_{n'}$. We have $h_l(x_{n}) = \sigma$, $h_l(x_{n+1}) = \sigma \cdot w_{\exists R.\End}$. Let $h_r(x_{n+1}) = \sigma v_1\dots v_{m+2}$, which is an instance of $\End$. Then $h_r(x_n) = \sigma v_1\dots v_{m+1}$, which is an instance of $\Row$.

Suppose $v_{m} = w_{\exists R. T^{\halt}_2}$ (any $k$ other than 2 is treated analogously). By \eqref{end}, $\textit{right}(T) = W$; by~\eqref{horizontal-halt}, $\textit{up}(T) = C$.
Suppose $w_{n'-1} = w_{\exists R. S_k}$. Now, we know that $k=1$. By \eqref{start-last-row}, $\textit{right}(S) = W$. Consider the atom $B_{n-1}(x_{n-1})$ from $\q^{r}_{n}$. Both $a w_0\cdots w_{n'-1}$ and $\sigma v_1\cdots v_{m}$ are instances of $B_{n-1}$. By \eqref{tile-hat} and \eqref{vertical-halt}, $B_{n-1} = \widehat S_1$ and $\textit{down}(T) = \textit{up}(S)$.
Suppose $v_{m-1} = w_{\exists R. U^{\halt}_2}$. By \eqref{horizontal-halt}, $\textit{right}(U) = \textit{left}(T)$ and $\textit{up}(U) = C$.
Suppose $w_{n'-2} = w_{\exists R. Q_1}$. By \eqref{horizontal-k}, $\textit{right}(Q) = \textit{left}(S)$.
Consider the atom $B_{n-2}(x_{n-2})$ from $\q^{r}_{n}$. Both $a w_0\cdots w_{n'-2}$ and $\sigma \cdots v_{m-1}$ are instances of $B_{n-2}$. By \eqref{tile-hat} and \eqref{vertical-halt}, $B_{n-2} = \widehat Q_1$ and $\textit{down}(U) = \textit{up}(Q)$.
We proceed in the same way until we reach $\sigma$ and $a w_0\cdots w_{n'-N-1}$, for $N=m$, both of which are instances of $B_{n-N-1} = \Row$. Thus, we have tiled the last two rows of the grid.

We proceed in this way until we have reached some variable $x_t$, for $t \geq 0$, of $\q^{r}_{n}$ that is mapped by $h_l$ to $a w_0 w_1$ (see Fig.~\ref{two-homomorphisms-rooted}). Note that this situation is guaranteed to occur. Indeed, $h_l(a) = a$, $h_l(x_0) \in \{ a, a w_0\}$, $h_l(x_1) \in \{ a, a w_0, a w_0 w_1\}$, etc. Clearly, the assumption that $h_l(x_i) \in \{ a, a w_0\}$ for all $i$ ($0 \leq i \leq n+1$) leads to a contradiction.
Let $h_r(x_t) = a w_0 \cdots w_s$, for some $s > 1$. Note that $s =
N+2$. By~\eqref{r-initial}, it follows that $a w_0 w_1$ is an instance of
$I_0$. Therefore, $B_t = \widehat{I}_0$ and, by~\eqref{vertical-k},
$a w_0\cdots w_s$ is an instance of $V_1$, for some tile $V$ such that
$\textit{down}(V) = \textit{up}(I)$.

Thus, we have a tiling as required since the vertical and horizontal
compatibility of the tiles is ensured by the construction above and by the fact
that the tile $I$ occurs in it as the initial tile.
\end{proof}
\begin{lemma}\label{qn-rejected-by-k1-r}
$\prod\!\Mod_{\KtworCQ}$ is con-$n\SigmarCQ$-homomorphically embeddable into $\I_{\KonerCQ}$ preserving $\{a\}$ for all
$n\geq 1$ iff there does not exist an rCQ $\q^{r}_m(y)$ such that $\prod\!\Mod_{\KtworCQ} \models \q^{r}_m(a)$.
\end{lemma}
\begin{proof}
  $(\Rightarrow)$ Suppose otherwise, that is, 
  $\prod \Mod_{\KtworCQ} \models \q^{r}_m(a)$ for some $m$. By the assumption,
  $\prod \Mod_{\KtworCQ}$ is con-$n\SigmarCQ$-homomorphically embeddable into
  $\I_{\KonerCQ}$ for $n=m+3$ (the length of $\q^{r}_m$).  So we have
  $\I_{\KonerCQ} \models \q^{r}_m(a)$, which is clearly impossible because none
  of the paths of $\I_{\KonerCQ}$ contains the full sequence of symbols
  mentioned in $\q^{r}_m(y)$.

$(\Leftarrow)$ Suppose $\prod \Mod_{\KtworCQ} \not\models \q^{r}_m(a)$ for all $m$.
Take any subinterpretation of $\prod \Mod_{\KtworCQ}$ whose domain contains $n$ elements connected to $a$.
Recall from the proof of Theorem~\ref{prop:char} that we can regard the $\SigmarCQ$-reduct of this subinterpretation as a $\SigmarCQ$-rCQ,
and so denote it by $\q(y)$. Clearly, $\q$ is tree shaped plus the atom $R(y, y)$.
We know that there is no $\SigmarCQ$-homomorphism from $\q^{r}_m(y)$ into $\q(y)$ for any $m$; in particular,
$\q(y)$ does not have a subquery of the form $\q^{r}_m(y)$. We have to show that $\I_{\KonerCQ} \models \q(a)$.
We show how to map $\q(y)$ starting from $a$.

We call a variable $x$ in $\q(y)$ a \emph{gap} if there exists no $B \in \SigmarCQ$ such that $B(x)$ is in $\q(y)$.
Since $\q(y)$ does not contain a subquery of the form $\q^{r}_m(y)$, we know that every path $\rho$ starting from $y$ in $\q(y)$ either:
\begin{description}\itemsep 0cm
\item[(a)] does not contain $\End(x)$, or
\item[(b)] contains $\End(x)$ and contains a gap $x'$ that occurs between the $y$ and $x$.
\end{description}
If all paths $\rho$ starting from $y$ in $\q(y)$ are of type \textbf{(a)} we map $\q(y)$ on the path $\pi_\omega$:
\begin{center}
  \begin{tikzpicture}[xscale=2, yscale=1.2, %
    emptyind/.style={fill=gray!50, inner sep=1.7},
    rotatelabels/.style={rotate=-14},
    abovesloped/.style={above,rotatelabels}]

    \foreach \al/\x/\y/\lab/\wh/\extra/\rot in {%
      a/0/0/{\qquad\qquad A,\Row,\widehat{T}_0}/above/constant/, %
      a2/0/-0.7/{}/below/emptyind/, %
      a3/0.7/-1/{\End,\Sigma_{0},E~~~}/below//rotatelabels, %
      a4/1.4/-1.3/{~~~\End,\Sigma_{0},E}/below//rotatelabels, %
      x1/1.5/0/{\Sigma_0,D}/above//, %
      x2/1.5/-0.7/{}/below/emptyind/, %
      x3/2.2/-1/{\End,\Sigma_{0},E~~~}/below//rotatelabels, %
      x4/2.9/-1.3/{~~~\End,\Sigma_{0},E}/below//rotatelabels, %
      y1/3/0/{\Sigma_0,D}/above//,%
      y2/3/-0.7/{}/above/emptyind/,%
      y3/3.7/-1/{\End,\Sigma_{0},E~~~}/below//rotatelabels,%
      y4/4.4/-1.3/{~~~\End,\Sigma_{0},E}/below//rotatelabels,%
      z1/4.5/0/{\Sigma_0,D}/above//,%
      z2/4.5/-0.7/{}/above/emptyind/,%
      z3/5.2/-1/{\End,\Sigma_{0},E~~~}/below//rotatelabels,%
      z4/5.9/-1.3/{~~~\End,\Sigma_{0},E}/below//rotatelabels%
    }{ \node[point, \extra, label={[\rot, inner sep=2]\wh:{\scriptsize
          $\lab$}}] (\al) at (\x,\y) {}; }

    \node[label={[inner xsep=1]left:{\scriptsize$a$}}] at (a) {};%

    \foreach \from/\to/\wh in {%
      a/a2/left, a2/a3/abovesloped, a3/a4/abovesloped, %
      a/x1/above, x1/x2/left, x2/x3/abovesloped, x3/x4/abovesloped, %
      x1/y1/above, y1/y2/left, y2/y3/abovesloped, y3/y4/abovesloped, %
      y1/z1/above, z1/z2/left, z2/z3/abovesloped, z3/z4/abovesloped%
    } {\draw[role] (\from) -- node[\wh] {\scriptsize $R$} (\to);}

    \draw[role] (a) to[out=230,in=130,looseness=12] node[left] {\scriptsize $R$} (a);

    % \foreach \al/\lab/\wh in {%
    % x2/\sigma_{\End}/right, x0/\sigma_{\Start}/right%, w0/\sigma/left%
    % } {\node[label={[red, inner sep=1]\wh:{\footnotesize $\lab$}}] at (\al)
    %   {};}
    \foreach \from in {%
      a4, x4, y4, z4%, w2, v11a, v00a, u0, u1a%
    } {\draw[dotted, thick] (\from) -- +(0.5,-0.2);}

    \draw[dotted, thick] (z1) -- +(1,0);

    \node at (-1,0) {$\I_{\KonerCQ}$:};
    \node[red] at (5.9,0) {$\pi_\omega$};
    \node[red] at (2.1,-1.5) {$\pi_1$};
    \node[red] at (3.6,-1.5) {$\pi_2$};
    \node[red] at (5.1,-1.5) {$\pi_3$};
    \node[red] at (6.6,-1.5) {$\pi_4$};
  \end{tikzpicture}
\end{center}
Otherwise, let $y$ be the current variable and $a$ the current image. Let
$x_1,\dots,x_k$ be all successor gaps and $z_1,\dots,z_l$ all successor
non-gaps of the current variable in $\q(y)$. We map all $x_i$ to the vertical
successor and all $z_i$ to the horizontal successor of the current image. All
the rest of the paths starting from $x_i$ can then be mapped to an appropriate
$\pi_i$. We then consider each $z_i$ as the current variable, and the point
where it has been mapped as the current image, and continue analogously. Thus,
the paths $\rho$ not containing gaps and $\End(x)$ atoms would result in being
mapped to $\pi_\omega$, while the paths with gaps would each result in being
mapped to an appropriate $\pi_i$.
\end{proof}

We now prove Theorem~\ref{thm:undecidability-rcq}~(\emph{ii}).
We set $\K_2 = \KtworCQ \cup \KonerCQ$ and show that the following are equivalent:
\begin{enumerate}
\item[(1)] $\KonerCQ$ $\SigmarCQ$-rCQ entails $\KtworCQ$;
\item[(2)] $\KonerCQ$ and $\K_2$ are $\SigmarCQ$-rCQ inseparable.
\end{enumerate}
Let $\I_{\KonerCQ}$ be the canonical model of $\KonerCQ$ and $\Mod_{\KtworCQ}$
the set of minimal models of $\KtworCQ$. Again, one can easily show that the
following set $\Mod_{\K_{2}}$ is complete for $\K_2$:
$$
\Mod_{\K_2} = \{ \I \uplus \I_{\KonerCQ} \mid \I \in \Mod_{\KtworCQ}\},
$$
where $\I \uplus \I_{\KonerCQ}$ is the interpretation that results from merging
the roots $a$ of $\I$ and $\I_{\KonerCQ}$. Now $(2) \Rightarrow (1)$ is
trivial. For the converse, suppose $\KonerCQ$ $\SigmarCQ$-rCQ entails
$\KtworCQ$. It directly follows that $\K_2$ $\SigmarCQ$-rCQ entails
$\KonerCQ$. So it remains to show that $\KonerCQ$ $\SigmarCQ$-rCQ entails
$\K_2$.  Suppose this is not the case. Without loss of generality, we may
assume that there is a $\SigmarCQ$-rCQ $\q(y)$, a ditree with one answer variable
$y$ not mentioning $D$ and $E$, such that $\K_2 \models \q(a)$ and $\KonerCQ
\not\models \q(a)$. We can assume $\q$ to be a \emph{smallest} rCQ with this
property. % ; in particular, no proper
% sub-rCQ of $\q$ separates $\KonerCQ$ and $\K_2$.
%
% Now, we cannot have $\KtworCQ \models \q$ because this would contradict the
% fact that $\KonerCQ$ $\SigmarCQ$-rCQ entails $\KtworCQ$.  Then $\KtworCQ
% \not\models \q$, and so there is $\I \in \Mod_{\KtworCQ}$ such that $\I \not
% \models \q$. On the other hand, we have $\I \uplus \I_{\KonerCQ} \models
% \q$. Take a homomorphism $h\colon \q \to \I \uplus \I_{\KonerCQ}$. As $\q$ is
% connected, $\I \not \models \q$ and $\I_{\KonerCQ} \not\models \q$, there is a
% variable $x$ in $\q$ such that $h(x) = a$.
%
% Let $X$ be the set of all variables $x$ such that $h(x) = a$.  For every
% variable $x \in X$, we remove $\exists x$ from the prefix of $\q$ if any.
% Denote by $\q'$ the maximal sub-rCQ of $\q$ such that $h(\q') \subseteq \I$
% (more precisely, $S(\avec{y})\in \q$ is in $\q'$ iff $h(\avec{y}) \subseteq
% \Delta^\I$). Clearly, $\q' \subsetneqq \q$. We prove that $\K_2 \models \q'$,
% that is in each model $\J \in \Mod_{\K_2}$, the variables in $X$ are mapped to
% $a$.
Consider the various cases of $\q(y)$:
\begin{itemize}
\item[--] $\q(y)$ does not contain $\End$ atoms: but then $\KonerCQ
  \models \q(a)$ (see the proof of Lemma~\ref{qn-rejected-by-k1-r}),
  contrary to our assumption.
\item[--] $\q(y)$ contains $\End$ atoms and, on each path from $y$ to an $\End$ atom,
  there is a variable $x$ that does not appear in $\q(y)$ in any atom of the
  form $B(x)$, for a concept name $B \in \Sigma$. But then  $\KonerCQ \models \q(a)$ (see the proof of Lemma~\ref{qn-rejected-by-k1-r}),
  contrary to our assumption.
\item[--] $\q(y)$ contains $\End$ atoms and a path from $y$ to
  an $\End$ atom such that each variable $x$ on this path appears in an atom of
  the form $B(x)$, for a concept name $B \in \Sigma$. Denote this path by
  $\q'(y)$, and observe that $\q'(y)$ is a query of the form $\q^{r}_n(y)$. Then
  $\KonerCQ \not\models \q'(a)$ by the construction of $\KonerCQ$, moreover there
  is no subquery $\q''$ of $\q'(y)$ such that there is a model $\I \in
  \Mod_{\KtworCQ}$ and $\I \uplus \I_{\KonerCQ} \models \q'(a)$ by mapping
  $\q''$ entirely into $\I_{\KonerCQ}$. So it must be that
  $\KtworCQ\models\q'(a)$. But now, as $\KonerCQ \models \KtworCQ$, we know that
  $\KtworCQ \not\models \q^{r}_n(a)$ for each $n$, which is again a contradiction.
\end{itemize}
The contradictions arise from the assumption that $\KonerCQ$ does not
$\SigmarCQ$-rCQ entail $\K_2$.

\section{Proof of Theorem~\ref{thm:TBoxbasicundecidability} for Rooted CQs}
We show that it is undecidable whether an $\mathcal{EL}$ TBox is $\Theta$-rCQ inseparable from an $\mathcal{ALC}$ TBox.
For the proof we require homomorphisms between ABoxes and the observation that they preserve certain answers. Let
$\A_{1}$ and $\A_{2}$ be ABoxes. A map $h$ from $\ind(\A_{1})$ to $\ind(\A_{2})$ is called an \emph{ABox-homomorphism}
if $A(a)\in \A_{1}$ implies $A(h(a))\in \A_{2}$ for all concept names $A$, and $R(a,b)\in \A_{1}$ implies $R(h(a),h(b))\in \A_{2}$
for all role names $R$. The following is shown in \cite{BaaderBL16}.
\begin{proposition}\label{prop:ABoxhom}
Let $\T$ be an \ALC TBox, $\A,\A'$ be ABoxes, and $h\colon\A \rightarrow \A'$ an ABox homomorphism. Then
\begin{itemize}
\item $\A$ is consistent with $\T$ if $\A'$ is consistent with $\T$, and
\item $(\T,\A)\models \q(\avec{a})$ implies $(\T,\A')\models \q(h(\avec{a}))$
  for all CQs $\q(\avec{x})$.
\end{itemize}
\end{proposition}
To prove the undecidability of the problem whether an $\mathcal{EL}$ TBox is $\Theta$-rCQ inseparable from an $\mathcal{ALC}$ TBox,
we use the TBoxes constructed in the proof of Theorem~\ref{thm:undecidability-rcq}. Recall the KBs $\KonerCQ= (\TonerCQ,\ArCQ)$,
$\KtworCQ=(\TtworCQ,\ArCQ)$ and $\K_{2}= (\T_{2},\ArCQ)$, where $\T_{2}=\TonerCQ \cup \TtworCQ$. Set $\Theta=(\Sigma_{1},\Sigma_{2})$,
where $\Sigma_{1}=\sig(\ArCQ)$ and $\Sigma_{2}=\SigmarCQ$. We aim to show that the following conditions are equivalent:
\begin{enumerate}
\item[(1)] $\KonerCQ$ and $\K_{2}$ are $\SigmarCQ$-rCQ inseparable;
\item[(2)] $\TonerCQ$ and $\T_{2}$ are $\Theta$-rCQ inseparable.
\end{enumerate}
The implication $(2) \Rightarrow (1)$ is straightforward: if $\KonerCQ$ and
$\K_{2}$ are not $\SigmarCQ$-CQ inseparable then the ABox $\ArCQ$ witnesses
that $\TonerCQ$ and $\T_{2}$ are not $\Theta$-rCQ inseparable.  Conversely,
suppose $\TonerCQ$ and $\T_{2}$ are not $\Theta$-rCQ inseparable. Take a
$\Sigma_{1}$-ABox $\A$ such that $(\TonerCQ,\A)$ and $(\T_{2},\A)$ are not
$\Sigma_{2}$-rCQ inseparable. Clearly, $(\T_{2},\A)$ $\Sigma_{2}$-rCQ entails $(\TonerCQ,\A)$.
Thus, $(\TonerCQ,\A)$ does not $\Sigma_{2}$-rCQ entail $(\T_{2},\A)$.
The canonical model $\I_{1}$ of the $\mathcal{EL}$ KB $(\TonerCQ,\A)$ can be
constructed by taking, for every $A(b)\in \A$, a copy of the canonical model
$\I_{\KonerCQ}$ and hooking the two $R$-successors of $a$ in $\I_{\KonerCQ}$
(together with the subinterpretations they root) as fresh $R$-successors to
$b$.
On the other hand, the class $\Mod$ of minimal models of $(\T_{2},\A)$ is
obtained from $\I_{1}$ by hooking to every $b$ with $A(b)\in \A$ a copy of a
minimal model $\I_{b}\in \Mod_{\KtworCQ}$ by identifying the root $a$ of
$\I_{b}$ with $b$.

Now consider a $\Sigma_{2}$-rCQ $\q(\avec{a})$ with
$(\TonerCQ,\A)\not\models \q(\avec{a})$ and $(\T_{2},\A)\models \q(\avec{a})$.
Suppose $\q(\avec{a})$ is the smallest rCQ with this property. Using the description 
of the canonical model $\I_{1}$ of $(\TonerCQ,\A)$ and the class $\Mod$ of minimal models of $(\T_{2},\A)$, one can show in
the same way as in the proof of Theorem~\ref{thm:undecidability-rcq}~(\emph{ii}) given in the appendix above that there
must be a path in $\q$ from an answer variable to an $\End$ atom such that each
variable $x$ on this path appears in an atom of the form $B(x)$ with
$B\in \SigmarCQ$. But then $\q$ contains a query of the form $\q^{r}_{n}(x)$ (see again the proof of Theorem~\ref{thm:undecidability-rcq}~(\emph{ii}))
such that $(\T_{2},\A)\models \q_{n}^{r}(a)$ for some individual $a$ and $n>0$.
Observe that the map $h \colon \ind(\A) \rightarrow \{a\}$ is an ABox-homomorphism from the ABox
$\A$ onto the ABox $\ArCQ$. It follows from Proposition~\ref{prop:ABoxhom} that
$(\T_{2},\ArCQ)\models \q_{n}^{r}(h(a))$, for some $n$. We know from the proof of Theorem~\ref{thm:undecidability-rcq}
that $\KonerCQ \not\models q_{n}^{r}(a)$. Thus, $\KonerCQ$
and $\K_{2}$ are not $\SigmarCQ$-rCQ inseparable, as required.

\bibliography{bibliography}
\bibliographystyle{model1-num-names}

\end{document}